\documentclass{article}

\usepackage{arxiv}

\usepackage[utf8]{inputenc} 
\usepackage[T1]{fontenc}    
\usepackage{hyperref}       
\usepackage{url}            
\usepackage{booktabs}       
\usepackage{amsfonts}       
\usepackage{nicefrac}       
\usepackage{microtype}      
\usepackage{lipsum}		
\usepackage{graphicx}
\usepackage[numbers]{natbib}
\usepackage{doi}
\usepackage{pgfplots}
\usepackage{subcaption}
\usepackage[utf8]{inputenc}
\usepackage{float}
\usepackage{wrapfig}
\usepackage{booktabs}
\usepackage{multicol}
\usepackage{multirow}
\usepackage{caption} 
\usepackage{amsmath}
\usepackage{relsize}
\usepackage{color,soul}
\usepackage{xcolor}
\usepackage[toc,page]{appendix}
\usepackage{amsthm}
\usepackage{amssymb}
\usepackage{cleveref}
\usepackage{enumitem}   

\usepackage{nicematrix}
\usepackage{diagbox}
\usepackage{makecell}
\usepackage{algpseudocode}
\usepackage{algorithm}
\usepackage{pifont}

\newtheorem{lemma}{Lemma}
\newtheorem{corollary}{Corollary}
\newcommand{\norm}[1]{\left\|#1\right\|_2}

\newcommand{\blmin}{\bar\lambda_{\min}}
\newcommand{\blmax}{\bar\lambda_{\max}}

\newcommand{\bLb}{\bar\Lambda}

\pgfplotsset{width=10cm,compat=1.9}

\usepgfplotslibrary{external}
\definecolor{light-gray}{gray}{0.95}

\title{ADPTNet: Adaptive with Prescriptive Timescales Non-Linear SSM for Sequence Modelling}

\author{ \href{https://orcid.org/0000-0003-2726-1860}{\includegraphics[scale=0.06]{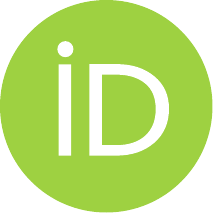}\hspace{1mm}Matei-Ioan Stan *} \\
	International Centre for Neuromorphic Systems\\
	The University of Manchester\\
	Manchester, United Kingdom \\
	\texttt{matei.stan@manchester.ac.uk} \\
	\And
	\href{https://orcid.org/0000-0003-1728-2828}{\includegraphics[scale=0.06]{orcid.pdf}\hspace{1mm} Oliver Rhodes} \\
    International Centre for Neuromorphic Systems\\
	The University of Manchester\\
	Manchester, United Kingdom  \\
	\texttt{oliver.rhodes@manchester.ac.uk} \\
}

\renewcommand{\headeright}{}
\renewcommand{\undertitle}{}
\renewcommand{\shorttitle}{ADPTNet: ADaptive with Prescriptive Timescales Non-Linear SSM for Sequence Modelling}

\hypersetup{
pdftitle={ADPTNet: ADaptive with Prescriptive Timescales Non-Linear SSM for Sequence Modelling},
pdfsubject={q-bio.NC, q-bio.QM},
pdfauthor={Matei-Ioan Stan},
pdfkeywords={Spiking Neural Networks, State Space Models, Sequence Modelling, Long Range Dependencies, Dynamical Systems, State Tracking, Selectivity},
}

\begin{document}
\maketitle

\begin{abstract}
A central aim of neuromorphic computing is to provide a viable alternative to highly energy-intensive Transformer-based AI.  However, efficient alternatives struggle to capture the set of qualities that have secured the Transformer's status as the de facto standard in sequence modelling. Any realistic contender must be data-adaptive, able to capture long-range dependencies, and GPU-parallelisable, but also non-linearly recurrent to enable complex reasoning. Based on evidence suggesting the auditory cortex operates on fixed timescales, this work proposes the ADaptive with Prescriptive Timescales Network (ADPTNet) as a potential solution to achieving all four properties simultaneously. ADPTNet is built around local topological conjugates, obtained by a novel combination of linear attention and Riemannian optimisation, applied to static global dynamics. This enables non-linear yet predictable long-term behaviour. Dynamical systems theory proofs provide theoretical guarantees for the parametric control of ADPTNet's timescales (its Lyapunov spectrum). ADPTNet improves performance on Selective Copying over Hawk, the existing method balancing long-range memory and adaptability, while also improving state tracking over linear SSMs like Mamba. On sequential CIFAR-10, ADPTNet matches linear SSM accuracy and outperforms existing selective models (incl. the Transformer), using fewer parameters. We also introduce a neuromorphic SpikingADPTNet, which achieves a new state-of-the-art accuracy on the Spiking Speech Commands dataset ($83.56\%\pm0.15$). Finally, ADPTNet's constant timescales enable two efficient, Jacobian-free extensions to the DEER parallel simulation algorithm (Conv and Forward DEER) that retain the same average convergence. Conv DEER adds no computational overhead beyond the network's forward pass and enables non-linear RNN parallelisation via iterated convolutions for the first time. 

\end{abstract}

\keywords{Spiking Neural Networks \and State Space Models \and Sequence Modelling \and Selectivity \and State Tracking \and Dynamical Systems \and Long Range Dependencies}

\section{Introduction} \label{sec:intro}

In recent years, sequence-based tasks have proven a hotbed of progress in deep learning research. This is best exemplified by the billion-dollar industry that has emerged around Large-Language Models (LLMs), essentially ever-improving token-sequence compression and decoding neural network algorithms. More specifically, this massive effort to build state-of-the-art LLMs has largely relied on developing variations of what has effectively become the standard architecture, Transformers \citep{vaswani2017attention}.

Some research on the remarkable success of Transformers in Natural Language Processing (NLP) applications has focused on how their core component, self-attention, overcomes vanishing and exploding gradients that affect Recurrent Neural Networks (RNNs) \citep{hochreiter2001gradient, pascanu2013difficulty}, enabling the use of longer context windows. Another strong argument for the Transformer's ubiquitous adoption has been the formulation of self-attention around massive and parallelisable matrix multiplications, which are well-positioned to take advantage of modern GPUs and thus are currently "winning" the hardware lottery \citep{hu2019introductory, hooker2021hardware}. A final research direction for explaining the Transformer's performance, which has attracted increasing interest over time, has been to formalise the role of the architecture's perceived "ad aptability". 

This latter property of self-attention layers has been described in literature using terms at various levels of abstraction. Some have defined adaptability in relation to the flexibility granted by direct modelling of token-to-token interactions \citep{lee2022fnet}. Others have focused on how self-attention implicitly defines the weights used for linear sequence-wise mixing as functions conditioned on the input tokens, thereby making it a data-controlled operation \citep{massaroli2020dissecting, poli2023hyena}. Links have also been drawn between attention key-value pairings and associative Hopfield Networks \citep{ramsauer2020hopfield}. Furthermore, associative storage and recall are thought to be key components of in-context learning, which refers to the emergent ability of Transformer-based LLMs to effectively predict unseen inputs based on information provided through prompts at inference-time \citep{arora2024zoology, olsson2022context}. 

The dominance of Transformer-based architectures, however, is not unchallenged, with a major current challenge being the infamous energy-intensity of training and deploying state-of-the-art LLM systems \citep{singh2025survey}. From an algorithmic perspective, this is partly due to quadratic-in-sequence-length compute and memory scaling of self-attention. In this regard, recurrent architectures have seen a resurgence of interest based on comparatively favourable sub-quadratic scaling. Most notably, linear Transformers \citep{tay2022efficient} and State Space Models (SSMs) \citep{gu2021efficiently} have emerged as overlapping research streams focused on matching or outperforming Transformers in associative recall and long-context sequence modelling, respectively. Consequently, contributions from both fields have provided efficiency gains in industrial LLM development \citep{botev2024recurrentgemma, team2025kimi, chandiramani2026nemotron}. However, fundamentally, they still rely on dense floating-point Multiply-Accumulate (MAC) operations on GPUs, much like Transformers, and thus inherit a similar energy-consumption paradigm \citep{bai2024beyond}. 

One might assume that this taxing energy consumption is an inherent drawback of intelligent systems, were it not for the evident counterexample of the human brain. The cortex vastly outperforms current frontier models, including in key areas of deep learning research, such as reasoning and few-shot generalisation, while using a fraction of the energy \citep{hasler2017special}. Neuromorphic computing aims to bridge this gap by developing algorithms and computational substrates that emulate the brain's efficacy and efficiency. Perhaps most famously, Spiking Neural Networks (SNNs), the third generation of neural network models \citep{maass1997networks} explicitly aim to reduce energy consumption by avoiding MAC operations through sparse, binary communication between biologically plausible neurons. SNNs are also intended for deployment on neuromorphic hardware, which can leverage spike-train sparsity and potentially avoid the von Neumann bottleneck of traditional hardware, including GPUs \citep{kudithipudi2025neuromorphic}. Consequently, neuromorphic systems have been shown to reduce energy consumption by orders of magnitude compared to mainstream counterparts \citep{davies2021advancing, rhodes2020real}.

While SNNs have a proven track record of reducing energy consumption, to date, it remains mostly confined to small-scale machine learning applications. As neuromorphic algorithms scale, they typically begin to sacrifice biological inspiration, creating an inherent tension in their definition. For instance, scaling SNNs in depth typically requires loosening the definition of spikes (e.g., "graded" integer spikes \citep{fang2021deep}), or replacing them altogether (e.g., ternary-weight linear layers with continuous activations \citep{zhu2024scalable, stan2024learning}), becoming harder and harder to distinguish from low-bit quantised RNNs. It is, therefore, worthwhile to ask how cortical computational principles can still contribute to scaling up SNNs in the age of massive foundation models and alleviate this tension. 

Realistically, a precondition for neuromorphic systems to catch up to the widespread adoption of current LLMs is to first match the Transformer's favourable qualities. For the purposes of this study, as mentioned above, these target qualities include long-range sequential dependency modelling, amenability to fast parallel simulation, and a notion of "adaptability". The present study attempts to address all three properties, drawing on recent developments in understanding temporal processing in the auditory cortex.

The brain performs computations using recurrent, non-linear dynamics \citep{douglas1995recurrent, singer2016does, vignesh2025review}, determined by a mix that includes, but is not limited to, individual neuronal dynamics \citep{gerstner2014neuronal}, synaptic connectivity \citep{marsh2024emergent, breakspear2007neuronal}, and neuromodulation \citep{nadim2014neuromodulation}. The auditory cortex is no exception, as it essentially implements non-linear mappings from auditory sensory inputs to higher-order neural representations \citep{keshishian2020estimating}. Crucially, however, recent evidence suggests that across both primary (i.e., acoustic processing) and non-primary (i.e., speech or music processing) areas, inputs are processed within almost fixed-duration temporal windows \citep{sabat2025neurons, norman2025temporal}. In other words, the dynamical systems within the auditory cortex have a non-linear fading memory \citep{boyd1985fading}, that leaks information over time at a roughly constant, inherent rate, irrespective of the current state of neural activity (e.g., individual neuron membrane voltages) or the incoming inputs. This is not a universal property of non-linear dynamical systems, as one cannot generally predict trajectories in phase space a priori \citep{sussillo2013opening}, including whether input information is forgotten at all. 

Traditional non-linear SNNs, much like RNNs, do not generally provide guarantees about time scales across phase space either \citep{sussillo2013opening, smith2021reverse}. The linear sub-threshold behaviour in popular neuron models, such as the Leaky Integrate-and-Fire (LIF) model \citep{gerstner2014neuronal}, can be parametrised to have pre-determined and constant memory properties in individual neurons (Section \ref{SNN_formula}). Furthermore, one can also parametrise recurrent synaptic connectivity matrices to guarantee certain linear memory properties \citep{hermans2010memory}. Nevertheless, once non-linear spiking is introduced, the analytical guarantees granted by linear dynamics no longer apply, making direct initialisation and parametrisation of network timescales difficult \citep{engelken2023gradient}. This is a well-known shortcoming of non-linear recurrent architectures, and is the root cause of the recently mentioned vanishing and exploding gradient problems. If one were to further increase the complexity of the network architecture, for example, by adding neuromodulation (i.e., changing recurrent weights at inference based on context \citep{alkilany2025neuromodulation}), it is safe to assume that timescale parametrisation would become even more difficult. 

These challenges in reliably controlling the timescales of non-linear dynamics have fuelled a recent wave of interest in SNNs with purely linear recurrence, taking advantage of progress made in effective initialisation and parametrisation schemes for linear SSMs in deep learning literature \citep{stan2024learning, meyer2025diagonal, shen2025spikingssms}. However, as interest in linearity has increased in both neuromorphic and deep learning research, evidence has also been mounting for previously unknown advantages of non-linear recurrence. Recurrent computation has been shown to improve reasoning capabilities while using fewer parameters compared to traditional vanilla parallel blocks such as Transformers and SSMs \citep{zhu2025scaling, jolicoeur2025less}. Non-linear RNNs layers have also been theoretically proven to model algorithmic state-tracking tasks and finite-state automata (e.g., tracking the state of a chess game), unlike Transformer or SSM layers \citep{merrill2024illusion, terzic2026structured}. This is a crucial observation, because state-tracking is devised as a synthetic test for complex reasoning capabilities, which have become a major focal point in LLM research \citep{merrill2023parallelism, jolicoeur2025less}. 

Much like non-linear dynamics, synaptic neuromodulation has also attracted recent attention in the RNN and SNN literature. It typically takes the form of LSTM-like gating \citep{hochreiter1997long, jing2019gated}, or subnetworks that control the gain scaling of recurrent weights \citep{costacurta2024structured, alkilany2025neuromodulation}, and has been shown to potentially help emulate Transformer "adaptability" in both linear \citep{gu2023mamba} and non-linear \citep{mishra2026m} recurrent architectures.

\begin{figure}[H]
    \centering
        \centering        
        \includegraphics[width=0.5\textwidth]{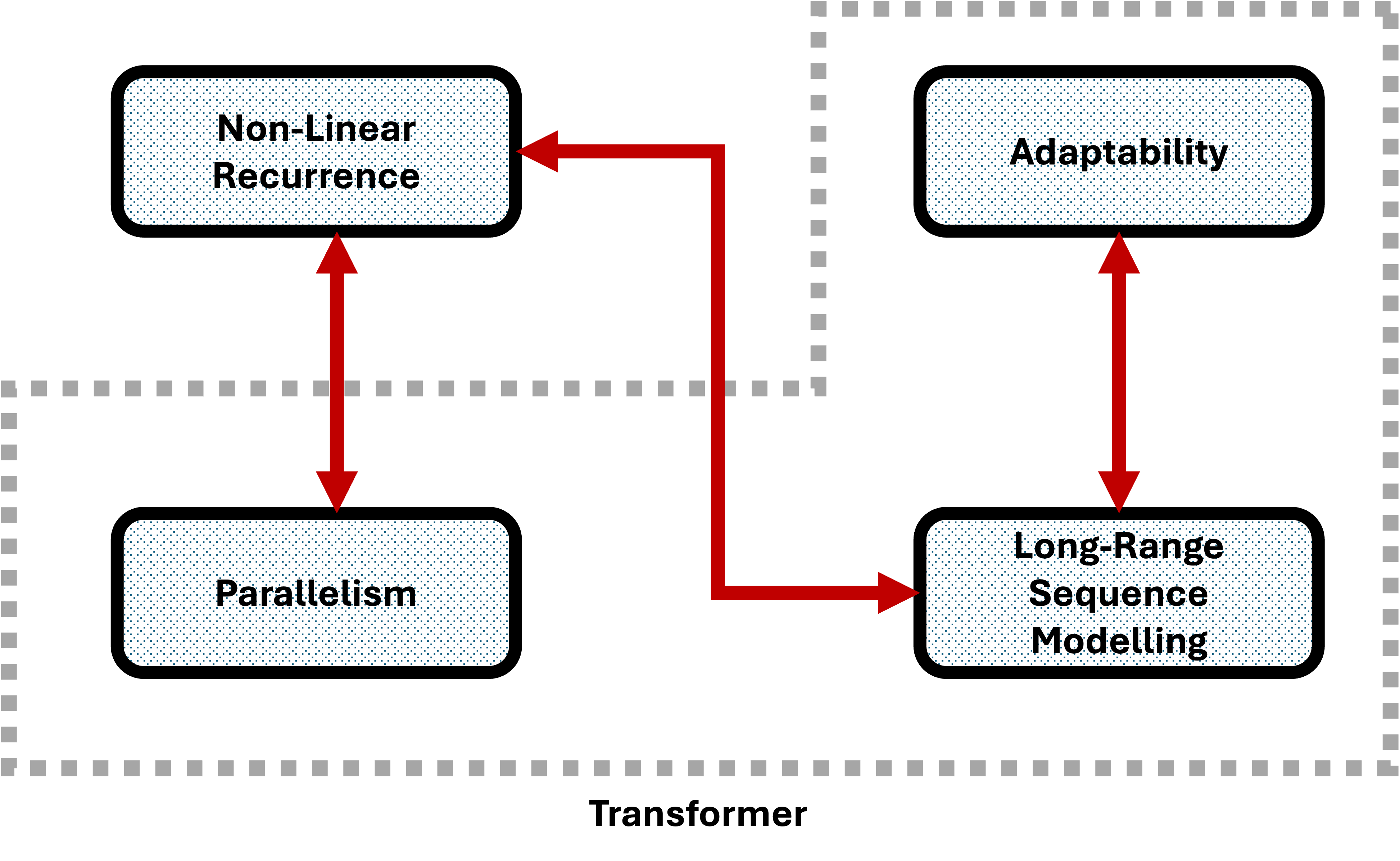}
        \hfill
        \caption{\textbf{Trade-Offs in RNNs and SSMs} This figure shows the four qualities required for state-of-the-art and versatile sequence modelling. Three of them are staples of Transformer performance: adaptability (ii), long-range sequence modelling (iii), and parallelisability (iv). The fourth, non-linear recurrence (i), is thought to enable performance unattainable by vanilla Transformers. A competitive RNN or SSM should possess all four. However, as the red arrows highlight, they are fundamentally at odds with each other. } 
        \label{fig:tradeoffs}
\end{figure}

\textbf{Contributions} To summarise, the overarching goal of this work is to progress neuromorphic computing toward a viable energy-efficient alternative to current large-scale AI systems. That implicitly means directly competing with the Transformer. This can be measured by performance on the current taxonomy of sequence modelling tasks, probing four specific qualities needed to match and even outperform the Transformer (Fig.~\ref{fig:tradeoffs}):

\begin{enumerate}[label=(\roman*)]

    \item  \textbf{Non-linear recurrent dynamics} are needed for reasoning, and their effectiveness is measured by state tracking ability. This is where architectures have a concrete lever to improve on the Transformer;
    
    \item \textbf{Adaptability}, measured by selectivity, is essential for in-context learning and language modelling more broadly;

    \item \textbf{Long-range memory} is necessary for processing long conversations, documents, or high-frequency/long duration signals in general. One cannot deploy a Transformer alternative if it cannot process the same context lengths; 

    \item \textbf{Parallelisability} is crucial for GPU-optimised, fast, and scalable training. Any architecture that cannot scale to the same large-scale training datasets as the Transformer immediately faces a fundamental hurdle to widespread adoption. 
\end{enumerate}

The main challenge is that pairing together all four qualities requires addressing a series of inherent trade-offs. Non-linear dynamics (i) typically require sequential simulation, preventing parallelism (iv). Furthermore, non-linear recurrence (i) and adaptability (ii) usually hinder long-range modelling performance (iii), since one cannot directly use state-of-the-art stable initialisation and parametrisation schemes from linear time-invariant SSMs. Thus, a clear research gap remains in tackling all trade-offs simultaneously, especially for neuromorphic and energy-efficient architectures. To address this, this study provides the following contributions:

\begin{enumerate}
    \item A novel recurrent architecture, the ADaptive with Prescriptive Timescales Network (ADPTNet), is proposed. ADPTNet draws inspiration from the auditory cortex to process inputs using input-invariant timescales while remaining adaptive through neuromodulation. The main building block of the proposed dynamics is the topological conjugation operation, which uses data-dependent similarity transforms constrained to the orthogonal manifold to modulate recurrent dynamics. This is the first study to extend linear Transformer and DeltaNet online update rules to Riemannian manifold optimisation. Furthermore, it is also the first to combine them with Dynamical Systems and Chaos Theory. Accordingly, this work provides theoretical guarantees on parametric control of the Lyapunov spectrum of ADPTNet dynamics as a function of its recurrent eigenspectrum. The theory is then validated empirically, showing how to use SSM initialisation and parametrisation techniques in non-linear ADPTNet without losing performance on long-context sequence modelling (iii). At the same time, the study also shows how the same topological conjugation is sufficiently non-linear (i) to enable state tracking, while also injecting selectivity in the network (ii) to the same level as established selective and recurrent architectures \citep{de2024griffinmixinggatedlinear}, tackling two of the three trade-offs mentioned above.   
    
    \item Two novel efficient parallelisation methods co-designed with ADPTNet and based on DEER \citep{lim2024parallelizing} are proposed. Conv-DEER is a Jacobian-free DEER variation that achieves similar average convergence to Quasi-DEER \citep{gonzalez2024towards}. It leverages ADPTNet's constant, predictable timescales to enable parallelisation via convolution, for the first time for non-linear RNNs. Within each iteration, Conv-DEER adds no computational overhead beyond the forward pass and scan/convolution, and uses only the network's static recurrent eigenspectrum. Forward-DEER is also proposed, reaching even closer convergence behaviour to Quasi-DEER, while still not computing any derivatives or Jacobians, relying solely on the forward dynamics of the network. 

    \item ADPTNet is used as a backbone for an SNN, yielding state-of-the-art accuracy on the Spiking Speech Commands classification task \citep{cramer2020heidelberg}, outperforming current state-of-the-art methods by over a percentage point ($83.56\% \pm 0.15$). 
    
\end{enumerate}

\section{Background} \label{sec:background}

\subsection{Spiking Neural Networks} \label{SNN_formula}

As previously mentioned, the most common backbone for SNNs is the LIF neuron model \citep{eshraghian2023training, gerstner2014neuronal}. In continuous time, sub-threshold LIF membrane voltage $u\in \mathbb{R}$ dynamics are equivalent to an RC circuit (resistance R, capacitance C) with a time constant of $\tau=RC$ (Eq. \ref{RC}) \footnote{Note that Equation~\ref{RC} uses $I$ to denote input current. This notation will later be overloaded to mean the identity matrix, as is convention in linear algebra.}. In practice, Euler discretisation is used for simulation (Eq. \ref{discrete_lif}) and the resting membrane voltage ($u_{rest}$) is set to $0$, yielding an exponential moving average with a decay rate ($\beta$) parametrised by the step size ($\Delta t$) and $\tau$ (Eq. \ref{beta}). As mentioned in Section \ref{sec:intro}, the memory horizon of neurons in the sub-threshold regime can be directly defined by $\beta$. However, once the membrane voltage crosses the firing threshold ($\theta$), a spike is fired, and the voltage is reset (Eq. \ref{spike_fn} and \ref{spiking}), anticipating the difficulties in predicting overall timescales (Section \ref{sec:intro}). In the context of building SNNs with vector states $u \in \mathbb{R}^N$, an additional recurrent feedback connection with trainable synaptic weights ($W_{rec}$) can also be added (Eq. \ref{rsnn}), sometimes referred to as Recurrent SNNs (RSNNs) \citep{bellec2020solution}. Finally, because spiking ($s$) is a non-differentiable Heaviside step function, vanilla backpropagation cannot be directly applied. One could address this by applying a local, brain-inspired learning rule such as Spike-Timing Dependent Plasticity (STDP) \citep{bengio2015stdp}. However, the community has converged on surrogate gradients as a de facto standard owing to their performance being closest to that of standard backpropagation \citep{neftci2019surrogate}. While the binary threshold function is applied in the forward pass, the derivative of a differentiable surrogate function $\sigma$ (e.g., sigmoid, arctan, etc.) (Eq. \ref{surrogate_grad}) is applied backwards, allowing backpropagation to otherwise be applied as usual. 

\begin{equation} \label{RC}
    \tau \frac{du(t)}{dt} =  -(u(t)-u_{rest}) + I(t)R, \space \tau = RC
\end{equation}

\begin{equation} \label{discrete_lif}
    u[t + 1] = \beta u[t] + (1-\beta)I[t+1]
\end{equation}

\begin{equation} \label{spiking}
    u[t+1] = u[t+1] - \theta s[t]
\end{equation}

\begin{equation} \label{beta}
    \beta = e^{\frac{-\Delta t}{\tau}}
\end{equation}

\begin{equation} \label{spike_fn}
s[t] = 
\begin{cases}
   1, \space  u[t] \geq \theta\\
   0, \space u[t] < \theta
\end{cases}
\end{equation}

\begin{equation} \label{rsnn}
    u[t+1] = \beta u[t] + (1-\beta) (I[t] + W_{rec}s[t])- \theta s[t]
\end{equation}

\begin{equation} \label{surrogate_grad}
    \frac{ds[t]}{du[t]} \approx \frac{d\sigma[t]}{du[t]}
\end{equation}

\subsection{State Space Models} \label{ssms}

SSMs have a long history of usage in dynamical systems and control theory \citep{dahleh2004lectures}, along with well-established applications such as neural activity modelling \citep{linderman2019hierarchical}. Much like a vanilla RNN (Eq. \ref{vanilla_rnn}), SSMs are generically a function of a recurrent state $x$ and incoming input $u$ (Eq. \ref{ssm_rec}), with output $y$ typically obtained using a readout matrix and an often left-out skip connection (Eq. \ref{ssm_readout}). However, for the purposes of this work, the SSMs considered do not include a non-linearity ($\sigma$) in the recurrence. Depending on whether model parameters also vary over time, SSMs can be Linear Time-Varying (LTV) or Linear Time-Invariant (LTI) (Eq.~\ref{eq:lti_ssm}). 

\begin{equation} \label{vanilla_rnn}
    f(x_{t+1}) = \sigma(W_{rec}x_t + W_{in}u_{t+1}) 
\end{equation}

\begin{equation} \label{ssm_rec} 
    x_{t+1} = A_{t+1} x_t + B_{t+1}u_{t+1}
\end{equation}

\begin{equation}  \label{ssm_readout}
    y_{t+1} = C_{t+1}x_{t+1} + D_{t+1}u_{t+1}
\end{equation}

\begin{equation} \label{eq:lti_ssm}
    \begin{split}
        x_{t+1} &= A x_t + Bu_{t+1} \\
        y_{t+1} &= C x_{t+1} + D u_{t+1}
    \end{split}
\end{equation}

In the context of deep learning, SSMs were first introduced in the Legendre Memory Unit (LMU) \citep{voelker2019legendre}, in which a linear memory network is coupled to the main RNN. Because of its linearity, the memory unit can be initialised to project the input signal onto a Legendre orthogonal polynomial basis, thereby granting theoretical guarantees for compressing a fixed-width sliding window of the past. Building on the LMU, the S4 architecture \citep{gu2021efficiently} relinquished the non-linear RNN in favour of a purely linear recurrent unit that also extended the initialisation schemes to other orthogonal polynomial bases (i.e., Laguerre, Chebyshev, etc.), and compression strategies beyond fixed windows, such as exponentially decaying \citep{gu2022train}. S4 was a landmark achievement for SSMs, as it was the first to be proven to outperform Transformers on long-range dependency modelling tasks. Follow-up studies have established that the exponentially decaying memory inductive bias is sufficient to achieve this performance \citep{li2022makes, ma2022mega}. Accordingly, the current consensus is that $A$ is typically a complex-valued diagonal matrix, with its eigenvalues distributed close to the unit circle \citep{gu2022parameterization, orvieto2023resurrecting}. One may notice a strong resemblance to sub-threshold LIF dynamics (Section~\ref{SNN_formula}), or their complex-valued variation, Resonate-and-Fire (RF) neurons \citep{orchard2021efficient}. Similar to LIF neurons, SSMs are typically first formulated in continuous time and then discretised for simulation. For notational simplicity, the SSM descriptions included here are in discrete time. 

All SSMs described so far in this section have been LTI architectures. However, from the perspective of Language Models, LTI SSMs are known to lag behind Transformers \citep{poli2023hyena, fu2022hungry} and lack the latter's adaptability (Section~\ref{sec:intro}). To account for this, LTV SSMs have also been developed, first introduced as Mamba \citep{gu2023mamba}. Here, $A_t$, $B_t$, $C_t$ are functions of the input $u_t$, typically the output of a Multi-layer Perceptron (MLP), and thus do not create non-linear dependencies between time steps. Mamba has made progress in closing the gap with Transformers for language modelling and has been widely adopted in hybrid architectures alongside them \citep{botev2024recurrentgemma}. Mamba's success on language tasks has, however, come at the cost of long-range dependency modelling performance, losing the theoretical guarantees of earlier LTI SSMs \citep{yu2026block}. One may also notice that the input-dependent parameters in Mamba are reminiscent of neuromodulated RNNs in neuromorphic literature \citep{costacurta2024structured, alkilany2025neuromodulation}, showcasing the apparent trade-off between adaptability and long-range dependencies referenced in Section~\ref{sec:intro}. 

A key to the success of SSMs has been the GPU-parallelisable nature of linear recurrences for training, while retaining efficient $O(1)$ sequential deployment. The output of an LTI system ($y$) is equivalent to the convolution of an input signal with a global sequence-long kernel defined by the powers of the recurrent matrix \citep{chilkuri2021parallelizing, gu2021combining} (Eq. \ref{sum_view_ssm}). Furthermore, these convolutions can be efficiently computed using element-wise multiplication in the Fourier domain, a near-linear operation in sequence length ($O(N \log(N))$) compared to the Transformer's quadratic scaling, when using Fast Fourier Transforms (FFTs) \citep{cooley1965algorithm}. Mamba-like LTV architectures are not equivalent to convolutions, and instead rely on parallel associative scans for parallelisation \citep{blelloch2002scans}. 

\begin{equation} \label{sum_view_ssm}
    y[t] = CA^tBu[0] + ... + CABu[t-1] + CBu[t]
\end{equation}

\begin{equation} \label{kernel}
    \kappa = (CB, CAB, CA^2B, ..., CA^LB)
\end{equation}

\begin{equation} \label{conv_view}
    y = u \circledast \kappa = FFT^{-1}(FFT(u) \ast FFT(\kappa))
\end{equation}

\subsection{Lyapunov Exponents} \label{sec:lyapunov_exponents}

\begin{figure}[H]
    \centering
        \begin{subfigure}[b]{0.3\textwidth}
            \centering        
            \includegraphics[width=\textwidth]{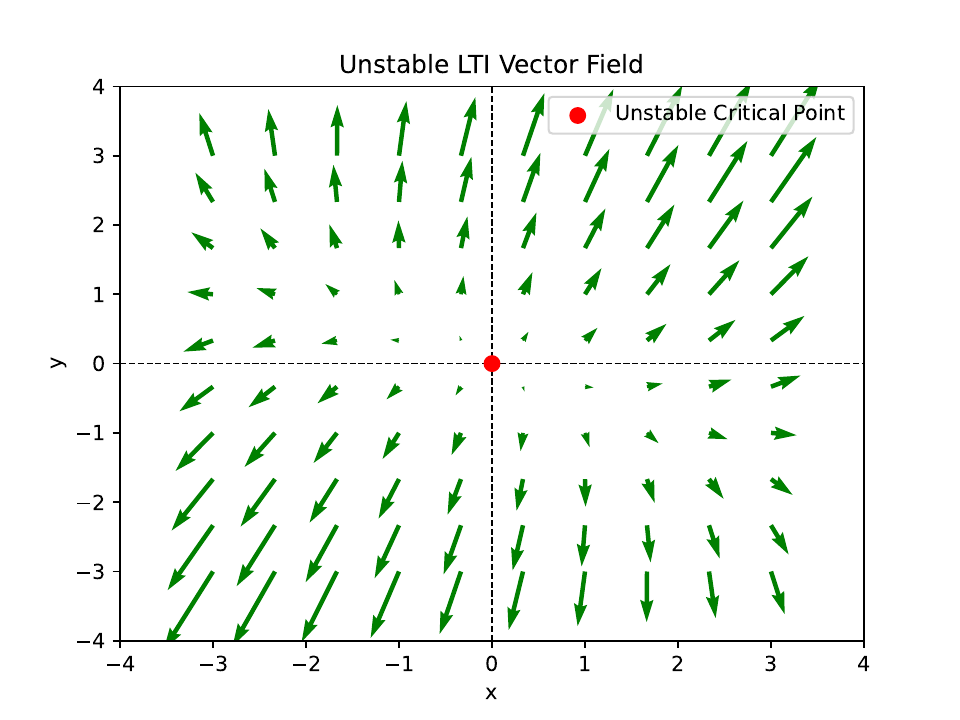}
            \caption{Unstable LTI System}
            \label{fig:unstable_lti}
         \end{subfigure}
        \begin{subfigure}[b]{0.3\textwidth}
            \centering        
            \includegraphics[width=\textwidth]{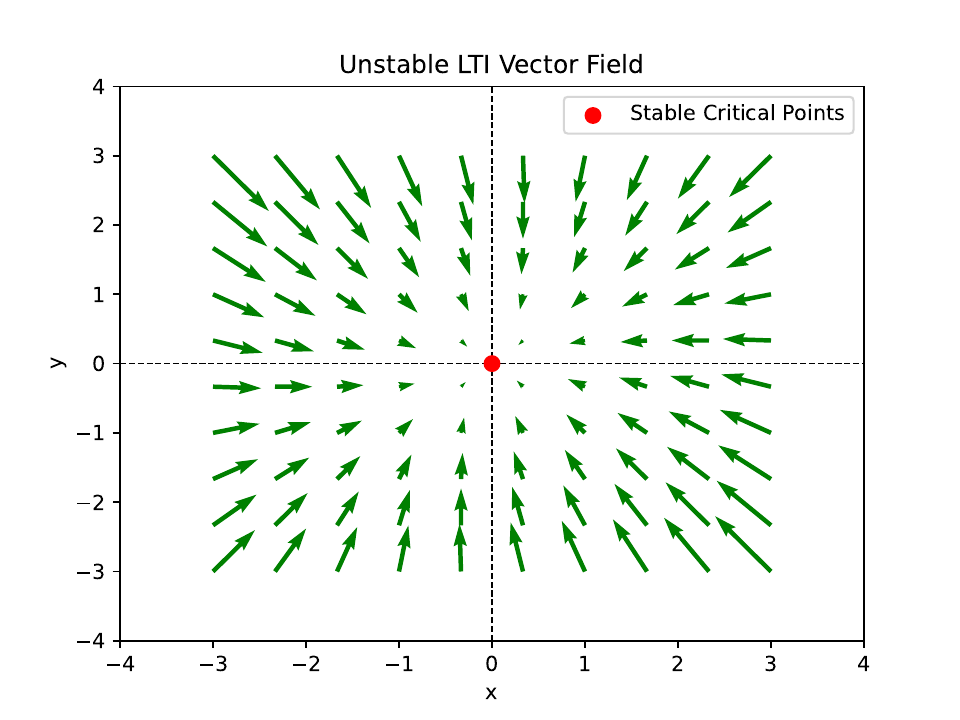}
            \caption{Stable LTI System}
            \label{fig:stable_lti}
        \end{subfigure}
        \begin{subfigure}[b]{0.3\textwidth}
            \centering        
            \includegraphics[width=\textwidth]{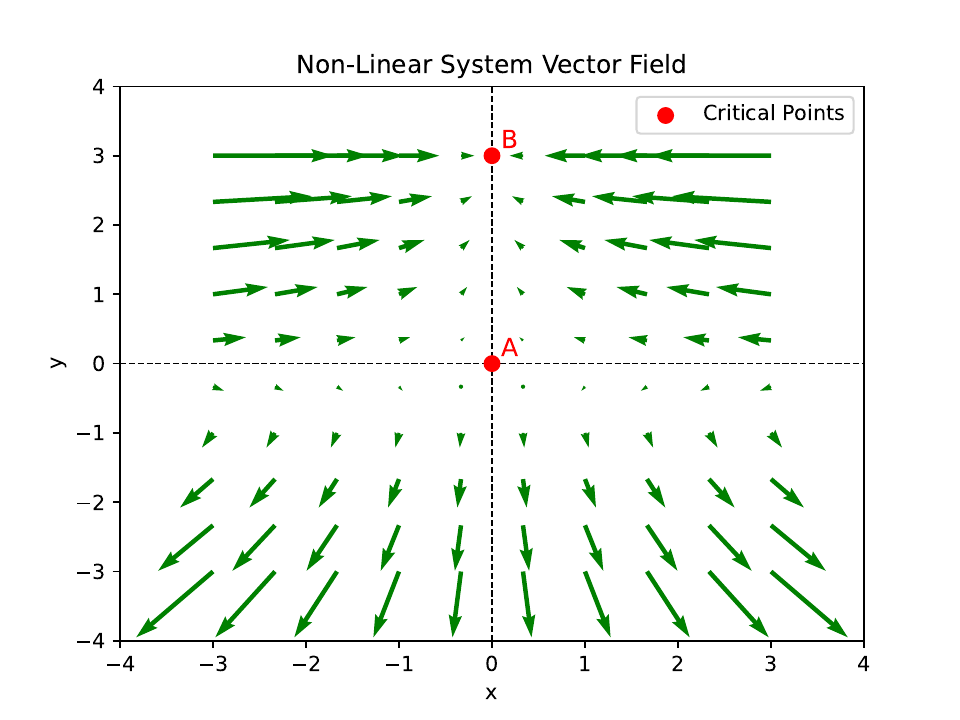}
            \caption{Non-Linear System}
            \label{fig:nonlinear}
        \end{subfigure}
        
     \hfill
        \caption{\textbf{Difference in predictability between LTI and Non-Linear Systems}. Subfigures~\ref{fig:unstable_lti} and ~\ref{fig:stable_lti} show the vector fields of unstable/stable 2-dimensional LTI systems ($\dot{x} = A_{11}x + A_{12}y, \space \dot{y} = A_{21}x + A_{22}y$), where $A$ has positive and negative eigenvalues, respectively. Subfigure~\ref{fig:nonlinear} shows the vector field of a non-linear dynamical system: $\dot{x}=-2x - 3xy, \space \dot{y}=3y - y^2$ \citep{herman2008second}. The system has two critical points: A is a stable point (sink) that attracts trajectories in its vicinity, while B is a saddle point that attracts and repels trajectories depending on their location. }
        \label{fig:dynamical_system}
\end{figure}

As highlighted in Sections~\ref{SNN_formula} and ~\ref{ssms}, the long-term dynamics of LTI systems are fully determined by the eigenvalues of recurrent matrices. In continuous time, if all the eigenvalues have negative real parts, then the system exponentially relaxes back to its resting state in the absence of an external driver (Figure ~\ref{fig:stable_lti}).  Conversely, if there are eigenvalues with positive real parts, inputs may persist indefinitely in memory, causing the system to diverge exponentially from its resting state over time (Figure~\ref{fig:unstable_lti}). Alternatively, in discrete settings, stability is tied to eigenspectra bounded within the unit circle. The eigenvalues also describe the rate of convergence/divergence from the resting state. 

LTV and non-linear systems do not generally benefit from a priori predictors of long-term behaviour. Consider the phase plane in Figure~\ref{fig:nonlinear}. Depending on the state of the system, trajectories may either converge to point B or converge/diverge around the saddle point A.  Moreover, the rate of convergence/divergence also differs by region. While one can approximate locally with Jacobian linearisation in the vicinity of critical points, long-term trajectories cannot be predicted. Instead, one can describe the behaviour of already observed trajectories using Lyapunov exponents. For a given discrete non-linear system $x_{t+1}=f(x_t)$, Lyapunov exponents ($\lambda_k$) are the averaged singular values ($\sigma_k$) of a system's long-term Jacobian (G) (Eq. ~\ref{eq:long_term_jacobian}) for a sampled trajectory (Eq. ~\ref{eq:lyapunov_exp}) \citep{geist1990comparison, engelken2023gradient}.  In particular, the Largest Lyapunov Exponent (LLE or $\lambda_1$) indicates whether the system behaviour is stable or chaotic. First, if the LLE is negative, a trajectory starting from a perturbed state  $x_t + \Delta$ eventually converges with its unperturbed counterpart. The rate of convergence is lower the closer the LLE is to zero, i.e., the system has a longer memory horizon. Second, if the LLE is positive, perturbations result in exponentially diverging trajectories and chaotic dynamics. 

\begin{equation}\label{eq:long_term_jacobian}
    G_{1:N}=\prod_t^{1:N}J_t,  \space J_t = \frac{dx_{t}}{dx_{t-1}}
\end{equation}

\begin{equation}\label{eq:lyapunov_exp}
    \lambda_k = lim_{N\rightarrow \infty}\frac{1}{N}log(\sigma(G_N)_k)
\end{equation}

Lyapunov exponents have been used not only to formalise vanishing and exploding gradient problems, but also to mitigate them. Most notably, Gradient Flossing \citep{engelken2023gradient, engelken2024analyzing} has been proposed as a method to directly control network timescales through optimisation over the Lyapunov spectrum. For example, one can use the LLE magnitude as a loss function to encourage a slowly decaying fading memory, as sought after in SSMs, and use it in conjunction with traditional Backpropagation-Through-Time (BPTT) \citep{werbos1990backpropagation}. However, separate Gradient Flossing epochs are typically required, and computing Jacobians and their singular values is computationally and memory-intensive, slowing training and reducing scalability. Moreover, Lyapunov exponents are descriptors of particular trajectories, which implies that input sequences outside the training set may still produce undesirable dynamics. Finally, as network complexity increases and sequences grow longer, it becomes increasingly difficult to train for precisely targeted Lyapunov exponents, which also tend to drift during training anyway. The solution proposed here aims in part to provide a more scalable and reliable alternative to Gradient Flossing. 

\subsection{Parallel Simulation Algorithms} \label{sec:paral_algorithms}

As highlighted in Section~\ref{sec:lyapunov_exponents}, non-linear dynamics cannot be predicted as easily as linear time-invariant systems. Consequently, computing outputs from non-linear RNNs and SNNs traditionally requires step-by-step sequential simulation, which cannot fully leverage GPUs and thus limits scalability compared to parallel architectures such as SSMs and Transformers \citep{kalchbrenner2016neural, orvieto2023resurrecting}. An emerging solution to bridging this gap has been to reframe RNN simulation through Newton's Fixed-Point Method, iteratively refining the entire trajectory in a single step \citep{lim2024parallelizing, gonzalez2024towards, danieli2025pararnn}. Given a generic discrete non-linear recurrence rule (Eq.~\ref{eq:generic_rec}) and a state-guess tensor for all time steps ($s \in R^{T\times N}$), one can define a residual tensor (Eq.~\ref{eq:residual}) which becomes $\mathbf{}{0}$ if and only if the state-guess is the target trajectory. Each Newton iteration minimises the Euclidean norm of the residual (Eq. \ref{eq:deer_loss}), by finding the optimal update step (Eq. \ref{eq:optimal_update}). The resulting sequence-wide iteration step (Eq.~\ref{eq:deer_step}) is now known as the DEER algorithm \citep{lim2024parallelizing}. While similar Newton-iteration-based parallel-in-time simulation methods have been discussed for decades \citep{gander201550}, DEER is notable for including a GPU-efficient way to compute the large Jacobian and its product with the residual required for the state update (Eq. \ref{eq:optimal_update}). Because the residual Jacobian matrix $\mathcal{J}$ is block lower bidiagonal, its inverse is populated, using the same terminology from Eq.~\ref{eq:long_term_jacobian}, by long-term Jacobians (Eq.~\ref{eq:big_jacobian}), and its product with the residual matrix is equivalent to an associative parallel scan. 

At any Newton iteration $k$, DEER is guaranteed to have converged on the first $k$ time steps of the simulation \citep{gonzalez2024towards}. Therefore, it takes at most as many steps as a step-by-step simulation. However, since DEER iterations are considerably more computationally intensive than applying a single RNN step, speed-ups are only possible if DEER converges in substantially fewer steps. As hinted by the presence of long-term Jacobians $G_{i:j}$ in $\mathcal{J}$, there is a direct connection between the LLE of the recurrent systems simulated (Section~\ref{sec:lyapunov_exponents}) and the convergence rate of DEER \citep{gonzalez2026predictability}. For instance, positive LLEs, i.e., chaotic dynamics, typically cannot be effectively parallelised with DEER. Intuitively, errors in earlier steps propagate across time and can affect all subsequent states, increasing the number of fixed-point iterations and approaching the total number of sequential steps. For the same reason, even with a negative LLE, the longer the system's memory horizon, the more future states are affected, and hence, more iterations are required for convergence.

\begin{equation} \label{eq:generic_rec}
    x_{t+1} = f(x_t), \space x_i \in \mathbb{R}^N
\end{equation}

\begin{equation} \label{eq:residual}
    r(s) = [f(s_0) - s_1, f(s1) - s2, ... , f(s_{T-1})-s_T], \space r(s)\in \mathbb{R}^{T \times N}
\end{equation}

\begin{equation} \label{eq:deer_loss}
    \mathcal{L}(s) = \frac{1}{2}||r(s)||^2
\end{equation}

\begin{equation} \label{eq:optimal_update}
    \Delta s= \underset{h}{\min}(\lVert r(s+h) \rVert_2^2) = \underset{h}{\min}(\lVert (r(s) + \frac{dr(s)}{ds}h \rVert_2^2) = -\left(\frac{dr(s)}{ds}\right)^{-1}r(s)
\end{equation}

\begin{equation} \label{eq:deer_step}
    s^{(i+1)} = s^{(i)} +\Delta s^{(i)} = s^{(i)} -\left(\frac{dr(s^{(i)})}{ds^{(i)}}\right)^{-1}r(s^{(i)})
\end{equation}

\begin{equation} \label{eq:big_jacobian}
    \mathcal{J}^{-1} = \left( \frac{dr(s)}{ds} \right)^{-1} = \begin{pmatrix}
        I& \mathbf{0} & \mathbf{0} & \dots & \mathbf{0}\\ 
        -J_2 & I & \mathbf{0} & \dots & \mathbf{0} \\
    \mathbf{0} & -J_3 & I & \dots & \mathbf{0} \\
        \vdots & & \ddots & & \vdots \\
        \mathbf{0} & \mathbf{0} & \dots & -J_T & I \\
    \end{pmatrix}^{-1} = \begin{pmatrix}
        I & \mathbf{0} & \mathbf{0} & \dots& \mathbf{0} \\
        G_{2:2} & I  & \mathbf{0} & \dots& \mathbf{0} \\
        G_{2:3} & G_{2:2} & I  & \dots& \mathbf{0} \\
        \vdots & & \ddots & & \vdots \\
        G_{2:T} & G_{2:T-1} & G_{2:T-2} & \dots & I
    \end{pmatrix}
\end{equation}


DEER suffers from two main drawbacks. Computing and storing full Jacobians $J_t$ for each time step $t$ in each batch during training limits scalability. Furthermore, the long-term Jacobians $G_{i:j}$ in $\mathcal{J}$ can become numerically unstable even if the system is asymptotically stable, depending on the presence of volatile transitory dynamics. To address the former, \citet{gonzalez2024towards} showed that Quasi-Newton methods, which replace the full $J_t$ Jacobians with diagonal approximations, have similar convergence properties to the original DEER algorithm, reducing compute and memory overheads. To tackle numerical instability, the study also proposes a Kalman filter-based approach that dampens all $J_t$ eigenvalues by a fixed factor, thereby preventing the $G_{i:j}$ terms from exploding. This work aims to provide more efficient heuristics for both challenges (see Section~\ref{sec:alt_deer}).

\subsection{Orthogonal Matrices} \label{sec:orthogonal}

The Orthogonal Group ($O(n)$) comprises matrices $Q$ with the property that $QQ^T=I$. From a neural network perspective, orthogonal matrices are of interest because, multiplied by any vector $v\in \mathbb{R}^n$, they conserve Euclidean norms and angles. In other words, because they have eigenvalues restricted to $\{\pm 1\}$ and singular values in $\{1\}$, a system that multiplies an input vector by a sequence of orthogonal matrices (e.g., $Q_TQ_{T-1}...Q_1v$) has a memory that neither decays nor explodes. Such a dynamical system, where all recurrent Jacobians $J_t$ are in $O_n$, would have a Lyapunov Spectrum (Section~\ref{sec:lyapunov_exponents}) equal to $\mathbf{0}$, a property referred to as criticality or the edge of chaos \citep{engelken2023lyapunov}.

One could be tempted to argue that, for example, an SSM with an orthogonal recurrent matrix $A$ (Eq.~\ref{ssm_rec}) would have favourable memory properties. However, since no information dissipates and new input signals $u$ continue to arrive, the state could grow asymptotically, leading to numerical instabilities \citep{orvieto2023resurrecting}. Therefore, orthogonal recurrent matrices have been mostly found in non-linear RNNs (Eq.~\ref{eq:ortho_rnn}) \citep{arjovsky2016unitary, helfrich2018orthogonal, wisdom2016full}. Because of the added element-wise non-linear function $\sigma$, the recurrent Jacobian of the network $J_t$ becomes a product between a diagonal and orthogonal matrix (Eq.~\ref{eq:ortho_jac}). If one uses a common activation function, such as a sigmoid, tanh or ReLU, the spectral norm (i.e., maximum singular value) of the diagonal derivative term $D_t$ is bounded by $1$. In turn, because the recurrent weight is orthogonal, the spectral norm of every recurrent $J_t$ and, implicitly, long-term $ G_{1:T}$ are also $\le 1$ (Eq.~\ref{eq:ortho_long_term}). Therefore, the Lyapunov spectrum is negative, and the system is stable with fading memory and no exploding gradients (Section~\ref{sec:lyapunov_exponents} and \citet{arjovsky2016unitary}). 

Here, it is worth emphasising that orthogonal RNNs cannot guarantee that all timescales of the network are slowly decaying, long-term memory. As Eq.~\ref{eq:ortho_long_term} suggests, there is no lower bound on the norm of each $J_t$.  For example, one could have all ReLU activations set to zero at once and thus all information forgotten. Orthogonal parametrisation does not solve such "dead neuron" pathologies \citep{douglas2018relu}. Furthermore, the power of initialisation schemes for SSMs lies in setting the distribution of all recurrent eigenvalues near the unit circle, thereby enabling multiple long-term timescales \citep{orvieto2023resurrecting}. Even if there are no time steps where $D_t$ is $\mathbf{0}$, then the inequality in Eq.~\ref{eq:ortho_long_term} only shows the upper bound of slow timescales being $1$, but does not help with setting the rate of decay for all other timescales besides the slowest. Regularisation methods like Gradient Flossing \citep{engelken2023gradient} can be added to mitigate this, but they cannot offer any guarantees about the system's specific timescales either (Section~\ref{sec:lyapunov_exponents}). The solution proposed in this work uses orthogonal matrices in a way that also guarantees a lower bound on the Lyapunov Spectrum.

\begin{equation} \label{eq:ortho_rnn}
    x_{t+1} = \sigma(Qx_t + W_{in}u_{t+1}) 
\end{equation}

\begin{equation} \label{eq:ortho_jac}
    J_t = \frac{d\sigma}{dx_{t-1}}Q = D_tQ 
\end{equation}

\begin{equation} \label{eq:ortho_long_term}
    \lVert G_{1:T} \rVert = \lVert \prod_t^{1:T}J_t \rVert = \lVert \prod_t^{1:T} D_tQ\rVert \le \prod_t^{1:T}\lVert D_tQ \rVert = \prod_t^{1:T}\lVert D_t \rVert
\end{equation}

Training neural networks that use orthogonal parametrisations takes advantage of the fact that the $O_n$ also carries a Riemannian manifold structure, which enables smooth optimisation \citep{absil2008optimization}. Over the general linear group $GL(n)$, one can define a dynamical system whose state $M \in \mathbb{R}^{n \times n}$ evolves in the steepest descent direction that minimises a scalar objective function $f$ (Eq.~\ref{eq:continuous_gds}). Discretisation with a step size $\eta$ yields the traditional gradient descent algorithm (Eq.~\ref{eq:vanilla_gds}). Importantly, the update step automatically produces a valid new matrix $M_t$ in $GL(n)$.

\begin{figure}[H]
    \centering
        \centering        
        \includegraphics[width=0.5\textwidth]{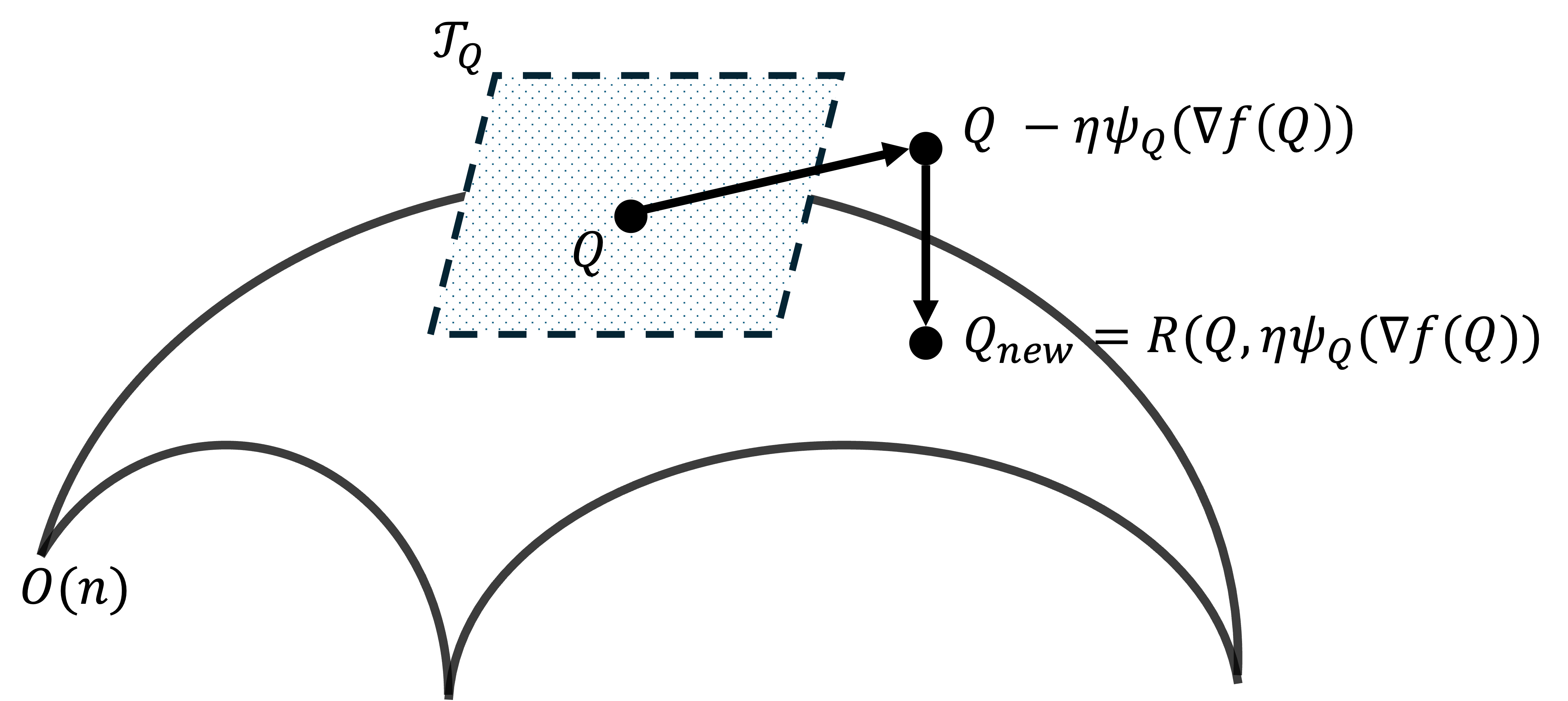}
        \hfill
        \caption{\textbf{Riemannian Gradient Descent over the Orthogonal Manifold} Given an initial orthonormal matrix $Q$ and a Euclidean gradient $\nabla f(Q)$, Riemannian Gradient Descent is performed by first projecting the Euclidean gradient to the tangent space to $O(n)$ at $Q$, i.e., $\mathcal{T}_Q$ \citep{absil2008optimization}. The updated $Q$ from the tangent space is then projected back to the manifold using a retraction function $R$}. 
        \label{fig:ortho_gd}
\end{figure}

That is not the case in Riemannian optimisation over orthogonal matrices, as taking a Euclidean gradient step could break the orthogonality constraint $(Q - \eta \nabla f(Q)) (Q - \eta \nabla f(Q))^T \neq I$.  To keep trajectories constrained to $O_n$, any point on it has to satisfy $QQ^T =I$, and therefore any infinitesimal direction of travel from that point has to respect Eq.~\ref{eq:manifold_cond}. In practice, this is obtained using a projection onto the tangent space to $O_n$ at $Q$ ($\mathcal{T}_QO_n$) (Eq.~\ref{eq:tangent_proj}), which is essentially a mapping to a skew-symmetric matrix (a matrix $X$ with the property that $X^T+X=0$). Once a valid directional derivative is computed, one can integrate over it to obtain the final modified orthogonal matrix. In discrete gradient descent, this is done by projecting the updated matrix back onto the manifold using a retraction ($\mathcal{R}(Q, \psi_Q(\nabla f(Q))$). The exact solution is the matrix exponential (Eq.~\ref{eq:exponential_retraction}) \citep{absil2008optimization}. This is evidently computationally expensive, so it is often substituted with a Taylor approximation in the form of the Cayley map (Eq.~\ref{eq:cayley_map}) \citep{ablin2022fast}. 

\begin{equation} \label{eq:continuous_gds}
    \frac{dM}{dt} = -\nabla f(M)
\end{equation}

\begin{equation} \label{eq:vanilla_gds}
    M_t = M_{t-1} - \eta \nabla f(M_{t-1})
\end{equation}

\begin{equation} \label{eq:manifold_cond}
    \frac{d(QQ^T)}{dt}=\frac{dI}{dt} \Rightarrow \frac{dQ}{dt}Q^T + Q\left( \frac{dQ}{dt} \right)^T = \mathbf{0}
\end{equation}

\begin{equation} \label{eq:tangent_proj}
    \psi_Q(\nabla f(Q)) = \frac{1}{2}(\nabla f(Q)Q^T - Q\nabla f(Q)^T)
\end{equation}

\begin{equation} \label{eq:exponential_retraction}
    \mathcal{R}_{exp}(Q, \eta\cdot\psi_Q(\nabla f(Q)) = exp(\eta \cdot\psi_Q(\nabla f(Q))Q
\end{equation}

\begin{equation} \label{eq:cayley_map}
    \mathcal{R}_{aprox}(Q, \eta \cdot \psi_Q(\nabla f(Q)) = (I - \frac{\eta}{2}\psi_Q(\nabla f(Q))(I + \frac{\eta}{2}\psi_Q(\nabla f(Q))Q
\end{equation}

\section{Related Work}

\subsection{Neuromodulation in Neural Networks} \label{sec:neuromodulation}

The term "neuromodulation" is used in this work to refer to neural architectures where parameters vary during training and inference as a function of the network input. The focus here is mostly on neuromodulated RNNs, where both the recurrent weights and the state of the network evolve as functions of the previous weights and state, and incoming inputs (Eq.~\ref{eq:generic_neuromodulation_rnn}). As previously highlighted in Sections~\ref{sec:intro} and \ref{ssms}, this typically takes the form of gating or hypernetworks \citep{ha2016hypernetworks}.

\begin{equation}\label{eq:generic_neuromodulation_rnn}
    \begin{aligned}
    W_{t+1} &= f(W_t, x_t, u_{t+1}) \\
    x_{t+1} &= g(W_tx_t, u_{t+1})
    \end{aligned}
\end{equation}

First popularised by LSTMs and GRUs \citep{hochreiter1997long, cho2014properties}, gating is now considered any operation based on Hadamard products that scales activations in an input-dependent manner. As argued in \citet{gu2023mamba}, the most basic example of it is the Gated Linear Unit activation \citep{dauphin2017language} (Eq.~\ref{eq:glu}), comprising two linear layers, where one also passes through a sigmoid function and element-wise scales the output of the other. In recurrent models, the goal of gating is most commonly to enable dynamic forgetting, in which the current state and/or input determine which information decays more quickly \citep{jing2019gated, gu2023mamba}. Hypernetworks go a step further by generating all model weights using a smaller subnetwork, rather than performing only element-wise scaling. In their original conception, the sub-networks only took fixed qualities, such as depth, as input to produce parameters. However, more recently, they have also become data-controlled \citep{chauhan2024brief}. Modern SSMs have been blurring the distinctions between the two research directions, with recurrent weight matrices more and more commonly becoming diagonal, thereby effectively reducing hypernetworks down to forget gates \citep{gu2023mamba, zhu2024scalable}. 

\begin{equation} \label{eq:glu}
    GLU(x) = (W_Ax+b_A) \odot \sigma(W_Bx+b_B)
\end{equation}

All forms of neuromodulation described so far are heuristic in nature: the goal is to explore the general effect of data-controlled parameters on model performance, for example, on tasks requiring selective forgetting \citep{jelassi2024repeat}, or continual learning \citep{von2019continual}. More recently, however, increased attention has been given to endowing neuromodulation with more principled structure, with the deep learning and neuromorphic research communities adopting slightly diverging goals and understandings. 

In mainstream deep learning, structured neuromodulation has been widely explored in associative recall (AR). AR tasks test a model's ability to store key-value pairs in working memory and then retrieve the correct value when prompted with a query \citep{arora2024zoology}. Because all keys, values, and queries are provided at inference time, the network is said to "learn" the associations "in-context". While content-addressable associative memory has been a popular subject of study since at least the advent of Hopfield Networks \citep{hopfield1982neural}, this modern framing emphasises performing AR online and at scale. To generalise, LLMs have to be able to dynamically answer prompts such as "The apple is red. What colour is the apple?" without necessarily relying on prior knowledge. The Transformer \citep{vaswani2017attention} is explicitly formulated in these terms. Its core operation, self-attention, linearly projects the input sequence ($X \in \mathbb{R}^{T \times d}$) to keys ($K=XW_k$), queries ($Q=XW_q$), and values ($V=XW_v$), where  $W_k, W_q, W_v\in \mathbb{R}^{d \times d_k}$. The product of the keys and queries yields the attention matrix, which then mixes all values across the sequence (Eq.~\ref{eq:self-attention}). Because the attention matrix is a function of the input, some have argued that the Transformer is a hypernetwork \cite{schug2025attention}. Furthermore, given the explicit associative memory formulation, it can also be considered an example of structured neuromodulation, using the terminology of this study. 

\begin{equation} \label{eq:self-attention}
     O = \mathtt{SelfAttention}(Q, K, V) = \mathtt{Softmax}(\frac{QK^T}{\sqrt{d_k}})V
\end{equation}

The $QK^T \in \mathbb{R}^{T \times T}$ product in self-attention (Eq.~\ref{eq:self-attention}) incurs a quadratic computational and memory complexity in sequence length $T$, making efficient alternatives to Transformers highly desirable. Linearised Transformers \citep{tay2022efficient, peng2023rwkv, zhai2021attention} aim to reduce this overhead by not computing a $T \times T$ matrix and replacing the row-wise $\mathtt{Softmax}$ over the attention matrix with non-linearities over $K$ and $Q$ individually \citep{katharopoulos2020transformers}. In the original Linear Transformer (Eq.~\ref{eq:aft}) \citep{katharopoulos2020transformers}, for example, the attention matrix is a sum over $d_k \times d_k$ outer products of individual key-value pairs ($k_t$, $v_t$).  $\mathtt{Softmax}$ normalisation is only applied between keys, and the query is transformed using a generic element-wise non-linearity (e.g., sigmoid). These changes reduce the computational and memory requirements to linear scaling in sequence length. 

However, interest in linearised Transformers of this original form waned under pressure from developments in hardware-optimised vanilla self-attention implementations \citep{dao2022flashattention} and the advent of SSMs (see Section~\ref{ssms}). One could argue that this ebb ended with the landmark connection between linearised Transformers and the delta rule \citep{widrow1988adaptive} in \citet{schlag2021linear}. The structure of the Linear Transformer, and related architectures, is amenable to a recurrent formulation (Eq.~\ref{eq:rec_linear_transformer}). However, unlike traditional RNNs, the recurrence is not over a vector state but rather over a so-called "fast weights" matrix $W$, which evolves over time during inference. As highlighted in \citet{schlag2021linear}, without the $\mathtt{Softmax}$ normalisation in Eq.~\ref{eq:rec_linear_transformer}, this weight-dynamics view of Linear Transformers is remarkably similar to using online gradient descent over a recall error correction objective (Eq.~\ref{eq:delta_rule}). The difference is that while Linear Attention keeps adding key-value pairs to memory, the gradient descent step first retrieves and removes the current value associated with the key (via the Householder transformation $I-k_{t+1}k_{t+1}^T$ \footnote{The key $k_{t+1}$ is typically normalised \citep{yang2024parallelizing}}) before writing a new associated value  ($k_{t+1}v_{t+1}^T$). The result is that only changes in the value vector $v_{t+1}$ are stored, hence the name of the delta rule. This architecture, referred to as DeltaNet \citep{yang2024parallelizing}, is now an important pillar of efficient Transformer development, with significant research focused on In-Context Learning and Test-Time Training directly adopting its online optimisation framing \citep{behrouz2025s}. Moreover, a priority has been to improve compression capabilities in theoretically grounded ways.  For instance, Lattice \citep{karami2025lattice} treats the columns of $W_t$ as individual memory slots and only stores new information that is orthogonal to each of them. MesaNet \citep{von2025mesanet} extends online optimisation from single gradient descent steps at each time step to a regression over all preceding time steps, producing an optimal readout matrix for the query. 

\begin{equation} \label{eq:aft}
    o_t = \frac{1}{ \left( \sum_{i=1}^{t}e^{k_t} \right)\cdot \sigma(q_t)}\left( \sum_{i=1}^t ( e^{k_i} \otimes v_i) \right) \sigma(q_t)
\end{equation}

\begin{equation} \label{eq:rec_linear_transformer}
    \begin{array}{l}
            W_{t+1} = W_t + e^{k_{t+1}}\otimes v_{t+1} \\
            z_{t+1} = z_t + e^{k_{t+1}} \\
            o_{t+1} = \frac{1}{z_{t+1}\cdot \sigma(q_{t+1})} W_{t+1}\sigma(q_{t+1})
    \end{array}
\end{equation}

\begin{equation} \label{eq:delta_rule}
    \begin{aligned}
    W_{t+1} &= W_t -\frac{\eta}{2} \nabla_{W_t}(\lVert W_tk_{t+1} - v_{t+1} \rVert^2) \\
    &= W_t - \eta(W_tk_{t+1}k_{t+1}^T - k_{t+1}v_{t+1}^T) \\
    &= W_t(I - \eta k_{t+1}k_{t+1}^T) + \eta k_{t+1}v_{t+1}^T \\
    \end{aligned}
\end{equation}

It is worth emphasising that Linear Transformer and DeltaNet-inspired models do not fully fit the neuromodulated RNN template from  Eq.~\ref{eq:generic_neuromodulation_rnn}. In these architectures, the output at each time step $t$ is the readout product of $W_{t}$ and $q_{t}$ (Eq.~\ref{eq:rec_linear_transformer}) and is only fed to the next layer, not the next time step. In other words, since the query only depends on the current input and is not recurrent, the weights $W_t$ are the state of the system. This differs from a neuromodulated RNN, where recurrence also applies to a vector state, making $W_t$ both a state and a modulator of the network's dynamics. Because linearised Transformer recurrences rely mostly on addition (writing key-value outer products to memory) and, when present, multiplications are restricted to Householder reflections, the main concern is keeping the norm of $W_t$ bounded. In contrast, a multiplicative interaction between recurrent states and weights in neuromodulated RNNs warrants careful consideration of long-term Jacobian stability when parametrising $W_t$ (Section~\ref{sec:lyapunov_exponents}). In practical terms, this means that in a recurrent Transformer, the eigenvalues of $W_t$ itself are typically irrelevant. In a neuromodulated RNN, however, they can play an important role in how well the model can perform (Section~\ref{sec:orthogonal}). This difference helps explain the approach neuromorphic research has taken regarding structured neuromodulation.

In a significant body of neuroscience-inspired literature, structured synaptic neuromodulation has taken the form of constraining $W_t$ in RNNs, rather than adopting an associative recall formalism as in deep learning research. For instance, relevant to this study, in \citet{zador2025walking}, neuromodulated recurrent weights are constrained to linear interpolations of basis points on a rigid matrix manifold (e.g., a straight line or an ellipse). In \citet{costacurta2024structured}, the recurrent weights $W_t$ are constrained to rank-K matrices, where each low-rank component is scaled by an individual neuromodulatory signal. 

Although not within a single, comparable architecture, the fundamental concepts behind DeltaNet are also present in neuromorphic research. Evidently, associative learning algorithms such as the Delta and Hebbian rules originated in neuroscience \citep{hebb2005organization, kang2024distinguishing}. Moreover, SNNs benefit from additional biologically plausible online and local rules that can leverage temporal information, such as STDP \citep{bengio2015stdp} or e-prop \citep{bellec2020solution}, which have yet to see substantial adoption in mainstream deep learning. However, all these learning algorithms are not typically paired with Transformer-like Key-Value-Query projections and are generally intended as full backpropagation replacements rather than inference-time augmentations. In-context gradient descent is present to some extent in SNNs, though. For instance, Parameter-Free Attention \citep{sun2025towards} minimises, at inference, a linear-separability loss that implements lateral inhibition for input currents to SNN layers. 

\subsection{Parallelisable Architectures} \label{sec:related_parallel}

As mentioned in Section~\ref{sec:intro}, GPU-parallelisation is one of the core advantages that have fuelled the popularity of Transformers, and thus is a prerequisite for any efficient alternative to be competitive. Furthermore, SSMs have shown through equivalent parallel and iterative formulations that one need not trade off training and deployment efficiency (Section~\ref{ssms}). In fact, parallel/recurrent duality has emerged as an essential characteristic of state-of-the-art efficient Transformer contenders. 

Both Linear Transformers and DeltaNet variants require only linear operations between time steps, so, unsuspectingly, they also exhibit parallel/sequential duality (Section~\ref{sec:neuromodulation}). However, naive full parallelisation can lead to similar quadratic-scaling issues to those of traditional self-attention. Recurrence rules such as Eq.~\ref{eq:rec_linear_transformer} or \ref{eq:delta_rule} assume causality, and therefore cannot be directly mapped to an efficient matrix multiplication with linear scaling in sequence length (e.g., $Q(K^TV)$, $K^TV \in \mathbb{R}^{d_k \times d_k}$), because a binary causal mask $M$ is required (Eq.~\ref{eq:causal_masking}) \citep{yang2023gated}. Instead, one can hedge between parallel and sequential forms through chunk-wise parallel simulation \citep{yang2024parallelizing, yang2025gated}. Each segment of length $C$ is processed in parallel at a quadratic in-chunk-size cost, with the final state $W_{C*k}$ of each chunk $k$ then recurrently fed into the next as its initial state. However, the overall complexity across the entire sequence length becomes sub-quadratic.

\begin{equation} \label{eq:causal_masking}
    O = (QK^T\odot M_{causal})V
\end{equation}

As highlighted in Section~\ref{sec:paral_algorithms}, there has also been a growing interest in parallelising non-linear dynamics. Much like the aforementioned DEER algorithm \citep{lim2024parallelizing} and its variations \citep{gonzalez2024towards, danieli2025pararnn}, many parallelisation methods for non-linear systems typically trade fully sequential, step-by-step execution for large, whole-sequence fixed-point iterations. Notably, \citet{schone2025implicit} introduces implicit SSMs, which iteratively apply an SSM layer to the entire sequence until convergence to a fixed point. While implicit SSMs have been shown to simulate certain non-linear dynamics, unlike DEER-based algorithms, they are not 1:1 interpretable equivalents to arbitrary non-linear RNNs. In implicit SSMs, non-linear dynamics are intrinsically encoded in fixed-point iterations, even if the simulation is sequential, i.e., each sequential simulation step still requires multiple fixed-point iterations for itself. 

While the ultimate goal of SNNs is generally considered deployment on specialised low-power hardware, the most popular training paradigm remains backpropagation (see Section~\ref{SNN_formula}). Hence, to improve scalability, GPU parallelism at training time has also gained interest within the neuromorphic community. The most direct, and perhaps most popular, strategy to achieve this has been to build SNNs based on mainstream parallel architectures. First introduced in \citet{stan2024learning}, SSM-based SNNs can be trained using the same convolutional approach as standard SSMs (see Section~\ref{ssms}) \citep{shen2025spikingssms, bal2025p}. Analogously, spiking counterparts to vanilla \citep{zhou2022spikformer} and linearised \citep{yao2023spike} Transformers have also been proposed, similarly inheriting parallelism from the baseline architectures. However, a direct spiking adaptation of DeltaNet has not yet been studied to the best of the authors' knowledge. 

DEER-like algorithms have also found counterparts in SNN literature. As the name suggests, Fixed-Point Parallel Training (FPT) \citep{feng2025efficient} follows the same principles as DEER (Section~\ref{sec:paral_algorithms}). However, unlike DEER, FPT does not require computing Jacobians $J$ or long-term Jacobians $G$ in $\mathcal{J}$ (see Eq.~\ref{eq:big_jacobian}). Instead, FPT takes advantage of LIF sub-threshold dynamics, and propagates spikes "in the future" using powers of the membrane decay constant $\beta$ (Eq.~\ref{eq:fpt_matrix}). To handle the spiking non-linearities, where DEER would compute $J_t$, FPT simply propagates the spikes themselves, or, during training (for numerical stability), a smooth approximation of the Heaviside step function $s(u-\theta)$ (e.g., sigmoid). Therefore, instead of using a residual tensor for the fixed-point iteration, FPT directly treats all membrane voltages across time $u\in \mathbb{R}^{T \times d}$ as weighted sums of previous output spikes $s(u-\theta)$ and input currents $i$ (Eq.~\ref{fpt}). While not computing any Jacobians reduces memory and computational overheads, it makes the critical assumption that the network has no recurrent weights (see Eq.~\ref{rsnn}), and thus that all neurons are independent. This effectively limits expressivity compared to vanilla RNNs. 

\begin{equation} \label{eq:fpt_matrix}
    \mathbf{B} = \begin{pmatrix}
        1 & \mathbf{0} & \mathbf{0} & \dots& \mathbf{0} \\
        \beta & 1  & \mathbf{0} & \dots& \mathbf{0} \\
        \beta^2 & \beta & 1  & \dots& \mathbf{0} \\
        \vdots & & \ddots & & \vdots \\
        \beta^{T-1} & \beta^{T-2} & \beta^{T-3} & \dots & 1
    \end{pmatrix}
\end{equation}

\begin{equation} \label{fpt}
    \begin{aligned}
    u_1 &= Bi\\
    u_{2} &= -\theta(\mathbf{B} - I)s(u_{1} -\theta) + \mathbf{B}i \\
    \vdots \\
    u_{k} &= -\theta(\mathbf{B} - I)s(u_{k-1} -\theta)  + \mathbf{B}i \\
    \end{aligned}
\end{equation}

It is worth noting that SNNs also benefit from an additional, idiosyncratic simulation paradigm. Event-based simulations reduce execution time by computing between-spike sub-threshold membrane evolution using closed-form solutions, thereby leveraging spiking sparsity to skip simulation steps \citep{wunderlich2021event}. Moreover, they rely on producing sequences of exact spike timings in continuous time rather than requiring temporal discretisation \citep{engelken2023sparseprop}. Event-based methods have also been combined with DEER-like Newton fixed-point iterations to produce all spikes in parallel, providing further speed-ups \citep{morrill2026bullet}. 

Because only spike times are computed in event-based simulation, training uses these continuous time points rather than surrogate gradients to account for spike non-differentiability during backpropagation \citep{wunderlich2021event, meszaros2025efficient}. However, in practice, this also means that backpropagation cannot effectively control the creation of new spikes, only shift existing ones forwards/backwards in time. Spikes can be deleted by moving them outside the task's time window, thereby causing neurons to become silent. Once silent, neurons will remain silent (i.e., dead neurons) because non-spiking sub-threshold intervals do not receive any learning signals. \citep{eshraghian2023training, wenig2026quadratic, klos2025smooth}. Therefore, while event-based simulation holds tremendous potential, surrogate gradient methods are still generally preferred for achieving state-of-the-art accuracy \citep{eshraghian2023training}.

In this context, event-based discretisation could be considered a middle ground. Traditional discretisation schemes, such as zero-order hold (ZOH) or Euler, generally employ fixed step sizes, imposing regular sampling on the data that can be processed by systems parametrised with them, such as SSMs. In contrast, neuromorphic sensors such as retina-inspired Dynamic Vision Sensor (DVS) Cameras \citep{amir2017low} produce irregular, asynchronous, and sparse events at high temporal resolution. A frame-based approach with regular sampling, such as a traditional SSM, would waste resources processing a large number of empty input frames. Similar to event-based simulation, event-based discretisation can be applied to SSMs to reduce computational overhead by evolving their linear dynamics in closed form between input spikes \citep{schone2024scalable}. However, these event-based SSMs operate only on precise input spike times, not on internal network spikes, unlike full event-based simulation. In other words, this method would not be used to speed up SNN execution in general. 

\subsection{Long Range Modelling in SNNs} \label{sec:long_range_snn}

Learning dependencies over long temporal horizons is difficult in non-linear recurrent architectures, owing to the inherent unpredictability of dynamical systems (see Sections~\ref{sec:lyapunov_exponents} and \ref{sec:orthogonal}).  Being non-linear RNNs themselves (Section~\ref{SNN_formula}), LIF-based SNNs are no exception. Accordingly, mirroring the progress from vanilla RNNs to LSTMs, Long Short-Term Memory Recurrent SNNs (LSNNs) were a significant milestone in tackling long sequences \citep{bellec2018long, yin2021accurate}. Not to be confused with the notion of adaptability from Section~\ref{sec:intro}, LSNNs use adaptive LIF neurons (ALIF), where the firing threshold temporarily rises $\theta[t]$ after each spike emission (Eq.~\ref{eq:alif}). Evolving with a slower timescale $\alpha$ than the membrane voltage decay rate $\beta$, the threshold becomes a form of linear long-term memory. Because higher thresholds increase firing sparsity, the neurons' silence patterns implicitly encode their memory \citep{salaj2021spike}. \citet{bittar2022surrogate} further generalised adaptive LIF (adLIF) neurons to slowly and linearly evolving currents $w$ instead of thresholds (Eq.~\ref{eq:adlif}). One may notice that both ALIF and adLIF could be considered ontological precursors to the LMU \citep{voelker2019legendre}, being non-linear RNNs with auxiliary linear memory units, in this case using heuristic parametrisation schemes. 

\begin{equation} \label{eq:alif}
    \begin{aligned}
         u[t+1] &= \beta u[t] + (1-\beta)(I[t+1] + W_{rec}s[t]) - \theta[t+1]s[t] \\
         \theta[t+1] &= \alpha \theta_[t] + (1-\alpha)s[t]
    \end{aligned}
\end{equation}

\begin{equation} \label{eq:adlif}
    \begin{aligned}
         u[t+1] &= \beta u[t] + (1-\beta)(I[t+1] + w[t]) - \theta s[t] \\
         w[t+1] &= \alpha w[t] + (1-\alpha)(u[t] + s[t])
    \end{aligned}
\end{equation}

Another dominant approach to increasing long-range processing capabilities in SNNs has been the use of axonal delays. In biological neural circuits, depending on factors including the cell type, conductance and size of the axon, action potentials can spend anywhere between below 1ms to over 100ms travelling to their post-synaptic destinations \citep{wang2008functional}. These heterogeneous delays have inspired computational theories regarding the expressive power of the spike patterns they produce \citep{izhikevich2006polychronization, izhikevich2025spiking}. Continuous-time delayed dynamical systems have even been theoretically proven to have infinite-dimensional state spaces \citep{erneux2009applied}. For the purposes of discretised SNNs, however, their main contribution to long-sequence processing is avoiding vanishing and exploding gradients. In very deep neural networks, one can add skip connections between layers ($\mathtt{Out}_k(x) = \mathtt{Layer}_k(x) + x$) to mechanistically enable better gradient propagation across depth \citep{he2016deep}. In essence, axonal delays achieve a similar effect over the unrolled temporal computation of SNNs, adding skip connections between time points $\tau$ steps apart (Eq.~\ref{eq:delay_lif}). At a high level, long-term Jacobians $G_{1:t}$ (Section~\ref{sec:lyapunov_exponents}) contain terms of the general form $W_{rec}\frac{ds}{du}$ with exponents at most $\frac{t}{\tau}$. Therefore, delays reduce their vanishing/exploding exponents by a constant factor, regardless of the $W_{rec}$ initialisation and parametrisation \citep{queant2025delrec}. That means that in the absence of additional gradient controls, delays must grow proportionally with the input sequence length to keep up the same long-term Jacobian properties. Therefore, because $\tau$-long delays require equally large state buffers, axonal delays are difficult to scale for state-of-the-art long sequence tasks (e.g., over 10,000 steps). 

\begin{equation}\label{eq:delay_lif}
    u[t+1] = \beta u[t] + (1-\beta)(I[t+1] + W_{rec}s[t-\tau]) - \theta s[t]
\end{equation}

Finally, leveraging modern SSM methods has also improved SNN performance on long sequences. As mentioned in Section~\ref{ssms}, SSMs are essentially leaky integrators akin to LIF sub-threshold dynamics. This shared computational primitive has led to the development of SSM-based SNNs, first comprehensively explored in \citet{stan2024learning}, which have shown that linear SNNs with SSM initialisation and parametrisation schemes can also retain strong long-range modelling performance. Follow-up studies have also integrated SSM concepts into non-linear SNN architectures such as adLIF by reintroducing resets after spiking \citep{fabre2025structured}. While these approaches benefit from increased sparsity from refractory periods and thus improved efficiency, recurrent non-linearities also mean that the linear memory properties of SSMs are no longer guaranteed, potentially diminishing their scalability with longer sequences.

\section{Methods} \label{sec:methods}

The goal of ADPTNet is to possess all three of the qualities highlighted in Section~\ref{sec:intro} as key to the Transformer's success: (a) adaptability, (b) long-range dependency modelling, and (c) parallelisability. This section details how the topological conjugate backbone of ADPTNet serves to attain them. 

\subsection{ADPTNet Recurrence Rule} \label{sec:ADPTNet_def}

In general, two dynamical systems with recurrence functions $\frac{dx}{dt} = f(x)$ and $\frac{dx}{dt} = g(x)$ are said to be topologically conjugate if there exists a bijection $h(x)$ such that $g(x)$ can be formulated as the composition $g(x) = (h^{-1}\circ f \circ h)(x)$. Relevant here, given two linear systems $\frac{dx}{dt} = Ax$ and $\frac{dx}{dt} = Bx$, if there exists a non-singular matrix $C$ such that $B = C^{-1}AC$, then the systems are also said to be topologically conjugate. As linear topological conjugacy equates to similarity transforms, the two systems share intrinsic qualities such as eigenspectra but differ in their eigenvectors/modes (Fig.~\ref{fig:top_conj}). In this context, ADPTNet employs topological conjugates by chaining them together to change flow direction locally while also retaining globally coherent properties (Fig.~\ref{fig:stitch_top_conj}). Furthermore, the similarity transforms are parametrised to be a function of the system's recurrent state, i.e. $C^{-1}(x)AC(x)$, matching the neuromodulation template set out in Section~\ref{sec:neuromodulation}. Considering a discrete-time setting, an RNN-like step function can be derived as Eq.~\ref{eq:tc_template_rec}. 

\begin{equation}  \label{eq:switching_linear_rec}
    \frac{dx}{dt} = \begin{cases}
   C_1^{-1}AC_1x, \space  t \le T_1\\
   C_2^{-1}AC_2x, \space T_1 < t < T_2 \\
   C_3^{-1}AC_3x, \space  T_2 \le t
\end{cases}
\end{equation}

\begin{equation}  \label{eq:tc_template_rec}
    x_{t+1}= C^{-1}(x_t)AC(x_t)x_t + Bu_{t+1}
\end{equation}

\begin{figure}[H]
    \centering
        \begin{subfigure}[b]{0.53\textwidth}
            \centering        
            \includegraphics[width=\textwidth]{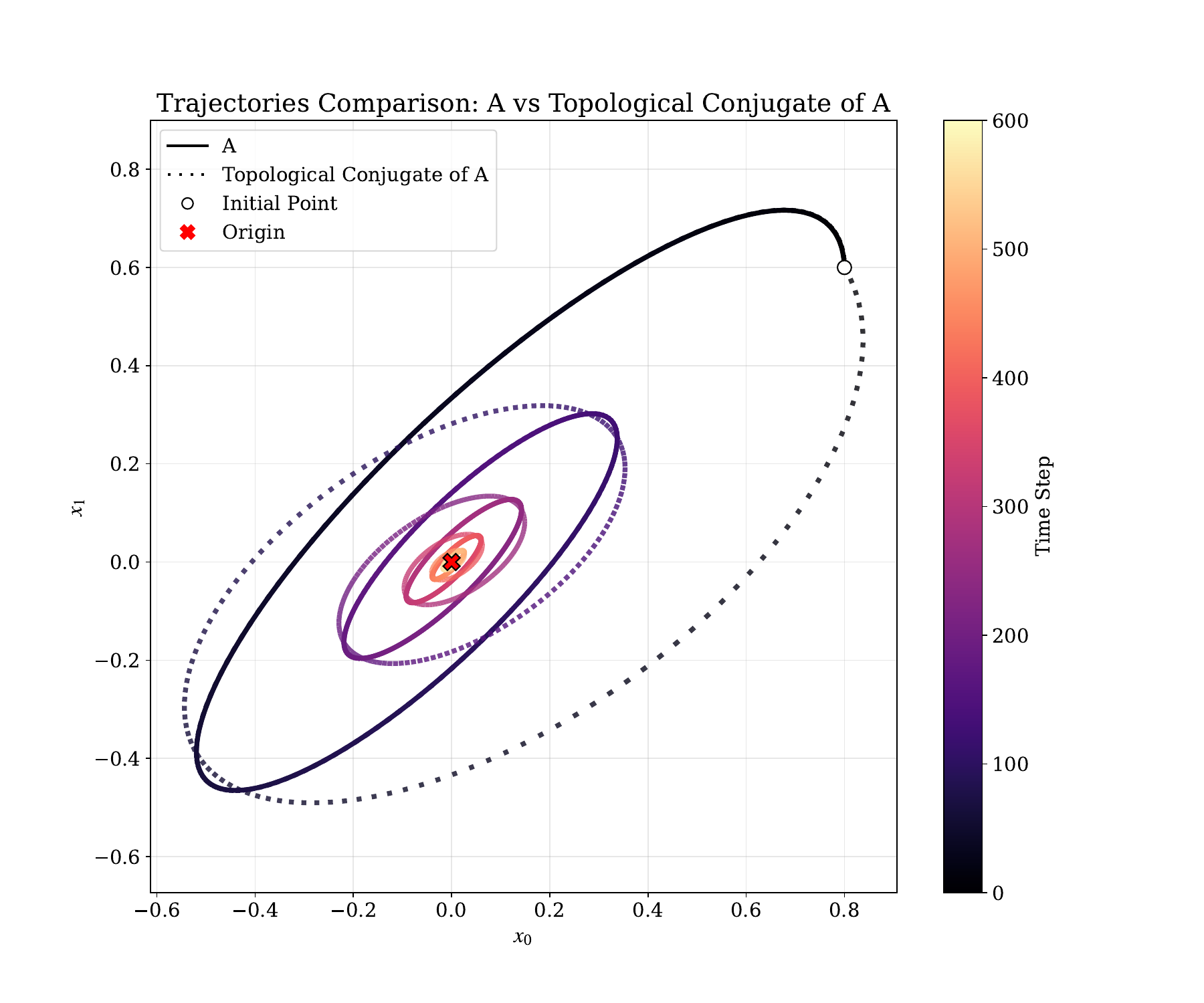}
            \caption{Linear Topological Conjugates}
            \label{fig:top_conj}
         \end{subfigure}
        \begin{subfigure}[b]{0.45\textwidth}
            \centering        
            \includegraphics[width=\textwidth]{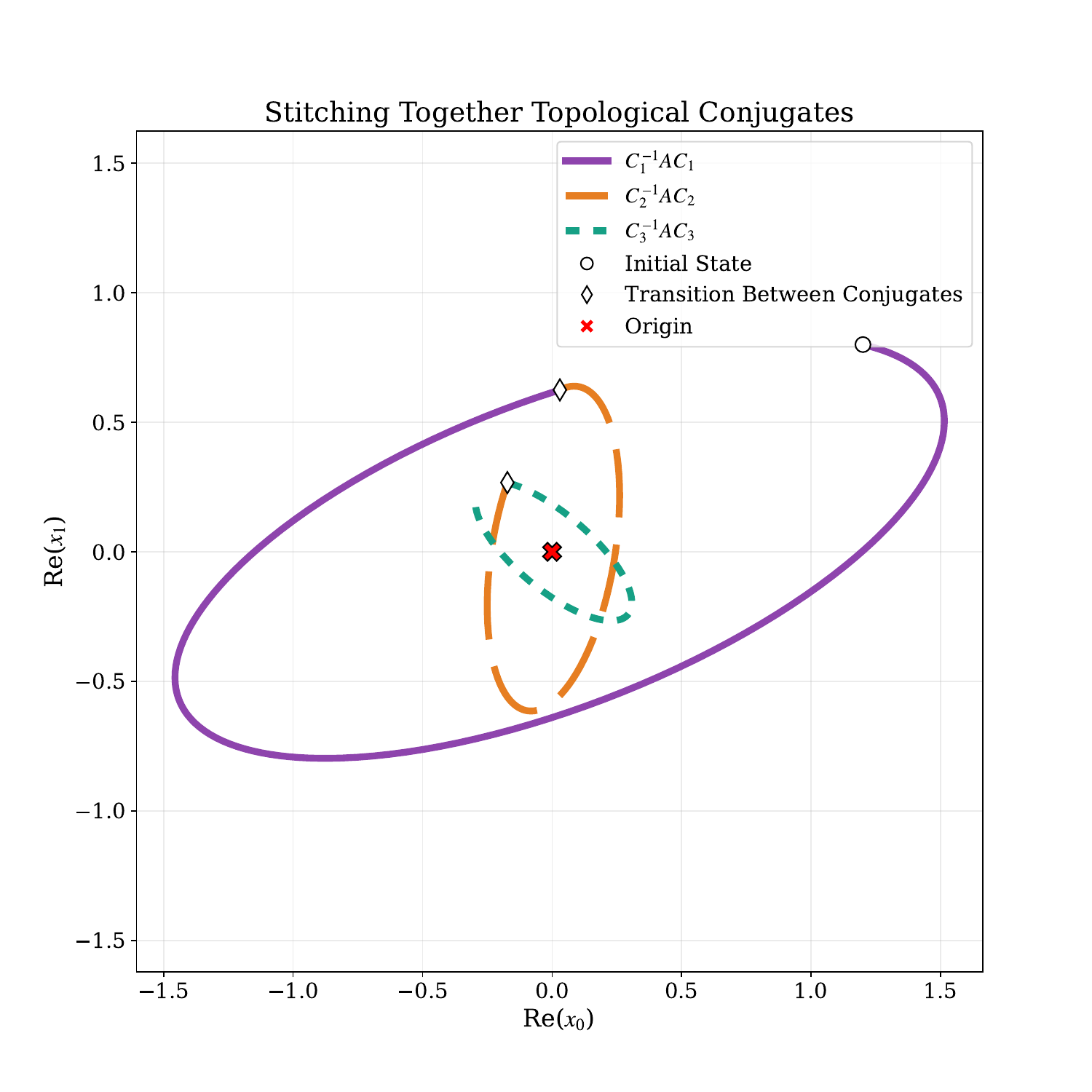}
            \caption{Switching between Topological Conjugates}
            \label{fig:stitch_top_conj}
        \end{subfigure}
        
     \hfill
        \caption{\textbf{Topological Conjugates in Linear Dynamics}. Subfigure~\ref{fig:top_conj} shows trajectories from two linear dynamical systems, one governed by $\frac{dx}{dt}=Ax$ and the other by a topologically conjugate equation $\frac{dx}{dt}=C^{-1}AC$. Notably, they share the same oscillation frequency and decay back to the origin at the same rate, as they share the same eigenvalues. However, the direction of travel, i.e., eigenmodes, differs. Subfigure~\ref{fig:stitch_top_conj} shows a toy example of how one can chain together (or "stitch") topological conjugates, changing eigenvector direction every 450ms, while keeping certain dynamical properties such as oscillation frequency constant (Eq.~\ref{eq:switching_linear_rec}).}
        \label{fig:topological_conjugates}
\end{figure}

Any potential implementation fitting this proposed parametrisation needs to be computationally inexpensive and effectively enable direct, fine-grained control over global dynamics, i.e., guarantee long-range memory. By definition, controlling global memory properties is a matter of parametrising $A$. The most computationally cheap and interpretable solution is to constrain $A$ to be a diagonal matrix containing the eigenvalues of the system ($\Lambda$). This also allows taking advantage of state-of-the-art diagonal SSM parametrisation, with the details of the specific methods employed here described in Section~\ref{sec:eigvalue_param}. 

The goal of defining $C$ as a function of the recurrent state is to enable adaptability in the system's dynamics. Therefore, one way to obtain $C$ from $x$ would be to follow current state-of-the-art practices in efficient Transformer alternatives (see Section~\ref{sec:neuromodulation}), and base it around outer products of key-value projections ($k \otimes v$). Concretely, taking inspiration from the Linear Transformer \citep{katharopoulos2020transformers}, $C$ can statefully evolve over time as in Eq.~\ref{eq:naive_c_kv}. It is important to highlight that here $k$ and $v$ are functions of the recurrent state $x$, rather than incoming inputs as in Linear Transformers. While appealing, this stateful and unconstrained $C$ approach poses several problems. 

\begin{equation}  \label{eq:naive_c_kv}
    C_{new} = C_{old} + k \otimes v, \space k = f_K(x), \space v = f_V(x)
\end{equation}

A first drawback is that, without careful consideration, accumulating key-value pairs in $C$ poses inherent challenges when computing its inverse $C^{-1}$. First, generic matrix inversion for $GL(n)$ has a high computational cost, with $\mathcal{O}(n^3)$ complexity. Moreover, matrix conditioning plays a significant role in the numerical stability of inversion, effectively constraining the floating-point precision admissible by the algorithm if not accounted for \citep{von2025mesanet}. The solution proposed here is to constrain $C$ to the Orthogonal Group $O(n)$, which enables efficient and stable inversion by transposing.

As detailed in Section~\ref{sec:orthogonal}, constraining matrices that evolve over time on the orthogonal manifold requires Riemannian Gradient Descent. In this case, the gradient $\nabla f(Q)$ \footnote{To align with notational convention for orthogonal matrices, $C$ is replaced from here on with $Q$} can be considered the k-v outer products, and thus the update rule for $Q$ becomes Eq.~\ref{eq:naive_ortho_update}. 

\begin{equation}  \label{eq:naive_ortho_update}
    Q_{new} = Q_{old}\texttt{exp}(S), \space S = kv^TQ^T - Qvk^T
\end{equation}

A second drawback with the premise of stateful $Q$ evolution as in Eq.~\ref{eq:naive_ortho_update}, is that it necessitates large $\frac{dQ_{new}}{dQ_{old}}$ Jacobians for backpropagation. Describing the credit assignment between all $n \times n$ elements in $Q_{new}$ with respect to all $n \times n$ entries in $Q_{old}$ imposes an infeasible $\mathcal{O}(n^4)$ memory cost. In general, Linear Transformers can avoid this drawback since the recurrent weight matrix does not have to be materialised sequentially (see Section~\ref{sec:neuromodulation}). Here, however, $Q$ Riemannian evolution is a non-linear function and thus requires full instantiation. In the interest of scalability and keeping network memory capacity transparently and entirely tied to the recurrent state $x$, for the purposes of this work, the memory of $Q$ itself is removed. From a theoretical perspective, this is equivalent to taking a Hebbian-like associative key-value step on the orthogonal manifold from the identity $I$ at each recurrent step of the network (Eq.~\ref{eq:Q_func}). 

\begin{equation} \label{eq:Q_func}
    Q = Iexp(kv^TI - I^Tvk^T) = exp(kv^T - vk^T), \space I = \mathtt{eye(n)}
\end{equation}

Taking all into consideration, the ADPTNet recurrence rule amounts to Eq.~\ref{eq:state_dependent_similarity} and ~\ref{eq:recurrence_eq}. To implement the matrix $\mathtt{exp}$ function required for the retraction back to the orthogonal manifold, the most widely used method is the Cayley Map as in Eq.~\ref{eq:cayley_map_rec}\footnote{The notation for skew-symmetric matrices, S, becomes overloaded in this work, since it may also refer to a spiking activation.} and \ref{eq:cayley_map_rec_T} \footnote{From now it is assumed that $M$, $Q$, $S$, $k$, and $v$ are always functions of $x$, and thus the notation is simplified to omit this detail.}, where $S$ is the skew-symmetric map applied to the $kv^T$ outer product, and $k$ and $v$ are normalised linear projections of the state $x$ (Eq.~\ref{eq:details_skew_kv}). The $\beta$ parameter in Eq.~\ref{eq:cayley_map_rec} is effectively the discretisation step size for travelling over the orthogonal manifold in the direction pointed by $S$. Because in this case the starting point is always the identity $I$, $\beta$ gains the additional function of directly tuning the off-diagonal interactions between the elements of $x$ within the dynamics of the network (Fig.~\ref{fig:beta_effect_interactions}). 

\begin{figure}[H]
    \centering
        \centering        
        \includegraphics[width=0.75\textwidth]{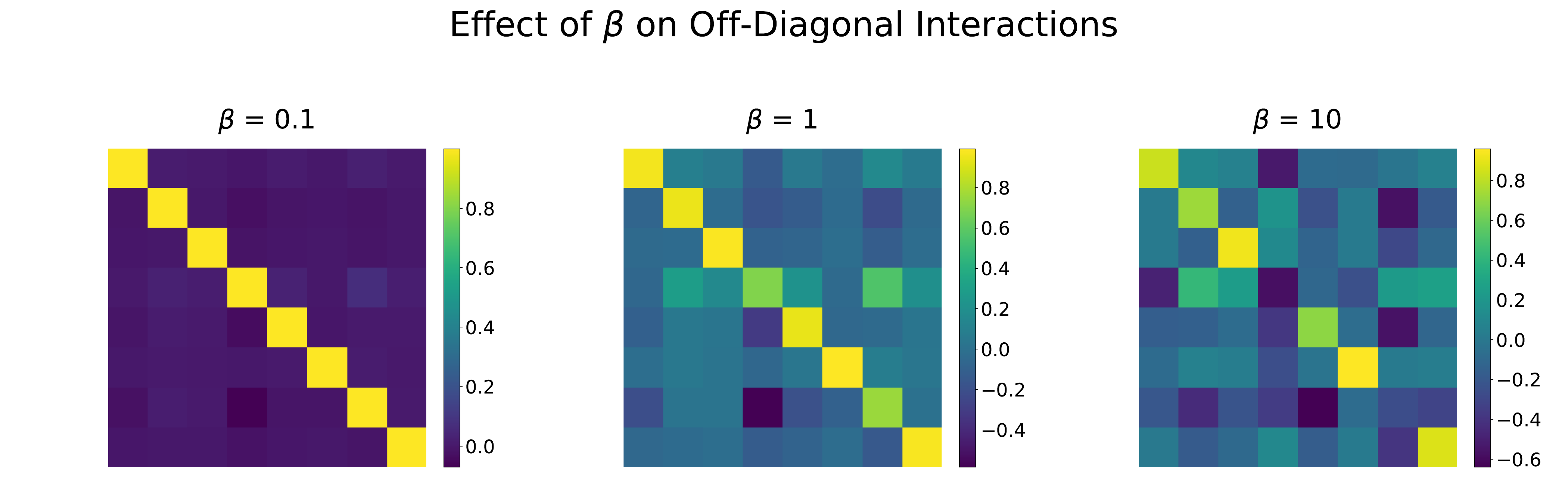}
        \hfill
        \caption{\textbf{Effect of $\beta$ on off-diagonal interactions} Intuitively, as one takes a larger step away from the identity $I$ on the orthogonal manifold (Eq.~\ref{eq:Q_func}), the resulting matrix becomes less and less dominated by its main diagonal, strengthening interactions between the neurons in the network recurrent state $x$. } 
        \label{fig:beta_effect_interactions}
\end{figure}
 
\begin{equation}\label{eq:state_dependent_similarity}
    M_t = Q_t \Lambda Q_t^T, \space Q_t \in O(n), \space \Lambda = \mathtt{diag}(\lambda_1, \dots, \lambda_n)
\end{equation} 

\begin{equation}\label{eq:recurrence_eq}
    x_{t+1} = M_tx_t + Bu_{t+1}
\end{equation}

\begin{equation} \label{eq:cayley_map_rec}
    Q = (I - \frac{\beta}{2}S)^{-1}(I + \frac{\beta}{2}S) 
\end{equation}

\begin{equation} \label{eq:cayley_map_rec_T}
    Q^T = (I - \frac{\beta}{2}S)(I + \frac{\beta}{2}S)^{-1}
\end{equation}

\begin{equation} \label{eq:details_skew_kv}
    S = kv^T - vk^T, k = \frac{Kx} {||Kx||}, v = \frac{Vx}{||Vx||}
\end{equation}

Having established the core mechanics of the network, it is important to emphasise how they differ from traditional RNNs. For simplicity, the example of a vanilla RNN with ReLU activations is used here, namely $x_{t+1} = \mathtt{ReLU}(W_{rec}x_t) + W_{in}u_{t+1}$. The baseline logic is that the recurrent weight matrix $W_{rec}$ stretches and rotates the state $x$, which then determines the outputs of the $\mathtt{ReLU}$ activations. The activations, in turn, form a diagonal $D$ that element-wise scales entries in $x$ by $\in \{ 0, 1\}$, i.e. $x_{t+1} = D(W_{rec}x_t)\odot(W_{rec}x_t) + W_{in} u$. The Lyapunov Spectrum (Section~\ref{sec:lyapunov_exponents}) of the network cannot be known a priori because $D$ entries change arbitrarily. In contrast, ADPTNet fixes and parametrises the entries of the diagonal scaling factors. To obtain non-linearity, the state $x$ is rotated by $Q^T$ to decay along different modes with the pre-determined decay rates. This distinction enables the claims in Lemma~\ref{lemma:J_singular_vals} and \ref{lemma:le_spectrum} to be made. 

\begin{lemma} \label{lemma:J_singular_vals}

For any given initialisation of $\Lambda$ where all $0 < \lambda_i < 1$, the singular values ($\sigma_i$) of the recurrent Jacobian matrix $J$ of the ADPTNet recurrence lie within the bounds $\lambda_i \pm 8\beta\lambda_{max}$ (Assuming K, V are orthogonal). 

\end{lemma}

\begin{proof}

As per the product rule, the recurrent Jacobian matrix $J$ comprises the matrix $M$ itself and its derivative with respect to the state $\frac{dM}{dx}$ (Eq.~\ref{eq:J}). The eigenspectrum of $M$ is known by definition, $\Lambda$, and because by construction $M$ is symmetric positive definite, they are also $M$'s singular values ($\lambda_i = \sigma_i(M)$). If the norm of the second term is sufficiently small, it is then possible to use Weyl's inequality to find bounds on the singular value spectrum of their sum, $J$ (Eq.~\ref{eq:weyl}). 

\begin{equation}\label{eq:J}
    J = M + \frac{dM}{dx}x
\end{equation}

\begin{equation}\label{eq:weyl}
    | \sigma_i(J) - \sigma_i(M) | < \left\|\frac{dM}{dx}x\right\|_2
\end{equation}

To find bounds on the norm of $\frac{dM}{dx}x$, it is necessary to reduce it with the chain rule to terms with known norms. Using the product rule again, $\frac{dM}{dx}x$ further breaks down into two terms for each of the two sides of the similarity transform (Eq.~\ref{eq:dMdx}). For the purposes of this proof, it is assumed that $Q$ is parametrised using the Cayley Map (Eq.~\ref{eq:cayley_map_rec}), and thus $\frac{dQ}{dx}$ and its transpose take the form of Eq.~\ref{eq:dQdx} and \ref{eq:dQTdx}. When chaining them together within derivatives of the similarity transform in  $\frac{dM}{dx}x$ (Eq.~\ref{eq:dqdx_lmb_qt}), it is worth observing that the $\beta$ hyperparameter directly scales the norm of $\frac{dM}{dx}x$.

\begin{equation} \label{eq:dMdx}
     \frac{dM}{dx} = \frac{dQ}{dx}\Lambda Q^T + Q\Lambda \frac{dQ}{dx}
\end{equation}

\begin{equation} \label{eq:dQdx}
    \frac{dQ}{dx} = \beta(I -  \frac{\beta}{2}S)^{-1}\frac{dS}{dx}(I -  \frac{\beta}{2}S)^{-1}
\end{equation}

\begin{equation}\label{eq:dQTdx}
    \frac{dQ^T}{dx} = -\beta(I +  \frac{\beta}{2}S)^{-1}\frac{dS}{dx}(I +  \frac{\beta}{2}S)^{-1}
\end{equation}

\begin{align}\label{eq:dqdx_lmb_qt}
        \frac{dQ}{dx}\Lambda Q^T  &= \beta(I -  \frac{\beta}{2}S)^{-1}\frac{dS}{dx}(I -  \frac{\beta}{2}S)^{-1}\Lambda Q^T \\
        &= \beta(I -  \frac{\beta}{2}S)^{-1}\frac{dS}{dx}(I -  \frac{\beta}{2}S)^{-1}\Lambda (I -  \frac{\beta}{2}S)(I +  \frac{\beta}{2}S)^{-1}
\end{align}

To simplify the notation, the $A$ and $B$ \footnote{While the symbol $B$ is overloaded because it is also used for the input weights in the recurrence formula for ADPTNet (Eq.~\ref{eq:recurrence_eq}), in the scope of this proof it is only used as the shorthand from Eq.~\ref{eq:sub_B}.} substitutions defined in Eq.~\ref{eq:sub_A} and \ref{eq:sub_B} are introduced. 

\begin{equation}\label{eq:sub_A}
    A = (I - \frac{\beta}{2}S)^{-1}
\end{equation}

\begin{align}\label{eq:sub_B}
    B = (I - \frac{\beta}{2}S)^{-1}\Lambda (I - \frac{\beta}{2}S)
\end{align}

Using the shorthand notation, the similarity transform derivative terms can be contracted to Eq.~\ref{eq:simple_dQdx} and \ref{eq:simple_dQTdx}, and thus the overall $\frac{dM}{dx}x$ term reduces to Eq.~\ref{eq:simple_dMdx}.

\begin{equation} \label{eq:simple_dQdx}
    \frac{dQ}{dx}\Lambda Q^T = \beta A \frac{dS}{dx} BA^T
\end{equation}

\begin{equation} \label{eq:simple_dQTdx}
    Q \Lambda\frac{dQ^T}{dx} = \beta A B^T \frac{dS}{dx}A^T
\end{equation}

\begin{equation} \label{eq:simple_dMdx}
    \frac{dM}{dx}x = \beta A (\frac{dS}{dx} B  a - B^T \frac{dS}{dx}a), a = A^Tx
\end{equation}

Next, going down the derivation chain, the derivative of the skew-symmetric map $\frac{dS}{dx}$ requires a tensor outer product $\otimes_{outer}$ (Eq.~\ref{eq:dSdx}) if computed on its own, since it is the Jacobian matrix of a matrix with respect to a vector. However, the tensors collapse back to matrices when multiplied with the rest of the terms in $\frac{dM}{dx}x$ (Eq.~\ref{eq:dSdxBa} and \ref{eq:BdSdxa}), avoiding the need to explicitly compute their individual norms.

\begin{equation} \label{eq:dSdx}
    \frac{dS}{dx} = \frac{dk}{dx}\otimes_{outer}v^T + k\otimes_{outer} \frac{dv}{dx} - \frac{dv}{dx}\otimes_{outer}k^T - v\otimes_{outer} \frac{dk}{dx} 
\end{equation}

\begin{align} \label{eq:dkdx}
    \frac{dk}{dx} &= \frac{1}{||Kx||} (I - kk^T)K
\end{align}

\begin{align} \label{eq:dvdx}
    \frac{dv}{dx} &= \frac{1}{||Vx||} (I - vv^T)V
\end{align}

\begin{equation} \label{eq:dSdxBa}
    \frac{dS}{dx}Ba = \frac{dk}{dx}(v^TBa) + (ka^T)B^T\frac{dv}{dx} - \frac{dv}{dx} (k^TBa) - (va^T)B^T\frac{dk}{dx}
\end{equation}

\begin{equation} \label{eq:BdSdxa}
    B^T\frac{dS}{dx}a = B^T \frac{dk}{dx}(v^Ta) + B^T(ka^T)\frac{dv}{dx} - B^T\frac{dv}{dx}(k^Ta) - B^T(va^T)\frac{dk}{dx}
\end{equation}

Given that $\left\|B\right\|_2 = \left\|B^T\right\|_2=\lambda_{max}$, $\left\|A\right\|_2 = \left\|A^T\right\|_2 = 1$, $\left\|\frac{dk}{dx} \right\|_2=\left\|\frac{dv}{dx} \right\|_2=\frac{1}{\left\|x \right\|_2}$, $\left\|v \right\|_2 = \left\|k \right\|_2 = 1$ and $\norm{a} \le \norm{A}\norm{x}$ then:

\begin{equation} \label{eq:lemma1_result}
    \begin{split}
            \norm{\frac{dM}{dx}x} 
    &\le \beta \norm{A} \norm{B} \norm{a} ( 4 \norm{v}\norm{\frac{dk}{dx}} + 4\norm{k}\norm{\frac{dv}{dx}}) \\
    &\le \beta\lambda_{max}\norm{A}\norm{x}(\frac{4}{\norm{x}} + \frac{4}{\norm{x}}) \\
    &= 8\beta\lambda_{max}
    \end{split}
\end{equation}

 Since the norm of $\frac{dM}{dx}x$ is upper-bounded by $8\beta\lambda_{max}$, for sufficiently small $\beta$, each of the singular values of $J$ is bounded within:

\begin{equation*}
    \max(0, \lambda_i - 8\beta\lambda_{max}) \le \sigma_i(J) \le \lambda_i + 8\beta\lambda_{max}, \space \forall i \in [1, \dots,  n]
\end{equation*}

\end{proof}

To simplify analysis going forward, it is assumed that $\lambda_{min}$, $\lambda_{max}$, and $\beta$ are constrained such that $\lambda_{min} - 8\beta\lambda_{max} > 0$. Having established spectral bounds for individual ADPTNet recurrent Jacobian matrices $J$, it is now possible to formalise theoretical guarantees on long-term behaviour:

\begin{corollary} \label{coro:lipschitz}
    The Lipschitz constant of the ADPTNet recurrent step function is $\lambda_{\max} + 8\beta \lambda_{\max}$, for $x \in \mathbb{R}^{n} \setminus \{\mathbf{0}\}^n$.
\end{corollary}

\begin{proof}
Since Lemma~\ref{lemma:J_singular_vals} derives an upper bound on the norm of the recurrent Jacobian $\left\| J \right\|_2$, this also produces, by definition, the Lipschitz constant $L = \sup \left\|J \right\|_2 = \left\|M + \frac{dM}{dx}x\right\|_2=\lambda_{\max} + 8 \beta\lambda_{max}$. 
\end{proof}

\begin{lemma} \label{lemma:le_spectrum}

For sufficiently many trajectory time steps $T$, the entire Lyapunov spectrum of ADPTNet dynamics is bounded within $[\ln(\min(\lambda_{min}-8\beta\lambda_{max})), \space \ln(\lambda_{max} + 8\beta\lambda_{max})]$. 

\end{lemma}

\begin{proof}

In Lemma \ref{lemma:J_singular_vals}, it is established that all ADPTNet recurrent Jacobian $J$ singular values $\sigma_i(J)$ lie within $\pm 8\beta\lambda_{max}$ of the eigenvalues $\lambda_i$. Therefore

\begin{equation*}
    \lambda_{min} - 8\beta\lambda_{max} \le \sigma_i(J) \le \lambda_{max} + 8\beta\lambda_{max}, \space \forall i \in [1, \dots, n]
\end{equation*}

For notational simplicity, let $\tilde{\sigma}_{min} = \lambda_{min} - 8\beta\lambda_{max}$ and $\tilde{\sigma}_{max} = \lambda_{max} + 8\beta\lambda_{max}$ denote the limits above. As with any arbitrary matrices, for any two Jacobian matrices $J_1$ and $J_{2}$ from time steps $1$ and $2$, all the singular values of their product are bounded:  

\begin{equation*}
     \tilde{\sigma}_{min} * \tilde{\sigma}_{min} \le \sigma_i(J_1J_{2}) \le  \tilde{\sigma}_{max} * \tilde{\sigma}_{max}, \space \forall i \in [1, \dots, n] 
\end{equation*}

If we assume that the long term Jacobian product $G_{k} = J_1J_2 \dots J_k$ (Section~\ref{sec:lyapunov_exponents}) has singular values bounded by $\tilde\sigma_{min}^k$ and $\tilde\sigma_{max}^k$, then 

\begin{equation*}
    \begin{split}
        \tilde\sigma_{min}^k * \tilde\sigma_{min} \le \sigma_{min}(G_k) \sigma_{min}(J_{k+1}) &\le \sigma_i(G_kJ_{k+1}) \le \sigma_{min}(G_k) \sigma_{min}(J_{k+1}) \le \tilde\sigma_{max}^k * \tilde\sigma_{max} \iff \\
        \tilde\sigma_{min}^{k+1} &\le \sigma_i(G_kJ_{k+1}) \le \tilde\sigma_{max}^{k+1} 
    \end{split}
\end{equation*}
    
Hence, by induction, ADPTNet long-term Jacobians $G_{k}$ always have singular values bounded by matching powers of $\tilde\sigma_{min}$ and $ \tilde\sigma_{max}$. Taking the natural logarithm of $G_k$ is then bounded as 

\begin{equation*}
    \begin{split}
        ln(\tilde\sigma_{min}^k) &\le ln(\sigma_i(G_k)) \le ln(\tilde\sigma_{max}^k) \iff \\
        ln(\tilde\sigma_{min}) k &\le ln(\sigma_i(G_k)) \le ln(\tilde\sigma_{max})k
    \end{split}
\end{equation*}

Considering the formula for obtaining the Lyapunov spectrum $\alpha_i$ \footnote{While the notational convention is to denote the Lyapunov spectrum by $\lambda$, to avoid confusion with the recurrent eigenspectrum of ADPTNet, $\alpha$ is used instead. } (see Section~\ref{sec:lyapunov_exponents}), it becomes apparent that $\forall$ Lyapunov exponents are also bounded 

\begin{equation*}
    \begin{split}
        \ln(\tilde\sigma_{min}) k \le ln(\sigma_i(G_k)) \le \ln(\tilde\sigma_{max})k &\iff \\
        \ln(\tilde\sigma_{min}) \le \frac{1}{k}ln(\sigma_i(G_k)) \le \ln(\tilde\sigma_{max}) &\iff \\ 
        \ln(\tilde\sigma_{min})\le \lim_{k\rightarrow \infty}\frac{1}{k}ln(\sigma_i(G_k)) \le \ln(\tilde\sigma_{max}) &\iff \\
        \ln(\tilde\sigma_{min})\le \alpha_i \le \ln(\tilde\sigma_{max}), \space \forall i \in [1, \dots, n]
    \end{split}
\end{equation*}

\end{proof}

Lemma~\ref{lemma:le_spectrum} is built on the assumption that the Lyapunov spectrum is derived from the singular values of the long-term Jacobian $G_k$. However, in practice this is not numerically stable \citep{gonzalez2026predictability, engelken2023lyapunov, engelken2023gradient}, and, as such, the Lyapunov spectrum is typically computed as the sum of individual $ln(\sigma_i(J_t))$ terms, which are always sorted. Therefore, given this implementation detail, it is possible to derive even stronger bounds on the relationship between recurrent eigenvalues and the Lyapunov spectrum in ADPTNet:

\begin{lemma}\label{lemma:le_exact_match}

When using numerically stable methods to obtain the Lyapunov Spectrum, for sufficiently small $\beta$ and a given recurrent eigenspectrum $\Lambda$ of ADPTNet, individual Lyapunov exponents are $\alpha_i \approx \ln(\lambda_i) + \Delta_i$, where $\Delta_i \in [-\frac{8\beta\lambda_{max}}{\lambda_i}, \frac{8\beta\lambda_{max}}{\lambda_i}]$, $\lambda_{max} = \lambda_{1} \ge \lambda_2 \ge \dots \ge \lambda_n = \lambda_{min}$, and $\alpha_{max} = \alpha_{1} \ge \alpha_ \ge \dots \ge \alpha_ = \alpha_{min}$. 

\end{lemma}

\begin{proof}

For simplicity, and without loss of generality, it is assumed that $\min_{i,j}|\lambda_i - \lambda_j| >8 \beta\lambda_{max}$, i.e., all eigenvalues are further apart than $8\beta\lambda_{max}$. Then, let the singular values of each recurrent Jacobian $J_t$ be $\tilde{\sigma}_1 \ge \tilde{\sigma}_2 \dots \ge \tilde{\sigma}_n$, where each $\tilde{\sigma}_i^{(t)} = \lambda_i + \Delta_i^{(t)}$, $-8\beta\lambda_{max}\le \Delta_i^{(t)} \le +8\beta\lambda_{max}$. As highglighted in Lemma~\ref{lemma:le_spectrum}, for any given product $J_tJ_{t-1}$, the only guarantees on its singular value spectrum are that $ \tilde{\sigma}_{\min}^2\le \sigma_i(J_tJ_{t-1}) \le \tilde{\sigma}_{\max}^2$. That is because matrix multiplication "mixes" the spectrum and only preserves upper/lower bounds for its extremes. Therefore, when computing the singular values of the final long-term Jacobian product $G_t$, one can still only bound its extremes. However, if you compute the spectrum at each time step, for example using QR decomposition \citep{engelken2023gradient}, it will be sorted (high to low) each time, preventing spectral mixing. 

Let $\alpha_i^{(t)}$ be the un-normalised partially-computed Lyapunov exponent at time $t$. Then computing it becomes:

\begin{equation*}
\begin{split}
    \alpha_i^{(t)} &= \alpha_i^{(t-1)} + \ln(\tilde{\sigma_i}^{(t)})
\end{split}
\end{equation*}

Assuming $\beta$ is small enough, then $\Delta_i^{(t)}$ can be treated as a perturbation, and the logarithm can be expanded with its Taylor series:

\begin{equation}
    \begin{split}
        \alpha_i^{(t)} &= \alpha_i^{(t-1)} + ln(\lambda_i + \Delta_i^{(t)}) \\
                       &= \alpha_i^{(t-1)} + ln(\lambda_i)+ \frac{\Delta_i^{(t)}}{\lambda_i}
    \end{split}
\end{equation}

The final, time-averaged over $T$ steps, Lyapunov exponent is:

\begin{equation}
    \begin{split}
        \alpha_i &= \frac{1}{T}\sum_{t=1}^{T}\left(ln(\lambda_i)+ \frac{\Delta_i^{(t)}}{\lambda_i} \right) \\
                 &= ln(\lambda_i) + \frac{1}{\lambda_i} \left( \frac{1}{T}\sum_{t=1}^{T} \Delta_i^{(t)} \right)
    \end{split}
\end{equation}

We denote the time-average as $\mu_i =  \frac{1}{T}\sum_{t=1}^{T} \Delta_i^{(t)}$. Since, each $-8\beta\lambda_{max}\le \Delta_i^{(t)} \le +8\beta\lambda_{max}$, then the average also satisfies $-8\beta\lambda_{max}\le \mu_i^{(t)} \le +8\beta\lambda_{max}$. Let $\Delta_i = \frac{\mu_i}{\lambda_i}$, then the initial statement in Lemma~\ref{lemma:le_exact_match} is proven. 

\end{proof}

\begin{corollary} \label{coro:deltas}
    If $\Delta_i^{(t)}$ are sampled from a zero-mean distribution (e.g., $\mathcal{U}(-\beta \lambda_{\max}, \beta \lambda_{\max}))$, and the sequence length $T$ is sufficiently large, then $\alpha \equiv ln(\Lambda)$. 
\end{corollary}

Lemmas~\ref{lemma:le_spectrum} and \ref{lemma:le_exact_match} differ only on an implementation technicality. Regardless of Lyapunov spectrum computational methodology, the LLE and the long-term behaviour of ADPTNet are still prescribed by the recurrent eigenvalues $\Lambda$ and the choice of $\beta$. 

While the topological conjugate in the ADPTNet provides non-linearity through the matrix exponential and the multiplicative interaction between $M(x)$ and $x$, one can inject additional non-linear processing by  pre-pending activation functions before the $KV$-projections. For instance, if the state $x$ passes through element-wise ReLU before being projected to $k$ and $v$ and normalised, the spectral bounds from Lemma~\ref{lemma:J_singular_vals} are unchanged. This is because, unless completely silent, a ReLU diagonal Jacobian has spectral norm $\equiv 1$, and thus does not influence $\norm{\frac{dM}{dx}x}$ through the derivation chain. This differs from traditional RNNs where non-linearities are typically bottlenecks in the propagation of the recurrent state. 

It should also be mentioned that "\textit {stitching}" together linear dynamics to obtain non-linear behaviour is reminiscent of well-established methods in machine learning such as Recurrent Switching Linear Dynamical Systems \citep{linderman2016recurrent} or the  Piecewise-Linear RNN understanding of ReLU networks \citep{brenner2024almost}. However, these methods typically focus on producing certain dynamical behaviours locally around critical points \citep{smith2021reverse}. In terms of global stability properties, the focus in this line of research is to enhance training methodology through regularisation \citep{engelken2023gradient} or teacher forcing \cite{hess2023generalized}, rather than by the inherent properties of the recurrent architecture, like the ADPTNet model proposed here.

\subsection{Eigenvalue Parametrisation and Positional Embeddings} \label{sec:eigvalue_param}

As highlighted in Section~\ref{ssms}, LTI SSMs have the favourable property of direct control of long-term dynamics through the eigenvalue parametrisation of the recurrent matrix. As established in Lemma~\ref{lemma:le_spectrum}, ADPTNet also enjoys theoretical guarantees on the direct influence of its recurrent eigenvalues $\Lambda$ (Eq.~\ref{eq:state_dependent_similarity}) over its long-term non-linear dynamics, as quantified through its Lyapunov spectrum. Therefore, the parametrisation of $\Lambda$ plays a crucial role in the model's performance. 

The general consensus in the SSM literature is that recurrent weights can be reduced to diagonal matrices containing their complex conjugate eigenspectrum \citep{muca2024theoretical}. Their parametrisation and initialisation are designed for stability in continuous time. Hence, eigenvalues are typically constrained through parametrisation to negative real parts (e.g., negative exponential). For simulation, the recurrent dynamics are discretised using techniques such as Zero-Order Hold (ZOH), and typically the discretisation step size $\Delta t$ is also introduced as a trainable parameter. $\Delta t$ also serves the secondary purpose of normalising the recurrent state and preventing magnitude explosion over long horizons \citep{orvieto2023resurrecting}. Altogether, using ZOH, discrete eigenvalue $\bar{\lambda}$ parametrisation takes the form of Eq.~\ref{eq:eig_parametrisation}, where $\Delta t$, $\lambda_i^{real}$ and $\lambda_i^{imag}$ are initialised to the logarithms of the desired discretisation real and imaginary parts of the eigenvalues. 

\begin{equation} \label{eq:eig_parametrisation}
    \bar{\lambda}_i = e^{-e^{\Delta t} (e^{\lambda_i^{real}} \pm i e^{\lambda_i^{imag}})}
\end{equation}

To ensure long-term slow decay, $\lambda_i$ are initialised close to the unit circle \citep{orvieto2023resurrecting, gu2022parameterization}. Therefore, $\Delta t$ is typically initialised in the range $[10^{-4}, 10^{-1}]$ or $[10^{-3}, 10^{-1}]$ depending on sequence length\footnote{Discretisation step size has the implicit task of normalising the weighted sum of the inputs that is an SSM hidden state. Therefore, for example, if sequence length is $\approx1000$ time steps, one can use a $\Delta t=10^{-1}$ as the slowest evolving timescale to roughly normalise the sum.}, and $\alpha_i^{real}=\ln(\frac{1}{2})$. For the purposes of this work, the distribution of the imaginary parts at initialisation follows the S4D-Inv scheme proposed by \citet{gu2022parameterization} (Eq.~\ref{eq:s4d-inv})

\begin{equation} \label{eq:s4d-inv}
    e^{\lambda_k^{imag}} = \frac{n}{\pi} \left( \frac{n}{2k + 1} - 1 \right), \space n = \dim(\Lambda) /2
\end{equation}

At this stage, it is worth emphasising the logic behind SSM complex eigenvalues. Considering, without loss of generality, an unrolled recurrent state $x_t$, i.e., viewed as the sum of inputs $u_k$ for $k < t$ (Section~\ref{ssms}), that uses a generic simplified view of the eigenvalue parametrisation in Eq.~\ref{eq:eig_parametrisation}, then $x_t$ can be rewritten as:

\begin{equation*}
    \begin{split}
            x_t &= e^{t(\Lambda^{real} + i\Lambda^{imag})}u_0 + e^{(t-1)(\Lambda^{real} + i\Lambda^{imag})}u_1 + \dots + e^{(\Lambda^{real} + i\Lambda^{imag})}u_{t-1} + u_t \\
            &= e^{(t\Lambda^{real}} (e^{ti\Lambda^{imag}}u_0) + 
             e^{(t-1)\Lambda^{real}}(e^{(t-1)i\Lambda^{imag}}u_1) + \dots + 
             e^{\Lambda^{real}}(e^{i\Lambda^{imag}}u_{t-1}) + (e^{0*i\Lambda^{imag}}u_t) \\
            &= e^{t\Lambda^{real}}R(u_0, t\Lambda^{imag}) + \dots + e^{\Lambda^{real}}R(u_{t-1}, \Lambda^{imag}) + R(u_t, 0) \\
            &= e^{t\Lambda^{real}}\tilde{u}_0+ \dots + e^{\Lambda^{real}}\tilde{u}_{t-1} + \tilde{u}_t, \space \tilde u_k = R(u_k, (t-k)\Lambda^{imag})
    \end{split}
\end{equation*}

Where, because in SSMs eigenvalues come in conjugate pairs by convention, the products $e^{i\Lambda^{imag}}u$ can be rewritten as real-valued block-diagonal rotation matrix products:

\begin{equation}\label{eq:rotary_embedding}
     e^{i\mathbf{\theta}}u = R(u, \mathbf{\theta}) = \begin{pmatrix}
        \cos(\theta_1)  & \sin(\theta_1) & 0      & \cdots & 0 & 0 \\
        -\sin(\theta_1) & \cos(\theta_1) & 0      & \cdots & 0 & 0 \\
        0               & 0              & \ddots &        & \vdots & \vdots \\
        \vdots          & \vdots         &        & \ddots & 0 & 0 \\
        0               & 0              & \cdots & 0      & \cos(\theta_n) & \sin(\theta_n) \\
        0               & 0              & \cdots & 0      & -\sin(\theta_n) & \cos(\theta_n)
    \end{pmatrix}
    \begin{pmatrix}
        u1 \\
        u2 \\
        \\
        \vdots
        \\
        \\
        \\
        u_{2n}
        
    \end{pmatrix}
\end{equation}

In this rewritten form, $x_t$ becomes the weighted sum of rotated inputs $\tilde{u}$. One can change the indexing of each rotation so that instead of having frequencies relative to the current time step $t$, $\tilde{u}_k = R(u_k, (t-k)\Lambda^{imag})$, they are relative to the initial state, i.e., $\tilde u_k = R(u_k, k\Lambda^{imag})$ without any meaningful loss of expressivity. In this new form, $\tilde u$ becomes the well-known Rotary Position Embedding (RoPE), widely employed in modern LLMs \citep{su2021roformer}. 

Therefore, complex-diagonal SSMs' recurrent states are a weighted sum of RoPE-like positionally embedded inputs, with an exponentially decaying memory. This is an important distinction to make in the context of parametrising ADPTNet. The Lyapunov spectrum that ADPTNet is theoretically guaranteed to control only describes exponential decay rates, not oscillations. Conversely, that means that while it is possible to use a complex-valued $\Lambda$ in Eq.~\ref{eq:state_dependent_similarity}, its oscillations may not materialise as intended in the non-linear dynamics. In contexts such as Section\ref{sec:trained_lyapunov_results}, to emphasise this explicit link between recurrent eigenvalues and the Lyapunov spectrum, decay rates and oscillation frequencies are functionally separated in ADPTNet, becoming:

\begin{equation}\label{eq:state_dependent_Q_redo}
    M_t = Q_t (\pm e^{\bar{\Lambda}^{real}})Q_t^T, \space Q \in O(n), \space \bar{\Lambda}^{real} = \mathtt{diag}(-e^{\Delta t\lambda^{real}_1}, \dots, -e^{\Delta t\lambda^{real}_n})
\end{equation}

\begin{equation}\label{eq:recurrence_fn_redo}
    \begin{split}
         x_{t} &= M_t x_{t-1} + \tilde u_t, \tilde u_k = R(Bu_k, k\bar{\Lambda}^{imag}), \space \bar{\Lambda}^{imag} = \mathtt{diag}(e^{\Delta t\lambda^{imag}_1}, \dots, e^{\Delta t\lambda^{imag}_{n/2}}) \\
         &= (M_tM_{t-1}\dots M_1)\tilde{u}_0 + \dots + M_{t}\tilde{u}_{t-1} + \tilde{u}_t
    \end{split}
\end{equation}

As opposed to vanilla RoPE, the frequency distribution in $\Lambda^{imag}$ is still governed by S4D-Inv initialisation and parametrisation. A similar separation of recurrent eigenvalue imaginary components into a RoPE-based layer is also present in Mamba-3 \citep{lahoti2026mamba}. However, there, the positional embeddings are also input-dependent. \citet{grazzi2025unlocking} recently suggested that including negative eigenvalues may also improve performance on certain tasks, and therefore, they are also included here for investigation. 

This resulting real-value eigenspectrum parametrisation  derived from S4D-Inv by separation from rotational positional embeddings, is referred to, in this study, as \textit{linspace} initialisation, since all $\lambda_i = ln(\frac{1}{2})$ and $\Delta t \in \mathtt{linspace}(\ln(\Delta t_{\min}), \ln(\Delta t_{\max}), 2n)$. Following \citep{gu2022parameterization}, another real-valued eigenvalue initialisation scheme which forgoes rotations completely is S4D-Real, which takes the form: 

\begin{equation} \label{eq:s4d_real}
    \lambda_i = \ln(i)
\end{equation}

Notably, while $\lambda_i$ is initialised deterministically here, $\Delta t _i \sim \mathcal{U}(\ln(\Delta t_{\min}), \ln(\Delta t_{\max}))$. Unless otherwise specified, experiments in this study default to the \textit{linspace} initialisation.

\subsection{Low-Parameter Linear Layers} \label{sec:low_param_linear}

As described in Section~\ref{sec:ADPTNet_def}, ADPTNet recurrence makes use of three linear projection matrices for $u$ at every time step: input ($B$), Key ($K$), and Value ($V$). This means that the width $h$, and, intrinsically, the number of timescales for the dynamics, come with a taxing parameter budget of $3h^2$ just for linear projections alone. This imposes a sharp trade-off between temporal feature expressivity/context compression and overparametrisation. By contrast, for instance, SSMs such as S4 \citep{gu2021efficiently} or S5 \citep{smith2023simplified} avoid this pitfall since they lack $K$ and $V$ projections. Therefore, to maintain comparability with such benchmark baselines in Section~\ref{sec:results}, ADPTNet layers need a more parameter-efficient parametrisation. This work explores two such methods. 

Firstly, Low-Rank Matrices (Eq.~\ref{eq:low_rank_matrix}) are ubiquitous in both machine learning and deep learning research through methods such as Singular Value Decomposition (SVD)-based data compression or Low-Rank Adaptation (LoRA) \citep{hu2021lora} for memory-efficient LLM fine-tuning, respectively. This factorisation reduces the number of parameters from $n^2$ to $2nr$, where typically $r \ll n$. Furthermore, low-rank vector-matrix multiplication $A(B^Tv)$ avoids materialising a full $n \times n$ matrix and thus also reduces memory and compute overheads.

\begin{equation} \label{eq:low_rank_matrix}
    M = AB^T,\space M \in \mathbb{R}^{n \times n}, \space A, B \in \mathbb{R}^{n \times r}, \space r \le n
\end{equation}

Secondly, Kronecker products (Eq.~\ref{eq:kronecker_product}) can also drastically reduce parameter counts, from $n^2$ to $N_a^2 + N_b^2$, where $N_a, N_b \ll n$. Similar to low-rank matrices, Kronecker factorisation also admits faster matrix-vector multiplication without full $n \times n$ materialisation, using the \textit{vec trick} (Eq.~\ref{eq:vec_trick}). In addition, they have also been extensively studied in neural networks, for instance, for Fisher information matrix approximation \cite{martens2015optimizing} or even alternatives to LoRA \citep{edalati2025krona}. 

\begin{equation} \label{eq:kronecker_product}
    M = A \otimes_{Kron}B = \begin{pmatrix}
        a_{11}B & a_{12}B & \dots & a_{1N_a}B \\
        a_{21}B & a_{22}B & \dots & a_{2N_a}B \\
        &  \ddots & & \\
         a_{N_a1}B & a_{N_a2}B & \dots & a_{N_aN_a}B \\
    \end{pmatrix}, \space M \in \mathbb{R}^{N_aN_b \times N_aN_b}, \space A \in \mathbb{R}^{N_a \times N_a}, \space B \in \mathbb{R}^{N_b \times N_b}
\end{equation}

\begin{equation} \label{eq:vec_trick}
    Mx = (A \otimes_{Kron}B)x = ((B^Tx.\mathtt{reshape}(N_a \times N_b))A).\mathtt{reshape}(N_aN_b), \space x \in \mathbb{R}^{N_aN_b}
\end{equation}

Both factorisations inject priors into the networks they parametrise, as visualised in Figure~\ref{fig:low-param_priors}. Figure~\ref{fig:low-param_priors} consists of a compression task, where the goal is to reconstruct the underlying image using low-parameter matrix factorisations. Evidently, low-rank factorisation restricts the weight space to low-dimensional subspaces. This is useful in applications such as Low-Rank RNNs \citep{mastrogiuseppe2018linking}, where the goal is specifically to create interpretable dynamics in a low-dimensional state space. However, as a general-purpose substitute for dense linear layers, this is potentially severely limiting. As shown in Figure~\ref{fig:low-param_priors}, this is intuitively equivalent to over-smoothing an image and losing the fine detail.

Likewise, while Kronecker products are full-rank, since $\mathtt{rank}(A \otimes_{Kron}B) = \mathtt{rank}(A) *\mathtt{rank}(B)$, they suffer from different strong inductive biases. Namely, the asymmetry in $A \otimes_{Kron} B$ creates a global-local resolution trade-off. As can be observed in Figure~\ref{fig:low-param_priors}, a larger $N_a$ places emphasis on variance between image patches, and enables a globally complex reconstruction. Conversely, a larger $N_b$ improves the local resolution of individual patches.  While it fails for a high-variance image such as the one in Figure~\ref{fig:low-param_priors}, it would lead to a higher-fidelity reconstruction of repeating patterns, for instance. In an abstract setting such as a generic linear layer, one cannot know a priori which bias is more appropriate. 

To hedge between these various inductive biases, this study employs the heuristic of combining the two: low-rank plus Kronecker product. The resulting LowParam linear layer (Eq.~\ref{eq:low_param_linear}) is full-rank and admits faster than dense matrix multiplication. As shown in Figure~\ref{fig:low-param_priors}, the low-rank adjustments can also help mitigate the harsh trade-off between $A$ and $B$ in the Kronecker product. While not as direct a sum as here, combining Kronecker products and Low-Rank matrices has been seen before in deep learning research, once again for LoRa-like LLM fine-tuning \citep{shen2025kron}. It should also be noted that while sparsification methods such as pruning are also heavily used, especially in neuromorphic literature \cite{bellec2017deep}, they typically do not provide any speed-ups compared to dense matrices \citep{gale2019state}. 

\begin{equation} \label{eq:low_param_linear}
\texttt{LowParam}(x) =   A(B^Tx) +   (C \otimes_{Kron} D)x
\end{equation}

\begin{figure}[H]
    \centering
        \centering        
        \includegraphics[width=\textwidth]{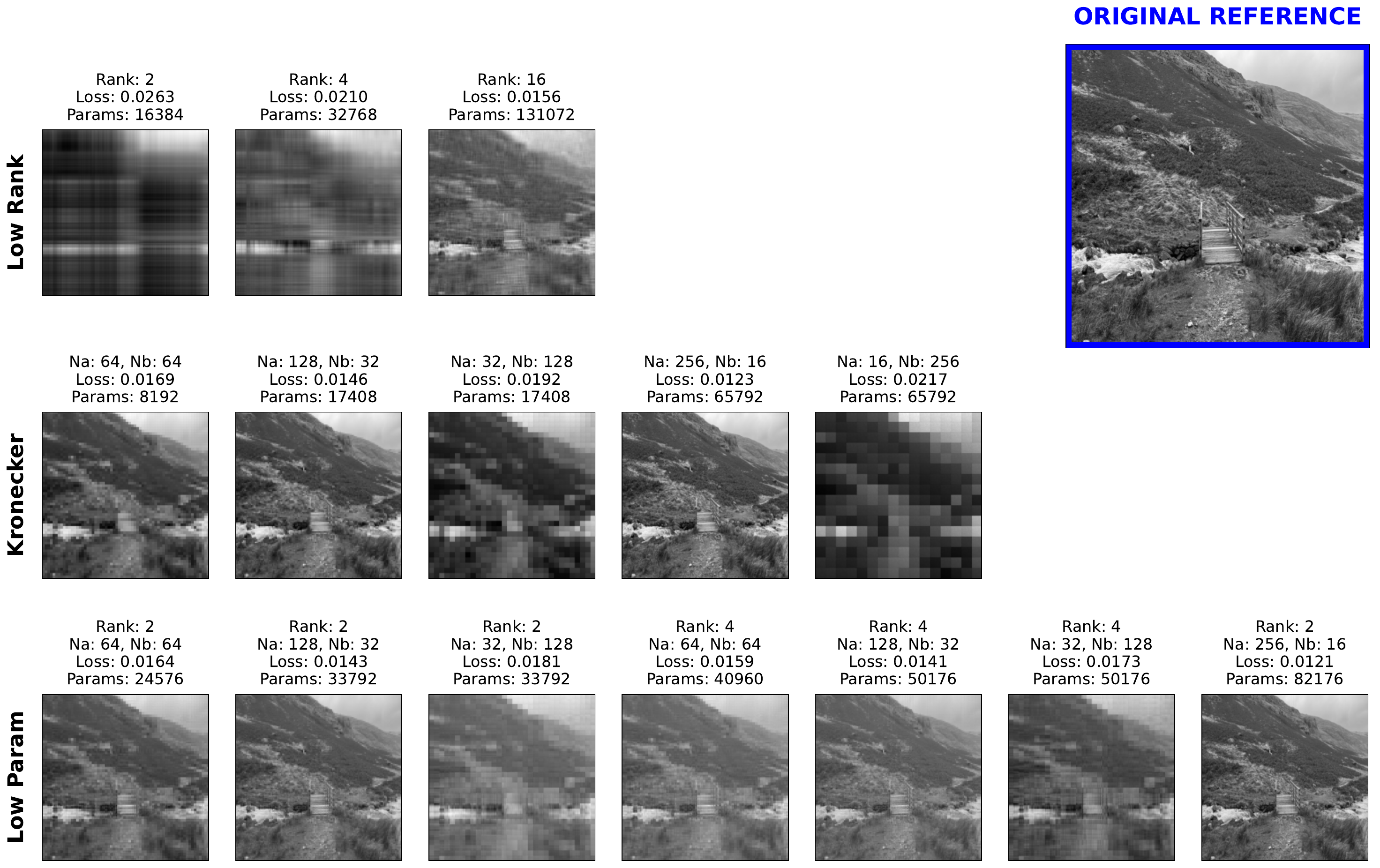}
        \hfill
        \caption{\textbf{Patterns of Image Reconstruction Error} Each row represents a matrix factorisation method, each producing different patterns of distortion when trained with gradient descent to reconstruct a target image on a low parameter budget. The reconstruction loss reported is Mean Squared Error (MSE). On one hand, Low-Rank matrices are shown to oversmooth the image on very low ranks $\in \{ 2, 4\}$. On the other hand, Kronecker products have much more accurate reconstructions on lower parameter budgets, but the performance varies massively depending on the ratio between $N_a$ and $N_b$. The proposed LowParam heuristic is shown to partially mitigate the extremes of Kronecker products with low-rank corrections. }
        \label{fig:low-param_priors}
\end{figure}

Since the goal is for the LowParam layer to behave similarly to a typical dense linear layer, its initialisation is constructed to match the output pre-activation distribution from a Kaiming-initialised weight matrix \citep{he2015delving}. For a square matrix, this entails sampling individual weights from $ \sim \mathcal{U} \left(-\sqrt{\frac{3}{n}}, +\sqrt{\frac{3}{n}}\right)$. The heuristic initialisation solution used here employs the LoRA initialisation scheme for the low-rank component, where $B \equiv \mathbf{0}$ and $A \sim \mathcal{N}(0, \frac{1}{rank})$. For the Kronecker product $C \otimes_{Kron}D$, $C \equiv\mathbf{1}$ and $D \sim \mathcal{U}\left(-\sqrt{\frac{1}{N_aN_b}}, \sqrt{\frac{1}{N_aN_b}}\right)$. As shown in Figure~\ref{fig:pre_activation_distribution}, the output distribution closely matches that of the baseline Kaiming Uniform-initialised dense, with a Kullback-Leibler (KL) divergence $\approx 10^{-2}$. 

\begin{figure}[H] 
    \centering
        \centering        
        \includegraphics[width=0.55\textwidth]{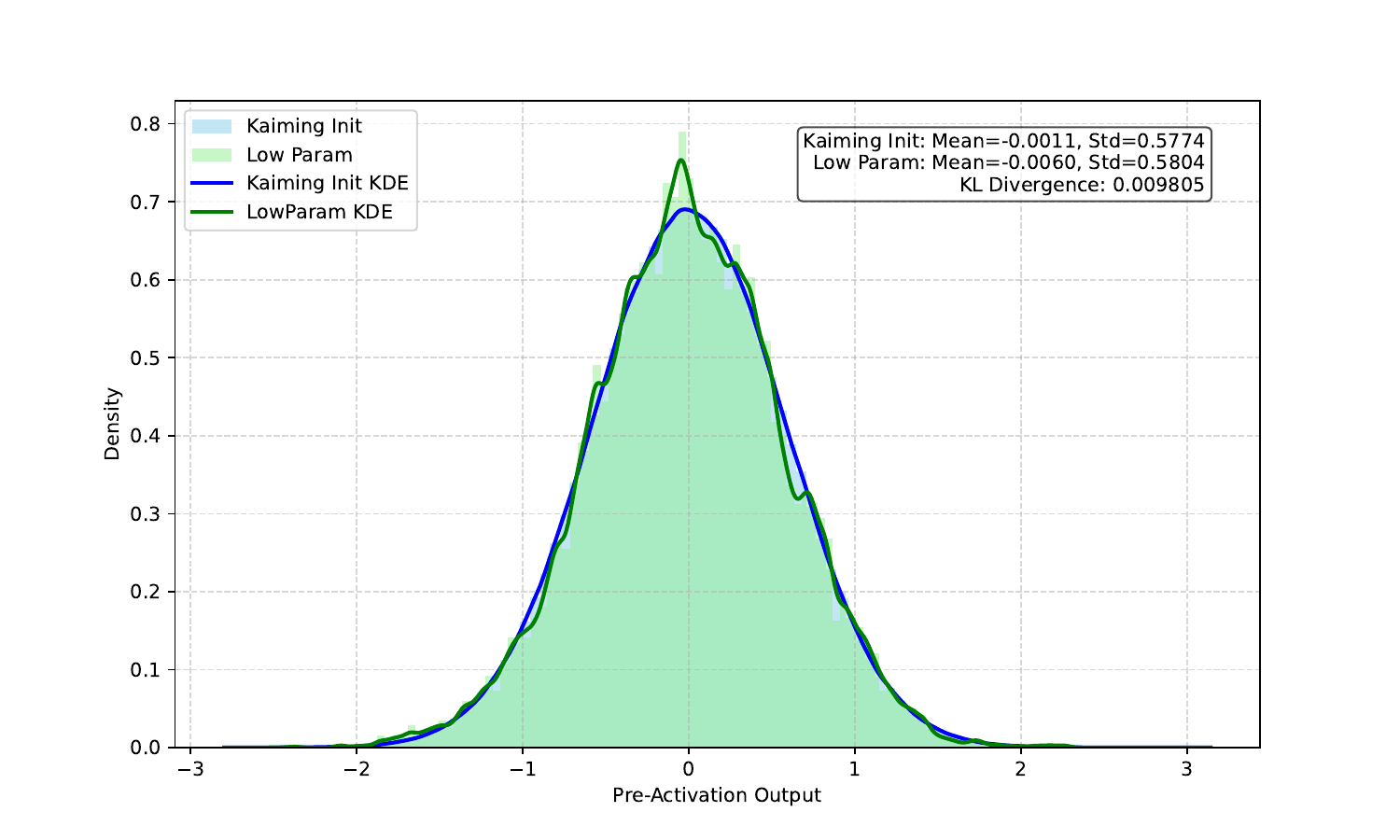}
        \hfill
        \caption{\textbf{Pre-Activation Distribution Comparison between LowParam and Kaiming Initialisation} The plot shows the output distribution of the proposed LowParam heuristic compared to a Kaiming-Uniform initialised dense matrix. The inputs are 1024 random vectors $\in \mathbb{R}^{1024}$ with entries sampled from $\mathcal{N}(0, 1)$. The LowParam layer uses $rank = 16$ for the low-rank component and $N_a, N_b=128, 8$ for the Kronecker product. While less smooth than the baseline, the initialisation scheme closely reproduces the mean and variance of the original, with a KL-divergence of $\approx 10^{-2}$ between the two distributions.}
        \label{fig:pre_activation_distribution}
\end{figure}

\subsection{Memory and Compute-Efficient Matrix Exponential} \label{sec:efficient_exp}

Section~\ref{sec:ADPTNet_def} highlighted that the purpose of using orthogonal matrices in the parametrisation of the topological conjugates is to avoid the computational inefficiency and numerical instability of generic matrix inverses. One may notice, however, that the Cayley Map (Eq.~\ref{eq:cayley_map}) used to map the key-value pairs to the orthogonal manifold also employs a matrix inverse. Because the skew-symmetric matrix $S=kv^T - vk^T$ is rank-2, it could be argued that the inverse does not incur a computational cost since it can be derived efficiently using the Woodbury Identity \citep{woodbury1950inverting}. However, in practice, in this particular case, it suffers from even worse conditioning than baseline 
inversion, and even degrades with wider networks \citep{hashemi2025instability}, thus making it effectively unserviceable. 

A more efficient and effective method for computing the orthogonal retraction is to take full advantage of the idiosyncratic rank-2 structure of $S$ in the matrix exponential function, rather than the Cayley Map approximation. To make the rank-2 structure explicit, $S$ can be rewritten as

\begin{equation*}
    S = EF^T, \space E = \begin{pmatrix}
        k & -v
    \end{pmatrix} \in \mathbb{R}^{n \times 2},
    \space F = \begin{pmatrix}
        v & k
    \end{pmatrix} \in \mathbb{R}^{n \times 2}
\end{equation*}

It is assumed that the $\beta$ step size parameter from Eq.~\ref{eq:cayley_map_rec} is absorbed into $k = \sqrt{\beta} \frac{Kx}{||Kx||}$ and $v = \sqrt{\beta} \frac{Vx}{||Vx||}$. 

To start, \textit{Theorem 1.35} from \citet{higham2008functions} (Eq.~\ref{eq:higham}) provides a means to compute  $f(M)$ for any function $f$, including exponentials, such that if $M\in \mathbb{R}^{n \times n}$ is rank $r < n$ ($M= AB^T, \space A, B \in \mathbb{R}^{n \times r}$), then the problem can be reduced to computing $f$ for a smaller $r \times r$ matrix instead of the full $n \times n$. In this case, the $\alpha$ scalar is set to zero, and $f$ is the matrix exponential function, resulting in Eq.~\ref{eq:eff_exp} for $e^S$. 

\begin{equation} \label{eq:higham}
    f\left(\alpha I_n + AB^T\right) = f(\alpha)I_n + A(B^TA)^{-1}(f(\alpha I_r + B^TA) - f(\alpha)I_r)B^T, \space \alpha \in \mathbb{R}, \space I_k = \mathtt{eye}(k)
\end{equation}

\begin{equation} \label{eq:eff_exp}
        e^{S} = I_n + E(F^TE)^{-1}(e^{F^TE} - I_2)F^T
\end{equation}

\begin{equation} \label{eq:core_matrix}
    C = F^TE = \begin{pmatrix}
        k^Tv & -v^Tv\\
        k^Tk & -v^Tk
    \end{pmatrix}
    =
    \begin{pmatrix}
        k\cdot v & -\beta \\
        \beta  & - k\cdot v
    \end{pmatrix}
\end{equation}

Computing the exponential retraction is thus reduced to finding the inverse and exponential of a $2 \times 2$ matrix $C = F^TE$ (Eq.~\ref{eq:core_matrix}). From Eq.~\ref{eq:core_matrix}, it should be noted that the determinant $\det(C)=\beta^2 - (k \cdot v)^2$ is always $\ge 0$ since $| k \cdot v| \le \left\| k \right\| \left\| v \right\|$. In practice, the determinant is strictly $>0$ since $K$ and $V$ are trained separately, producing distinct $k$ and $v$, and a small perturbation can also be introduced to enforce strict positivity. Because $Tr(C) = 0$ and the determinant is positive, computing $e^C$ simplifies to

\begin{equation} \label{eq:core_exp}
    e^C = \cos(\sqrt{\det(C)})I_2 + \frac{\sin(\sqrt{\det(C}))}{\sqrt{\det(C)}}C
\end{equation}

And $C^{-1} = \frac{-1}{\det(C)}C$. Because of the explicit formulation of $e^S$ in terms of low-rank components, the final orthogonal $Q$ never has to be materialised in full $n \times n$, allowing for memory and compute-efficient matrix-vector multiplication:

\begin{equation*}
    Qx = e^Sx = (I_n + EC^{-1}(e^C - I_2)F^T)x = x + E ((C^{-1}(e^C - I_2))(F^Tx))
\end{equation*}


\subsection{ADPTNet Parallelisation} \label{sec:alt_deer}

\subsubsection{ADPTNet DEER Convergence Properties}

In Lemma~\ref{lemma:le_spectrum}, it is established that the Lyapunov spectrum of ADPTNet is determined by the recurrent eigenvalues $\Lambda$ and the $\beta$ manifold step size parameter. This level of parametric control has secondary benefits in determining the convergence properties of the DEER algorithm  \citep{lim2024parallelizing, gonzalez2024towards} (see Section~\ref{sec:paral_algorithms}) when applied to ADPTNet. 

DEER displays different convergence properties depending on the region of the trajectory guess $s$ state space (Eq.~\ref{eq:residual}). As proven in \citet{gonzalez2026predictability}, DEER is globally guaranteed to converge to the correct trajectory $s^*$ at least linearly. In other words, the distance of the current guess $s^{(i)}$ to $s^*$ decays by at least a global constant factor $\in (0, 1)$ between Newton iterations. Furthermore, DEER can even converge quadratically, i.e., practically in a constant number of iterations, if the residual $r(s^{(i)})$ is sufficiently small. The quadratic convergence condition, as stated in \textit{Theorem 5} from \citet{gonzalez2026predictability}, is defined by the Lipschitz constant $L$ of the recurrent Jacobian and the LLE $\alpha_{max}$  of the system being parallelised: 

\begin{equation} \label{eq:quad_basin}
    \left\|r(s^{(i)}) \right\|_2 < \frac{2}{a^2L}\left(\frac{e^{\alpha_{max}} - 1}{e^{\alpha_{max} T} - 1}\right)^2
\end{equation}

In Eq.~\ref{eq:quad_basin}, $T$ denotes the total number of time steps of the trajectory, and $a \ge 1$ is a constant designed to capture the distortion effect of transient dynamics that may push the observed LLE to be larger than its actual long-term value (i.e.,~overshoot):

\begin{equation} \label{eq:DEER_As}
    \left\| J_{t + k -1} \dots J_t\right\| \le ae^{\alpha_{max}}, \space t > 1, \space k \ge 0
\end{equation}

It should be observed that since the LLE $\alpha$ is present in the definition of the quadratic convergence region, the $\Lambda$ parametrisation in ADPTNet exerts direct influence over it. In practical terms, a lower $\lambda_{max}$ and a sufficiently low $\beta$ can increase the domain where DEER converges quadratically. Moreover, it should be re-emphasised that the Lipschitz constant $L$ for DEER is inherited from the Lipschitz-ness of the recurrent Jacobian $J$ of the dynamics being parallelised (\textit{Theorem 2} from \citet{gonzalez2026predictability}). Concretely, L is defined as either Eq.~\ref{eq:lipschitz_def_1} or Eq.~\ref{eq:lipschitz_def_2} \citep{kim2021lipschitz}: 

\begin{equation} \label{eq:lipschitz_def_1}
    \left\| J(x) - J(y) \right\|_2 \le L \left\| x - y\right\|
\end{equation}

\begin{equation}\label{eq:lipschitz_def_2}
    L = \sup_{x\in \mathbb{R}^n} \left\| \frac{dJ(x)}{dx} \right\|_2
\end{equation} 

Based on Eq.~\ref{eq:lipschitz_def_2}, for ADPTNet dynamics, using Eq.~\ref{eq:J} - \ref{eq:simple_dMdx} and the $A$ and $B$ shorthand notation from Section~\ref{sec:ADPTNet_def}, Lipschitz-ness can be derived as:

\begin{equation} \label{eq:ADPTNet_lip}
    \begin{split}
        L &= \sup_{x\in \mathbb{R}^n}\left\| \frac{dJ(x)}{dx} \right\|_2 \\
          &= \sup_{x\in \mathbb{R}^n}\left\| \frac{d(M + \frac{dM}{dx}x)}{dx} \right\|_2 \\
          &= \sup_{x\in \mathbb{R}^n}\left\|  \frac{dM}{dx} + \left( \frac{dM}{dx} + \frac{d^2M}{dx^2}x \right) \right\|_2 \\
          &= \sup_{x\in \mathbb{R}^n} 
          \left\| 2\left( \beta A \left(\frac{dS}{dx} B - B^T \frac{dS}{dx}\right)A\right) +  \beta \frac{d}{dx} \left( A \left(\frac{dS}{dx} B - B^T \frac{dS}{dx} \right) A\right)x \right\|_2 \\
          &=\beta \sup_{x\in \mathbb{R}^n} 
          \left\| 2\left(A \left(\frac{dS}{dx} B - B^T \frac{dS}{dx} \right) A\right) + \frac{d}{dx} \left( A \left(\frac{dS}{dx} B - B^T \frac{dS}{dx}\right)A\right)x \right\|_2 \\
    \end{split}
\end{equation}

Without having to actually evaluate the upper bound on the norm in Eq.~\ref{eq:ADPTNet_lip}, it can be seen that $L \propto \beta$ ($L = \beta c$, for a constant $c \ge 0$). Therefore, the size of the quadratic convergence region for DEER is itself $\propto \frac{1}{\beta}$. In addition, as highlighted in Lemma~\ref{lemma:le_spectrum}, the LLE $\alpha_{max}$ is always bounded $\in \left[ \ln(\tilde{\sigma}_{min}), \ln(\tilde{\sigma}_{max}) \right]$, regardless of the trajectory length or region in state space.  Therefore, one can define the constant $0\le d\le1$ such that

\begin{equation*}
    \alpha_{max} = d\ln(\tilde{\sigma}_{min}) + (1-d)\ln(\tilde{\sigma}_{max}) 
\end{equation*}

Altogether, the formula for the DEER quadratic convergence basin as a function of ADPTNet parametrisation can be summarised as:

\begin{equation} \label{eq:ADPTNet_quad_conv_basin}
        \left\|r(s^{(i)} \right\|_2 < \frac{2}{a^2\beta c}\left(\frac{e^{ d\ln(\tilde{\sigma}_{min}) + (1-d)\ln(\tilde{\sigma}_{max})} - 1}{e^{(d\ln(\tilde{\sigma}_{min}) + (1-d)\ln(\tilde{\sigma}_{max}))T} - 1}\right)^2
\end{equation}

Eq.~\ref{eq:ADPTNet_quad_conv_basin} has two important implications. Firstly, the lower the extremal eigenvalue parameters $\lambda_{min}$ and $\lambda_{max}$, the lower the bounds on the LLE $\alpha_{max}$ and thus the larger the basin of quadratic convergence. This is intuitive since a lower LLE effectively imposes a more quickly decaying memory for the recurrent dynamics. In turn, any error in the current state guess at time $t$, $s_t^{(i)}$, will influence fewer subsequent $k$ time steps $s_{t:t+k}^{(i+1)}$ in the next Newton iteration (i.e., $\alpha_{max} \propto k$, where $k$ is the number of subsequent time steps affect by the current one).

Conversely, if the LLE is positive, the recurrent dynamics become chaotic, and the quadratic convergence region vanishes. In practice, this chaotic regime forces a number of Newton iterations close to or equal to the total number of time steps $T$ being simulated. If $\beta$ is selected sufficiently low, ADPTNet prevents this by parametrising the eigenspectrum $\Lambda \in [0,\space 1]^{n}$ (see Section~\ref{sec:eigvalue_param}). 

Secondly, a lower $\beta$ implies a larger basin of quadratic convergence. Consider the extreme case where $\beta = 0$, then $L = 0$. $L=0$ is equivalent to a linear system which converges in a single DEER iteration, with an infinitely-sized quadratic convergence region. This can be easily verified since $\beta=0$ implies that $Q=I_n$, and the similarity transform in Eq.~\ref{eq:state_dependent_similarity} becomes $I_n\Lambda I_n^T = \Lambda$, equivalent to an LTI SSM that can be easily parallelised with a single convolution/parallel scan. 

To conclude this section, it should also be mentioned that ADPTNet can also be reduced to linear dynamics by enforcing a constant eigenspectrum $\lambda_{min} = \lambda_{max} = \lambda$. In that case, the recurrent dynamics become $Q\lambda I_nQ^Tx=\lambda QQ^Tx = \lambda x$, regardless of the value of $\beta$. 

\subsubsection{Conv-DEER and Forward-DEER} \label{sec:conv_fwd_deer}

Computing the full $n \times n$ recurrent Jacobians $J_t$ required for vanilla DEER is both computationally and memory intensive, effectively limiting its scalability. \citet{gonzalez2024towards} introduced Quasi-DEER to account for this, replacing the full Jacobian with its diagonal $\mathtt{diag}(J_t)$.

In ADPTNet dynamics, for sufficiently small $\beta$, the Jacobian $J_t  = M + \frac{dM}{dx}x$ is equivalent to the forward dynamics $M$ plus a perturbation $\beta P$ (see Section~\ref{sec:ADPTNet_def}). Taking advantage of this, ADPTNet admits two efficient, Jacobian-free, Quasi-DEER heuristic approximations. In this context, this is understood to mean an approximation scheme for $\mathtt{diag}(J_t)$ which does not require computing any derivatives. 

The most inexpensive approximation approach proposed here is to replace $\mathtt{diag}(J_t)$ with the recurrent eigenvalue spectrum $\Lambda$. This comes at a minimal computational cost, since $\Lambda$ is known a priori for all time steps and only requires to be materialised through ZOH discretisation once per batch \footnote{For more details about exponential parametrisation see Section~\ref{sec:eigvalue_param}}. Furthermore, sharing $\Lambda$ across batched execution provides immediate memory savings proportional to the batch size. Using constant $\Lambda$ at each time step also converts each DEER iteration into an LTI system that can be computed using global and efficient $\mathcal{O}(T\log(T))$ convolutions based on FFTs (see Section~\ref{ssms}), in addition to parallel associative scans. Hence, this method is referred to as Conv-DEER. This opens the door to taking advantage of highly optimised CUDA libraries such as FlashFFTConv \citep{fu2024flashfftconv}, or even potentially using dedicated FFT hardware accelerators \citep{garrido2022survey}. 

To understand why Conv-DEER could be a viable heuristic, it is important to revisit how DEER is implemented in practice \citep{gonzalez2024towards}. First, the term $b$ is computed for each time step:

\begin{equation*}
    b_t = f(s_{t-1}^{(i)}) - A_ts_{t-1}^{(i)}
\end{equation*}

Where $f$ is the recurrent function and $A_t$ is the recurrent Jacobian $J_t$ for DEER, $\mathtt{diag}(J_t)$ for Quasi-DEER, and $\Lambda$ for Conv-DEER. The next trajectory guess $s^{(i+1)}$ is then computed as:

\begin{equation} \label{eq:deer_weighted_sum}
    s_t^{(i+1)} = A_{t-1}\dots A_1b_0 + A_{t-1}\dots A_2b_1 +\dots  + A_{t-1}b_{t-2} + b_{t-1}
\end{equation}

In vanilla DEER, the products of the form $A_tA_{t-1} \dots A_{t-k+1}$ are exactly long-term Jacobians $J_t \dots J_{t-k+1}$ whose effect on $b_{t-k}$, e.g., exponential decay, is described in some capacity by the Lyapunov spectrum and especially the LLE (see Eq.~\ref{eq:DEER_As}). As stated in Lemma~\ref{lemma:le_spectrum}, in the case of ADPTNet the bounds on the entire Lyapunov spectrum are constrained by the choice of $\Lambda$ and $\beta$. Therefore, intuitively, using $\Lambda^{k-1}$ instead of $A_tA_{t-1} \dots A_{t-k+1}$ should produce a similar long-range propagation effect, subject to the choice of $\beta$. 

The second heuristic proposed here is to take the diagonal of the forward recurrent dynamics $\mathtt{diag}(M_t)$ as an approximation to the full Jacobian diagonal $\mathtt{diag}(J_t)$ from Quasi-DEER. Accordingly, this method is referred to as Forward-DEER. For ADPTNet using the low-rank-aware matrix exponential from Section~\ref{sec:efficient_exp}, this does not require materialising the full recurrent matrix $M_t$.  The notation in Eq.\ref{eq:eff_exp} can be updated to emphasise how the exponential is an adjustment to the identity with left $L$ and right $R$ low-rank decomposition terms:

\begin{equation*}
    \begin{split}
        e^{S} &= I_n + E(F^TE)^{-1}(e^{F^TE} - I_2)F^T \\
              &= I_n + LR^T, \space L = E(F^TE)^{-1}(e^{F^TE} - I_2) \in \mathbb{R}^{n \times 2}, \space R = F
    \end{split}
\end{equation*}

The full matrix $M_t$ and its diagonal become: 

\begin{equation*}
\begin{split}
    M &= (I_n + LR^T)\Lambda (I_n + RL^T) \\
      &= \Lambda + \Lambda R L^T + LR^T\Lambda + LR^T\Lambda RL^T
\end{split}
\end{equation*}

\begin{equation} \label{eq:diag_M}
    \mathtt{diag}(M) = \Lambda + 2\Lambda(l_1\odot r_1 + l_2 \odot r_2) + (\hat{l}_1l_1 + \hat{l}_2l_2), \space \hat{L} = L((R^T\Lambda)R)
\end{equation}

Where lower-case $l$, $r$, and $\hat{l}$ denote individual columns of $L$, $R$, and $\hat{L}$ respectively. It should be noted that no full $n \times n$ matrix-matrix or matrix-vector multiplications are required to compute $\mathtt{diag}(M_t)$. 

Both Jacobian-free heuristics proposed in this study evidently introduce approximation errors compared to vanilla or even Quasi-DEER. In this context, the step-wise approximation errors $\epsilon_t^{approx}$ are quantified as the norm of the difference between approximated $\tilde{b}_t$ and the actual $b_t$ obtained with vanilla DEER $A_t = J_t$:

\begin{equation*}
    \begin{split}
         \epsilon_t^{approx} &= \left\| b_t - \tilde{b}_t\right\|_2 \\
                             &=  \left\| f(s_{t-1}) - J_ts_{t-1} - f(s_{t-1}) + A_ts_{t-1}\right\|_2 \\
                             &=  \left\|  (A_t - J_t)s_{t-1} \right\|_2 \\
                             &=  \left\| (A_t - M_t - \frac{dM_t}{ds_{t-1}}s_{t-1})s_{t-1} \right\|_2 \\
    \end{split}
\end{equation*}

On the last line, $J_t$ is instantiated with the ADPTNet Jacobian $J_t = M + \frac{dM_t}{dx_{t-1}}x_{t-1}$ from Section~\ref{sec:ADPTNet_def}. Starting with Conv-DEER where $A_t = \Lambda$ the approximation error takes the form:

\begin{equation} \label{eq:approx_error}
    \begin{split}
        \epsilon_t^{approx} &=  \left\| (\Lambda - M_t - \frac{dM_t}{ds_{t-1}}s_{t-1})s_{t-1} \right\|_2 \\
                            &\le \left (\left\| \Lambda - Q_t\Lambda Q_t^T\right\|_2 + \left\|\frac{dM_t}{ds_{t-1}}s_{t-1}\right\|_2 \right) \left\| s_{t-1}\right\|_2 \\
                            &= (\epsilon_t^{sim} + \epsilon_t^{der}) \left\| s_{t-1}\right\|_2 \\
    \end{split}
\end{equation}

Here, $\epsilon_t^{sim}$ and $\epsilon_t^{der}$ notations are introduced for the approximation errors resulting from the unaccounted effect of the similarity transform (i.e., "misalignment") and missing derivative term, respectively. $\epsilon^{der}_t$ is upper-bounded by $8\beta\lambda_{max}$, as derived in the proof of Lemma~\ref{lemma:J_singular_vals}. For $\epsilon_t^{sim}$, it is important to reiterate that $\Lambda$ is positive real-valued and diagonal, and so it is also, by extension, Hermitian. $Q$ from the similarity transform is orthogonal and thus unitary as well. Therefore, $\Lambda$ and $Q\Lambda Q^T$ exist in the same unitary orbit, and the following bound on $\epsilon_t^{sim}$ applies \citep{hiai1989distance, davidson1988estimating}:

\begin{equation} \label{eq:unitary_orbit_distance}
    \epsilon_t^{sim} = \left\| \Lambda - Q_t\Lambda Q_t^T\right\|_2 \le \lambda_{max} - \lambda_{min}
\end{equation}

Forward-DEER trivially presents the same derivative approximation error $\epsilon_t^{der}$ as Conv-DEER, however, instead of a misalignment error, it introduces an off-diagonal entry error $\epsilon_t^{diag} = \left\| \mathtt{diag}(M_t) - M_t\right\|$. As showcased in Figure~\ref{fig:beta_effect_interactions}, the magnitude of the off-diagonal entries in $M_t$ is directly controlled by $\beta$.

To highlight further distinctions between Forward and Conv-DEER, one can take Quasi-DEER as a baseline and consider the low-rank-aware computation of the matrix exponential:

\begin{equation*}
    \begin{split}
        \mathtt{diag}(J_t) &= \mathtt{diag}(M_t) + \mathtt{diag}(\frac{dM_t}{ds_{t-1}}s_{t-1}) \\
                           &= \Lambda + 2\Lambda(l_1\odot r_1 + l_2 \odot r_2) + (\hat{l}_1l_1 + \hat{l}_2l_2) +  \mathtt{diag}(\frac{dM_t}{ds_{t-1}}s_{t-1}) \\
    \end{split}
\end{equation*}

It can be observed that Conv and Forward-DEER are effectively truncated approximations of the diagonal of the full Jacobian from Quasi-DEER. Furthermore, the difference between the two approximations $\epsilon_t^{fwd} = \left\| 2\Lambda(l_1\odot r_1 + l_2 \odot r_2) + (\hat{l}_1l_1 + \hat{l}_2l_2)\right\|_2 \propto \lambda_{max}$ since $\hat{L}$ also contains $\Lambda$ in its composition. It should also be noted that when using the efficient matrix exponential from Section~\ref{sec:efficient_exp}, the $k$ and $v$ vectors comprising $E$, $F$, and thus $L$ and $R$ as well, each have norm $\sqrt{\beta}$. Hence $\beta$ also determines the bounds of the $\epsilon_t^{fwd}$ approximation error.

So far, the focus has been placed on \textit{local} approximation errors, stemming from computing individual $b_t$ terms. However, each state estimate $s_t^{(i+1)}$ is a weighted sum of all $b_k^{(i)}$ for $k<t$ (Eq.~\ref{eq:deer_weighted_sum}). Therefore, all local errors spread to subsequent time steps:

\begin{equation}\label{eq:weighted_sum_local_errors}
    \tilde{s}_t^{(i+1)} = A_{t-1}\dots A_1\epsilon_0^{approx} + A_{t-1}\dots A_2\epsilon_1^{approx} +\dots  + A_{t-1}\epsilon_{t-2}^{approx}  + \epsilon_{t-1}^{approx}
\end{equation}

Here, $\tilde{s}$ denotes the accumulation of approximation errors in the trajectory guess. The $\epsilon_k^{approx}$ terms are used with a wider scope than originally in Eq.~\ref{eq:approx_error}, since it now differs based on the approximation and baseline methods being compared (see Table~\ref{table:error_summary} for summary) The influence of local error terms $\epsilon_k^{approx}$ over future time steps $\tilde{s}_{k+j}^{(i+1)}$ depends on the long-term products $A_{k+j-1}\dots{A_k}$. This "\textit{transmission}" term can itself introduce errors in the DEER iteration. For instance, if one compares Conv-DEER with vanilla DEER, the $A$-products will themselves differ by at most $(\lambda_{max}- \lambda_{min} + 8\beta\lambda_{max})^{j}$ (Eq.~\ref{eq:approx_error} and \ref{eq:unitary_orbit_distance}). However, given Lemma~\ref{lemma:le_spectrum}, the overall effect of $J_{k+j-1}\dots J_k$ compared to $\Lambda^j$ is tightly related to the LLE of the recurrent dynamics. In other words, the magnitude of a local $\epsilon_k^{approx}$, after being propagated long term over $j$ time steps, will decay at a similar exponential rate $\approx e^{\alpha_{max}j}$. Therefore, the scope here will be limited to measuring the overall magnitude of accumulating local $\epsilon_k^{approx}$ terms, omitting "transmission" errors. 

\begin{table}[h!] 
    \centering
    \renewcommand{\arraystretch}{1.5}
    \begin{tabular}{|l|c|c|}
        \hline
        \diagbox{Baseline}{Approx. Method} & Conv-DEER & Forward-DEER \\
        \hline
        Vanilla DEER & \gape{\makecell{$ \epsilon_k^{sim} \le \lambda_{max} - \lambda_{min}$  \\
        $ \epsilon_k^{der} \le 8\beta\lambda_{max}$}} & N/A \\
        \hline
        Quasi-DEER & \gape{\makecell{
        $\epsilon_k^{fwd} = \left\| 2\Lambda(l_1\odot r_1 + l_2 \odot r_2) + (\hat{l}_1l_1 + \hat{l}_2l_2)\right\|_2 \propto \lambda_{max}, \space \beta$ \\
        $\epsilon_k^{der} \le 8\beta\lambda_{max}$
        }} & \gape{\makecell{ $\epsilon_k^{der} \le 8\beta\lambda_{max}$}}\\
        \hline
        Forward-DEER &\gape{\makecell{
        $\epsilon_k^{fwd} = \left\| 2\Lambda(l_1\odot r_1 + l_2 \odot r_2) + (\hat{l}_1l_1 + \hat{l}_2l_2)\right\|_2 \propto \lambda_{max}, \space \beta$  }} & N/A \\
        \hline
    \end{tabular}
    \caption{Summary of the different local approximation error terms that make up $\epsilon_k^{approx}$ depending on the baseline established method (DEER and Quasi-DEER) and the approximation methods proposed here (Conv and Forward-DEER). The error terms listed are also multiplied by the norm of the state $s_{k-1}$, to obtain $\epsilon_k^{approx} = (\epsilon + \dots) \left\| s_{k-1}\right\|_2$.}
    \label{table:error_summary}
\end{table}

Considering the approximation of the long-term effect of $A_{k+j-1}\dots A_{k} \approx e^{\alpha_{max}j} \le \tilde{\sigma}_{max}^j$, the accumulation of local approximation errors at a given time step $t$ becomes Eq.~\ref{eq:error_per_timestep}. The overall upper bound on the accumulation over all time steps $T$ of the trajectory of local errors $\tilde{S}= \Sigma_{j=1}^T \tilde{s}_j^{(i+1)}$ is Eq.~\ref{eq:total_approx_error}.

\begin{equation} \label{eq:error_per_timestep}
    \tilde{s}_t^{(i+1)} = \tilde{\sigma}_{max}^{t-1}\epsilon_0^{approx} +  \tilde{\sigma}_{max}^{t-2}\epsilon_1^{approx} + \dots + \tilde{\sigma}_{max}\epsilon_{t-2}^{approx} + \epsilon_{t-1}^{approx}
\end{equation}

\begin{equation} \label{eq:total_approx_error}
\begin{split}
            \tilde{S}^{(i+1)} &= (\tilde{\sigma}_{max}^{T-1} + \dots + \tilde{\sigma}_{max}^{1} + 1)\epsilon_0^{approx} + (\tilde{\sigma}_{max}^{T-2} +\dots + \tilde{\sigma}_{max}^{1} + 1)\epsilon_1^{approx} \\
            &\quad + \dots + (\tilde{\sigma}_{max}+1)\epsilon_{t-2}^{approx} + \epsilon_{t-1}^{approx} \\
                    &= \frac{\tilde{\sigma}_{max}^{T} - 1}{\tilde{\sigma}_{max} - 1}\epsilon_0^{approx} + \frac{\tilde{\sigma}_{max}^{T-1} - 1}{\tilde{\sigma}_{max} - 1}\epsilon_1^{approx} \\
                    &\quad + \dots + \frac{\tilde{\sigma}_{max}^{2} - 1}{\tilde{\sigma}_{max} - 1}\epsilon_{t-2}^{approx} +
                    \frac{\tilde{\sigma}_{max}^{1} - 1}{\tilde{\sigma}_{max} - 1}\epsilon_{t-1}^{approx} \\
                    &= \frac{1}{\tilde{\sigma}_{max} - 1}( 
                    (\tilde{\sigma}_{max}^{T} - 1)\epsilon_0^{approx} + (\tilde{\sigma}_{max}^{T-1} - 1)\epsilon_1^{approx}\\
                    &\quad +\dots + (\tilde{\sigma}_{max}^{2} - 1)\epsilon_{t-2}^{approx} + (\tilde{\sigma}_{max}^{1} - 1)\epsilon_{t-1}^{approx} )
\end{split}
\end{equation}

The main takeaway from laying out the sketch of a loose upper bound on accumulating local approximation errors in Eq.~\ref{eq:total_approx_error} is that, perhaps unsurprisingly, it increases with sequence length $T$ and $\tilde{\sigma}_{\max}$, which, in turn, is $\lambda_{max} + 8\beta\lambda_{max}$. In addition, if the overview in Table~\ref{table:error_summary} is also considered, it is possible to make a number of predictions regarding the performance of the proposed approximation schemes. 

Conv-DEER should generally require more Newton iterations than vanilla, Quasi, and Forward DEER. In particular, considering $\epsilon^{sim}$ and $\epsilon^{der}$ error terms, its convergence should be impacted by the spread of the recurrent eigenvalues ($\lambda_{max} - \lambda_{min}$), $\beta$, and the LLE (also determined by $\lambda_{max}$ and $\beta$). This expectation also rests on the intuition that a higher spread of the eigenvalues allows a smaller $\beta$ to have higher variance in decay rates "chosen" per element of the recurrent state $x_t$. Forward-DEER should be expected to be less sensitive to eigenvalue spread and generally track closer to Quasi-DEER than Conv-DEER.  However, a larger LLE also entails, by definition, longer-term memory, which generally "smears" state-guess errors further into the future. Thus, both approximations, much like DEER in general, incur a penalty in convergence speed when the network has long-range memory. 

\subsubsection{Damping} \label{sec:scale_elk_damping}

Gauss-Newton iterative methods such as DEER are known to suffer from instability. Accordingly, as first introduced in \citet{gonzalez2024towards}, a notable DEER stabilisation strategy is to leverage trust regions, i.e. the Levenberg-Marquardt method. The resulting solution, Evaluating Levenberg-Marquardt with
Kalman (ELK), effectively dampens the eigenvalues of the recurrent Jacobians $J_t$ to avoid "exploding" behaviour when computing the $\mathcal{J}^{-1}\text{r}$ product in DEER (see Section~\ref{sec:paral_algorithms}). Furthermore, to minimise computational overhead, as also proposed in \citet{gonzalez2024towards}, eigenvalue damping can also be achieved by simply multiplying recurrent Jacobians by a scalar constant $\in [0, 1]$ (Scale-ELK). 

With the added approximation errors introduced by the Conv and Forward-DEER methods proposed here, numerical instability may also be a challenge. As established in Section~\ref{eq:eig_parametrisation}, ADPTNet is constructed with stable recurrent Jacobians by design. However, throughout DEER convergence, intermediary trajectory guesses $s^{(i)}$ are not subject to the same normalisation constraints as in valid ADPTNet trajectories $s^{*}$. Moreover, Section~\ref{sec:conv_fwd_deer} shows that the magnitude of local approximation errors depends on the norm of these state guesses $s_t$ as well as parameters such as $\beta$ and $\lambda_{\max}$. Propagating unnormalised local errors $\epsilon^{\text{approx}}$ in Eq.~\ref{eq:deer_weighted_sum} may in itself cause instabilities, and thus damping may be necessary. 

Let $\tilde{A_t}^{\text{conv}}$ and $\tilde{A_t}^{\text{fwd}}$ denote Jacobian-free approximations of Quasi-DEER $\text{diag}(J_t)$ from Conv or Forward-DEER respectively. Using Quasi DEER in conjunction with Scale-ELK damping $\mathtt{Damping}(A_t) = kA_t, \space k\in [0, 1]$ on ADPTNet recurrent Jacobians results in:

 \begin{equation} \label{eq:damping_error_and_info}
 \begin{split}
          k(\text{diag}(J_t)) &= k(\text{diag}(M_t)  + \text{diag}(\frac{dM_t}{dx_{t-1}}x_{t-1})) \\
                              &= k(\Lambda +  2\Lambda(l_1\odot r_1 + l_2 \odot r_2) + (\hat{l}_1l_1 + \hat{l}_2l_2) +  \text{diag}(\frac{dM_t}{dx_{t-1}}x_{t-1})) \\
                              &= k (\tilde{A_t}^{\text{conv}} + 2\Lambda(l_1\odot r_1 + l_2 \odot r_2) + (\hat{l}_1l_1 + \hat{l}_2l_2) + \text{diag}(\frac{dM_t}{dx_{t-1}}x_{t-1})) \\
                              &= k(\tilde{A_t}^{\text{fwd}} + \text{diag}(\frac{dM_t}{dx_{t-1}}x_{t-1})) \\
 \end{split}
 \end{equation}

As established in Lemma~\ref{lemma:J_singular_vals}, the bounds on the norm of the derivative term $\frac{dM_t}{dx_{t-1}}x_{t-1}$ and, implicitly, its diagonal, are determined by $\beta$. If $\beta$ is chosen sufficiently small, then $\left\| \frac{dM_t}{dx_{t-1}}x_{t-1}\right\| \ll \tilde A_t^{\text{fwd}}$. Therefore, for a certain damping $k$, $k(\text{diag}(J_t)) \approx k\tilde A_t^{\text{fwd}}$. While including additional approximation error, a similar logic can also be applied to $\tilde{A_t}^{\text{conv}}$. In other words, damping should align Jacobian-free approximations with Quasi-DEER (and similarly vanilla DEER) and lead to more and more similar convergence behaviour as $k$ increases.

\subsubsection{Block-Wise DEER} \label{sec:block_deer}

As proven in \citet{gonzalez2024towards}, even in the worst-case, DEER and its more efficient approximations are guaranteed to converge correctly on the first $i$ states of the trajectory within the first $i$ Newton iterations. If $j$ more Newton iterations are required to converge to the fully correct trajectory $s^*$, compute and memory will still be wasted recomputing the already converged first $i$ steps, at least. When training, for example, all DEER intermediary trajectory guesses have to be stored for back-propagation, imposing a particularly heavy memory cost (Subfigure ~\ref{fig:deer_baseline}). 

In contrast, fully sequential simulation, by definition, only has to compute each time step once, i.e., no wasted resources. However, this evidently comes at the cost of no GPU parallelism and slower wall-clock simulation\footnote{In the event that the dynamics are effectively parallelisable by DEER \citep{gonzalez2026predictability}.}. Therefore, one solution towards a Pareto-optimal balancing between memory utilisation and GPU parallelism during training is to execute DEER in a block-wise sequential manner (Subfigure~\ref{fig:block_savings}). 

Intuitively, the original $T$-length sequence is split into $\frac{T}{B}$ $B$-sized blocks. DEER or DEER-approximate parallelisation is then applied to each individual block. The blocks are sequentially computed and chained together by passing the final state from one as the initial state to the next. In sum, the algorithm pseudocode is summarised in Algorithm~\ref{alg:block_wise_deer}.

\begin{algorithm}
\caption{Block-Wise DEER}\label{alg:block_wise_deer}
\begin{algorithmic}
\Require $T > 0$, $B$, \texttt{inputs}$\in \mathbb{R}^{\text{Batch Size} \times D \times T}$, \texttt{initial state} $\in \mathbb{R}^{\text{Batch Size} \times D}$, \texttt{states}$\in \mathbb{R}^{\text{Batch Size} \times D \times T}$, $\mathtt{n}_{\text{iters}}$
\Ensure $T \% B=0$
\State $N_{\text{blocks}} \gets T / B$
\State \texttt{input blocks} $ \gets \mathtt{inputs.chunk(N_{blocks}, dim=-1)}$
\State \texttt{states blocks} $ \gets \mathtt{states.chunk(N_{blocks}, dim=-1)}$
\State \texttt{states out} $\gets$ \texttt{empty list}
\State $i \gets 0$
\While{$i < n_{\text{blocks}}$}
    \State $j \gets 0$
    \State \texttt{states guess} $\gets$ \texttt{state blocks[i]}
    \While{$j <  n_{\text{iters}}$}
        \State \texttt{states guess} $\gets$ \texttt{DEER(states guess, initial guess, input blocks[i])}
        \State $j \gets j + 1$
    \EndWhile
    \State $i \gets i + 1$
    \State \texttt{initial state = states guess[..., -1]}
    \State \texttt{states out.append(states guess)}
\EndWhile

\Return \texttt{{cat(states out, dim=-1)}}
\end{algorithmic}
\end{algorithm}

\begin{figure}[H]
    \centering
        \begin{subfigure}[b]{0.5\textwidth}
            \centering        
            \includegraphics[width=\textwidth]{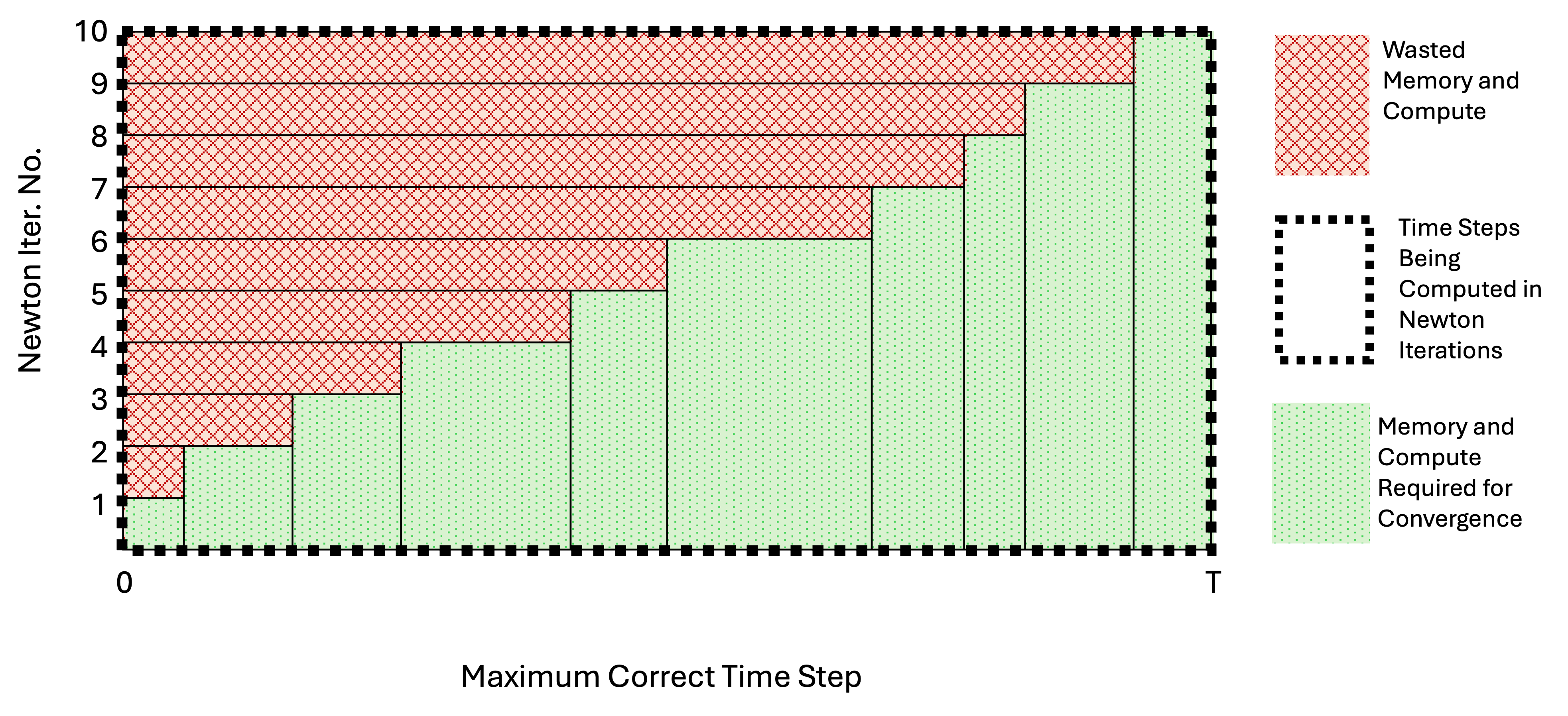}
            \caption{DEER}
            \label{fig:deer_baseline}
         \end{subfigure}
        \begin{subfigure}[b]{0.45\textwidth}
            \centering        
            \includegraphics[width=\textwidth]{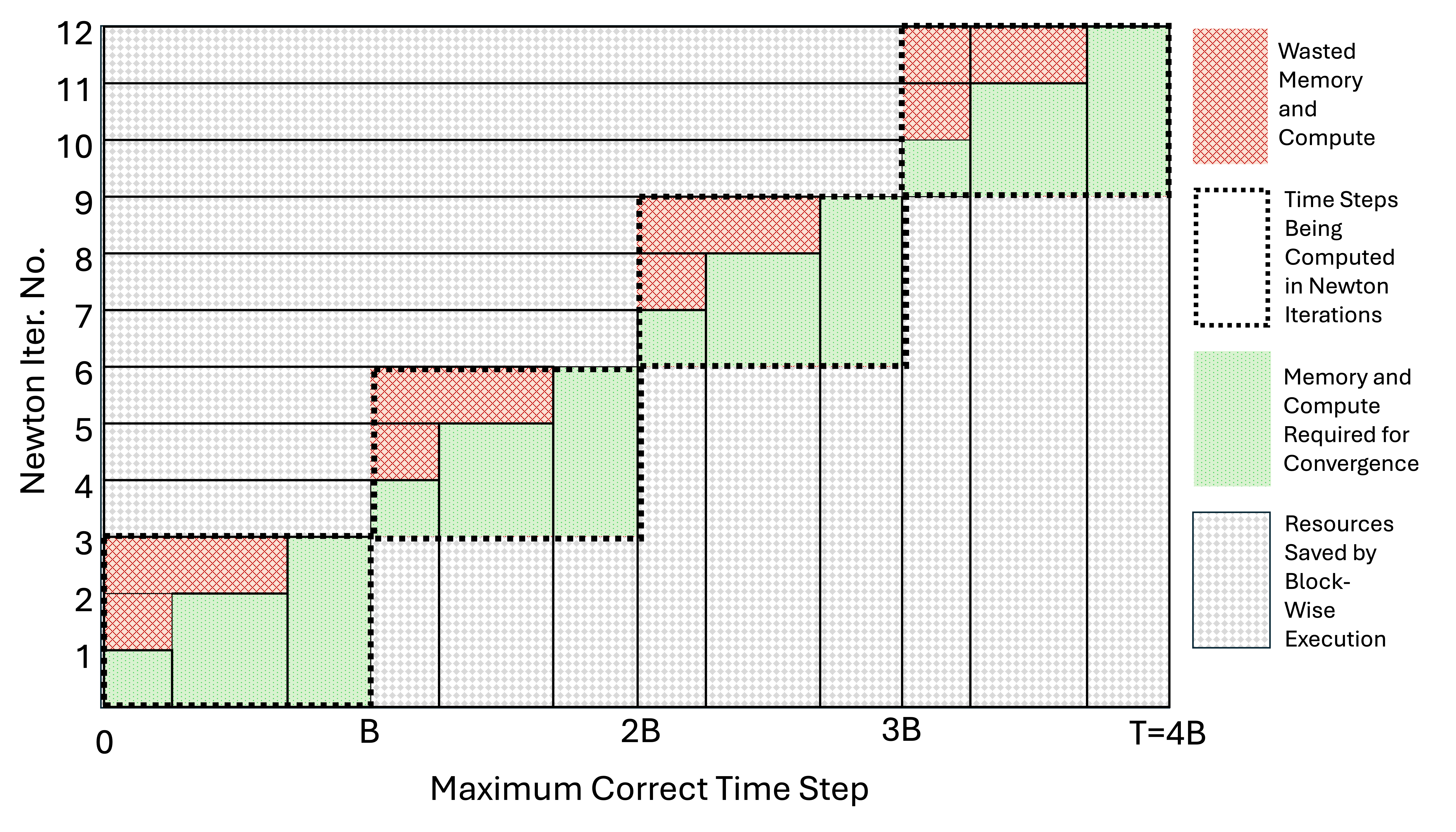}
            \caption{Block-Wise DEER}
            \label{fig:block_savings}
        \end{subfigure}
        
     \hfill
        \caption{\textbf{Saving memory and compute with Block-Wise DEER execution}. Subfigure~\ref{fig:deer_baseline} shows how DEER uses compute and, especially, memory to store intermediary copies of already converged states (the red shaded area). Subfigure~\ref{fig:block_savings} shows how splitting the sequence into 4 equally split sequential blocks, as an example, results in roughly $4\times $ reduction in memory required for storing intermediary Newton iterates.}
        \label{fig:block_wise_deer_comparison}
\end{figure}

Block-wise execution is a prevalent strategy in maximising GPU utilisation in parallel architectures (see Section~\ref{sec:related_parallel}) \citep{dao2022flashattention, fu2024flashfftconv, yang2025gated, liu2023blockwise}. However, the form of Block-Wise DEER investigated here is most similar to the ParaRNN \citep{danieli2025pararnn}. There, block-wise sequential processing for DEER-like algorithms is discussed in more detail with respect to hardware I/O optimisation, rather than DEER idiosyncrasies. 

\subsection{State Expansion}\label{sec:state_expansion}

In this study, unless otherwise specified, inputs $u_t$ in the ADPTNet recurrent step are dense linear projections of the overall input to the layer $o_t\in\mathbb{R}^{h}$ ($u_t = Wo_t, W \in \mathbb{R}^{h \times h}$). However, SSMs such as S4, the LRU, or Mamba \citep{gu2021efficiently, orvieto2023resurrecting, gu2023mamba} typically rely on parameter-efficient projections that expand the size of the recurrent state to $n$ by a factor $n_{\text{state}}$ compared to the hidden dimension $h$ of the network ($n=n_{\text{state}}h$), to allow for higher memory capacity. The projections are implemented with vector-scalar multiplications per input dimension $u_{i:i+n_{\text{state}}-1}[t] = w_i \cdot o_i[t], \space w \in \mathbb{R}^n_{\text{state}}$ \footnote{It should be noted, for notational convenience, that $u_t$ refers to the entire input vector at time step $t$, whereas $u_i[t]$ denotes the $i^{\text{th}}$ dimension of the input at time $t$. }.

\subsection{Input Normalisation} \label{sec:selective_ADPTNet}

A crucial consideration in systems with slowly fading memory, such as SSMs or ADPTNet, is preventing recurrent state magnitude explosion. In fact, \citet{orvieto2023resurrecting} showed how input and state normalisation at initialisation can determine whether SSMs converge beyond random accuracy or not on long-range sequence modelling tasks. One common solution is to adopt leaky integration (see Section~\ref{SNN_formula}) or exponential moving average (EMA) parametrisation \citep{ma2022mega}:

\begin{equation} \label{eq:generic_ema}
    x_t = \bLb x_{t-1} + (1-\bLb) u
\end{equation}

Another normalisation technique, used in \citet{orvieto2023resurrecting} to construct the Linear Recurrent Unit (LRU), is: 

\begin{equation} \label{eq:lru_normalisation}
    x_t = \bLb x_{t-1} + (\sqrt{I_n - \bLb ^2}) u_t
\end{equation}

Finally, one can also tie input normalisation to the step size $\Delta t$, by extending ZOH discretisation to inputs as proposed in \citep{gu2022parameterization}: 

\begin{equation}\label{eq:zoh_normalisation}
    x_t = \bLb x_{t-1} + \frac{(e^{-e^{\Delta t}e^{{\Lambda}}}-1)}{e^\Lambda}u_t
\end{equation}

For simplicity, from now on let $\gamma \in \mathbb{R}^{n}$ denote the element-wise input normalisation term in the ADPTNet recurrent step, either $\gamma_i = 1 - \lambda_i$ for EMA, $\gamma_i = \sqrt{1 - \lambda_i ^2}$ for LRU normalisation, or $\gamma_i = \frac{(e^{-e^{\Delta t_i}e^{{\lambda_i}}}-1)}{e^\lambda_i}$ for ZOH. In addition, following \citet{orvieto2023resurrecting}, $\gamma$ can either be set as a function of $\Lambda$ throughout training, or simply initialised in relation to $\Lambda$ and then allowed to be trained independently.

In the context of ADPTNet, a further consideration is the similarity transform applied to the recurrent eigenspectrum $Q\Lambda Q^T$. As a consequence of the rotations, there is no element-wise correspondence between $x_{t-1}$ and $u_t$ to perfectly match individual $\gamma$ terms. However, several strategies can be investigated for mitigation in the eventuality it is necessary.

Firstly, one could apply a rotation to the input as well to preserve element-wise matching between $\Lambda$ and $1-\Lambda$. This rotation-aware normalisation could take the form of:

\begin{equation} \label{eq:rotation_aware_normalisation}
    x_t = Q_t (\Lambda Q_t^Tx_{t-1} + \gamma \odot  u_t)
\end{equation}

It should be observed that any form of rotation-aware normalisation introduces an additional term in the recurrent Jacobian that has a norm proportional to the input $u_t$ (Equation~\ref{eq:broken_jacobian_ran}). Hence, this breaks the Jacobian norm bound guarantees derived in Lemma~\ref{lemma:J_singular_vals} and thus does not trivially inherit the theoretical properties of ADPTNet. 

\begin{equation} \label{eq:broken_jacobian_ran}
    J = M + \frac{dM}{dx}x + \mathbf{\frac{dQ}{dx} (\gamma \odot u)}
\end{equation}

A second solution would be not to consider the effect of the EMA mismatch. It could be argued that the 2-norm bounds of the recurrence step are invariant to any rotation applied to the input, and thus the recurrence can remain unchanged: 

\begin{equation*}
    x_t = Q_t \Lambda Q_t^Tx_{t-1} + \gamma \odot u_t
\end{equation*}

Finally, it is important to reiterate that \citet{orvieto2023resurrecting} found significant performance degradation due to a lack of proper normalisation when testing on long sequences, with eigenvalues close to the unit circle. Conversely, if $\Lambda$ is initialised for quick decay, the problem should be less pressing. Therefore, one could take inspiration from adLIF neurons (Subsection~\ref{sec:long_range_snn}) \citep{bittar2022surrogate} to separate long-range and short-range timescales ($\Lambda$) into an ADPTNet non-linear module, and a long-term EMA linear memory ($\Gamma$): 

\begin{equation} \label{eq:coupled_ADPTNet}
    \begin{split}
        x_t &= M_tx_{t-1} + \gamma \odot (u_t + \text{gain} \odot w_{t-1}) \\
        w_t &=  \Gamma w_{t-1} + (1-\Gamma)x_t
    \end{split}
\end{equation}

The model resulting from Eq.~\ref{eq:coupled_ADPTNet}, named here \textit{CoupledADPTNet}, reverses the tradition of focusing on linear memory units to complement non-linear RNNs \citep{voelker2019legendre}. Here, the secondary purpose of the CoupledADPTNet is to study the robustness of the proposed non-linear dynamics to coupling with another system as a function of the feedback gain term, rather than the long-term memory capacity of the linear subnetwork, which is essentially a standard SSM.  

\subsection{Input Selectivity} \label{sec:coupled_ADPTNet}

On tasks requiring in-context adaptation or selectivity, to match the formulation of selective SSMs such as Mamba, ADPTNet can be equipped with a data-dependent element-wise input gate ($i$):

\begin{equation} \label{eq:input_gate_ADPTNet}
    x_t = M_tx_{t-1} + i(u_t) \odot \gamma \odot u_t
\end{equation}

Following the convention from Mamba, the input gate $i(u_t)=\sigma(UD^Tu_t)$, where low-rank $U, D\in\mathbb{R}^{n \times r}$, and $\sigma:=\mathtt{softplus}$.

\subsection{ADPTNet Taylor Series Approximation} \label{sec:taylor_ADPTNet}

For low $\beta$, the computational cost of the ADPTNet recurrence (Sec.~\ref{sec:ADPTNet_def}) can be reduced by approximating the similarity transform with its truncated Taylor series expansion. For any arbitrary matrix $X$ and orthogonal matrix $Q=e^{\beta S}$, where $S$ is skew-symmetric, using the Baker-Campbell-Hausdorff formula \citep{campbell1896law}, a similarity transform can be expanded as:

\begin{equation}
    QXQ^T = e^{\beta S}Xe^{-\beta S} = X + \beta [X, S] + \beta^2 * [X, [X, S]] + \dots
\end{equation}

Where $[X, S]$ denotes the commutator $XS - SX$. Accordingly, the ADPTNet recurrence, for sufficiently low $\beta$, can be first-order approximated as: 

\begin{equation} \label{eq:taylor_ADPTNet}
\begin{split}
     x_{t+1} &= M_{\text{Taylor}}x_t + \gamma u_{t+1} \\
     M_{\text{Taylor}} &= \bLb + \beta(\bLb(kv^T - vk^T) - (kv^T - vk^T)\bLb)
\end{split}
\end{equation}

It should be noted that matrix exponentials are no longer required. Furthermore, the diagonal of the commutator term cancels out in the subtraction, and thus $M_{\text{Taylor}}$'s diagonal is exactly the parametrised eigenspectrum, while off-diagonal entries are directly scaled by $\beta$. Hence, for \textit{TaylorADPTNet}, Conv and Forward DEER are identical. 

\subsection{Linear ADPTNet}\label{sec:linear_ADPTNet}

One can notice that the topological conjugation backbone of ADPTNet's recurrence does not inherently have to be non-linear. Removing the state dependence, the derivative term $\frac{dM}{dx}x$ vanishes, and the Lyapunov spectrum of the now LTV system is identical to its eigenspectrum. Therefore, to isolate the expressivity of the similarity transform itself, separate from any non-linear recurrence effects, \textit{LinearADPTNet} can be used: 

\begin{equation}\label{eq:linear_ADPTNet}
    \begin{split}
        x_{t+1} &= M_{\text{Linear}}x_{t} + \gamma u_{t+1} \\
    M_{\text{Linear}} &= Q(u_{t+1})\bLb Q(u_{t+1})^T, \space k = \frac{Ku_{t+1}}{\norm{Ku_{t+1}}}, \space v = \frac{Vu_{t+1}}{\norm{Vu_{t+1}}}
    \end{split}
\end{equation}

Importantly, as an LTV system, LinearADPTNet can be parallelised using a single parallel associative scan \citep{gu2023mamba}. However, this entails materialising all recurrent $n \times n$ $M_{\text{Linear}}$ matrices, at a memory cost of $\mathcal{O}(Tn^2)$, where $T$ is the sequence length. Furthermore, all materialised recurrent weights have to be multiplied with $\mathcal{O}(Tn^3)$ computational cost. Therefore, one could still use Conv or Forward DEER in this context, to propagate off-diagonal entries. The $k$ and $v$ projection can be computed once at the first iteration and cached for the entire simulation, since they do not depend on state trajectory guesses. Then, the memory cost becomes $\mathcal{O}(Tn)$ for all DEER iterations\footnote{This cost assumes gradient checkpointing is applied during training, and intermediary DEER iterates are not stored for the backward pass. The result applies universally during inference.}. In terms of computation, the new scaling $\mathcal{O}(n_{\text{iter}}Tn)$ is advantageous as long as the number of DEER iterations required for convergence $n_{\text{iter}} < n$. To the author's knowledge, Quasi-DEER or any diagonal form of DEER has not previously been used a means for more efficient full-matrix LTV system parallelisation. 

As with non-linear ADPTNet, for small enough $\beta$, LinearADPTNet can also be approximated to first order with a Taylor Series expansion. Interestingly, in this case, since $\frac{dM}{dx}x = 0$, Quasi, Conv, and Forward DEER are all equivalent. 

\subsection{Spiking ADPTNet}\label{sec:spiking_ADPTNet}

To materialise the neuromorphic inspiration from the auditory cortex (Section~\ref{sec:intro}), ADPTNet layers can be equipped with position-wise spiking activations for information propagation across across model depth \citep{stan2024learning}. Using Eq.~\ref{spike_fn}, the output of each layer becomes: 

\begin{equation}
    \text{out}_i\text{[t]} = s(x_i[t]) = \begin{cases}
        0, \space x_i[t] < \theta \\
        1, \space x_i[t] \ge \theta
    \end{cases}
\end{equation}

In Section~\ref{sec:ssc}, where ADPTNet-based SNNs are employed, following \citet{fabre2025structured}, the Boxcar surrogate gradient is used in the backward pass: 

\begin{equation}
\frac{ds}{dx} = \begin{cases}
    0, \space x \notin [-0.5, 0.5) \\
    1, \space x \in (-0.5, 0.5)
\end{cases}
\end{equation}

\subsection{Fine-Grained Control over ADPTNet Non-Linearity} \label{sec:controlled_dot}

The efficient matrix exponential function from Section~\ref{sec:efficient_exp} reveals how the non-linearity, or the level of off-diagonal interaction, is heavily influenced by the choice of $\beta$, but not entirely. More specifically, if one revisits Eq.~\ref{eq:eff_exp} and \ref{eq:core_exp}, it can be observed that the off-diagonal $E$ and $F$-derived entries are scaled by $\sin(\sqrt{\beta ^2 - \beta^2(k\cdot v)^2})$, where the $\sqrt{\beta}$ scaling applied to the normalised $k$ and $v$ is explicitly taken out. As $\beta$ approaches $0$, the $\sin$ also follows, and thus off-diagonal non-linear integrations vanish. However, similarly, if the dot product $k \cdot v$ approaches $1$, i.e., $k$ and $v$ become identical, the non-linear interaction also disappears. Therefore, one could completely control the level of non-linearity in ADPTNet layers by parametrically setting $k\cdot v$ to a value $p$:

\begin{equation}\label{eq:parametric_kv_dot}
\begin{split}
    \hat{k} &=  \frac{(I-vv^T)k}{\norm{(I-vv^T)k}} \\
    k &\leftarrow (1-p^2)\hat{k} + pv
\end{split}
\end{equation}

\section{Results}\label{sec:results}


        

\subsection{Lyapunov Spectrum} \label{sec:lyapunov_results}

Lemmas~\ref{lemma:le_spectrum} and ~\ref{lemma:le_exact_match} provide claims regarding the relationship between the parametrised "strength" of the systems' non-linearity as expressed through $\beta$, the recurrent eigenvalue spectrum $\Lambda$ and the effective timescales of the non-linear ADPTNet dynamics measured through its Lyapunov spectrum. This section provides empirical evidence for these specific claims while also testing their robustness to less strict parametrisation assumptions.

\subsubsection{Untrained Orthogonal RNN Baseline}

As briefly discussed in Section~\ref{sec:orthogonal}, one can parametrise the recurrent weights of RNNs with orthogonal matrices to ensure a degree of stability and long-term memory \citep{arjovsky2016unitary}. Before showcasing empirical evidence supporting the theoretical properties of ADPTNet from Section~\ref{sec:ADPTNet_def}, it is helpful to examine Orthogonal RNNs as a baseline.

Consider a vanilla recurrence rule for Orthogonal RNNs:

\begin{equation} \label{eq:orthogonal_rnn_expriment}
    x_{t} = \mathtt{ReLU}(Qx_{t-1} + Bu_t)
\end{equation}

Where $Q \in O(n)$, $B \in \mathbb{R}^{n \times n}$, $x, u \in \mathbb{R}^{n}$. As long as there are no silent time steps, where the ReLU activation is $\equiv \mathbf{0}$, Orthogonal RNNs of this form are guaranteed to have at least the LLE $\alpha_{max} = 0$. Crucially, there are no theoretical bounds for the rest of the Lyapunov spectrum in terms of how low they can be. Figure~\ref{fig:orthogonal_rnn_lyapunov} visualises how this limitation materialises in practice. The Figure shows the entire Lyapunov spectrum\footnote{Computed using a numerically stable algorithm as described in Lemma~\ref{lemma:le_exact_match}} of a single-layer untrained Orthogonal RNN with 128 hidden units. The network is driven by random inputs $\sim \mathcal{N}(0, 1)$ of lengths $\in \{64, 4096, 16384\}$. It should be mentioned that the longer the input sequence is, the better the long-term dynamics of the system can be observed and Lyapunov spectrum estimated \citep{engelken2023lyapunov}.

Here, the most relevant takeaway is that while the entire singular value spectrum of $\sigma(Q)=1$, the Lyapunov spectrum only partially aligns with $\ln(\sigma(Q)) =0$. In fact, as can be seen in Figure~\ref{fig:orthogonal_rnn_lyapunov}, only approximately half of the spectrum is close to $0$, tied to the sparsity of ReLU activations at initialisation. Moreover, while the network can be regularised during training, there are no parametrised fine-grained controls over the sparsity of the activations. Therefore, as is also observed in Figure~\ref{fig:orthogonal_rnn_lyapunov}, short-term time scales are effectively vanishing, with no controls over their decay rate ($e^{-40}$). Even if a smoother activation function (e.g., sigmoid, GELU) were considered, one would still not have any more control over the lower bounds of the spectrum, and would also lose the LLE properties of ReLU. 

\begin{figure}[H] 
    \centering
        \centering        
        \includegraphics[width=\textwidth]{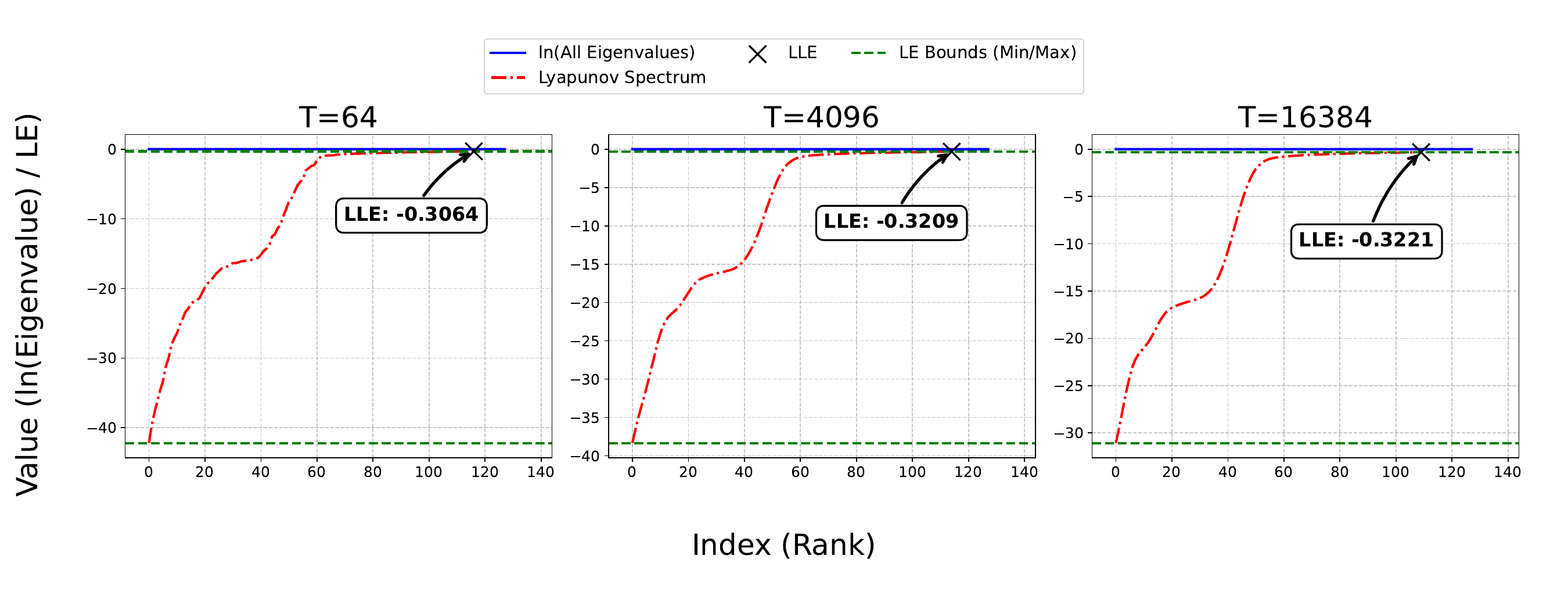}
        \hfill
        \caption{\textbf{Lyapunov Spectrum of Orthogonal RNN at Initialisation} Each subfigure shows the estimated Lyapunov spectrum of a single-layer untrained Orthogonal RNN (Eq.~\ref{eq:orthogonal_rnn_expriment}) at initialisation with hidden size $=128$ for a given random input of length $\in \{64, 4096, 16384\}$. The blue line shows the logarithm of the eigenspectrum of the orthogonal recurrent matrix, which $= \mathbf{1}$, and is only partially aligned with the Lyapunov spectrum of the system. }
        \label{fig:orthogonal_rnn_lyapunov}
\end{figure}

\subsubsection{ADPTNet at Initialisation} \label{sec:untrained_lyapunov_results}

This section examines the robustness of the claims regarding the parametric control over ADPTNet's Lyapunov spectrum at initialisation. Here, untrained networks are driven by inputs that are randomly drawn $\sim \mathcal{N}(0, 1)$. Unless otherwise specified, all networks tested used the configuration in Table~\ref{table:baseline_untrained_config}. Algorithm~\ref{algo:stable_lyapunov_comp} is used to compute the Lyapunov Spectra, extracting Jacobian singular values at each time step. 

\begin{table}[h!] 
    \centering
    \renewcommand{\arraystretch}{1.5}
    \begin{tabular}{|c|c|c|c|c|c|c|c|c|c|c|}
        \hline
        Rot.-Aware & Neg. Eigvals. & $\gamma$ & RoPE & D & Exp. & KV &  $(\lambda_{\min},\lambda_{\max})$ & $(\Delta t_{\min}, \Delta t_{\max})$\\
        \hline

        \ding{55} & \ding{55} & $\sqrt{1 -\bar\lambda^2}$ & \checkmark & 16 & Low Rank & Ortho. & $(2, 100)$ & $(10^{-4}, 10^{-4})$\\
        \hline
    \end{tabular}
    \caption{\textbf{Baseline Untrained ADPTNet Configuration} \textit{Rotation-Awareness} refers to the correction for the element-wise mismatch from the similarity transform (Eq.~\ref{eq:rotation_aware_normalisation}). \textit{Negative Eigenvalues} can be included as mentioned in Section~\ref{sec:eigvalue_param}. The \textit{RoPE} column indicates whether the custom RoPE embeddings with S4D-Inv parametrisation are included. Otherwise, no positional embeddings are used. \textit{D} refers to the hidden state dimension. \textit{Matrix Exponential (Exp.)} refers to whether the low-rank structure-aware formula (Section~\ref{sec:efficient_exp}) or the Cayley Map is used (Section~\ref{sec:orthogonal}). $\gamma $ refers to the input normalisation scheme used, EMA or LRU (Section~\ref{sec:selective_ADPTNet}). \textit{KV} refers to the matrix parametrisation used for the $K$ and $V$ linear projections: baseline unconstrained (Kaiming initialisation), orthogonal, or LowParam (Section~\ref{sec:low_param_linear}). Using the ZOH discretisation scheme (Section~\ref{sec:eigvalue_param}) eigenvalues $\bar{\lambda_i} = e^{-\lambda_i\Delta t_i}$, where $\lambda_i$ is $\sim \mathcal{U}(\lambda_{\min}, \lambda_{\max})$ and $\Delta t_i$ is obtained with $\mathtt{linspace}(\Delta t_{\min},\Delta t_{\max}, D)$. }
\label{table:baseline_untrained_config}
\end{table}

\textbf{Sequence Length and Hidden State Size} Figure~\ref{fig:time_ablation} shows the effect of the input sequence length on the measured Lyapunov spectrum. Firstly, and most crucially, it should be noted that all Lyapunov exponents are well contained within the spread of the logarithms of the recurrent eigenspectrum, implicitly meeting the bounds set out in Lemma ~\ref{lemma:le_spectrum}. Secondly, the hidden state size $D$ of an ADPTNet layer does not affect the alignment between the observed Lyapunov spectrum and the parametrised eigenvalues. Finally, as expected, the input sequence length distorts the Lyapunov exponents measured from the network, with short sequences effectively "squeezing" the spectrum. Intuitively, there is not sufficient long-term information to extract the most slowly evolving timescales, as by definition Lyapunov exponents become more accurate with $T\to \infty$ (Section~\ref{sec:lyapunov_exponents}). Furthermore, in the longest sequence length configuration, $\alpha_i \approx \ln (\bar\lambda_i)$, supporting the claims in Lemma~\ref{lemma:le_exact_match} as well. 

\textbf{Robustness of $K$ and $V$ Linear Projection Parametrisation to Choice of $\beta$} In Lemma~\ref{lemma:J_singular_vals}, it is assumed that $K$ and $V$ are constrained to orthogonal matrices which results in $\left\| K\right\|_2 = \left\| V\right\|_2 = 1$. This is helpful because it allows the norms of several derivative terms to cancel and produce the final $8\beta \bar\lambda_{max}$ factor, which is independent of network width, state $x$ norm, etc. However, in practice, constraining dense and full-rank linear layer weights to the orthogonal manifold is computationally expensive and slow, while also effectively halving the degrees of freedom ($\frac{D(D+1)}{2}$) compared to unconstrained $GL(D)$ matrices. Assuming an arbitrary invertible parametrisation , the norms of the $\frac{dk}{dx}$ and $\frac{dv}{dx}$ terms inside $\norm{\frac{dM}{dx}x}$ in Eq.~\ref{eq:lemma1_result} become: 

\begin{equation*}
\norm{Kx} \ge \sigma_{\min}(K)\norm{x} \Rightarrow \frac{1}{\norm{Kx}} \le \frac{1}{\sigma_{\min}(K)\norm{x}} \Rightarrow \frac{\norm{K}\norm{x}}{\norm{Kx}} \le \frac{\sigma_{\max}(K)\norm{x}}{\sigma_{\min}(K)\norm{x}} = \kappa(K) \Rightarrow
\end{equation*}

\begin{equation} \label{eq:KV_update_lemm1}
\begin{split}
       \norm{\frac{dM}{dx}x} &\le \beta \lambda_{\max}\norm{\frac{dk}{dx} + \frac{dv}{dx}}\norm{x} \\
       &\le  4\beta\lambda_{\max}\left(\frac{\norm{K}\norm{x}}{\norm{Kx}} + \frac{\norm{V}\norm{x}}{\norm{Vx}}\right) \le 4\beta\lambda_{\max}(\kappa(K) + \kappa(V))\\
\end{split}
\end{equation}

 Where $\kappa(K)=\frac{\sigma_{\max}(K)}{\sigma_{\min}(K)}$ denotes the condition number of $K$. In other words, if $K$ or $V$ are ill-conditioned at initialisation or become ill-conditioned during training, the bounds on the Lyapunov spectrum of the network also become looser. As with other random matrices, Kaiming-initialised dense linear weights have condition numbers that, on average, scale linearly with network width, but may have large outliers \cite{edelman1988eigenvalues}. Furthermore, if the weights are low-rank or non-invertible, the bounds are undetermined $< \infty$. Using the initialisation scheme from Section~\ref{sec:low_param_linear}, a LowParam linear layer would in theory have an undefined condition number\footnote{The low-rank component would be all-\textbf{0} at initialisation. The remaining Kronecker component is a product between a dense random matrix and an all-\textbf{1} matrix, which results in an undefined division by zero when computing the condition number.}. Either unpredictable or undefined, one cannot directly predict the conditioning of the two alternative orthogonal parametrisations considered here. 

Regardless, Figure~\ref{fig:kv_ablation} shows that ADPTNet dynamics are generally robust to this choice. More specifically, for a lower $\beta=0.125$, all parametrisation schemes display Lyapunov spectra $\ln(\bar\lambda_{\min}) \le \alpha_i \le \ln(\bar\lambda_{\max})$. With $\beta=1$, all parametrisations are still generally close to the original $\ln(\bar\Lambda)$. A more significant gap between orthogonal and alternatives becomes apparent with a large, and perhaps unrealistic, $\beta=1000$. In terms of the most quickly decaying timescales, while orthogonal constraints show a $\approx -0.002$ "undershoot", LowParam and dense have an order of magnitude higher discrepancies of $< -0.04$ and $<-0.01$, respectively. Moreover, both alternatives technically display chaotic dynamics since their LLEs are $>0$, while the orthogonal-constrained configuration is still stable. It is important to note, however, that all configurations are still well within the bounds prescribed by Lemma~\ref{lemma:le_spectrum}. In the $\beta=1$ case, the orthogonal parametrisation for which they are guaranteed, the theoretical bounds would be roughly $- \infty < \alpha_i < \ln(\bar\lambda_{\max} + 8*\bar\lambda_{\max})$, which are evidently far outside the observed behaviour and thus respected, and likewise for the more extreme case $\beta =1000$. One could argue that this provides evidence that mitigation methods such as spectral normalisation are not necessary for ADPTNet, especially in the more realistic range of $\beta \ll 1$. 

\begin{figure}[H] 
    \centering
        \centering        
        \includegraphics[width=\textwidth]{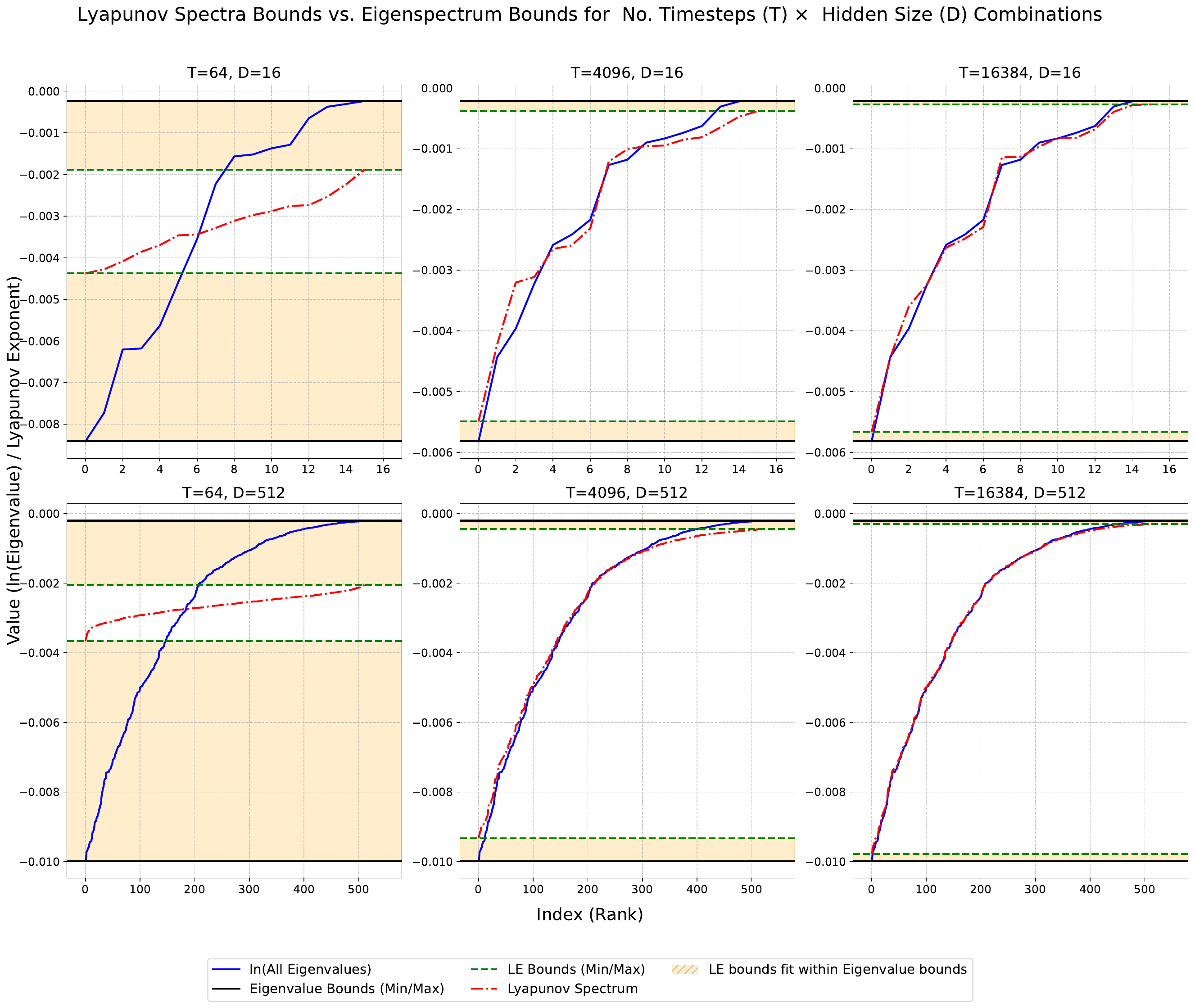}
        \hfill
        \caption{\textbf{Effect of Time Steps and Hidden Size on Lyapunov Spectrum} Each subfigure in the grid represents the recurrent eigenvalues $\Lambda$ and Lyapunov spectrum of untrained ADPTNet layers driven by random input of different sequence lengths $T \in \{ 64, 4096, 16384\}$. Each subfigure is obtained using a different randomly initialised network. Both D configurations have $\beta =0.125$. It should be observed how all Lyapunov exponents $\alpha_i \approx \ln (\lambda_i)$ over long sequences, and are universally contained within the bounds of the recurrent eigenspectrum as shown by the yellow-shaded areas.}
        \label{fig:time_ablation}
\end{figure}

\begin{figure}[H] 
    \centering
        \centering        
        \includegraphics[width=\textwidth]{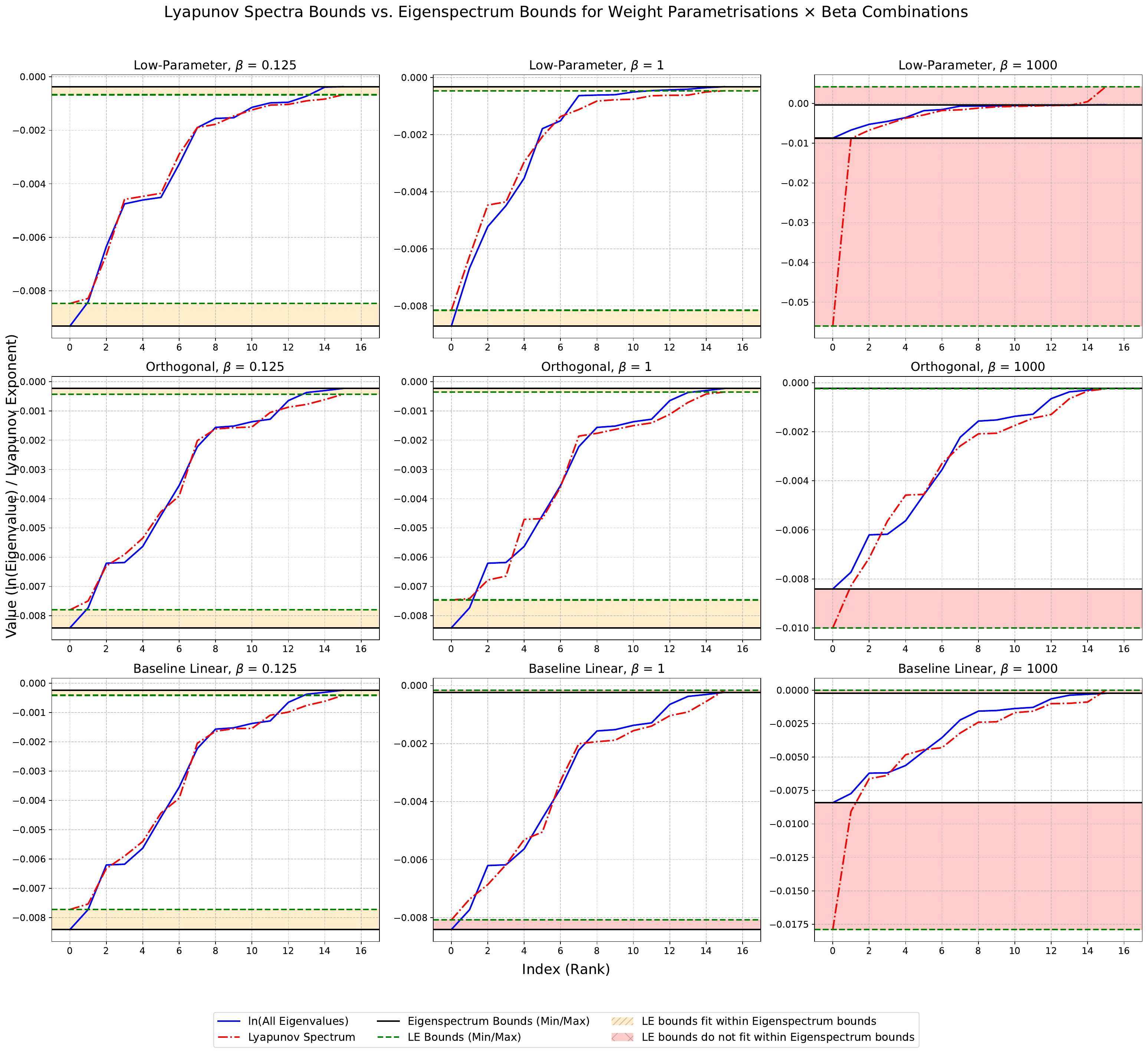}
        \hfill
        \caption{\textbf{Interaction of KV Weight Parametrisation and Choice of $\beta$} Each subfigure shows the Lyapunov spectrum and $\ln(\bar \Lambda)$ for random ADPTNet networks at initialisation with different parametrisation schemes. The LowParam setting uses $\text{rank}=2$ and $N_a = N_b=4$. Inputs are driven by $T=4096$ random inputs. Notably, only for very large $\beta=1000$ do the dynamics of orthogonal constraints start to diverge significantly from the baseline dense linear with Kaiming initialisation and LowParam.}
        \label{fig:kv_ablation}
\end{figure}

\textbf{Effect of Matrix Exponential Function} The proofs for the Lemmas~\ref{lemma:le_spectrum} and \ref{lemma:le_exact_match} rely on the assumption that the matrix exponential required for mapping skew-symmetric matrices to the orthogonal manifold is computed using the Cayley Map. However, in practice, the Cayley map is prohibitively expensive in terms of compute and memory and, thus, an alternative that takes into account the low-rank structure of the skew-symmetric mapping for the $k \otimes v$ outer product is preferable (see Section~\ref{sec:efficient_exp}). However, since the Cayley map is only a truncated approximation of the matrix exponential, one could argue that switching exponential functions may affect the theoretical claims in Section~\ref{sec:ADPTNet_def}. 

Figure~\ref{fig:cayley_ablation} provides evidence to suggest that the ADPTNet recurrence is resilient to this switch. It can be observed that the low-rank-aware exponential function produces very close, if not nearly identical, Lyapunov spectra relative to the parametrised $\ln(\bar\Lambda)$. Even in the large $\beta=1000$ setting, the two functions produce "undershoots" of similar proportions $\approx -0.002$. 

\begin{figure}[H] 
    \centering
        \centering        
        \includegraphics[width=\textwidth]{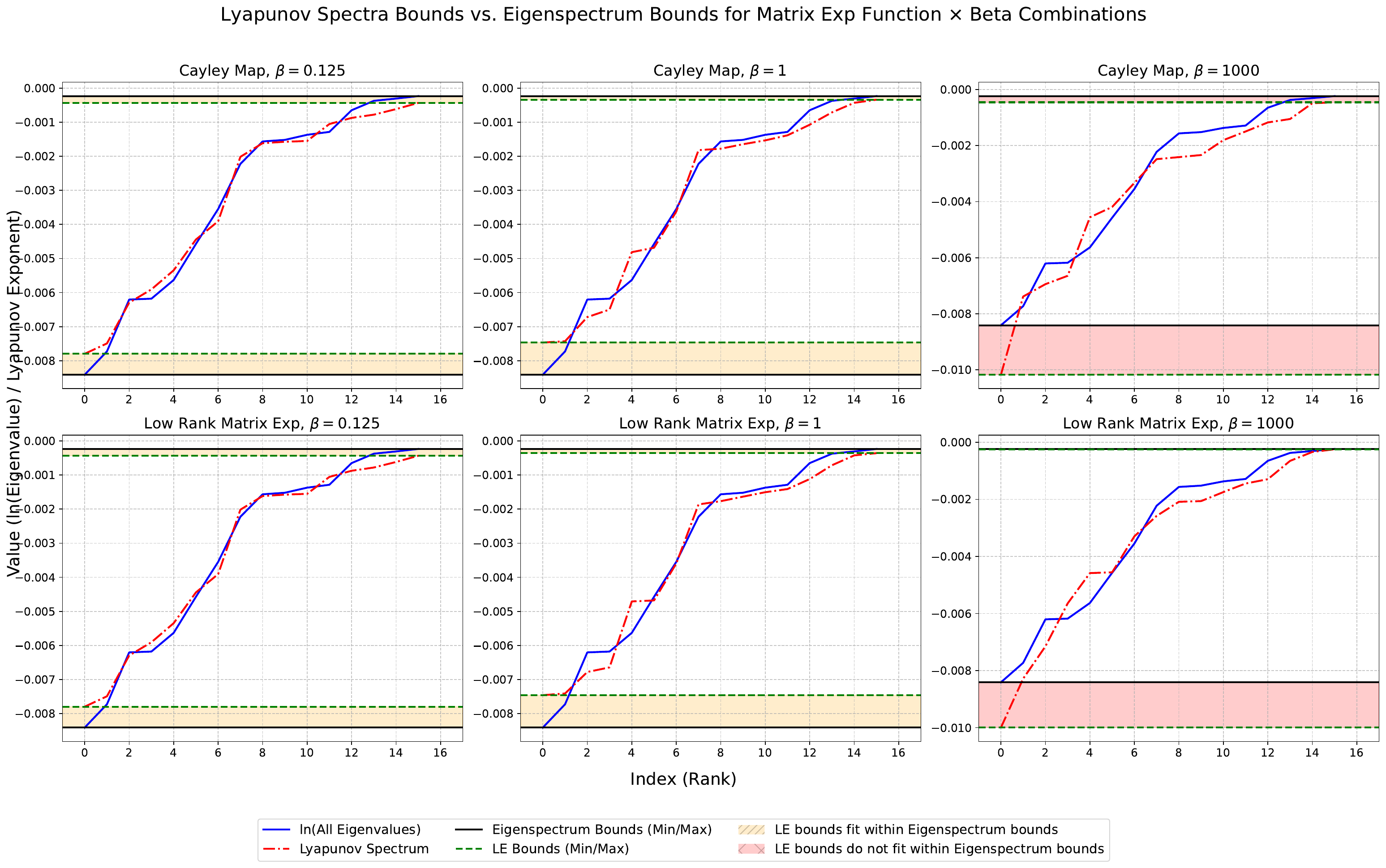}
        \hfill
        \caption{\textbf{Effect of Matrix Exponential Function Choice} The two rows of the grid show the two different matrix exponential functions tested: the Cayley map and the low-rank structure-aware exponential, as described in Section~\ref{sec:efficient_exp}. All randomly initialised networks in the grid are driven by $T=4096$-long random inputs. It is apparent that for any scale of $\beta$, both exponentiation methods produce very similar behaviour in terms of Lyapunov spectra. }
        \label{fig:cayley_ablation}
\end{figure}

\textbf{Effect of Recurrent Condition Number} Since the ADPTNet recurrent matrices $M_t$ are symmetric positive definite by construction, their singular values are $\equiv \bar\Lambda$. Therefore, their 2-norm condition number is $=\frac{\bar\lambda_{\max}}{\bar\lambda_{\min}}$. Figure~\ref{fig:cond_ablation} shows how this ratio does not influence the eigenspectrum's relationship to the observed Lyapunov exponents. While the $\approx 1.01$ setting shows slightly higher discordance between the parametrised and measured spectra, compared to the other condition numbers, alignment appears to be driven mostly by $\beta$. As in Figure~\ref{fig:kv_ablation} and \ref{fig:cayley_ablation}, a large $\beta=1000$ causes more "undershoot" compared to the $\beta=1$. This evidence suggests that regardless of the spread of parametric timescales in the network (e.g., if the spread increases during training), that should not affect the resulting Lyapunov spectrum.

\begin{figure}[H] 
    \centering
        \centering        
        \includegraphics[width=\textwidth]{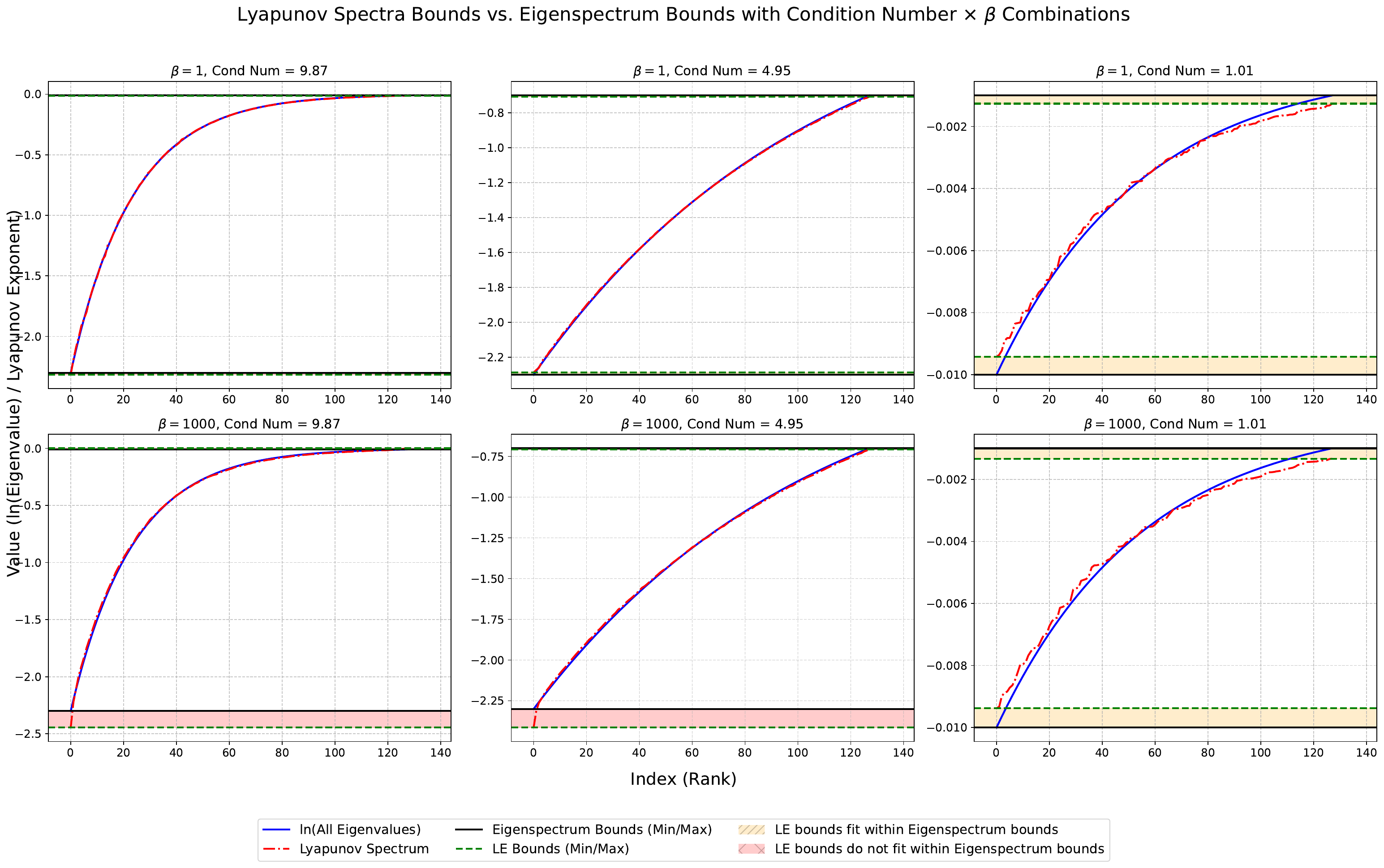}
        \hfill
        \caption{\textbf{Effect of Recurrent Weights Condition Number} Each subfigure shows the effect of different ratios $\frac{\bar\lambda_{\max}}{\bar\lambda_{\min}}$ on the aligment between the eigenvalue and Lyapunov spectra. Each random network with $D=128$ takes in $T=4096$ long random inputs. There are three condition number configurations, obtained by setting $\lambda_{\min}=\lambda_{\max}=1$ and $(\Delta t_{\min}, \Delta t_{\max}) \in \{(10^{-2}, 2.3), (0.7, 2.3), (10^{-3}, 10^{-2})\}$. It should be noted that alignment is driven more by the choice of $\beta$ than by the condition number.}
        \label{fig:cond_ablation}
\end{figure}

\subsubsection{Trained ADPTNet on Copy Memory Task} \label{sec:trained_lyapunov_results}

\begin{table}[h!] 
    \centering
    \renewcommand{\arraystretch}{1.5}
    \begin{tabular}{|c|c|c|c|c|c|c|c|c|c|c|}
        \hline
         Neg. Eigvals. & $\gamma$ & RoPE & D & Exp. & KV &  $(\lambda_{\min},\lambda_{\max})$ & $(\Delta t_{\min}, \Delta t_{\max})$ & $\beta$\\
        \hline

        \ding{55}  & $1 -|\bar\lambda|$ & \checkmark & 128 & Low Rank & Dense & $0.5$ & $(10^{-3}, 10^{-1})$ & 0.0125\\
        \hline
    \end{tabular}
    \caption{\textbf{Baseline Trained ADPTNet Configuration} Column headers follow the same conventions as in Table~\ref{table:baseline_untrained_config}. Notably, $\gamma$ is parametrised as a function of $\bar\lambda$ throughout training, not just initialised to the value. While \textit{Rotational-awareness} is omitted from the table, it is assumed to be absent from all models tested. $KV$ Dense refers to unconstrained dense linear layers with Kaiming initialisation. }
\label{table:baseline_trained_config}
\end{table}

\textbf{Copy Memory Task} This section explores how robust the alignment between parametrised recurrent eigenvalues and observed Lyapunov exponents is throughout training epochs. Table~\ref{table:baseline_trained_config} shows the baseline configuration used for all models tested, unless otherwise specified. All networks are trained on the Copy Memory Task \citep{hochreiter1997long}. The task consists of a sequence of target tokens, followed by a series of distractors, and finally output prompting tokens to elicit recollection of the target tokens (Figure~\ref{fig:copy_memory}). Following \citet{romero2021ckconv}, all models are trained to store 10 tokens in memory from a vocabulary of integers $\in \{1, 2, \dots, 8\}$. The task difficulty is modulated by the number of distractors (e.g., 100, 1000), since the loss computed from the final 10 time steps of the recall must be backpropagated through time over increasing numbers of time steps. This section emphasises the behaviour of ADPTNet Lyapunov spectra over training iterations, rather than raw accuracy. Typically, delayed copying is not intended as a challenge in itself, but rather, it is a pass/fail test of whether information can travel across time in sequence models. Here, passing is quantified as achieving $>90\%$ accuracy within at most 100 training epochs with a learning rate of $\eta=0.004$, cosine scheduling, and Adam optimiser \citep{kingma2014adam}. All reported models pass this standard. 

\begin{figure}[H] 
    \centering
        \centering        
        \includegraphics[width=0.75\textwidth]{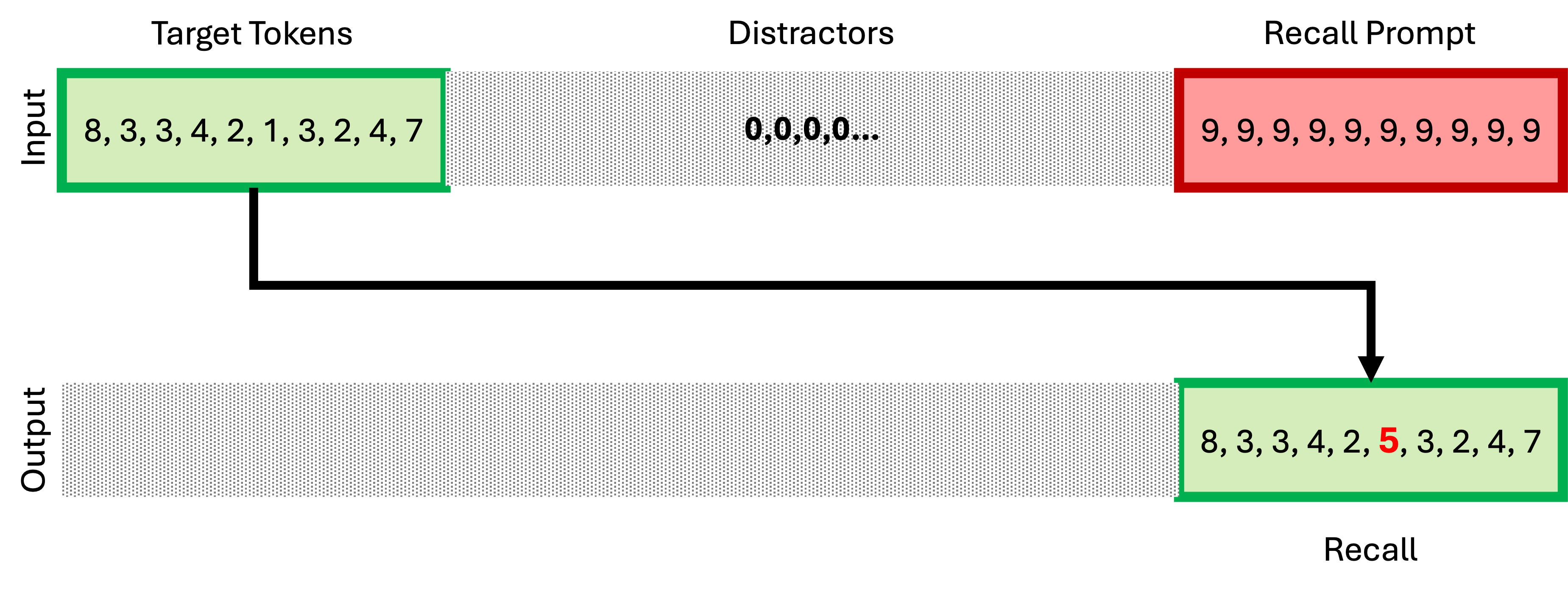}
        \hfill
        \caption{\textbf{Copy Memory Task} All models in this section are trained on the Copy Memory Task, which measures long-term recall ability. Target tokens (integers $\in \{1, 2,\dots, 8\}$) need to be retrieved in the correct order. The loss is then computed on the final recall window and backpropagated through the distraction steps.}
        \label{fig:copy_memory}
\end{figure}

\textbf{Effect of $\beta$ on Trained ADPTNet} Figure~\ref{fig:beta_trained_ablation} shows how the choice of $\beta$ affects $\ln(\bar\Lambda)$ and Lyapunov exponents over training iterations. Firstly, it can be observed that a lower $\beta$ induces a sharper slope and increased spread in the $\bar\Lambda$ distribution. Concretely, at epoch 5, $\ln(\bar\lambda_{\min})$ is an order of magnitude lower for $\beta=0.0125$ compared to $\beta \in \{0.125, 1\}$, going from $-2.5$  compared to $-0.35$ and $-0.14$ respectively. The trend continues up to the final epoch 50, with higher $\beta$ progressively increasing $\ln(\bar\lambda_{\min})$ orders of magnitude from $\approx-7$ to $\approx-2.5$ and $\approx -0.35$, for $\beta=0.0125$, $0.125$, and $1$, respectively. Intuitively, a lower $\beta$ implies a lower rotation angle in the similarity transform $Q\bar\Lambda Q^T$, where the plane of rotation is dictated by $k \otimes v$.  Therefore, to produce more non-linear behaviour, the network may learn to spread the eigenvalues further apart, achieving higher output variance with a limited rotation angle $\beta$.

Regarding the Lyapunov spectrum, once again, increasing $\beta$ decreases alignment with the $\ln(\bar\lambda)$ distribution. Both $\beta=0.0125$ and $\beta = 0.125$ have Lyapunov exponents within the log-spread $\ln(\bar\Lambda)$, evidently meeting the bounds in Lemma~\ref{lemma:le_spectrum}. For $\beta=1$, the LLE is positive, so the dynamics are slightly chaotic. However, the LLE is still $<0.5$, and the theoretical bound $\ln(\bar\lambda_{\max} + 8\beta\bar\lambda_{\max}) \approx \ln(9) = 2.19$, which is evidently larger than the observed value. Therefore, the theoretical bounds hold up over training and are robust to the choice of $\beta$. 

\begin{table}[h!] 
    \centering
    \renewcommand{\arraystretch}{1.5}
    \begin{tabular}{|c|c|c|c|}
        \hline
         \diagbox{Epoch}{KV-Param} & Dense & Orthogonal & LowParam \\
        \hline
         5 & \gape{\makecell{LLE $=0.00344$  \\ Bound $\approx 0.0952$}} &                       \gape{\makecell{LLE $=0.00048$  \\ Bound $\approx 0.095 $}} &                       \gape{\makecell{LLE $=0.00779$  \\ Bound $\approx 0.095$}} \\
         \hline
         25 &\gape{\makecell{LLE $=0.00835$  \\ Bound $\approx 0.0953$}} &
             \gape{\makecell{LLE $=0.01043$  \\ Bound $\approx 0.0953$}} &
             \gape{\makecell{LLE $=0.00723$  \\ Bound $\approx 0.0953$}} \\
         \hline
         50 &\gape{\makecell{LLE $=0.0108$  \\ Bound $\approx 0.0953$}} &
             \gape{\makecell{LLE $=0.0098$  \\ Bound $\approx 0.0953$}} &
             \gape{\makecell{LLE $=0.01026$  \\ Bound $\approx 0.0953$}} \\
        \hline
    \end{tabular}
    \caption{\textbf{LLEs and their Maximum Theoretical Bounds} Each cell in the table contains the LLE associated with each KV-parametrisation scheme along with the theoretical bound on it obtained with Lemma~\ref{lemma:le_spectrum}. While the bounds $\ln(\bar\lambda_{\max}+8\beta\bar\lambda_{\max})$ are technically specific to orthogonal weights, it can be observed that all parametrisations are well within them.}
\label{table:LLE_and_bounds}
\end{table}

\textbf{Effect of KV Parametrisation on ADPTNet Trained Dynamics} As shown for ADPTNet initialisation (Fig.~\ref{fig:kv_ablation}), Figure~\ref{fig:kv_trained_ablation} shows the same effect persisting over training epochs. The Lyapunov spectra closely track the recurrent eigenspectra, with the lower limits of all Lyapunov exponents remaining above their parametrised counterparts across all configurations and epochs. It can, however, also be observed that the LLEs do "overshoot" $\ln(\bar\lambda_{\max})$ by $\ll0.05$, especially towards the end of training. The mismatch remains within the bounds prescribed by Lemma~\ref{lemma:le_spectrum} as detailed in Table~\ref{table:LLE_and_bounds}.

\begin{figure}[H] 
    \centering
        \centering        
        \includegraphics[width=\textwidth]{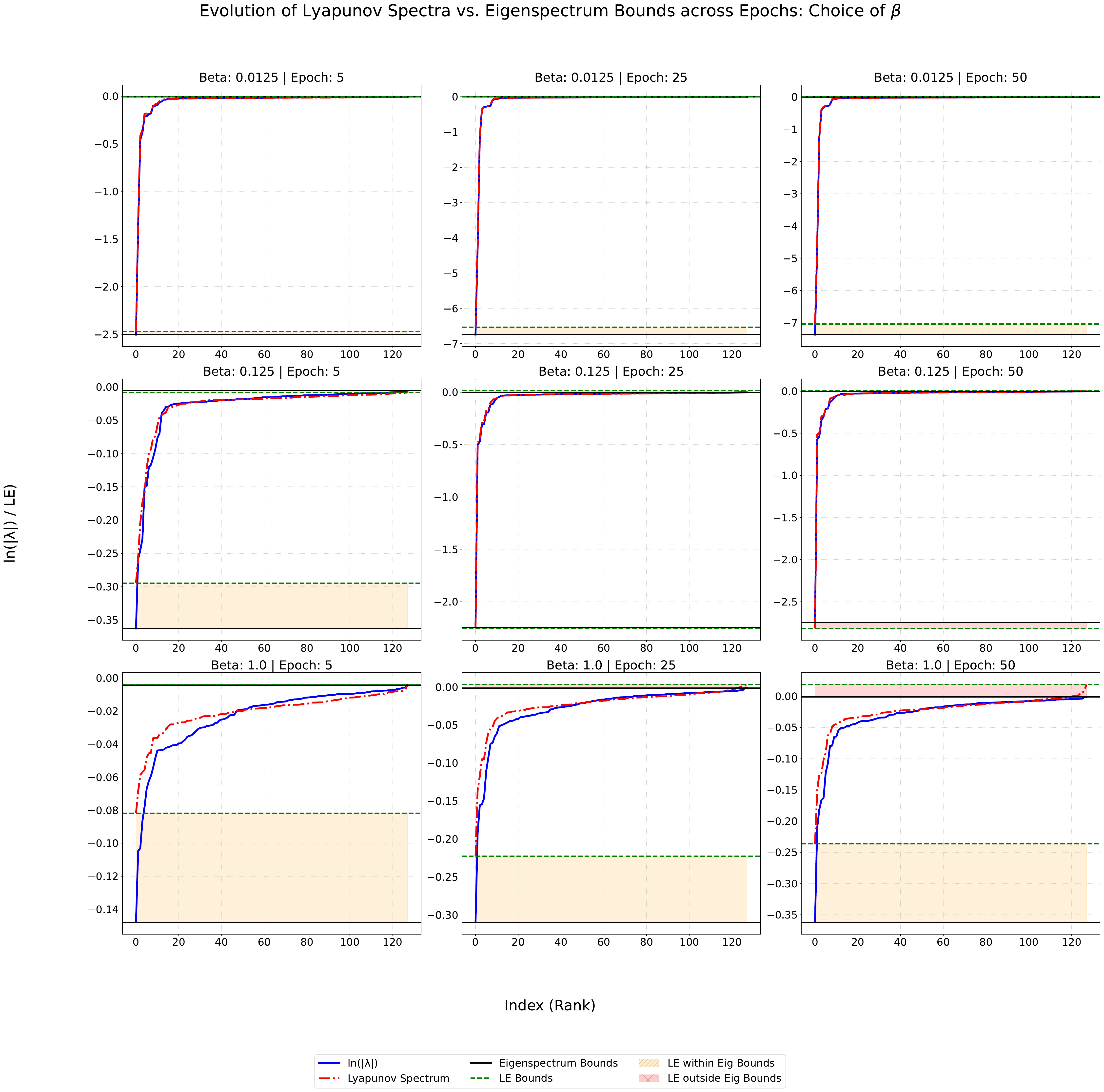}
        \hfill
        \caption{\textbf{Effect of $\beta$ over ADPTNet Trained Dynamics} Each row represents a trained ADPTNet model with a given $\beta$, and the columns are snapshots from different epochs (epoch 50 is the final one). The networks are trained on the Copy Memory task with 44 distraction tokens for a total sequence length of $T=64$. The initial learning rate is $\eta=0.005$. It should be observed how the bounds from Lemma~\ref{lemma:le_spectrum} hold up over training and how lower $\beta$ induces a higher eigenvalue spread. }
        \label{fig:beta_trained_ablation}
\end{figure}

\begin{figure}[H] 
    \centering
        \centering        
        \includegraphics[width=\textwidth]{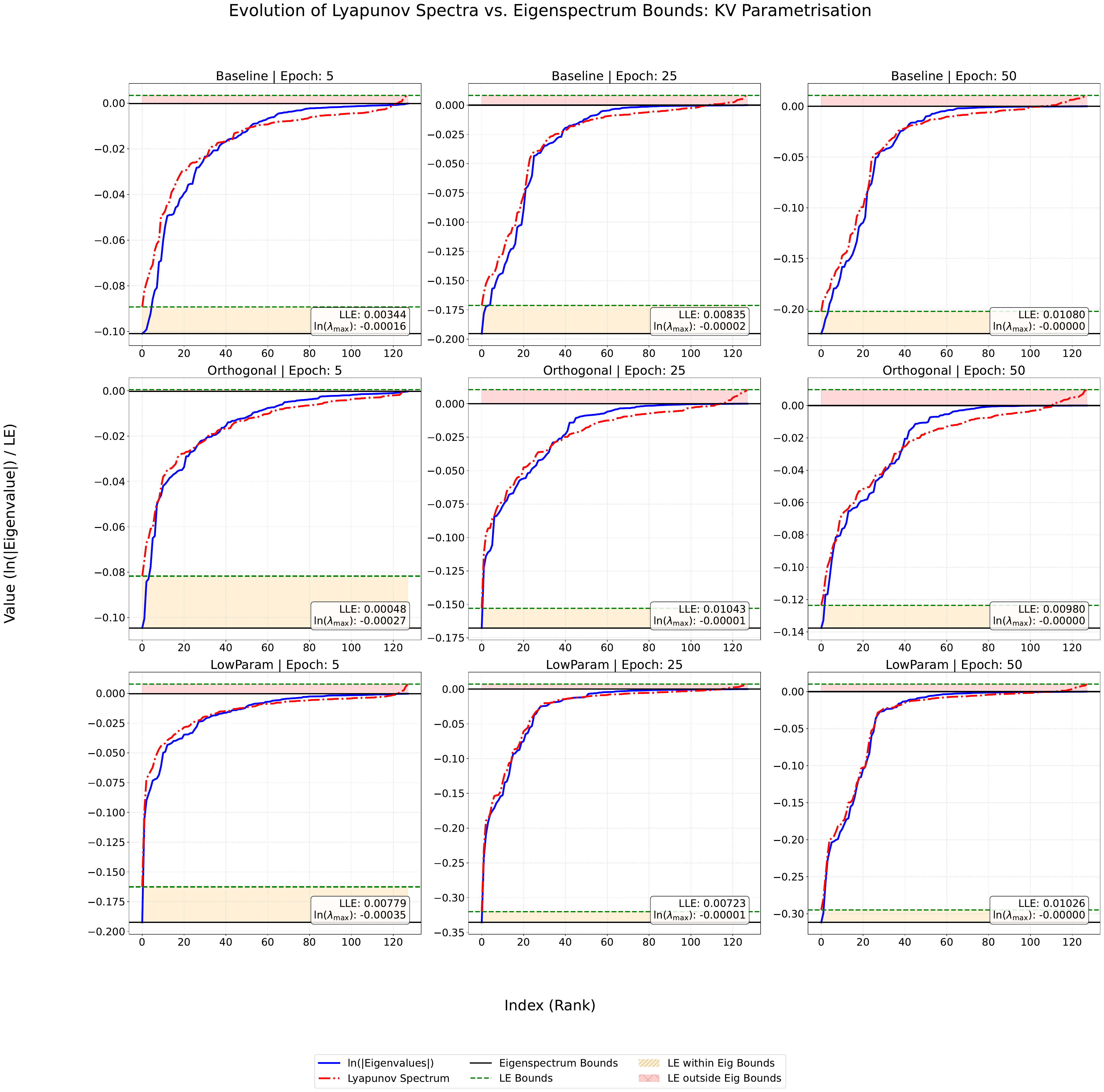}
        \hfill
        \caption{\textbf{Effect of KV-parametrisation over Lyapunov Spectrum During Training} Each subfigure in the grid shows the Lyapunov spectrum and the $ln(\bar \Lambda)$ for each KV parametrisation option. The networks are trained on the Copy Memory task with 108 distraction tokens and a total of $T=128$ time steps. Baseline denotes dense and unconstrained linear layers, and for the LowParam configuration, $\text{rank}=32$, $N_a=32$ and $N_b=4$. As also reinforced in Table~\ref{table:LLE_and_bounds}, all configurations have Lyapunov spectra within the bounds from Lemma~\ref{lemma:le_spectrum}.}
        \label{fig:kv_trained_ablation}
\end{figure}

\textbf{Effect of Including Negative Eigenvalues} As mentioned in Section~\ref{sec:eigvalue_param}, for the purposes of this study, ADPTNet recurrent eigenvalues are constrained parametrically to $ |\bar \lambda_i |\le 1$ by double exponentials, following \citet{gu2022parameterization} (Eq.~\ref{eq:state_dependent_Q_redo}). Because the $M = Q\bar\Lambda Q^T$ structure creates symmetric matrices, even with negative $\bar\lambda_i$, singular values are still $\sigma_i(M) = |\lambda_i|$. Consequently, the singular values $\tilde \sigma_i$ of the recurrent Jacobian $J_t$ are invariant to this choice, and, thus, alignment with the Lyapunov spectrum should not be affected. Figure~\ref{fig:trained_pos_vs_neg_eig} provides empirical evidence supporting this. In particular, Subfigure~\ref{fig:neg_vs_pos_eig_epoch100} shows that trained networks with either completely positive eigenspectra or negative eigenvalues show no visible differences in alignment with their respective Lyapunov spectra. However, building on the evidence from Figure~\ref{fig:beta_trained_ablation}, the networks do display different spectral distributions.

In Figure~\ref{fig:neg_vs_pos_eig_epoch100}, the fully-positive eigenspectrum configuration registers a $|\bar\lambda_{\min}| \approx e ^ {-7}$, orders of magnitude lower than the "negative" configuration's $|\bar\lambda_{\min}| \approx e ^ {-0.04}$. The LLE of the negative eigenvalue condition is slightly positive, which indicates weakly chaotic dynamics, compared to the stability of the fully positive condition. Furthermore, Subfigure~\ref{fig:neg_vs_pos_eig_spread} shows how both configurations maximise eigenvalue spread over training epochs. For the "positive" setting, maximising spread means pushing eigenvalues to the boundaries of $[0, 1]$. The "negative" setting can instead cluster eigenvalues around the $\pm1$ extremities. As highlighted in Section~\ref{sec:conv_fwd_deer}, the topological conjugates $Q \bar\Lambda Q^T$ live within the same unitary orbit, which means that the distance between $M(x)$ and $M(y)$ for any given $x$ and $y$ is bounded by the eigenvalue spread: 

\begin{equation} \label{eq:unitary_max_dist}
    \norm{M(x) - M(y)} = \norm{Q(x)\bar\Lambda Q(x)^T - Q(y)\bar\Lambda Q(y)^T } \le |\bar\lambda_{\max} - \bar\lambda_{\min}|
\end{equation}

Therefore, as argued before, ADPTNet layers may be learning to maximise variance, or the "sharpness of the turns" (Fig.~\ref{fig:topological_conjugates}), possible at each time step.

\begin{figure}[H]
    \centering
        \begin{subfigure}[b]{0.85\textwidth}
            \centering        
            \includegraphics[width=\textwidth]{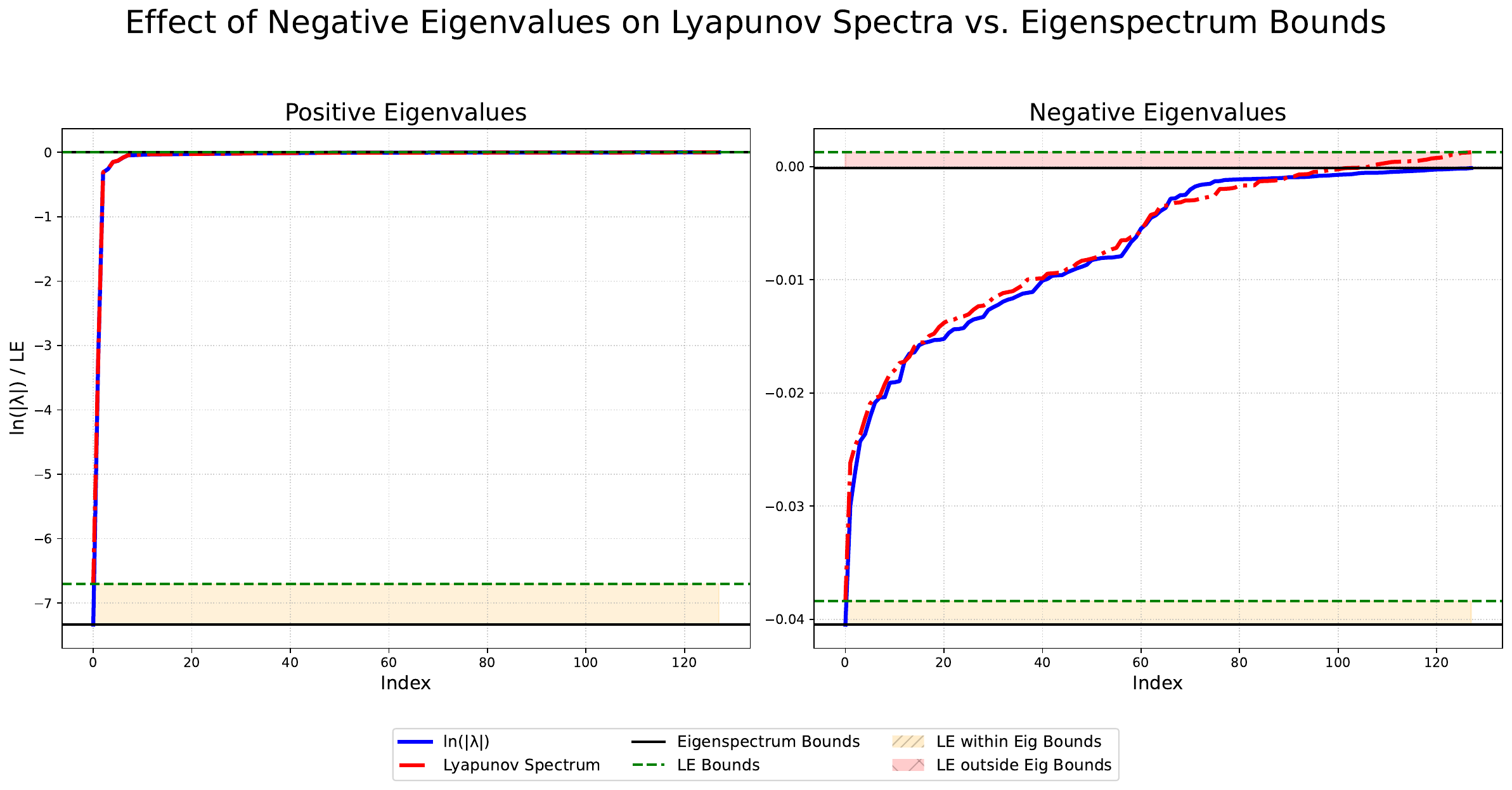}
            \caption{Fully Trained Singular Value and Lyapunov Spectra}
            \label{fig:neg_vs_pos_eig_epoch100}
         \end{subfigure}
        \begin{subfigure}[b]{0.45\textwidth}
            \centering        
            \includegraphics[width=\textwidth]{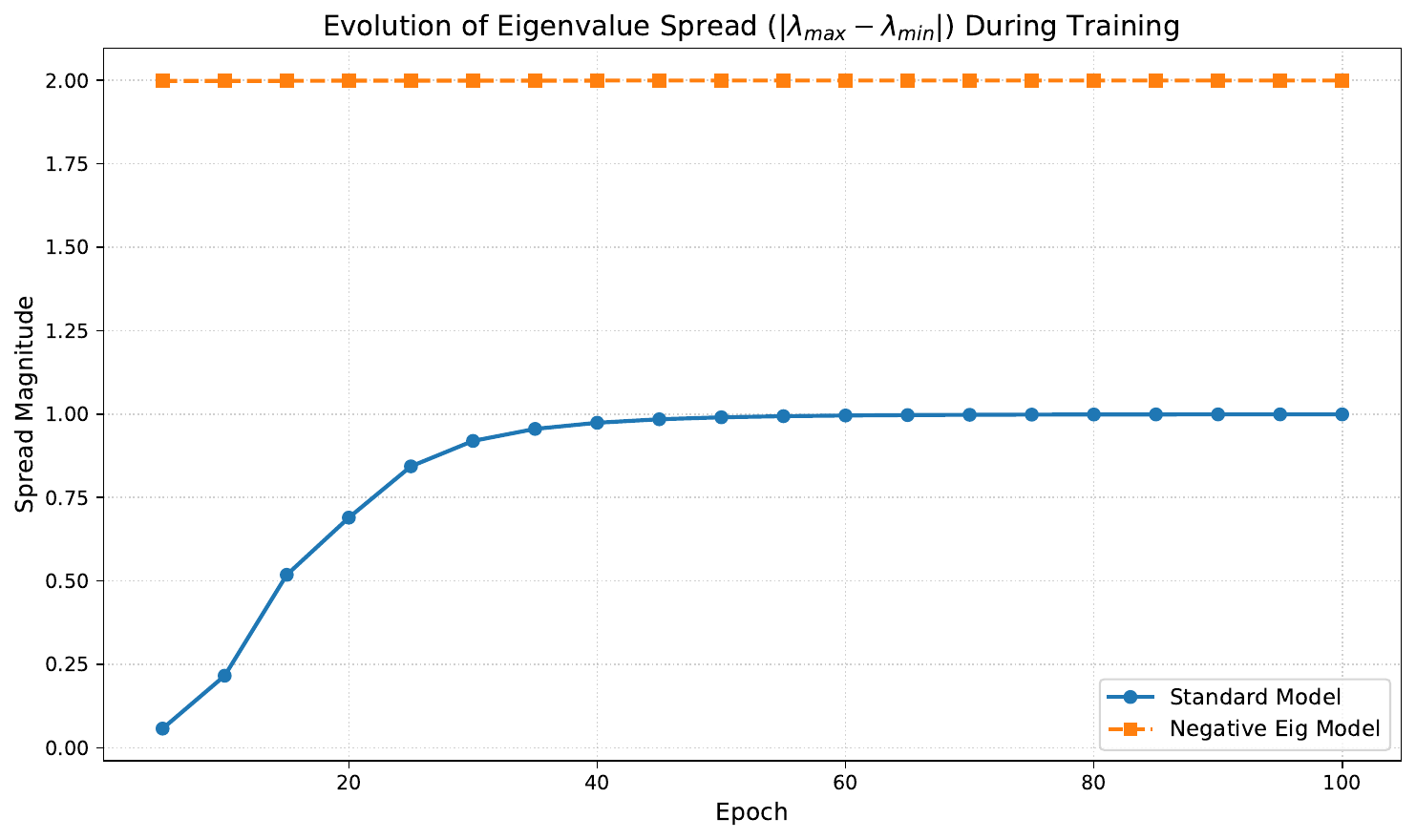}
            \caption{Eigenvalue Spread }
            \label{fig:neg_vs_pos_eig_spread}
        \end{subfigure}
        
     \hfill
        \caption{\textbf{Effect of Including Negative Eigenvalues of Trained ADPTNet Spectra}. Subfigure~\ref{fig:neg_vs_pos_eig_epoch100} shows the different distributions of log-singular values and Lyapunov exponents for a network with a fully positive eigenspectrum and a second network with 32 eigenvalues constrained to be negative throughout training. Both networks are trained on the Copy Memory task with 1004 distraction tokens (total sequence length $T=1024$). The "positive" model reaches $95\%$ test accuracy at Epoch 100, while the "negative" model reaches $90\%$. The distribution snapshots are taken at the end of training (Epoch 100). Subfigure~\ref{fig:neg_vs_pos_eig_spread} shows the evolution of the eigenvalue spread for the two model configurations throughout training. It should be noted that the Lyapunov spectra are close to the log-singular values for both settings. However, the distribution has a lower minimum for the "positive" case. It can also be observed that both spreads tend to their respective maximums over training. }
        \label{fig:trained_pos_vs_neg_eig}
\end{figure}

\textbf{$\mathbf{\Delta_i^{(t)}}$ Distribution Changes over Training Epochs} So far, all observed Lyapunov spectra strongly align with their parametrised $\ln(|\bar\Lambda|)$ counterparts. Since only the stable Lyapunov spectrum computation algorithm (Algorithm~\ref{algo:stable_lyapunov_comp}) has been used so far, all evidence thus far points toward confirming Lemma~\ref{lemma:le_exact_match}. In Lemma~\ref{lemma:le_exact_match}, the notation of $\Delta_i^{(t)} = \tilde\sigma_i - \sigma_i$ was introduced, describing the difference (i.e., perturbation) between the singular values of the recurrent Jacobian $J_t$ and the parametrised $|\bar\Lambda|$. As noted in Corollary~\ref{coro:deltas}, if the distribution of such perturbations over sequence length has mean $=0$, then they effectively cancel out in the long term, and the Lyapunov spectrum matches $\ln(\bar\Lambda)$.

Figure~\ref{fig:delta_distribution} provides evidence supporting this claim. As shown, throughout training, the distribution of $\Delta_i^{(t)}$ is centred around $0$, leading to the strong match between log-singular values and Lyapunov exponents. However, the variance of the perturbations does increase between Epoch 5 and the end of training (Epoch 100). Besides $\beta$, the size of the perturbations also depends on the various $k \cdot v$ dot products inside the $J_t$ formulation. Therefore, as training progresses, ADPTNet layers may also learn to control the distance between $k$ and $v$ and the plane of rotation in $Q$ more precisely. This further supports the argument that by increasing eigenvalue spread, the network dynamics tend towards increased non-linearity throughout training.

\begin{figure}[H] 
    \centering
        \centering        
        \includegraphics[width=0.8\textwidth]{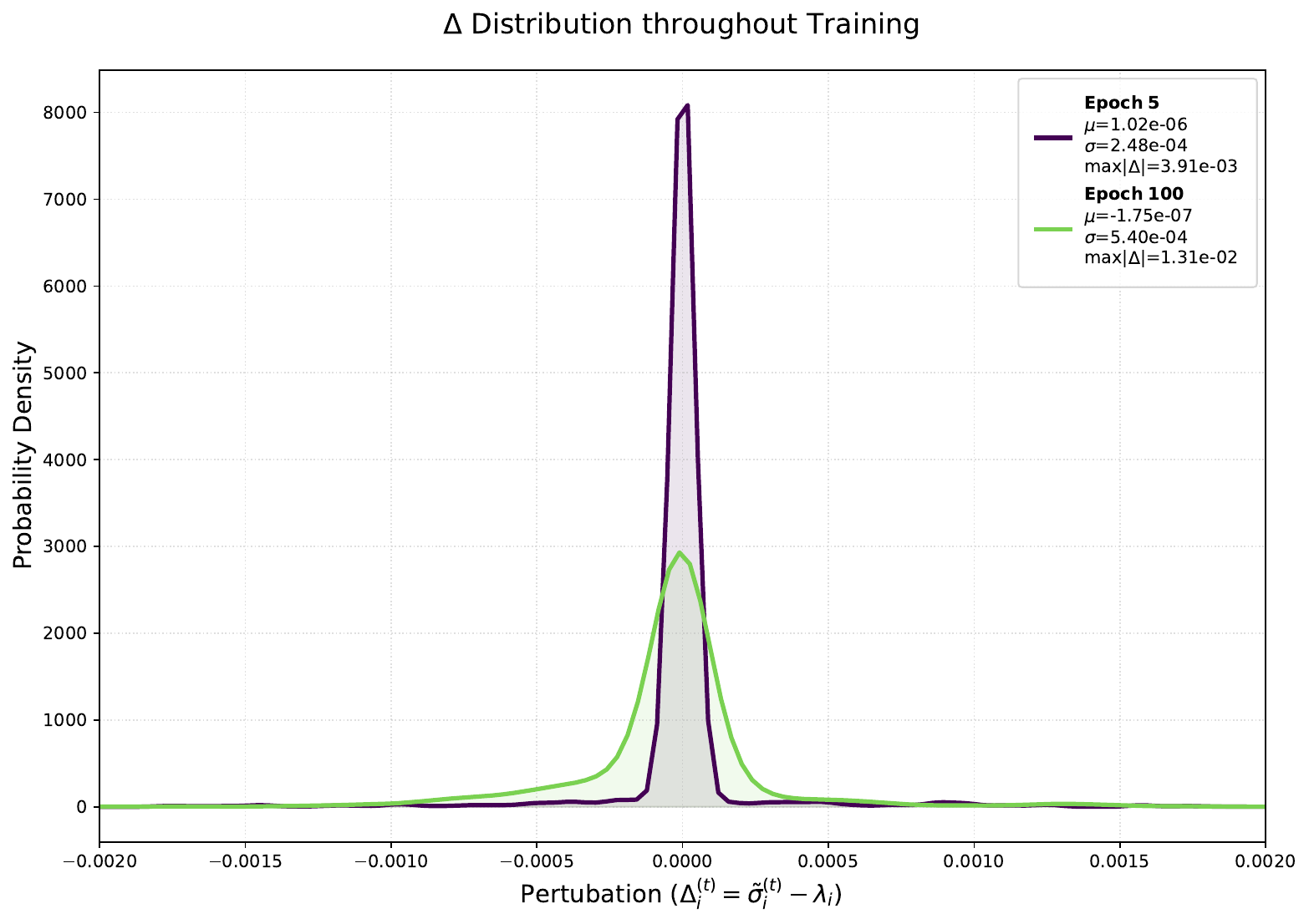}
        \hfill
        \caption{\textbf{$\mathbf{\Delta_i^{(t)}}$ Distribution Throughout Training} $\Delta _i^{(t)}$ perturbations are computed at each time step using as the difference $\sigma_i(J_t) - \bar\lambda_i = \tilde\sigma_i^{(t)}- \bar\lambda_i$, and averaged over sequence length for a Copy Memory task sample. $\mu$ denotes the mean over time steps, while $\sigma$ represents the standard deviation. The perturbations are collected from the "positive" model in Figure~\ref{fig:trained_pos_vs_neg_eig}. The key trends to observe are that for both Epoch 5 and Epoch 100 (final epoch) the mean of the perturbation distribution is 0. In addition, $\Delta_i^{(t)} $variance increases between the two epochs visualised. }
        \label{fig:delta_distribution}
\end{figure}

\subsection{Parallelisation with Conv and Forward DEER} \label{sec:deer_ablation_results}

As in Section~\ref{sec:lyapunov_results}, this section is split between analysing single-layer ADPTNet models at initialisation and throughout training. Using the insights from Section~\ref{sec:alt_deer}, the goal is to assess the degree to which the approximation errors of Conv and Forward-DEER compared to Quasi-DEER impact parallelisation performance in practice. 

\subsection{Conv and Forward DEER Computational Overhead}

Before analysing convergence properties for the Jacobian-free DEER variations proposed in this study, it is important to underscore the limited computational overhead they incur. Here, computational overhead is defined as additional Multiply-Accumulate (MAC) operations required to obtain $A_t$ (Section~\ref{sec:alt_deer}), after computing $M_tx_{t-1}$ for the forward dynamics of the network. This does not include the actual cost of computing DEER iteration parallel scans until convergence, as those vary by variable factors such as ADPTNet layer dynamics or Scale-ELK damping (see Section~\ref{sec:scale_elk_damping}). In Table~\ref{table:conv_forward_flops}, it can be seen that since Conv DEER consists of $A_t = \bLb$, no additional floating point operations are required, as the discretised eigenspectrum is already necessary to compute $M_tx_{t-1}$ in the ADPTNet recurrence and thus can be reused. Forward DEER introduces only a linearly scaling MAC overhead. Evidently, Quasi-DEER would also incur a linearly scaling cost, with a higher constant owing to the terms in $\frac{dM}{dx}x$ in Eq.~\ref{eq:simple_dMdx}. 

\begin{table}[h!] 
    \centering
    \renewcommand{\arraystretch}{1.5}
        \begin{tabular}{l c}
            \hline
            DEER Approx. & MACs \\
            \hline
            Conv & 0  \\
            
            Forward & $20*D$ \\
        \hline
        \end{tabular}
    \caption{\textbf{Computational Overhead for Jacobian-Free DEER Variations} Conv-DEER does not require any additional MACs to obtain its $A_t$ terms, while Forward DEER scales linearly. The MAC count estimate for Forward-DEER is obtained using the \texttt{torchprofile} library and the implementation of Eq.~\ref{eq:diag_M} in Appendix~\ref{appendix:code_fwd_macs}. Here, $D$ refers to the model hidden state size.}
\label{table:conv_forward_flops}
\end{table}

\subsubsection{ADPTNet Parallelisation at Initialisation}

\begin{table}[h!] 
    \centering
    \renewcommand{\arraystretch}{1.5}
    \begin{tabular}{|c|c|c|c|c|c|c|c|c|c|}
        \hline
         Neg. Eigvals. & $\gamma$ & RoPE & D & Exp. & KV &  $(\lambda_{\min},\lambda_{\max})$ & T\\
        \hline

        \ding{55}  & $1 -|\bar\lambda|$ & \checkmark & 64 & Low Rank & Dense & $0.5$ & 1024 \\
        \hline
        
    \end{tabular}
    \caption{\textbf{Baseline Untrained ADPTNet Configuration for DEER Parallelisation Ablations} Overall, the baseline configuration here is close to the trained model configuration in Section~\ref{sec:trained_lyapunov_results}. $T$ refers to the total number of time steps in the random sequences fed into the networks as inputs. To collect convergence metrics, the networks are then simulated with these random drivers following \citet{gonzalez2024towards}.}
\label{table:baseline_deer_untrained_config}
\end{table}

Similar to Section~\ref{sec:untrained_lyapunov_results}, this section explores the DEER parallelisation characteristics of single untrained ADPTNet layers at initialisation. Inputs are once again random signals drawn $\sim \mathcal{N}(0, 1)$. The default configuration for all experiments is in Table~\ref{table:baseline_deer_untrained_config}.

\textbf{Eigenvalue Spread and $\mathbf{\beta}$ Ablation} In Section~\ref{sec:conv_fwd_deer} it is hypothesised that increasing $\beta$ decreases the size of the basin of quadratic convergence for DEER. In practice, this would mean that the increased non-linearity of a larger $\beta$ should also increase the number of Newton iterations required for DEER convergence. Furthermore, considering the various local approximation error terms (Table~\ref{table:error_summary}), one could also predict that increasing $\bar\lambda_{\max}$ and the spread $|\bar\lambda_{\max} - \bar\lambda_{\min}|$ should also increase the required iterations for convergence, and also increase the disparities between Quasi-DEER and the Jacobian-free alternatives proposed here (Conv and Forward-DEER). Figure~\ref{fig:eigenvalue_dt_min_max} provides empirical evidence supporting these hypotheses. 

Firstly, Subfigure~\ref{fig:dt_min_deer} shows how decreasing $\Delta t_{\min}$ (i.e., increasing $\bar\lambda_{\max} = e^{-\lambda_{\min}\Delta t_{\min}}$) also increases the number of DEER iterations required for all variants tested. Furthermore, the increase in iterations required for convergence is larger for larger $\beta$, and a larger $\beta$ in itself raises the total number of iterations required for the same $\Delta t_{\min}$. In other words, a larger $\beta$ increases the non-linearity of the system, and combined with a larger $\bar\lambda_{\max}$, the bounds for the LLE also increase (Lemma~\ref{lemma:le_spectrum}). These two factors (increased non-linearity and larger LLE) are already known to slow down DEER convergence \citep{gonzalez2026predictability}. However, unlike other non-linear RNNs, the evidence in Figure~\ref{fig:eigenvalue_dt_min_max} suggests that ADPTNet can parametrically control them and, in turn, DEER convergence speed. It is worth noting that the differences between Quasi, Conv, and Forward DEER become more pronounced with lower $\Delta t_{\min}$ and larger $\beta$. Namely, for $\Delta t_{\min} = 10^{-3}$ and $\beta=100$, Forward and Conv-DEER have progressively higher penalties compared to Quasi-DEER, reflecting their increasingly aggressive truncations of the terms making up Quasi-DEER's $\text{diag}(J_t)$ (Table~\ref{eq:approx_error}). Additionally, since $\Delta t_{\max}$ is held constant between configurations, the eigenvalue spread increases, which may contribute to worse convergence in the Jacobian-free methods. 

Secondly, Subfigure~\ref{fig:dt_max_deer} shows a complementary story to Subfigure~\ref{fig:dt_min_deer}. Holding $\Delta t_{\min}$ constant and decreasing $\Delta t_{\max}$ effectively decreases eigenvalue spread, while also stretching all timescales of the system for longer-term memory. Subfigure~\ref{fig:dt_max_deer} shows how this decrease in spread leads to faster average convergence across all DEER variations. In fact, it should be noted that, for the final setting where $\Delta t_{\max} = \Delta t_{\min} = 10^{-3}$, the layers are actually linear, and thus only require a single DEER iteration by definition (see Section~\ref{sec:alt_deer}). This is reflected in the average convergence for Forward and Quasi DEER. However, the mean iterations required for Conv-DEER on this configuration are slightly $>1$, which suggests that numerical precision issues can emerge even in trivial cases. 

Finally, Subfigure~\ref{fig:deer_iter_dist_min_max} adds nuance to the differences in convergence behaviour between the DEER versions. Its two histograms show the distribution of DEER iterations required for convergence as different random input sequences are fed to the network for simulation. The vertical lines show the median number of iterations required for convergence across all seeds for each algorithm. For $\beta=1$, all algorithms have comparable median iterations to convergence, but Conv-DEER shows higher maximum outliers than Forward and Quasi DEER. Differences are further amplified by $\beta=100$, with significantly larger outliers for all algorithms, but also a clearer ordering for both mean and median convergence between DEER versions: Quasi < Forward < Conv. This is unsurprising, as it again confirms the accumulation of approximation errors caused by truncating Jacobian diagonal constituent terms. 

\begin{figure}[H]
    \centering
        \begin{subfigure}[b]{0.45\textwidth}
            \centering        
            \includegraphics[width=\textwidth]{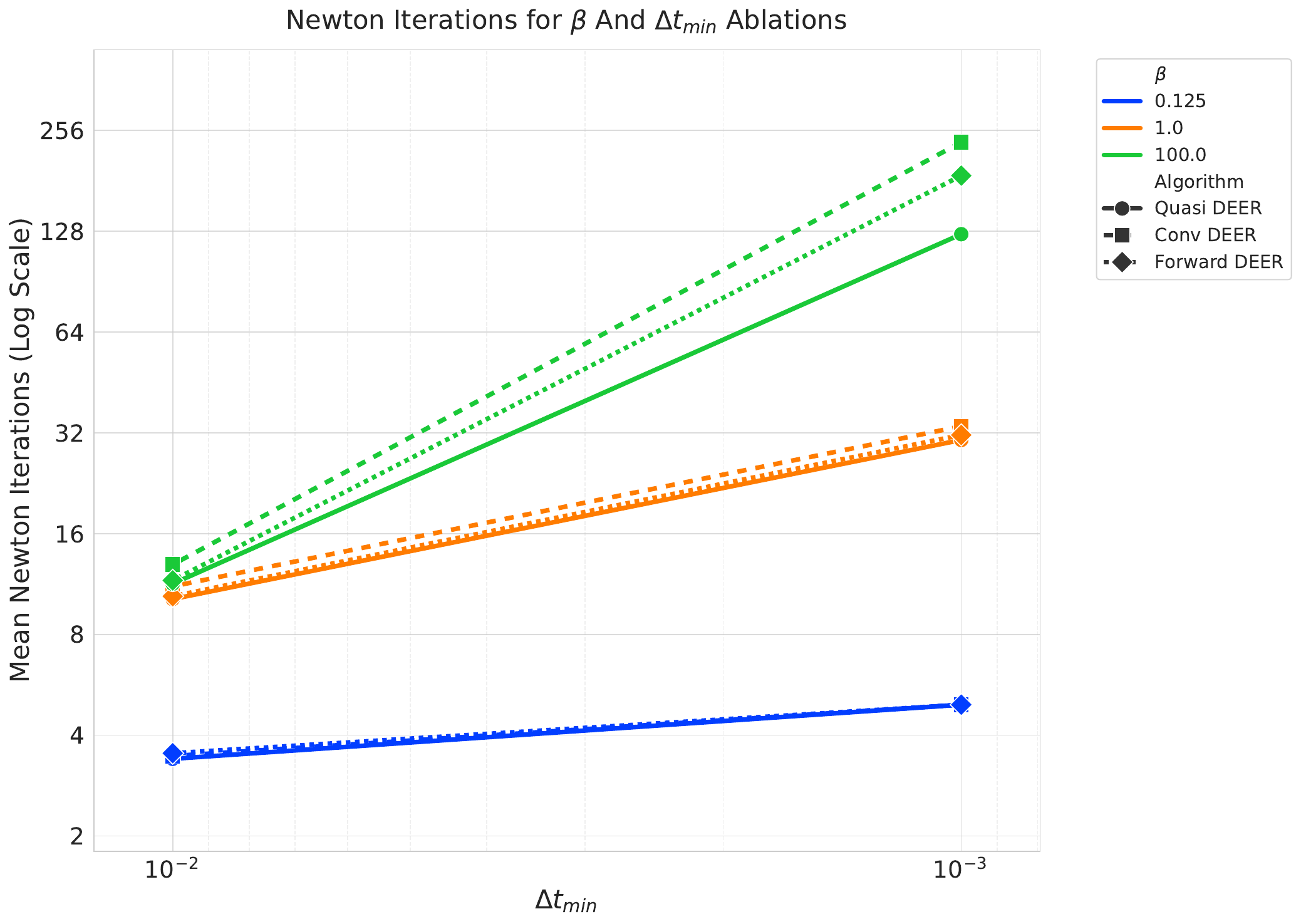}
            \caption{Increasing Long-Term Memory}
            \label{fig:dt_min_deer}
         \end{subfigure}
        \begin{subfigure}[b]{0.45\textwidth}
            \centering        
            \includegraphics[width=\textwidth]{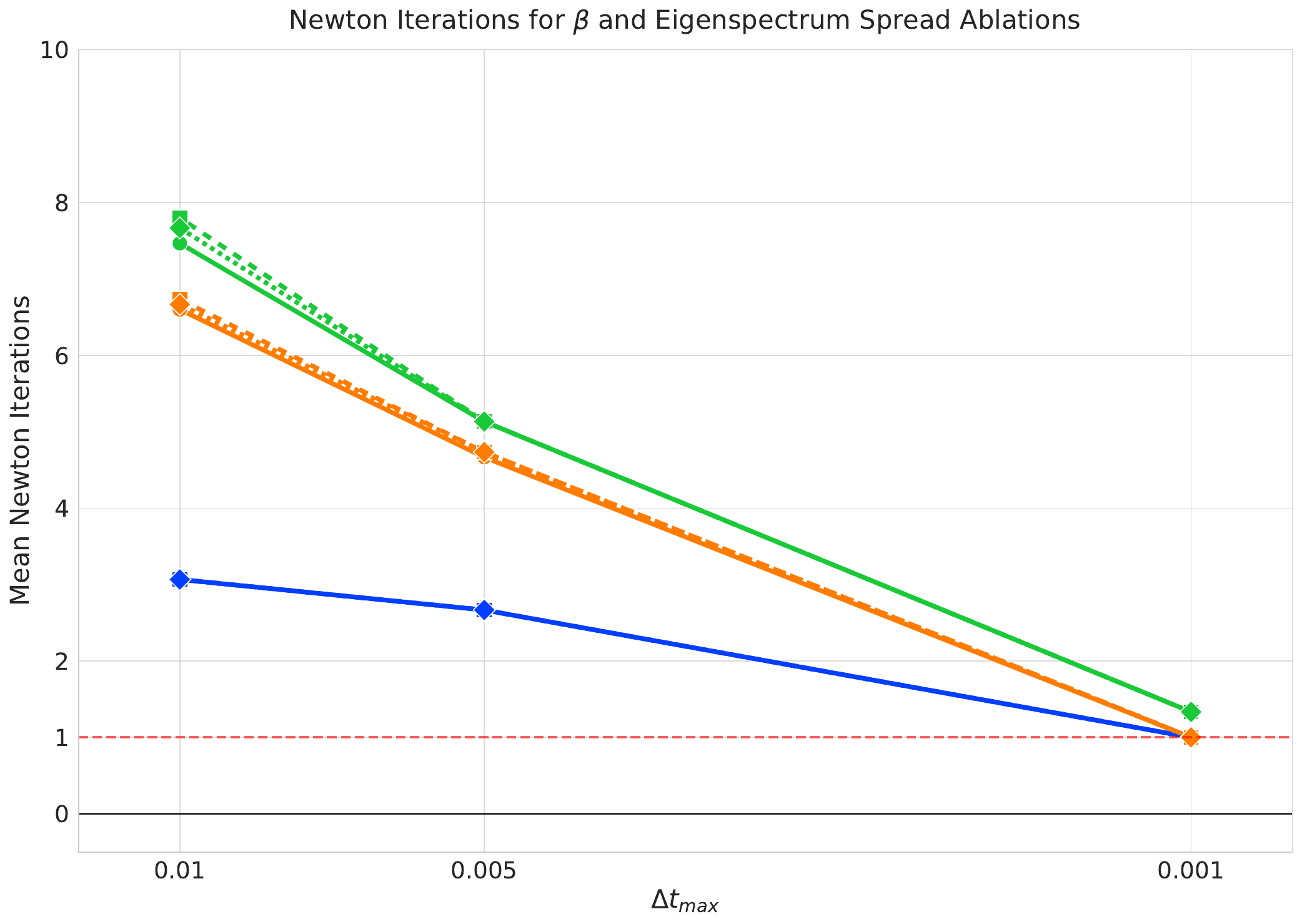}
            \caption{Decreasing Eigenvalue Spread}
            \label{fig:dt_max_deer}
        \end{subfigure}
        \begin{subfigure}[b]{0.9\textwidth}
            \centering        
            \includegraphics[width=\textwidth]{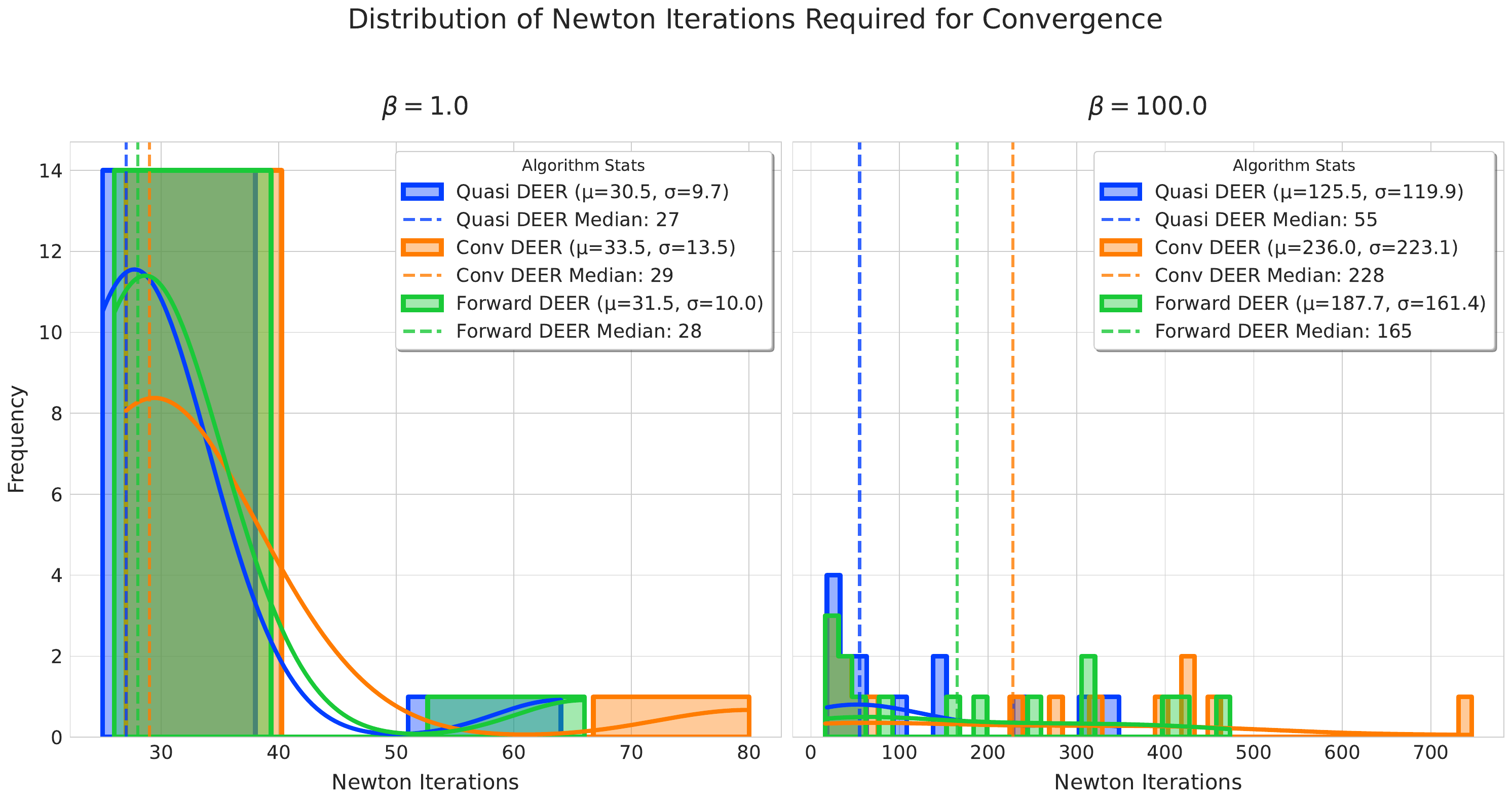}
            \caption{DEER Iteration Distribution}
            \label{fig:deer_iter_dist_min_max}
        \end{subfigure}
        
     \hfill
    \caption{\textbf{Eingspectrum Effects on DEER Convergence}. All Subfigures are produced using 15 random samples for each configuration. For Subfigure~\ref{fig:dt_min_deer}, $\Delta t_{\min}$ controls the longest time scale of the layer through $\blmax$ and takes values $\in\{ 10^{-2}, 10^{-3}\}$. $\Delta t_{\max}$ controls the shortest timescale through $\blmin$ and is constant at $10^{-1}$. The subfigure highlights how increasing the eigenvalue spread and $\beta$ leads to more DEER iterations across all algorithms tested. Subfigure~\ref{fig:dt_max_deer} holds $\Delta t_{\min}$ fixed at $10^{-3}$, while varying $\Delta t_{\max} \in \{ 0.01, 0.005, 0.001\}$. This shows that DEER iterations for long-range, slowly decaying dynamics can be reduced by reducing the eigenvalue spread along with $\beta$. Subfigure~\ref{fig:deer_iter_dist_min_max} uses the same samples and models as Subfigure~\ref{fig:dt_min_deer} and shows that Forward, and to a great extent Conv DEER, suffer from more high-iteration-count outliers. }
        \label{fig:eigenvalue_dt_min_max}
\end{figure}

\textbf{DEER Iteration Scaling with Sequence Length} The total number of time steps in the trajectory $T$ appears explicitly in the expression for the size of the quadratic convergence zone (Eq.~\ref{eq:quad_basin} adapted from \citet{gonzalez2026predictability}). This formalises the intuition that a longer sequence would require more DEER iterations to converge. Figure~\ref{fig:deer_seq_len} clearly reflects this notion as well. Furthermore, it shows that the mean convergence iterations for Conv and Forward DEER largely match those of baseline Quasi-DEER, with a slight sign of upward divergence on the longest $T=16\text{K}$ setting for Conv DEER. Focusing on the maximum outliers for each algorithm, however, one can notice that while they increase with a sharper slope than the mean for all DEER variants, they are universally higher for the Jacobian-free approximations. Still, they are in the same order of magnitude.

\begin{figure}[H] 
    \centering
        \centering        
        \includegraphics[width=0.75\textwidth]{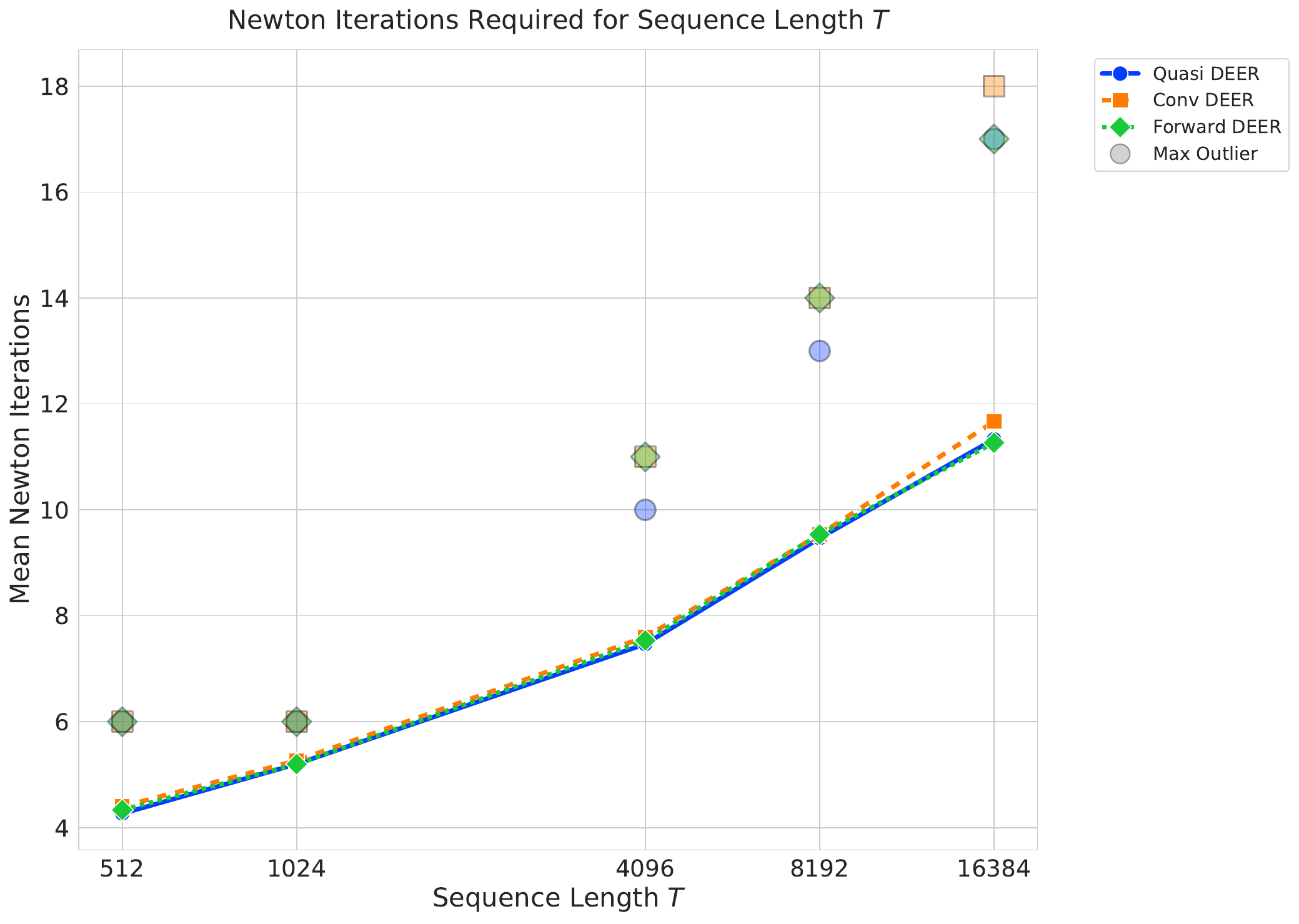}
        \hfill
        \caption{\textbf{Scaling with Sequence Length} This figure shows the averaged number of iterations required for convergence over 15 random samples of varying length $T \in \{ 512, 1204, 4096, 8192, 16384\}$. $\beta=0.125$ , $\Delta t_{\min}=10^{-4}$, and $\Delta t_{\max}=10^{-1}$ for all sequence length configurations. The main takeaway is that, unsurprisingly, the number of Newton iterations needed for convergence grows with sequence length. However, importantly, the average number of convergence iterations is the same between Conv, Forward, and Quasi DEER. }
        \label{fig:deer_seq_len}
\end{figure}

\textbf{Block-Wise DEER Simulation} As highlighted in Section~\ref{sec:block_deer}, by construction, all DEER algorithms suffer from a degree of redundant computation. This resource waste can be mitigated by introducing sequential computation between blocks of parallel DEER execution. Figure~\ref{fig:pareto_blocks} visualises the trade-off that occurs. Subfigures~\ref{fig:pareto_4096} and ~\ref{fig:pareto_8192} show how increasing the number of blocks decreases the memory consumption at the cost of increasing simulation wall-clock time. Moreover, for both sequence lengths shown, memory savings appear to plateau, with seemingly exponentially diminishing returns, while simulation times increase steadily and linearly. Subfigure~\ref{fig:pareto_iters} displays how, by increasing the number of blocks and implicitly reducing each block's length, the number of DEER iterations required to converge, per block, also decreases. Together with Subplots~\ref{fig:pareto_4096} and \ref{fig:pareto_8192}, this suggests that, in this case, the savings in iterations per block do not outweigh the slowness of sequential processing. Finally, Subfigure~\ref{fig:pareto_ratio} offers more detail into the relative performance improvements of Conv over Forward DEER. Most importantly, Conv DEER reduces execution time by $> 25\%$ compared to Forward DEER. Still, mirroring the general trend seen in Subfigures~\ref{fig:pareto_4096} and \ref{fig:pareto_8192}, peak memory allocation advantages vanish with increased block counts. Since both use the same implementation of parallel scans\footnote{Parallel associative scans are computed in PyTorch using \texttt{https://github.com/proger/accelerated-scan}}, the efficiency gains are largely driven by the lower computational overhead in computing $A_t$ terms for Conv DEER (Table~\ref{table:conv_forward_flops}). It is important to emphasise that these findings are strongly tied to \texttt{PyTorch} idiosyncrasies. \citet{danieli2025pararnn} provides a more in-depth analysis of hardware I/O-aware optimisations for DEER-like algorithms.

\begin{figure}[H]
    \centering
        \begin{subfigure}[b]{0.475\textwidth}
            \centering        
            \includegraphics[width=\textwidth]{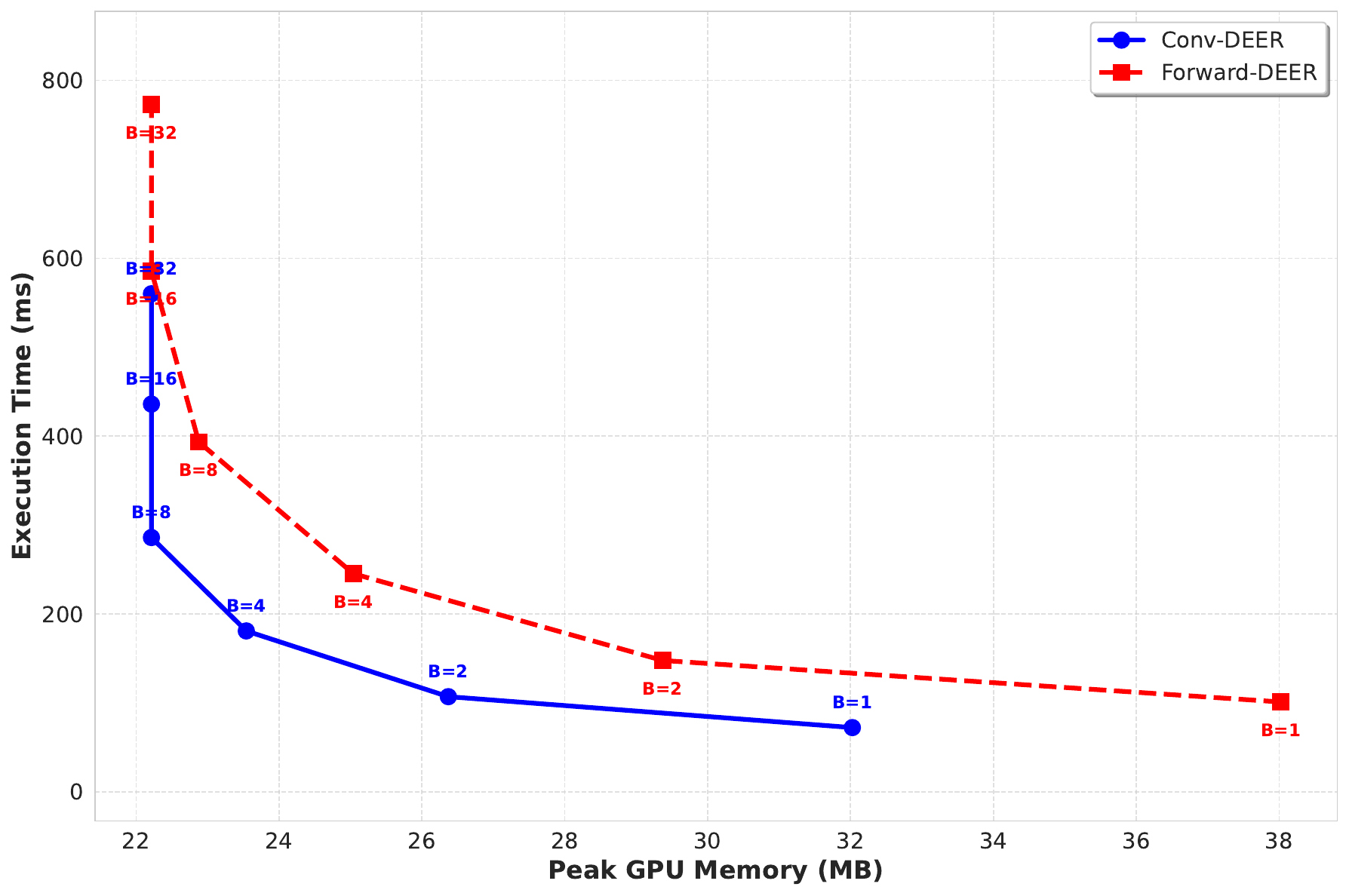}
            \caption{Memory vs Speed ($T=4096$)}
            \label{fig:pareto_4096}
         \end{subfigure}
         \hfill
        \begin{subfigure}[b]{0.475\textwidth}
            \centering        
            \includegraphics[width=\textwidth]{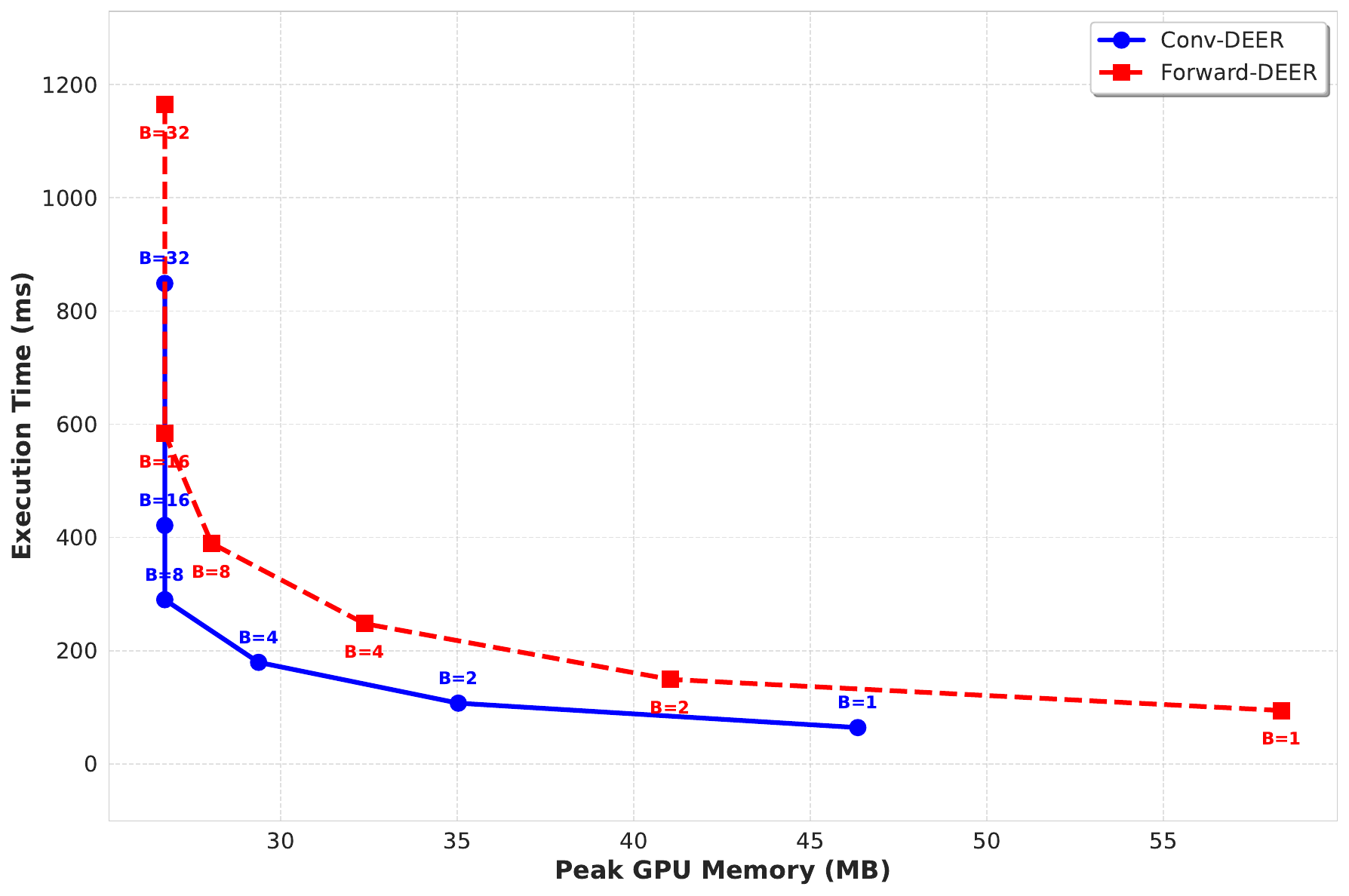}
            \caption{Memory vs Speed ($T=8192$)}
            \label{fig:pareto_8192}
        \end{subfigure}    
         \hfill
        \begin{subfigure}[b]{0.475\textwidth}
            \centering        
            \includegraphics[width=\textwidth]{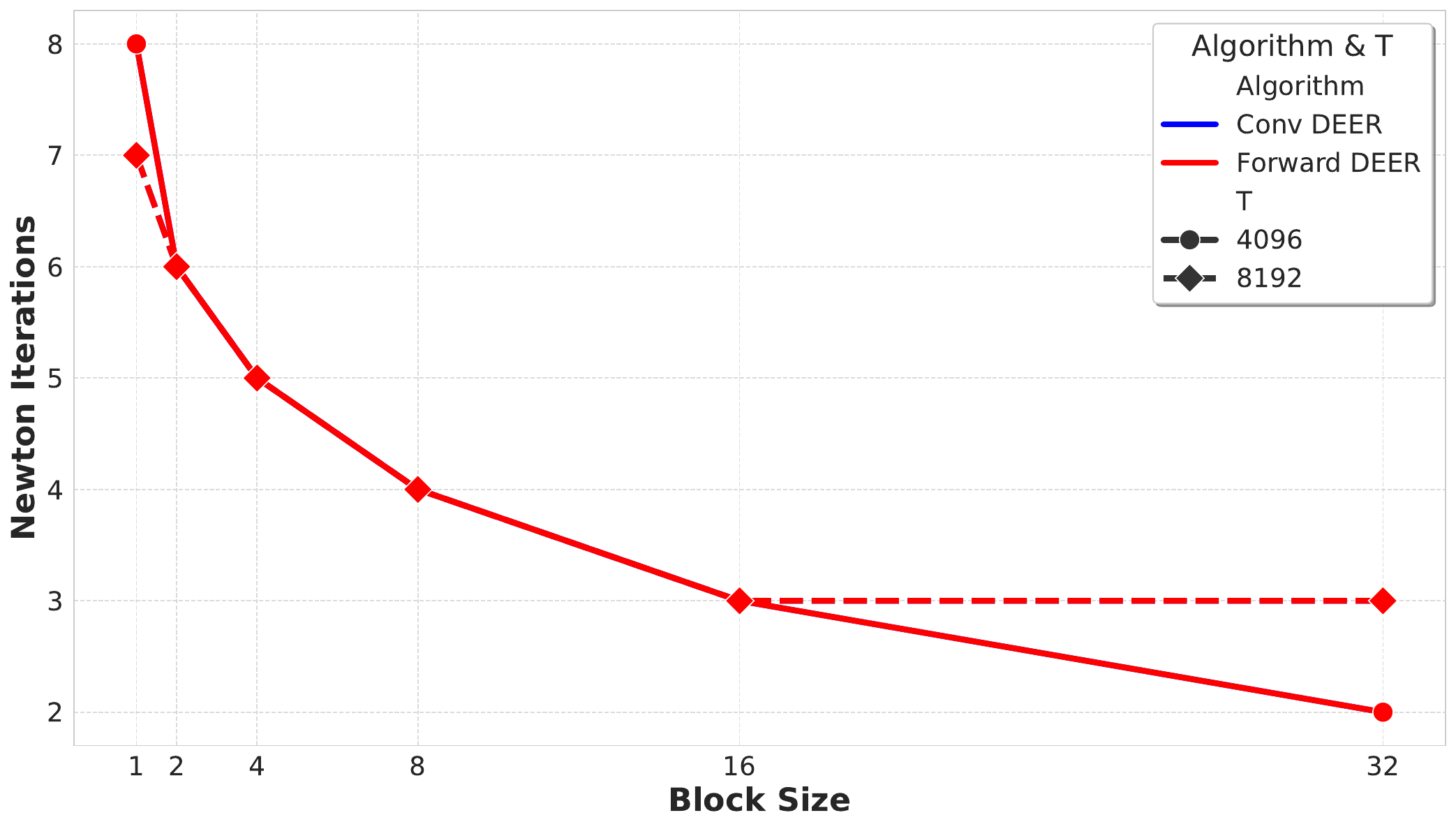}
            \caption{Iterations Required for Convergence}
            \label{fig:pareto_iters}
        \end{subfigure}   
         \hfill
        \begin{subfigure}[b]{0.475\textwidth}
            \centering        
            \includegraphics[width=\textwidth]{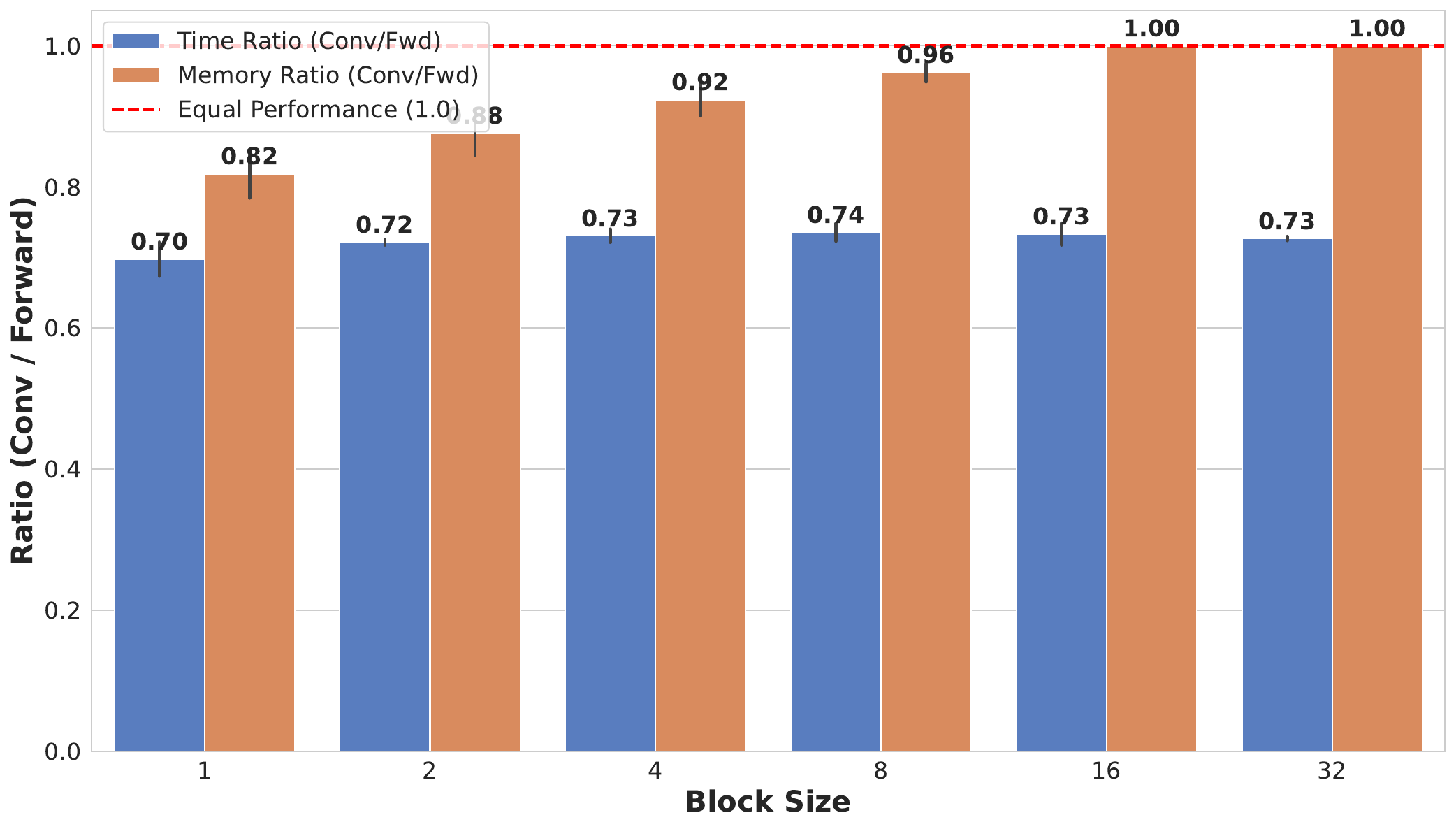}
            \caption{Conv vs Forward DEER Performance}
            \label{fig:pareto_ratio}
        \end{subfigure}        
        \hfill
    \caption{\textbf{Block-Wise DEER Trade-Offs} Subfigures~\ref{fig:pareto_4096} and \ref{fig:pareto_8192} show the wall clock time and peak GPU memory utilisation depending on the different number of sequential blocks ($B \in \{1, 2, 4, 8, 16, 32\}$) used for Conv/Forward DEER simulation. A random input is sampled for each sequence length $T \in \{4096, 8192\}$, along with a randomly initialised ADPTNet layer. For all configurations, $\beta=0.125$, $\Delta t_{\min}=10^{-4}$, and $\Delta t_{\max}=10^{-1}$. Subfigure~\ref{fig:pareto_iters} shows the number of DEER iterations used for each block/ sequence length configuration $(B, T)$. For $B>1$, the number of DEER iterations per sequential block may vary. Therefore, the iteration counts in Subfigure~\ref{fig:pareto_iters} are the median number of iterations to convergence across blocks in a $(B, T)$ simulation. This median is then used in all blocks in each $(B, T)$ setting to obtain the time and memory metrics. Subfigure~\ref{fig:pareto_ratio} shows the efficiency gains in simulation time and peak memory usage of Conv DEER compared to Forward DEER. Note that the Conv DEER implementation here is suboptimal because it recomputes $\bLb$ for $A_t$ instead of caching it from the forward pass. In contrast, Forward DEER reuses the same intermediary results from the forward pass where possible, minimising overhead in accordance with Table~\ref{table:conv_forward_flops}. Peak memory utilisation is obtained with \texttt{torch.cuda\_max\_memory\_allocated()}, while simulation time is obtained as an average of 7 trials after 2 warm-up runs. The recordings are taken without saving gradients for the backwards pass.}
    \label{fig:pareto_blocks}
\end{figure}

\textbf{Scale-ELK Damping} The hypothesis presented in Section~\ref{sec:scale_elk_damping} is that applying a scalar damping factor to Jacobian-free Quasi DEER approximations would "squeeze" the local approximation errors introduced by the $\text{diag}(J_t)$ truncation (Table~\ref{table:error_summary}). Figure~\ref{fig:damping_delta_and_iter} shows this phenomenon in effect. Subfigure~\ref{fig:damping_iter} displays how increasing the damping effect first removes high-iteration-count outliers for both Conv and Forward DEER. Then, all algorithms progressively enter an over-damping regime where iterations to convergence rise because actual information, rather than local approximation errors, vanishes before it propagates to target future time steps (Eq.~\ref{eq:damping_error_and_info}). Subfigure~\ref{fig:damping_delta} brings additional detail to the gradually growing alignment between Conv/Forward DEER and Quasi DEER as the damping effect increases. Interestingly, on the particular settings tested here, only mild damping of $0.999$ or $0.99$ is sufficient to remove most outliers for both Conv and Forward DEER. Still, Conv-DEER retains a slightly higher mismatch in iterations to convergence than Forward DEER, relative to Quasi DEER, across all damping-factor configurations tested. This perhaps reflects the non-trivial influence of the $\epsilon^{fwd}$ error term (Table~\ref{table:error_summary}).

\begin{figure}[H]
    \centering
        \begin{subfigure}[b]{\textwidth}
            \centering        
            \includegraphics[width=\textwidth]{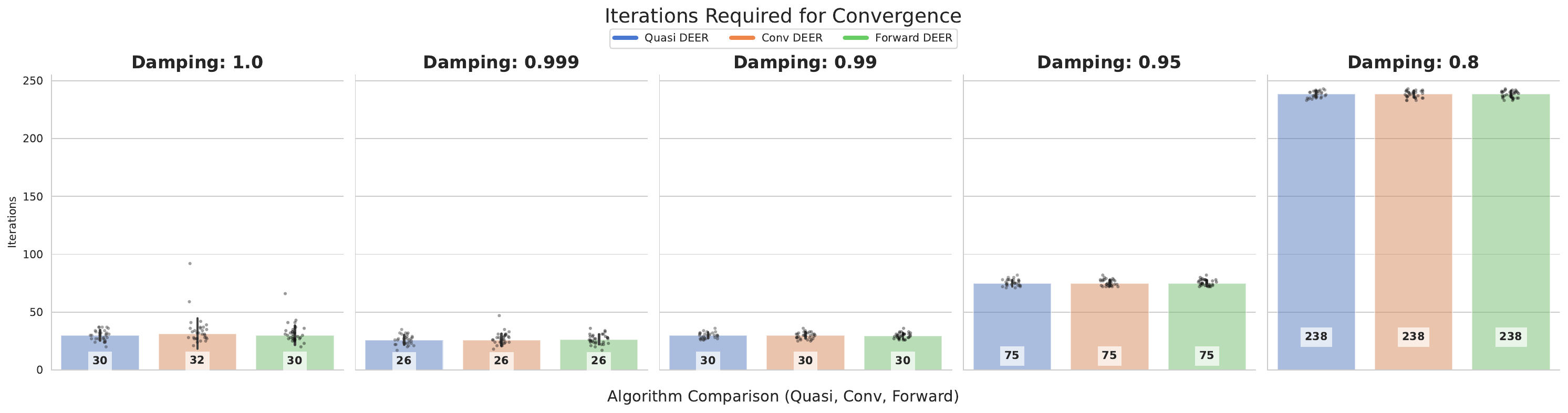}
            \caption{Total Iterations for Convergence}
            \label{fig:damping_iter}
         \end{subfigure}
        \begin{subfigure}[b]{\textwidth}
            \centering        
            \includegraphics[width=\textwidth]{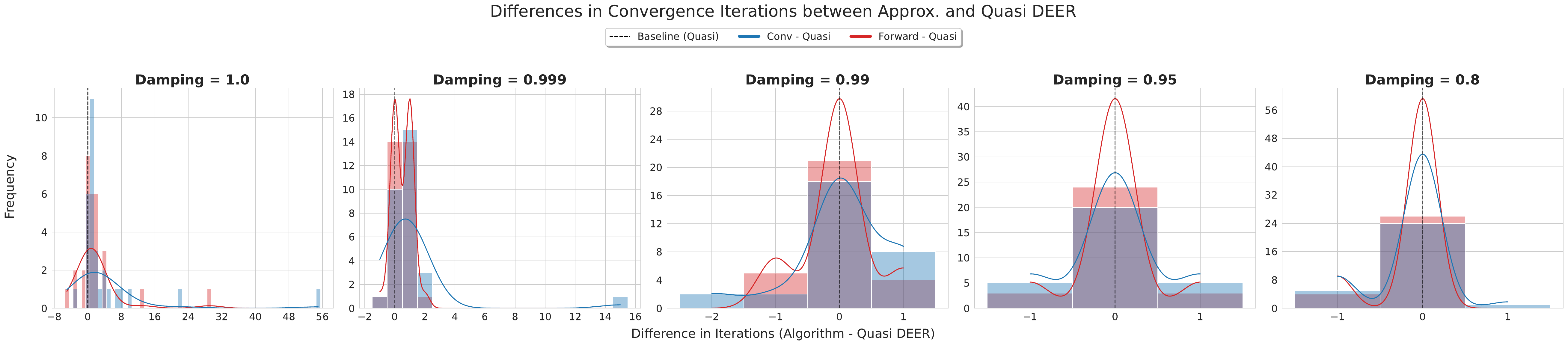}
            \caption{Reducing Approximation Errors}
            \label{fig:damping_delta}
        \end{subfigure}        
     \hfill
        \caption{\textbf{Effect of Damping on DEER Convergence}. Subfigure~\ref{fig:damping_iter} shows the effect of Scale-ELK damping on the absolute number of iterations required for convergence. The 30 dots in each column represent the convergence iterations for different random input samples, while the column shows their median. Each sample is also obtained with its own randomly initialised ADPTNet layers, with 16 eigenvalues $\bLb$ set to negative, leading to a higher median count for $T=1024$ than reported in Figure~\ref{fig:deer_seq_len}. All models tested have $\beta=0.125$, $\Delta t_{\min}=10^{-4}$, and $\Delta t_{\max}=10^{-1}$. Subfigure~\ref{fig:damping_delta} shows the relative difference in iterations to convergence between Conv/Forward DEER and Quasi DEER. The smooth curve shows the probability density. It can be noted that increasing the damping effect quickly reduces the number of outliers. }
        \label{fig:damping_delta_and_iter}
\end{figure}

\subsubsection{ADPTNet Parallelisation during Training}\label{sec:deer_trianing_results}

In Section~\ref{sec:trained_lyapunov_results}, it is shown that the spread of $\bLb$ values tends to increase over training epochs. At the same time, Figure~\ref{fig:deer_iter_dist_min_max} underlines how a higher eigenvalue spread leads to slower DEER convergence. Therefore, one could expect that DEER convergence would slow down throughout training. In Figure~\ref{fig:trained_damping} it can be seen that this is indeed likely to happen in practice. Taking all subfigures into consideration, it can be observed how (in the absence of damping), the sharp increase in DEER iterations during the early training epochs coincides with the fastest rate of growth in the $\bLb$ spread (i.e., during the first $\approx 20-30$ epochs). In addition, the fully parallel simulations in Subfigure~\ref{fig:trained_damped_iter} appear prone to occasionally collapsing into sequential-step parity (DEER iterations = $T$) for the Jacobian-free DEER versions. Scale-ELK damping seems to eliminate extreme outliers from Forward-DEER simulations, at the cost of considerably more baseline iterations to convergence. In the case of Conv-DEER, however, even severe damping, which leads to orders of magnitude higher baseline convergence iterations, does not fully eliminate $T$-iteration outliers. In contrast, block-wise DEER simulations appear stable and even equivalent for all DEER variants, without requiring any damping (Subfigure~\ref{fig:trained_block_iter}). Given that block-wise execution may be the only viable option for training on memory-constrained hardware, this raises the prospect of Conv and Forward DEER as viable, faster and more efficient alternatives to Quasi-DEER. 

\begin{figure}[H]
    \centering
        \begin{subfigure}[b]{\textwidth}
            \centering        
            \includegraphics[width=\textwidth]{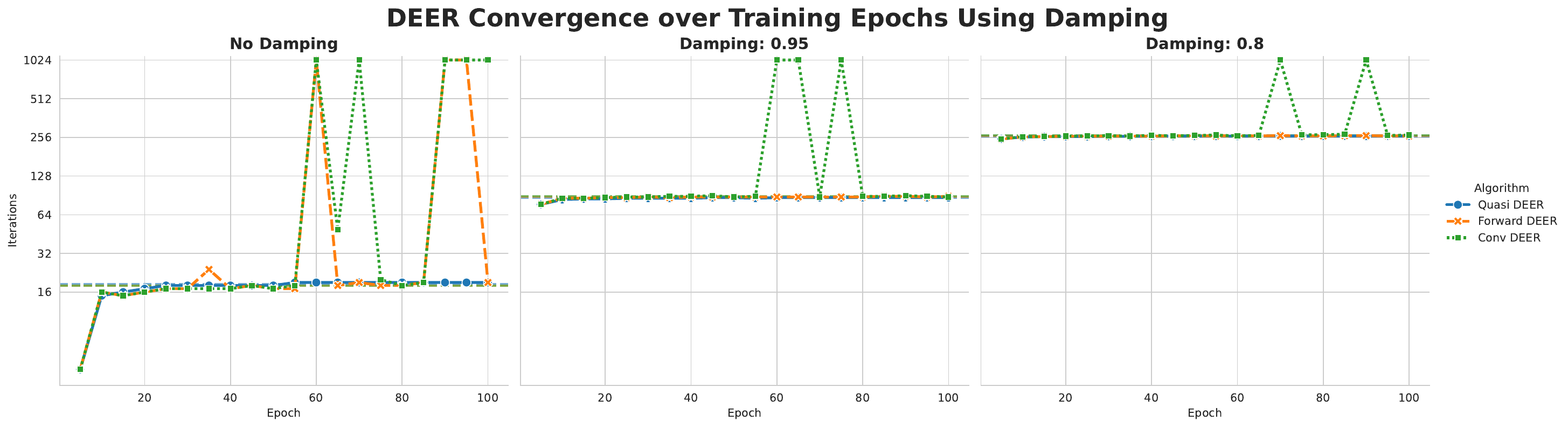}
            \caption{Damping Fully Parallel DEER}
            \label{fig:trained_damped_iter}
         \end{subfigure}
         \hfill
        \begin{subfigure}[b]{0.5\textwidth}
            \centering        
            \includegraphics[width=\textwidth]{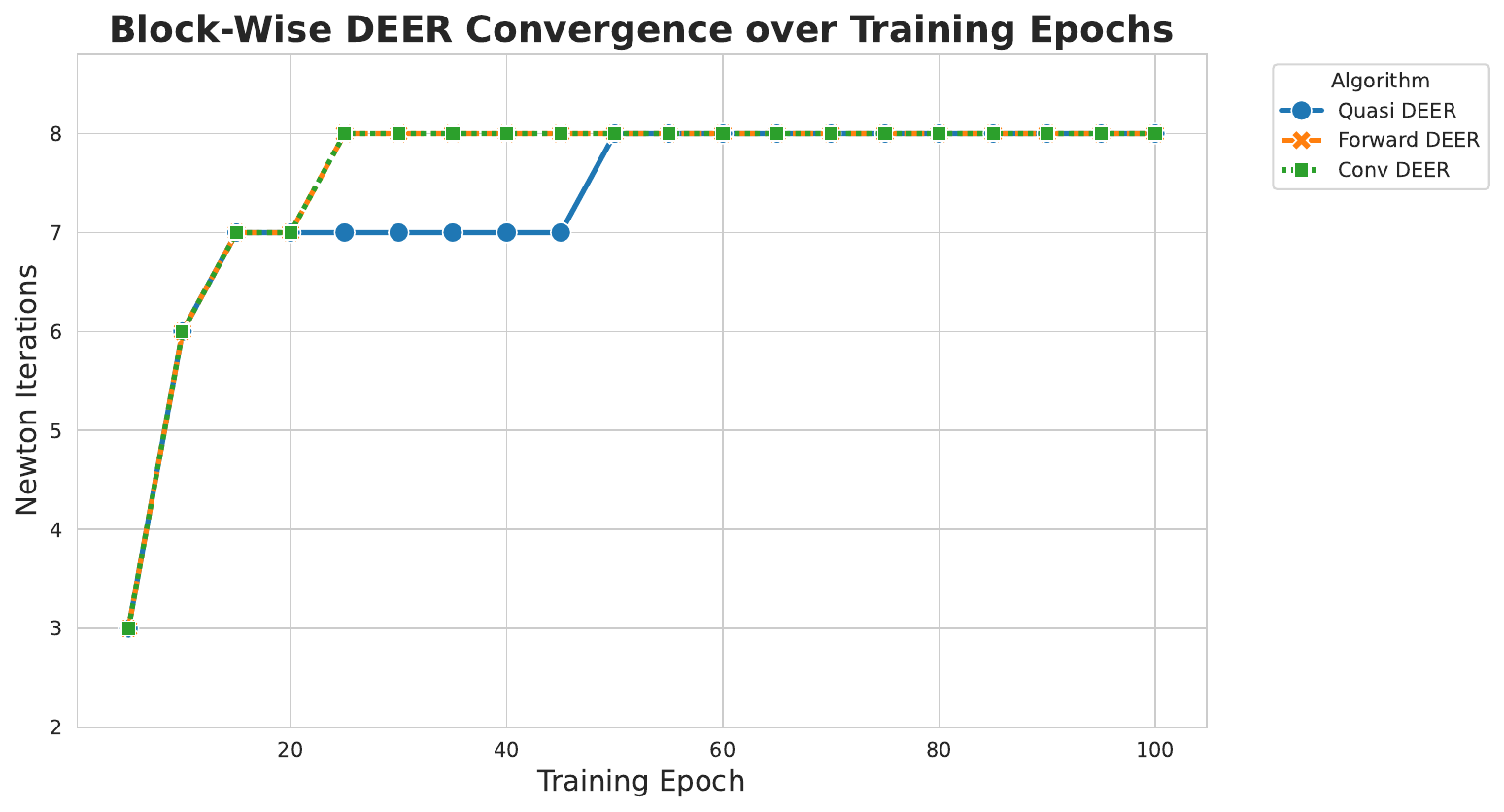}
            \caption{Block-Wise DEER}
            \label{fig:trained_block_iter}
        \end{subfigure}        
        \hfill
        \begin{subfigure}[b]{0.45\textwidth}
            \centering        
            \includegraphics[width=\textwidth]{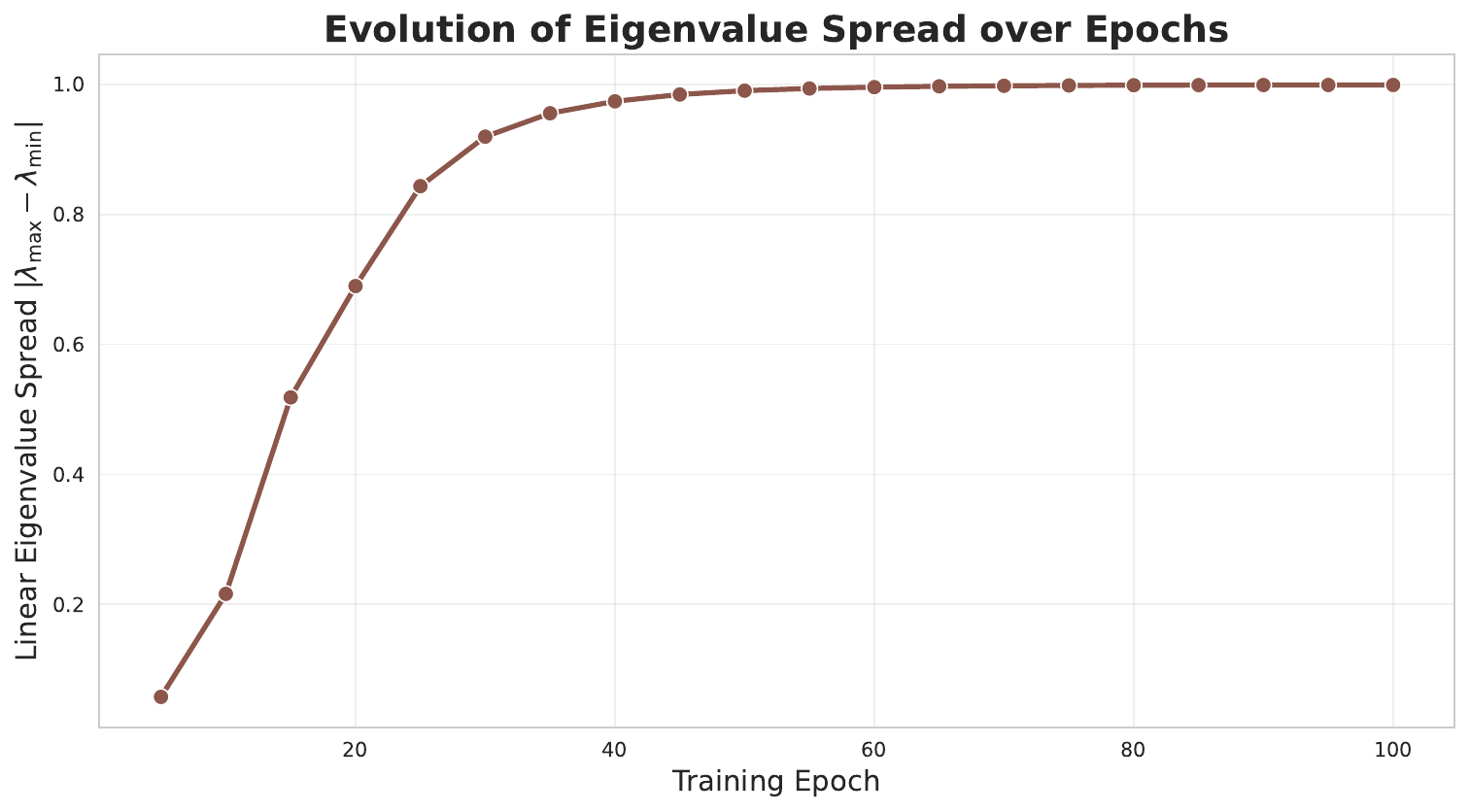}
            \caption{Eigenvalue Spread Evolution}
            \label{fig:trained_damping_spread}
        \end{subfigure}        
     \hfill
    \caption{\textbf{DEER Convergence over Training Epochs}. The convergence data in this figure are obtained from the same fully positive $\bLb$ model snapshots used to produce the spectral analysis in Figure~\ref{fig:trained_pos_vs_neg_eig} (Table~\ref{table:baseline_trained_config}). That also means that the training iterations correspond to learning the $T=1024$ Copy Memory Task. Consequently, Subfigure~\ref{fig:trained_damping_spread} is a reproduction of the data in Subfigure~\ref{fig:neg_vs_pos_eig_spread} included here to aid convergence data analysis. For the block-wise simulations in Subfigure~~\ref{fig:trained_block_iter}, the number of sequential blocks is $B=8$. The shown convergence iteration counts are obtained from the same randomly selected Copy Memory test sample. The main observations are that convergence iterations for all DEER variants increase over epochs in line with increasing $\bLb$ spread. Additionally, in this particular case, Conv-DEER appears unstable for full DEER parallelism in Subfigure~\ref{fig:trained_damped_iter}, while block-wise execution in Subfigure~\ref{fig:trained_block_iter} appears to mitigate the convergence pathology. } 
    \label{fig:trained_damping}
\end{figure}

\subsection{Measuring Non-Linearity with State Tracking} \label{sec:state_tracking_results}

Recall from Section~\ref{ssms} that SSMs place non-linear activation functions position-wise, between layers. Therefore, a single SSM is, in theory, restricted to modelling fading memory models that eventually decay back to a singular resting state / fixed point. However, despite their linear recurrence, there is substantial evidence to suggest that fixed-depth stacked SSMs are capable of modelling non-linear dynamical systems \citep{orvieto2023resurrecting, hu2025state}, such as Mackey-Glass \citep{chilkuri2021parallelizing, abreu2024q}. In other words, for practical purposes, SSM systems can model non-linear dynamics. Conversely, dynamical systems are not fully adequate to answer the question of the extent to which the ADPTNet topological conjugate recurrence induces non-linearity. 

Instead, in recent years, the de facto standard methodology for probing non-linear sequential modelling has become \textit{state tracking} \citep{merrill2023parallelism}. As the name suggests, state tracking refers to the capability of a single-layer sequence model to track changes to a state over time. For instance, \citet{merrill2024illusion} provides the analogy of tracking a game of chess. Sequential piece moves affect the state of the board at each time step, and the network needs to correctly encode these state changes all throughout the game to correctly output its end-state. The same principle can be more abstractly encapsulated by the word problem over the symmetric group $S_n$\footnote{$S$ is once again being overloaded. In SNN contexts, it represents the spiking activation; in the ADPTNet recurrence, it represents a skew-symmetric matrix; and here it is the standard notation for the symmetric group} \citep{merrill2024illusion}. The symmetric group encodes permutations applied to a set of items. For example, given the set of symbols, $[a, b, c, d, e]$, the $S_5$ group encodes all possible permutations(e.g., $\mathtt{swap}(1, 3) \rightarrow [c, b, a, d, e] \in S_5$). The world problem over $S_5$ consists of tracking a sequence of permutations applied to the base set (e.g., what is the end result of sequentially applying $\mathtt{swap}(1, 3) \rightarrow \mathtt{swap}(2, 4) \rightarrow \dots$). 

Crucially, in a seminal finding, \citet{merrill2023parallelism} theoretically proved that state tracking tasks such as $S_5$ cannot be solved by fully parallel architectures such as Transformers without scaling network depth with sequence length. In contrast, traditional non-linear RNNs can perfectly simulate $S_5$ using a single layer. \citet{merrill2024illusion} extended the result to SSMs, which despite the recurrent formulation, lack the expressivity of non-linear RNNs and cannot solve state-tracking at fixed depth either. Table~\ref{table:mamba_state_tracking} summarises concrete challenges that Mamba, for instance, faces when applied to the $S_5$ task, as reported in \citet{schone2025implicit}. 

The same contrast between linear and non-linear recurrences can be observed in Table~\ref{table:baselines_state_tracking}, where a vanilla RNN generalise to 2x and 4x the training sequence length. At the same time, a single-layer,linear, and diagonal ADPTNet ($\beta=0$) cannot surpass $\approx 31.2\%$ accuracy on the training sequence length, and collapses to random accuracy on longer sequences. Furthermore, not even increase the width of the system can improve performance. \citet{grazzi2025unlocking} showed how including negative eigenvalues in linear SSMs can help the networks learn parity state tracking tasks (e.g., determine the final sign after mulitplication with sequence of $\{ \pm1\}$: $1, -1, 1, 1, -1, \dots$ ). Interestingly, even for $\beta=0$ on a non-parity state tracking task such as $S_5$,  there appears to be a slight improvement of performance when setting half the eigenspectrum negative. 

\begin{table}[H]\centering\small\setlength{\tabcolsep}{6pt}
\begin{tabular}{c c c c}
\toprule
Model & Seq.Len. & N. Layers Required & Extrapolation \\
\midrule
\multirow{2}{*}{Mamba} & 8 & 4 & \texttimes \\
 & 32 & 16 & \texttimes \\
\midrule
\multirow{2}{*}{RNN} & 8 & 1 &  \checkmark \\
& 32 & 1 &  \checkmark \\
\bottomrule
\end{tabular}
\caption{\textbf{Mamba / RNN State Tracking Performance Summarised from \citet{schone2025implicit}} This table summarises the number of layers required for Mamba and non-linear RNNs to converge on the $S_5$ task, as reported in \citet{schone2025implicit}. Of note, Mamba requires more layers as sequence length grows, while not being able to extrapolate to longer sequences at test time (i.e., 2x or 4x the training length). In contrast, a single layer RNN can not only solve the task, but also extrapolate to longer input sequences. }
\label{table:mamba_state_tracking}
\end{table}

\begin{table}[H]\centering\small\setlength{\tabcolsep}{6pt}
\begin{tabular}{l l r r c rrr}
\toprule
\multirow{2}{*}{Eig.\ Init.} & \multirow{2}{*}{Model} & \multirow{2}{*}{Width} & \multirow{2}{*}{Sparsity} & \multirow{2}{*}{Neg.\ Eig.} & \multicolumn{3}{c}{Accuracy (\%)} \\
\cmidrule(lr){6-8}
 & & & & & $L{=}8$ & $L{=}16$ & $L{=}32$ \\
 & & & & & \footnotesize$1\times$ & \footnotesize$2\times$ & \footnotesize$4\times$ \\
\midrule
  \multirow{2}{*}{--} & ReLU RNN & 512 & -- & -- & 100.0 & 100.0 & 99.6 \\
 & ReLU RNN & 864 & -- & -- & 100.0 & 100.0 & 100.0 \\
 \midrule
 \multirow{3}{*}{Linspace}
 &  ADPTNet ($\beta=0$) & 512 & 0\% & \checkmark & 31.2 & 1.1 & 0.9 \\
 &  ADPTNet ($\beta=0$) & 624 & 0\% & \checkmark & 31.2 & 1.0 & 0.9 \\
\ & ADPTNet ($\beta=0$) & 512 & 75\% & \checkmark & 26.5 & 0.8 & 0.9 \\
\midrule
\multirow{4}{*}{S4D-Real} & ADPTNet ($\beta=0$) & 512 & 0\% & \checkmark & 28.6 & 1.2 & 0.8 \\
 & ADPTNet ($\beta=0$) & 512 & 0\% & \texttimes & 26.7 & 0.9 & 0.8 \\
 & ADPTNet ($\beta=0$) & 512 & 75\% & \checkmark & 26.9 & 1.1 & 0.8 \\
 & ADPTNet ($\beta=0$) & 512 & 75\% & \texttimes & 25.3 & 0.9 & 0.8 \\
\bottomrule
\end{tabular}
\caption{\textbf{Baseline $\mathbf{S_5}$ Accuracy} Results are obtained by averaging over three seeds for each configuration and following the training setup from Appendix~\ref{appendix:s5_training_config}. One can observe the strong alignment with the results in Table~\ref{table:mamba_state_tracking} with respect to the stark gap in accuracy between linear and non-linear recurrence. }
\label{table:baselines_state_tracking}
\end{table}

Table~\ref{table:s5_8_results} summarises the impact on state tracking abilities that the choice of ADPTNet initialisation and parametrisation can have. Overall, it can be observed that the best predictor of $S_5$ accuracy is $\beta$, with performance monotonically increasing with a higher $\beta$, regardless or eigenspectrum initialisation/parametrisation. Moreover, for the largest setting $\beta=4$, almost all eigenspectrum choices enable signs of extrapolation up to $4\times$ the training sequence length. This evidence supports the intuition that as $\beta$ increases, the "sharpness of turns" of the topological conjugates with $\beta$-bound Lispchitz-ness (Corollary~\ref{coro:lipschitz}), also increases to the point where the system is sufficiently non-linear to model state tracking. Stated differently, the choice of $\beta$ is effectively a gradual parametric "dial" for increasing the non-linearity of the recurrent dynamics. 

While $\beta$ dominates $S_5$ performance, it does not singularly control it. A second visible factor is the choice of eigenspectrum parametrisation, which comprises whether $75\%$ of the eigenvalues are masked to $0$ permanently, and whether half of the non-zero eigenvalues are constrained to be negative. For $\beta < 1$, negative eigenvalues appear to improve accuracy regardless of sparsity, while for $\beta=1$, sparsity appears to becomes the dominant driver of performance. Interestingly, for $\beta=4$, sparsity and negative eigenvalues become mutually exclusive, with both reaching comparable peak accuracies only when separated from the other. The effect does appear conditional on the initialisation used, with Linspace reaching its highest performance with a sparse and positive spectrum, while S4D-Real does so on the reverse. In terms of maximum extrapolation performance, Linspace initialisation outperforms S4D-Real. 

Hence, it can be argued that sparsity and negative eigenvalues engender non-linearity through different inductive biases. On one hand, sparsity can be seen as forcing the network to learn ReLU-like forgetting, with the similarity transform effectively rotating information into the 0-masked dimensions to be discarded. The "surviving" information is then projected onto a low-dimensional state space akin to the action of low-rank RNNs \citep{mastrogiuseppe2018linking}. On the other hand, negative eigenvalues maximise the possible spread of the eigenspectrum. This, in turn, maximises the output variance from the similarity transforms, increasing model expressivity. At the lowest tested $\beta=0.00125$, both mechanisms are necessary to boost accuracy slightly over the linear $\beta=0$ baseline ($32.4\% > 31.2\%$). However, with higher $\beta$, as previously highlighted, they become incompatible.

\begin{table}[H]\centering\small\setlength{\tabcolsep}{6pt}
\begin{tabular}{l r r c rrr}
\toprule
\multirow{2}{*}{Eig.\ Init.} & \multirow{2}{*}{Sparsity} & \multirow{2}{*}{$\beta$} & \multirow{2}{*}{Neg.\ Eig.} & \multicolumn{3}{c}{Accuracy (\%)} \\
\cmidrule(lr){5-7}
 & & & & $L{=}8$ & $L{=}16$ & $L{=}32$ \\
 & & & & \footnotesize$1\times$ & \footnotesize$2\times$ & \footnotesize$4\times$ \\
\midrule
\multirow{8}{*}{Linspace} & \multirow{2}{*}{0\%} & 4 & \checkmark & 43.4 & 1.7 & 0.8 \\
 &  & 4 & \texttimes & 51.5 & 1.3 & 0.9 \\
\cmidrule(lr){2-7}
 & \multirow{6}{*}{75\%} & 0.00125 & \checkmark & 28.8 & 0.8 & 0.8 \\
 &  & 0.0125 & \checkmark & 36.2 & 1.2 & 1.1 \\
 &  & 0.125 & \checkmark & 43.6 & 0.9 & 0.8 \\
 &  & 1 & \checkmark & 76.9 & 30.5 & 1.5 \\
 &  & 4 & \checkmark & 79.9 & 44.2 & 8.7 \\
 &  & 4 & \texttimes & 98.1 & 62.4 & 2.9 \\
\midrule
\multirow{24}{*}{S4D-Real} & \multirow{12}{*}{0\%} & 0.00125 & \checkmark & 30.3 & 1.3 & 0.8 \\
 &  & 0.00125 & \texttimes & 29.4 & 1.2 & 0.8 \\
 &  & 0.0125 & \checkmark & 50.3 & 1.4 & 0.8 \\
 &  & 0.0125 & \texttimes & 42.5 & 1.3 & 0.8 \\
 &  & 0.125 & \checkmark & 50.7 & 1.7 & 0.8 \\
 &  & 0.125 & \texttimes & 45.1 & 1.3 & 0.8 \\
 &  & 0.5 & \checkmark & 49.6 & 1.6 & 0.8 \\
 &  & 0.5 & \texttimes & 52.0 & 1.5 & 0.9 \\
 &  & 1 & \checkmark & 64.3 & 18.1 & 5.5 \\
 &  & 1 & \texttimes & 56.2 & 1.5 & 0.9 \\
 &  & 4 & \checkmark & 85.3 & 60.3 & 23.0 \\
 &  & 4 & \texttimes & 66.1 & 27.4 & 2.3 \\
\cmidrule(lr){2-7}
 & \multirow{12}{*}{75\%} & 0.00125 & \checkmark & 32.4 & 1.4 & 0.9 \\
 &  & 0.00125 & \texttimes & 29.4 & 1.4 & 0.8 \\
 &  & 0.0125 & \checkmark & 48.4 & 1.3 & 0.8 \\
 &  & 0.0125 & \texttimes & 43.8 & 1.2 & 0.8 \\
 &  & 0.125 & \checkmark & 57.3 & 1.6 & 0.8 \\
 &  & 0.125 & \texttimes & 44.4 & 1.5 & 0.8 \\
 &  & 0.5 & \checkmark & 53.6 & 2.5 & 0.8 \\
 &  & 0.5 & \texttimes & 51.2 & 1.4 & 0.9 \\
 &  & 1 & \checkmark & 55.9 & 1.7 & 0.9 \\
 &  & 1 & \texttimes & 53.7 & 1.3 & 0.9 \\
 &  & 4 & \checkmark & 64.5 & 29.5 & 10.2 \\
 &  & 4 & \texttimes & 83.3 & 45.5 & 4.2 \\
\bottomrule
\end{tabular}
\caption{\textbf{Effect of $\mathbf{\beta}$ and Eigenspectrum on ADPTNet State Tracking Performance} Each result in this table is obtained by averaging over three seeds, as in Table~\ref{table:baselines_state_tracking}. In addition, the training configuration is in Appendix~\ref{appendix:s5_training_config}. It can be seen that the most reliable indicator of $S_5$ word problem accuracy is the choice of $\beta$. To maximise performance, however, eigenvalue initialisation and parametrisation must also be considered. }
\label{table:s5_8_results}
\end{table}

Table~\ref{table:s5_32_results} shows the effect of including additional $k$ and $v$ processing steps on $S_5$-word problem accuracy, this time on training sequence length $L=32$. As mentioned in Section~\ref{sec:ADPTNet_def}, applying a ReLU non-linearity before $K$ and $ V$ projections does not affect long-term behaviour guarantees related to the Lyapunov spectrum (Lemma~\ref{lemma:le_spectrum}). However, as shown in Table~\ref{table:s5_32_results}, it does make a substantial difference in state-tracking ability. More specifically, including ReLU doubles accuracy on the training length (perfectly learning the task with $100\%$ accuracy), and almost completely bridges the gap to vanilla RNNs for $2\times$ length extrapolation ($97\%$). Even in the $4\times$ extrapolation setting, maximum accuracy is visibly higher with ReLU than all results on the shorter sequences in Table~\ref{table:s5_8_results}. 

Here, it is important to recall that the orthogonal matrices used in the ADPTNet recurrence are obtained via Riemannian gradient descent steps from the origin $I$, of step size $\beta$ and in a high-dimensional direction defined by $k$ and $v$ outer and inner products. Therefore, there are two hyperparameters for ADPTNet non-linearity and expressivity that do not pertain to $\beta$ or the inclusion of ReLU activations. Firstly, as argued in Section~\ref{sec:controlled_dot}, the strength of off-diagonal interactions is controlled, in part, by the dot-product similarity between $k$ and $v$. The $<k, v>$ setting in Table~\ref{table:s5_32_results} uses the a priori dot product control method from Section~\ref{sec:controlled_dot} to either enforce or not enforce $k\cdot v=0.5$. As a result, it can be seen that strict constraints on the variance of the dot product do reduce model expressivity and hurt $S_5$ performance to some extent. For example, compared to $49.2\%$ accuracy on $L=32$ for the baseline configuration (no LayerNorm, ReLU or $k \cdot v$ control), including $k\cdot v$ reduces accuracy by almost $6\%$ ($43.4\%$).  

Secondly, LayerNorm can be applied to $Kx$ and $Vx$ before unit-normalisation to constrain their variance to 1. In other words, the variance in the directions to travel across the manifold is constrained. Table~\ref{table:s5_32_results} shows how including this effectively erases state-tracking abilities. Conversely, this suggests that ADPTNet actively learns to control the interaction between $k$ and $v$ to achieve state tracking. However, including ReLU largely mitigates constraints on both the inner-product and outer-product. 

\begin{table}[H]\centering\small\setlength{\tabcolsep}{6pt}
\begin{tabular}{c c c rrr r}
\toprule
\multirow{3}{*}{LN} & \multirow{3}{*}{ReLU} & \multirow{3}{*}{$\langle k, v\rangle$} & \multicolumn{3}{c}{Mean Accuracy (\%)} & Best (\%) \\
\cmidrule(lr){4-6} \cmidrule(lr){7-7}
 & & & $L{=}32$ & $L{=}64$ & $L{=}128$ & $L{=}128$ \\
 & & & \footnotesize$1\times$ & \footnotesize$2\times$ & \footnotesize$4\times$ & \footnotesize$4\times$ \\
\midrule
\texttimes  & \texttimes  & \texttimes  &  49.2 & 26.3 &  9.9 & 10.3 \\
\checkmark  & \texttimes  & \texttimes  &   5.3 &  3.6 &  2.5 &  2.9 \\
\texttimes  & \checkmark  & \texttimes  & 100.0 & 97.0 &  5.1 & 36.8 \\
\texttimes  & \texttimes  & \checkmark  &  43.4 & 21.8 &  8.7 &  8.9 \\
\checkmark  & \checkmark  & \texttimes  & 100.0 & 93.6 &  0.8 &  8.8 \\
\texttimes  & \checkmark  & \checkmark  & 100.0 & 89.8 & 15.8 & 29.7 \\
\bottomrule
\end{tabular}
\caption{\textbf{Ablation on L=32 $\mathbf{S_5}$} As with Tables~\ref{table:baselines_state_tracking} and \ref{table:s5_8_results}, results are obtained by averaging over 3 seeds and following the training configuration from Appendix~\ref{appendix:s5_training_config}. \textit{LN} refers to whether LayerNorm is applied to $Kx$ and $Vx$ before unit normalisation. \textit{ReLU} refers to including a ReLU activation in $K\mathtt{ReLU}(x)$ and $V\mathtt{ReLU}(x)$. $<k, v>$ denotes whether the dot product between $k$ and $v$ is constrained to 0.5 using the methods from Section~\ref{sec:controlled_dot}. Overall, ReLU can be observed to significantly benefit state tracking performance, mitigating even the effects of adding LN and $<k, v>$ constraints, which otherwise can severely impact accuracy. }
\label{table:s5_32_results}
\end{table}

\subsection{Measuring Adaptability with Selective Copy} \label{sec:selective_copy_results}

Sections~\ref{sec:trained_lyapunov_results} and \ref{sec:deer_trianing_results} employ the Copy Memory task (Fig.~\ref{fig:copy_memory}), which measures a network's ability to learn fixed delay lines. While it effectively probes long-range dependencies, it requires no \textit{selectivity} or \textit{adaptability}. Instead, \citet{gu2023mamba} showed that \textit{Selective Copying} task performance can distinguish data-dependent, adaptive networks, such as Mamba, from fixed-parameter LTI SSMs such as S4. In Selective Copying, target tokens are randomly spread throughout the input sequences, interleaved with distractors (Fig.~\ref{fig:selective_copy_memory}). As in the fixed Copy Memory task, the network is prompted to recall all the target tokens at the end of the input sequence. \citet{gu2023mamba} presented evidence that, while S4 struggles to generalise on Selective Copy, Mamba can fully solve it given the same parameter budget. The most widely accepted explanation is that a non-data-dependent network effectively dilutes its memory with distractors and cannot discard unhelpful information. Conversely, adaptive/selective models dynamically discard distractors. 

Tables~\ref{table:linear_ADPTNet_selective_results} and \ref{table:non-linear_selective_results} show how ADPTNet performs on the Selective Copying task. Here, the focus is on ablating the effect of $\beta$ on selectivity rather than establishing absolute performance metrics compared to existing methods. In this regard, both tables mirror the state-tracking results from Table~\ref{table:s5_8_results}. Namely, there is a robust and monotonic improvement in Selective Copying accuracy stemming from higher $\beta$. The effect is present regardless of whether the topological conjugacy operation is linear (i.e., input-driven, see Section~\ref{sec:linear_ADPTNet}) or non-linear (i.e. recurrent state-driven). Remarkably, Table~\ref{table:linear_ADPTNet_selective_results} shows how even relatively small $\beta = 0.00125$, which effectively forces off-diagonal interactions to be weak couplings/perturbations, shows visible improvements in selectivity over the diagonal $\beta=0$ baseline. The non-linear ADPTNet recurrence in Table~\ref{table:non-linear_selective_results} similarly shows dramatic improvements in selectivity with a relatively low $\beta=0.0125$, more than doubling Selective Copy accuracy. It is important to note that the results between the two tables discussed here are not directly comparable, since Table~\ref{table:linear_ADPTNet_selective_results} is the result of a larger model with $ L=1024$ sequence length, compared to the $L=128$ sequence length in Table~\ref{table:non-linear_selective_results}. 
\begin{figure}[H] 
    \centering
        \centering        
        \includegraphics[width=0.75\textwidth]{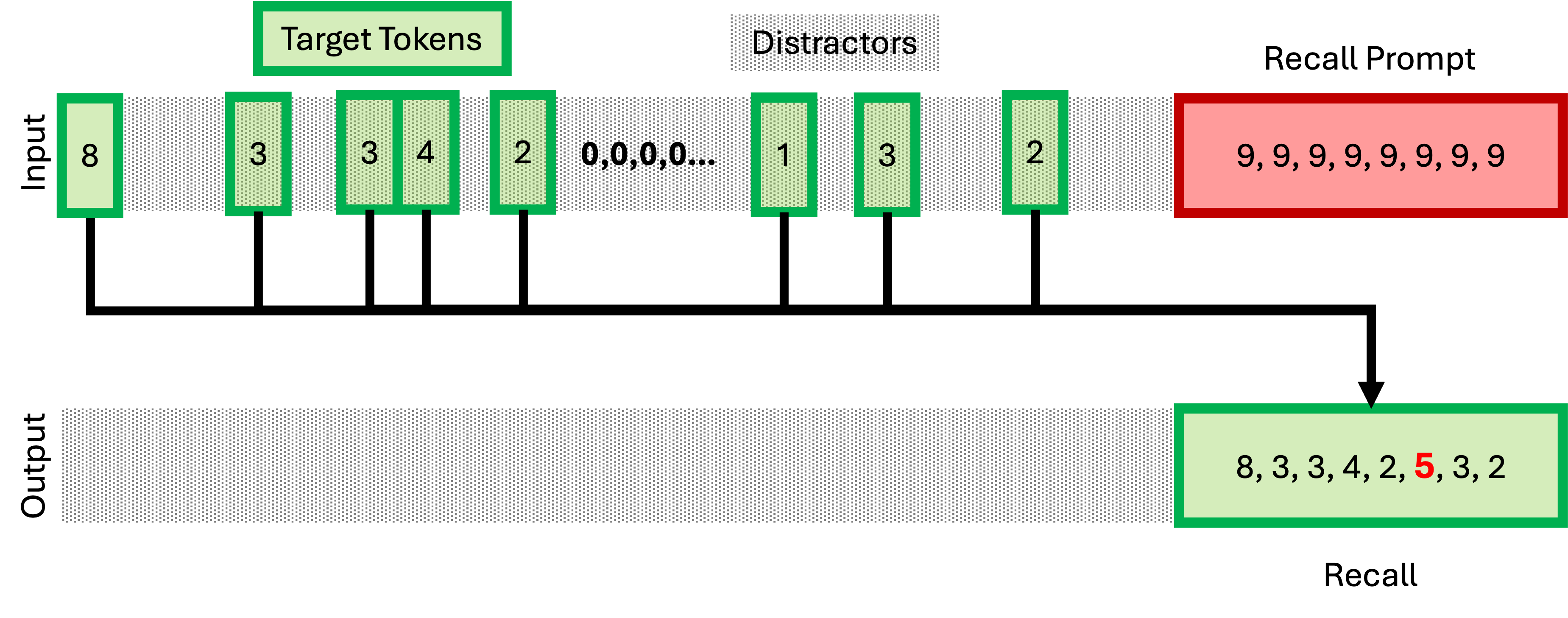}
        \hfill
        \caption{\textbf{Selective Copy Task} The Selective Copying task inherits the same general formulation as the Copy Memory task (Fig.~\ref{fig:copy_memory}). The difference is that here the target tokens are spread out randomly among the distractors. }
        \label{fig:selective_copy_memory}
\end{figure}

\begin{table}[H]
    \centering
    \renewcommand{\arraystretch}{1.5}
    \begin{subtable}[t]{0.46\textwidth}
        \centering
        \begin{tabular}{|c|c|}
            \hline
            Model & Acc (\%) \\
            \hline
            Baseline ($\beta=0$) & 46.3\\
            \hline
            $\beta = 0.00125$ & 64\\
            \hline
            $\beta = 0.0125$ & 69\\
            \hline
            $\beta = 0.125$ & 75.6\\
            \hline
        \end{tabular}
        \caption{\textbf{LinearADPTNet}}
        \label{table:linear_ADPTNet_selective_results}
    \end{subtable}
    \hfill
    \begin{subtable}[t]{0.46\textwidth}
        \centering
        \begin{tabular}{|c|c|}
            \hline
            Model & Acc (\%) \\
            \hline
            Baseline ($\beta=0$) & 36\\
            \hline
            $\beta = 0.0125$ & 80\\
            \hline
        \end{tabular}
        \caption{\textbf{Non-Linear ADPTNet}}
        \label{table:non-linear_selective_results}
    \end{subtable}
    \caption{\textbf{Effect of $\mathbf{\beta}$ on Selective Copy Accuracy} Details of the experimental setup are available in Appendix~\ref{appendix:selective_copy_setup}. It should be mentioned that each baseline in Tables~\ref{table:linear_ADPTNet_selective_results} and \ref{table:non-linear_selective_results} is obtained by taking the maximum accuracy over three seeds across six learning rate $\times $ weight decay configurations. The main observation from both tables is that using a larger $\beta$ improves Selective Copying accuracy. Table~\ref{table:linear_ADPTNet_selective_results} is obtained using sequence length $L=1024$, while Table~\ref{table:non-linear_selective_results} uses $L=128$. In addition, Table~\ref{table:linear_ADPTNet_selective_results} uses a model hidden size of $h=128$, and an state expansion of $n_{\text{state}=2}$ (see Section~\ref{sec:state_expansion}). For Table~\ref{table:non-linear_selective_results}, hidden size is $h=64$ and state expansion is $n_{\text{state}}=4$. It is important to highlight that the specific hyperparameter selection here is not optimised for producing the absolute highest Selective Copy accuracy. Instead, it purposefully underparametrises models to the point where differences in expressivity can be isolated. Evidently, drastically scaling the number of parameters could allow non-selective architectures to simply memorise all possible target token positions. 
}
    \label{table:ADPTNet_sc_results}
\end{table}

\subsubsection{ADPTNet Task-Dependent Adaptation}

Given how closely related the Copy Memory and the Selective Copying tasks are, one can use them to probe how ADPTNet learns to use the topological conjugate for different purposes. Namely, by setting the same vocabulary size, number of target tokens and distractors, model size, and training hyperparameters, it is possible to study how the exact same randomly initialised network evolves depending on the learning objective: selectivity or fixed delay-line learning. The specific shared hyperparameters between the twin experimental conditions are included in Appendix~\ref{appendix:selective_copy_setup}. 

\textbf{Lyapunov Spectrum} A key question worth answering is whether training for selectivity affects the Lyapunov spectrum's alignment with the parametrised recurrent eigenspectrum. In this regard, Figure~\ref{fig:sel_copy_lyapunov} shows how, irrespective of the task objective or network layer, the Lyapunov spectrum rests within the bounds of $\ln(|\bLb|)$, and thus implicitly meets the prediction of Lemma~\ref{lemma:le_spectrum}. Interestingly, as one would intuitively suspect, the eigenspectra resulting from the Selective Copying tasks have lower minima than those from Copy Memory. Lower eigenvalues yield shorter timescales, allowing quicker decay and disposal of distractors. It should also be noted that while Figure~\ref{fig:sel_copy_lyapunov} appears to show a stronger "squeezing effect" of the Lyapunov spectrum compared to the eigenvalues, it is only so relative to the spread of the spectrum. Specifically, on all layers for both tasks, the absolute difference between the minimum Lyapunov exponent $\alpha_{\min}$ and $\ln(|\bLb|)_{\min}$ is on the order of $\approx 1^{-2}$. 

\begin{figure}[H]
    \centering
        \begin{subfigure}[b]{0.475\textwidth}
            \centering        
            \includegraphics[width=\textwidth]{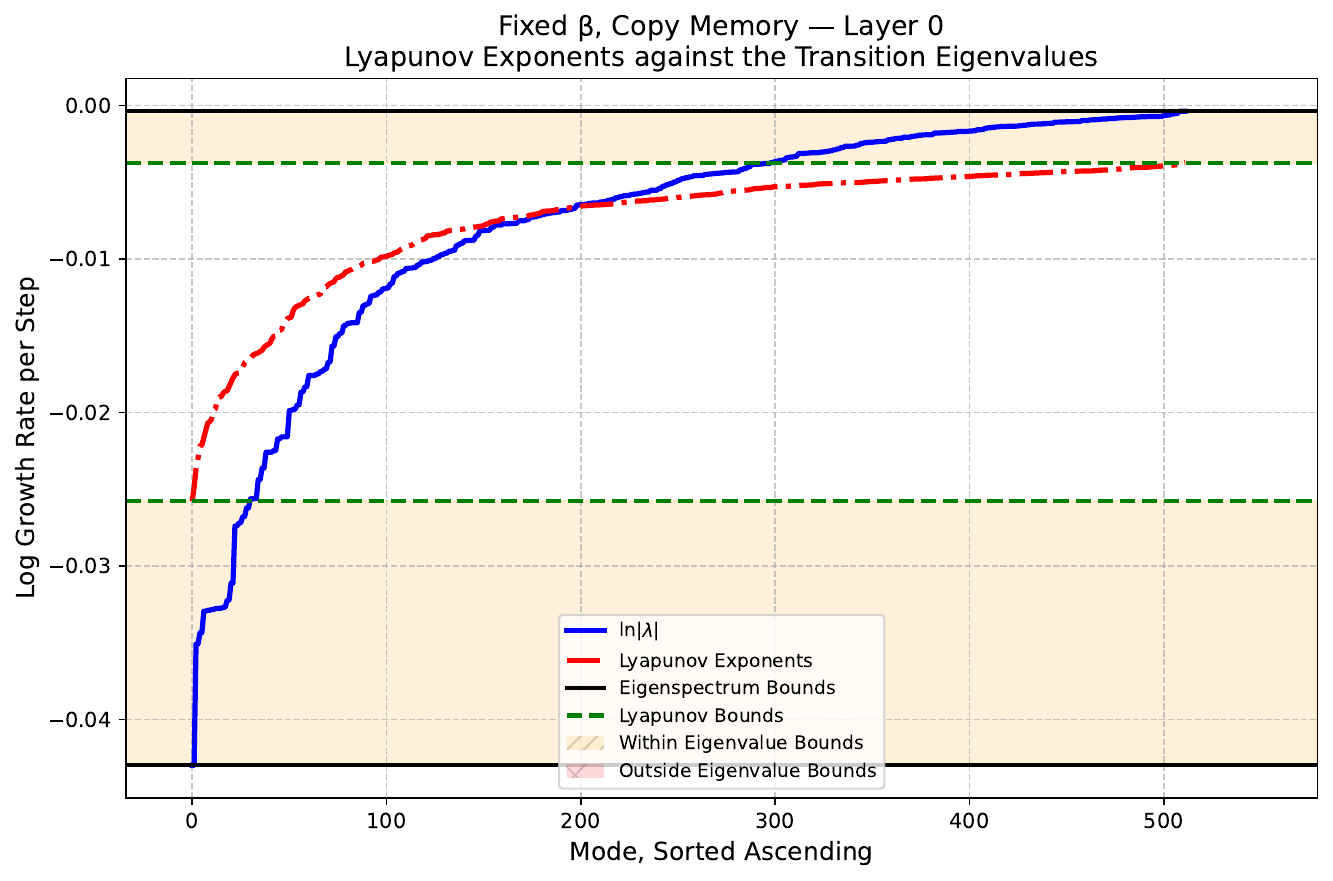}
            \caption{Copy Memory Task - First Layer}
            \label{fig:cm_le_0}
         \end{subfigure}
         \hfill
        \begin{subfigure}[b]{0.475\textwidth}
            \centering        
            \includegraphics[width=\textwidth]{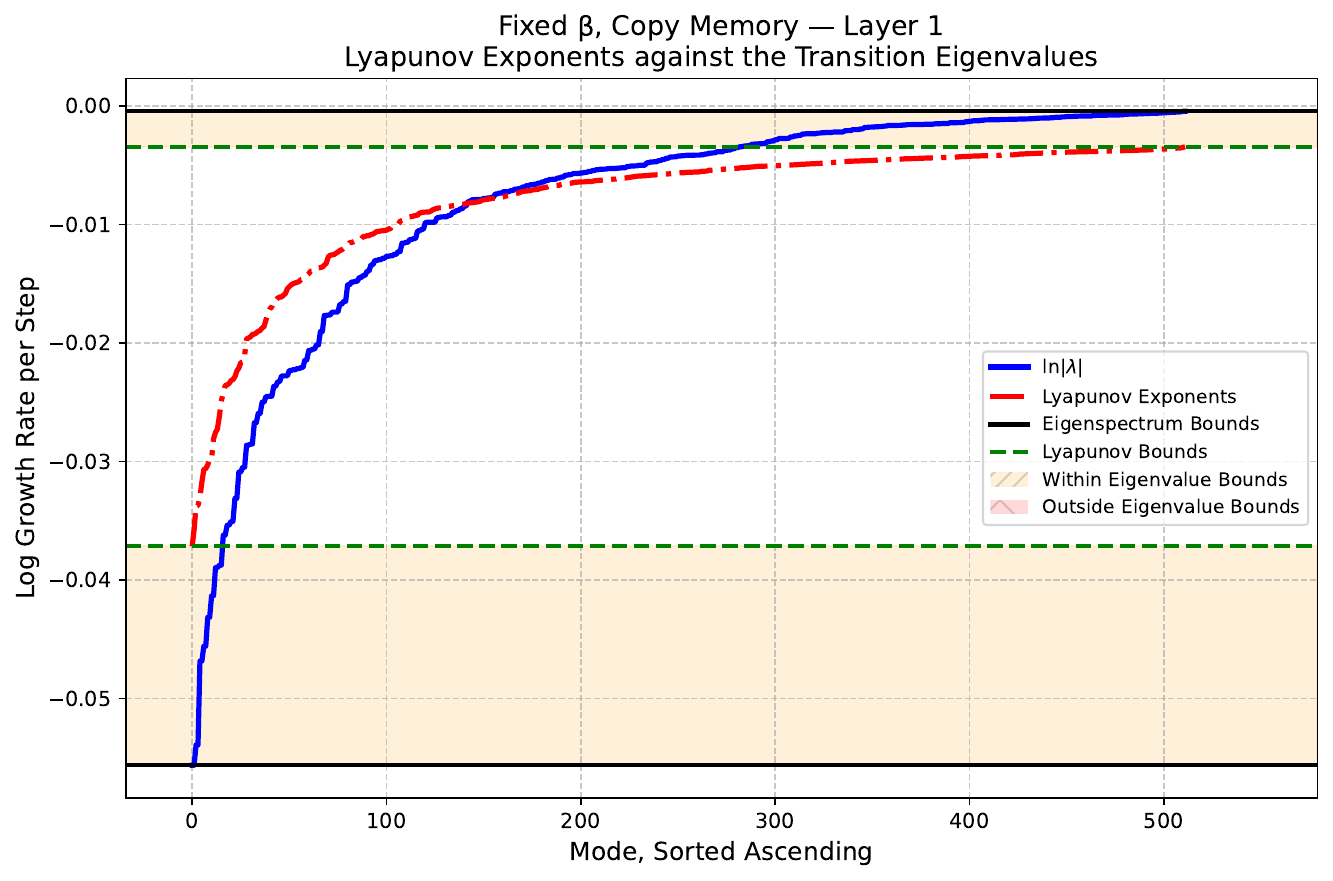}
            \caption{Copy Memory Task - Second Layer}
            \label{fig:cm_le_1}
        \end{subfigure}    
         \hfill
        \begin{subfigure}[b]{0.475\textwidth}
            \centering        
            \includegraphics[width=\textwidth]{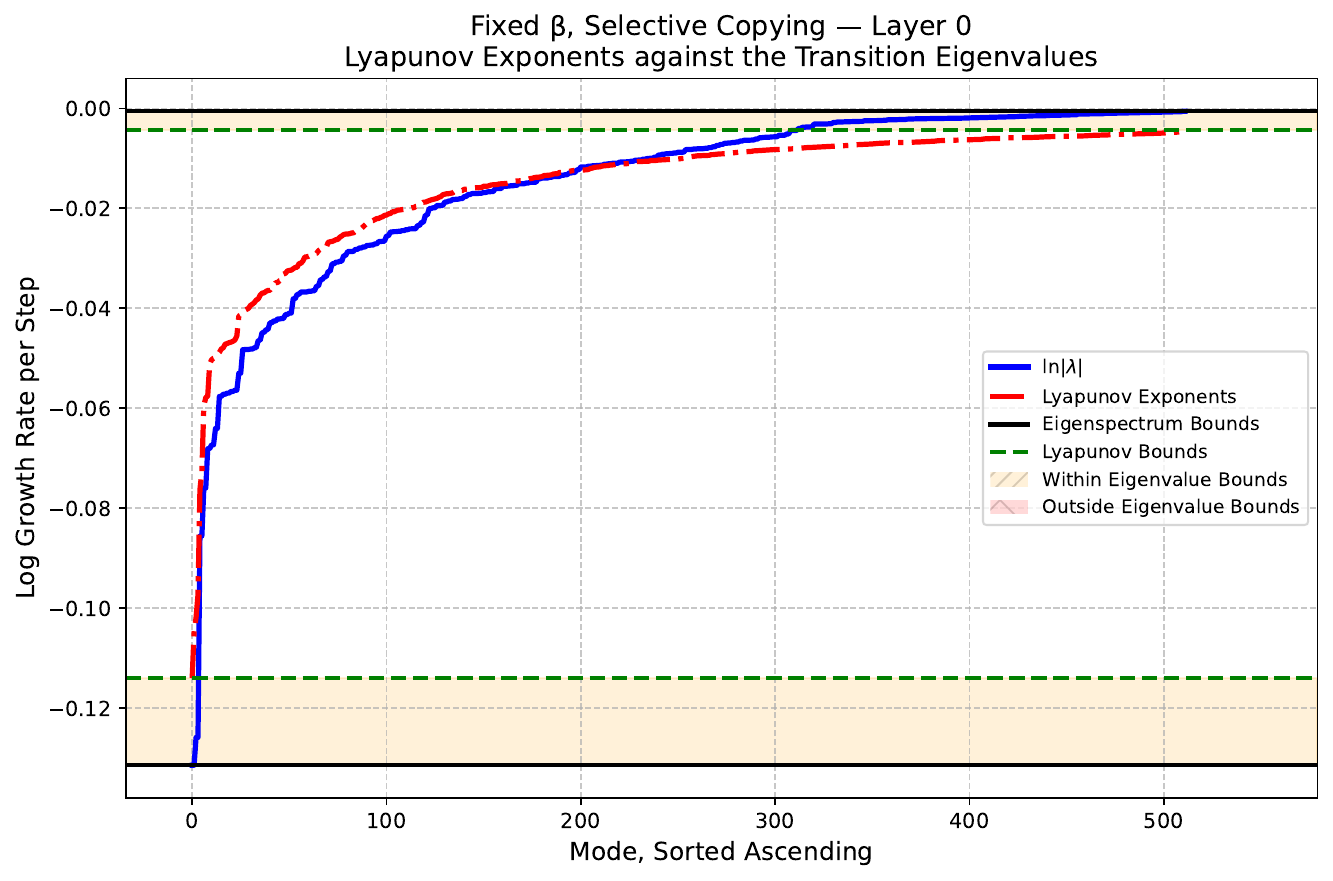}
            \caption{Selective Copy - First Layer}
            \label{fig:sc_le_0}
        \end{subfigure}   
         \hfill
        \begin{subfigure}[b]{0.475\textwidth}
            \centering        
            \includegraphics[width=\textwidth]{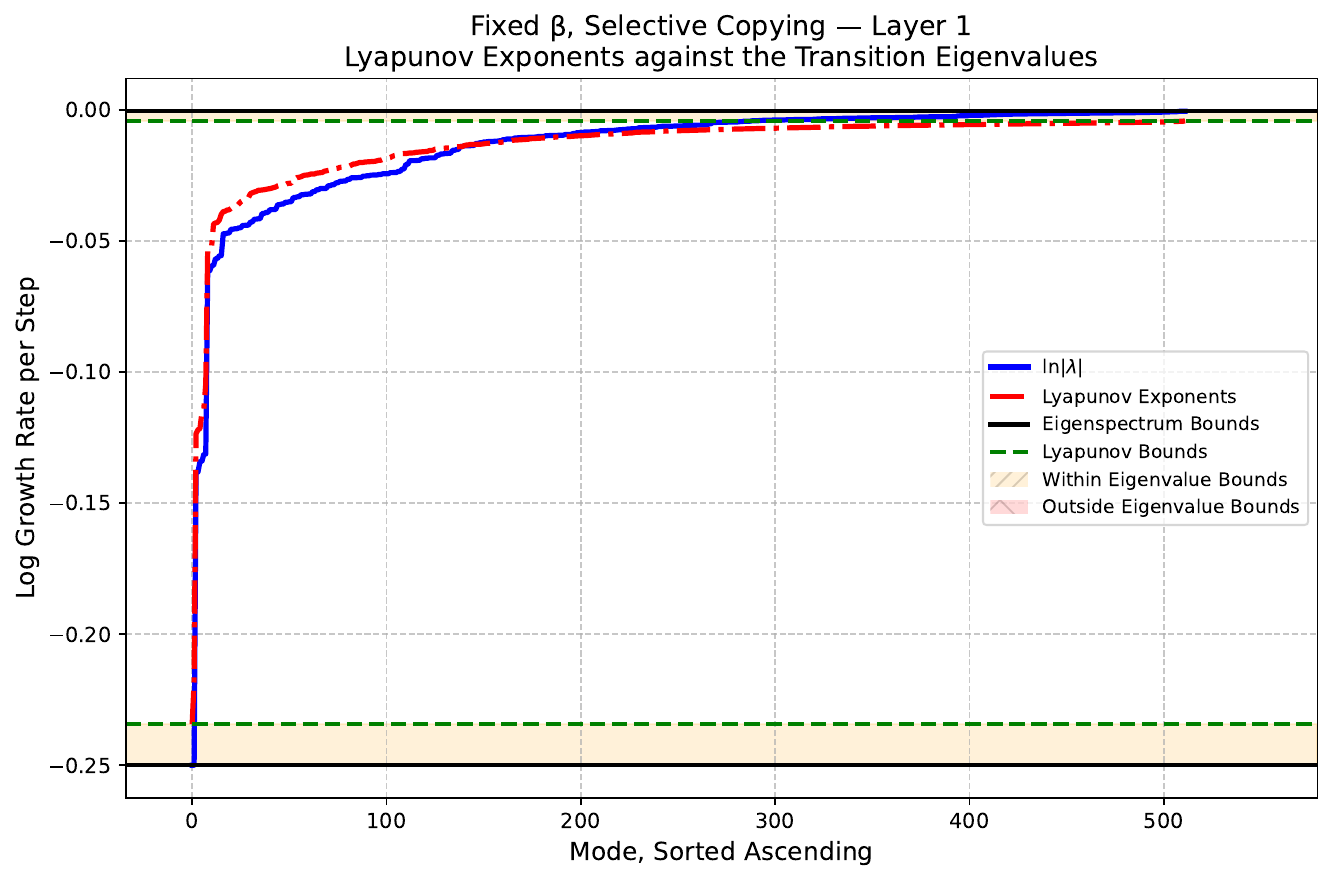}
            \caption{Selective Copy - Second Layer}
            \label{fig:sc_le_1}
        \end{subfigure}        
        \hfill
    \caption{\textbf{Lyapunov Spectrum learning for Matching Copying Tasks}. Each subfigure here shows the Lyapunov spectrum compared to the eigenspectrum for each layer in both networks trained on Selective Copying and Copy Memory tasks with matching training configurations and model hyperparameters (Appendix~\ref{appendix:selective_copy_setup}). Note that the empirical Lyapunov exponents are bounded and closely match $\ln(|\bLb|)$ across all layers examined. To obtain the spectra, a random sample is chosen from both datasets for each corresponding model. }
    \label{fig:sel_copy_lyapunov}
\end{figure}

\textbf{Trainable $\mathbf{\beta}$} Another question is whether, if ADPTNet layers are permitted to include $\beta$ in the training, the learning objective might push it towards making the network more/less diagonal-dominated, non-linear, and selective. Using the same training and model setup from Appendix~\ref{appendix:selective_copy_setup}, networks are now initialised with $\beta=0.0125$ and allowed to train, with clipping applied to keep it within $[0.0001, 0.025]$. Table~\ref{table:trainable_beta_sc} shows that, regardless of task, learning pushes $\beta$ towards its maximum, saturating the upper bound clipping value. 

This suggests that the similarity transforms are being actively used to learn the task, whether it is Selective Copying or Copy Memory. In other words, the models are not simply learning to circumvent the dynamic off-diagonal couplings. Still, Table~\ref  {table:trainable_beta_sc} does also show how including trainable $\beta$ only improves accuracy over the fixed baseline Selective Copying, following the trend of higher accuracies for higher $\beta$  from Table~\ref{table:ADPTNet_sc_results}. 

\begin{table}[H]
    \centering
    \begin{tabular}{c c c c}
        \hline
        Task & Layer & Learned $\beta$ & Acc. Diff.\\
        \midrule
        \multirow{2}{*}{Copy Memory}  & 0 & 0.026 & \multirow{2}{*}{-1.5\%}\\
        & 1 & 0.029 \\
        \midrule
         \multirow{2}{*}{Selective Copy} & 0 & 0.027 & \multirow{2}{*}{+2.4\%}\\
         & 1 & 0.027 \\
       \midrule
    \end{tabular}
    \caption{\textbf{Trained $\mathbf{\beta}$ across Tasks and Layers} Both networks are initialised with $\beta=0.0125$, and it can be observed that training pushes it to saturation of the upper clipping bound $0.025$. The \textit{Acc. Diff.} column shows the change in final accuracy compared to the fixed $\beta$ models used for Figure~\ref{fig:sel_copy_lyapunov}. In absolute terms, fixed $\beta$ resulted in $89.69\%$ accuracy for Copy Memory and $63.70\%$ for Selective Copy. For trainable $\beta$, the accuracies are $88.20\%$ for Copy Memory and $66.07\%$ for Selective Copying.}
    \label{table:trainable_beta_sc}
\end{table}

\textbf{Recurrent Weights Entropy} In this context, normalised Shannon entropy \citep{shannon1948mathematical} $H$ measures the degree of uniformity across the recurrent matrix entries, with $H = 1$ being achieved for all constant entries, and $H=0$ for a single large entry that dominates. It should be mentioned that since the diagonal of $M$ (Eq.~\ref{eq:state_dependent_similarity}) is dominant regardless of task, it is excluded from this computation. Hence, for the off-diagonal entries $|m_i|$ from $M \in \mathbb{C}^{n\times n}$, the discrete probabilities $p_i, i = 1, 2, \dots, t, t = n(n-1)$ and the normalised entropy $H$ are shown in Equation~\ref{eq:entropy}.

\begin{equation}\label{eq:entropy}
    \begin{split}
         p_i &= \frac{|m_i| ^2}{\sum_j^t|m_j|^2} \\
         H &= -\frac{\sum_i^tp_i\log(p_i)}{\log (t)}
    \end{split}
\end{equation}

Figures~\ref{fig:materialise_matrices_layer0} and ~\ref{fig:materialise_matrices_layer1} show how the Copy Memory and Selective Copy tasks shape the data-dependent ADPTNet recurrent matrices across depth and training epochs. In this regard, the most striking effect observable is the higher degree of uniformity emerging from training on Selective Copy. On both layers, Selective Copying engenders higher off-diagonal entropies compared to Copy Memory at the end of training, ($0.77 > 0.66$ for the first layer and $0.71> 0.61$ for the second layer). Visually, Figure~\ref{fig:materialise_matrices_layer0} in particular shows the largest discrepancy emerging from training on the two tasks, with the Selective Copying-induced first layer dynamics evidently showing a denser "all-to-all" coupling pattern between recurrent neurons, compared to Copy Memory. Since the weights are data-dependent, slight differences are also visible at initialisation. Namely, while the overall patterns largely match layer-wise between the two tasks, Selective Copy displays persistently higher amplitudes in both Figures~\ref{fig:materialise_matrices_layer0} and \ref{fig:materialise_matrices_layer1}. 

These observations, combined with the more quickly decaying time-scales present for Selective Copy (Figure~\ref{fig:sel_copy_lyapunov}), point towards the networks effectively learning to not only create "forget" channels, but also learning to move distractor information to those channels for disposal.  In other words, some neurons learn to have quickly fading memory, and the network dynamically assigns information to those neurons to induce selective forgetting via changing connectivity. As a concrete example, on Copy Memory, the minimum log-eigenvalue is $\ln(\bar\lambda_{\min})\approx -0.04$ and the sample recurrent matrix entropy is $\approx 0.61$ for the second layer of the network. By contrast, for Selective Copying, the second layer has $\ln(\bar\lambda_{\min})\approx -0.25$ and a sample recurrent matrix entropy of $\approx 0.71$. Therefore, Selective Copying induces higher entropy (more uniform connectivity) and a lower minimum timescale (faster decay), consistent with the hypothesis of the network routing information to "forget" neurons. 

\begin{figure}[H]
    \centering
        \begin{subfigure}[b]{\textwidth}
            \centering        
            \includegraphics[width=\textwidth]{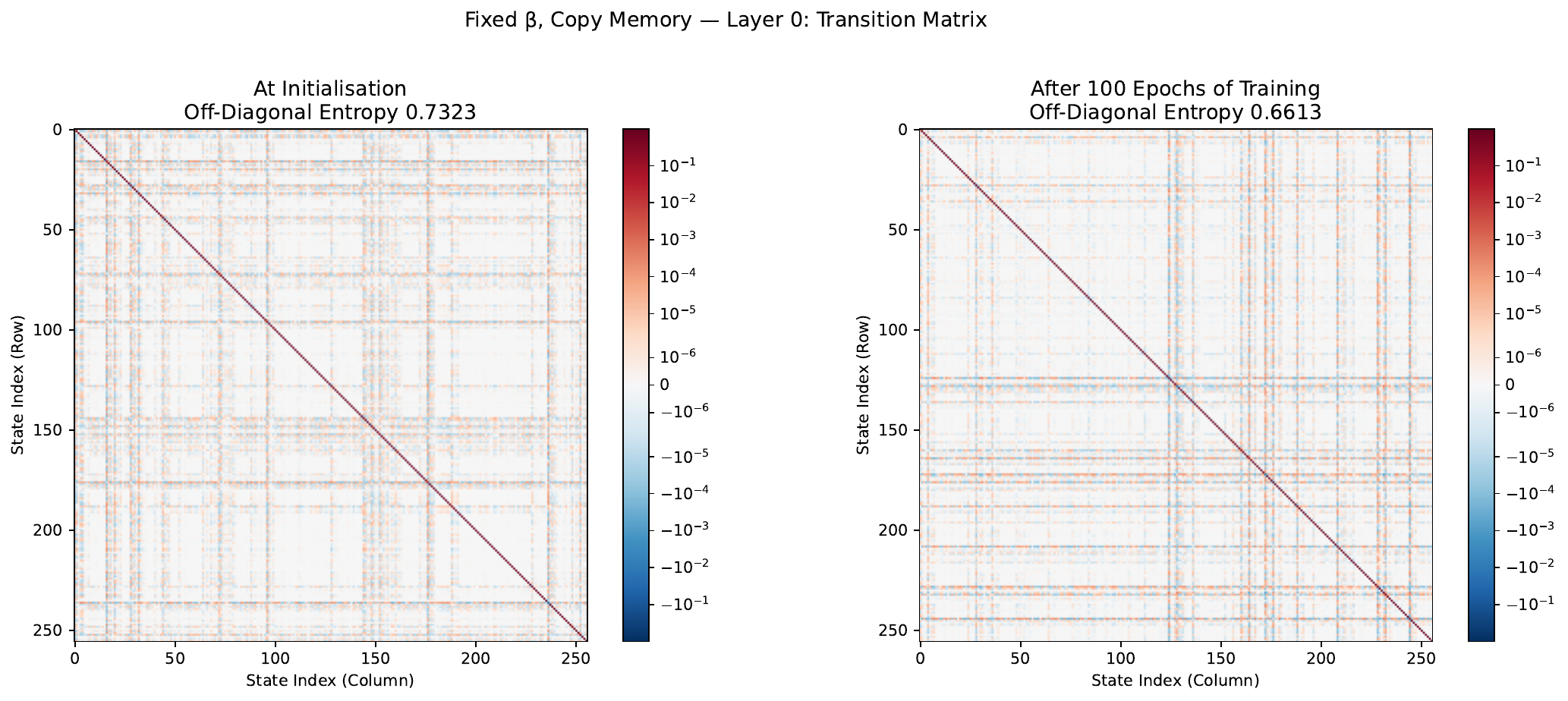}
            \caption{Copy Memory - First Layer}
            \label{fig:cm_matrix_layer0}
         \end{subfigure}
         \hfill
        \begin{subfigure}[b]{\textwidth}
            \centering        
            \includegraphics[width=\textwidth]{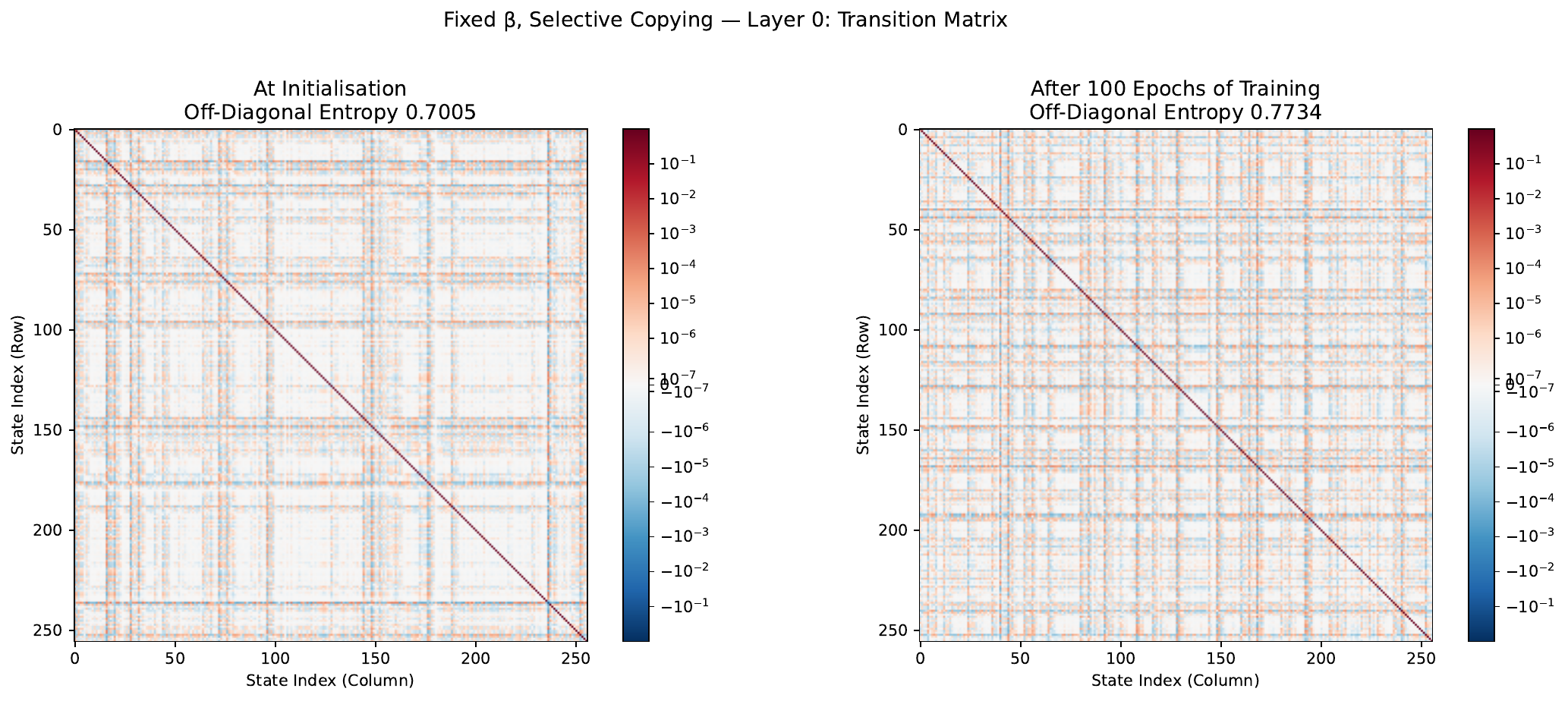}
            \caption{Selective Copy - First Layer}
            \label{fig:sc_matrix_layer0}
        \end{subfigure}    
         \hfill
    \caption{\textbf{Recurrent Matrix $\mathbf{M}$ Materialisation during Training - First Layers}. Figures~\ref{fig:materialise_matrices_layer0} and \ref{fig:materialise_matrices_layer1} show how the data-dependent recurrent matrices $M$ adapt depending on the task, Copy Memory or Selective Copy. Here, the same models and training experiments are used as in Fig.~\ref{fig:sel_copy_lyapunov}. On the left, the matrices are extracted at initialisation (before training) from their respective networks, while on the right the matrices are materialised at the end of training. Each entry in the matrices is the coupling strength between two recurrent neurons, so each matrix also implicitly represents a network topology. The main differences pertain to how evenly-distributed the off-diagonal entries are, as quantified by the normalised Shannon entropy (Eq.~\ref{eq:entropy}). It can be observed that Selective Copying leads to higher overall entropy, and thus more evenly distributed all-to-all recurrent connectivity.  }
    \label{fig:materialise_matrices_layer0}
\end{figure}

\begin{figure}[H]
    \centering
        \begin{subfigure}[b]{\textwidth}
            \centering        
            \includegraphics[width=\textwidth]{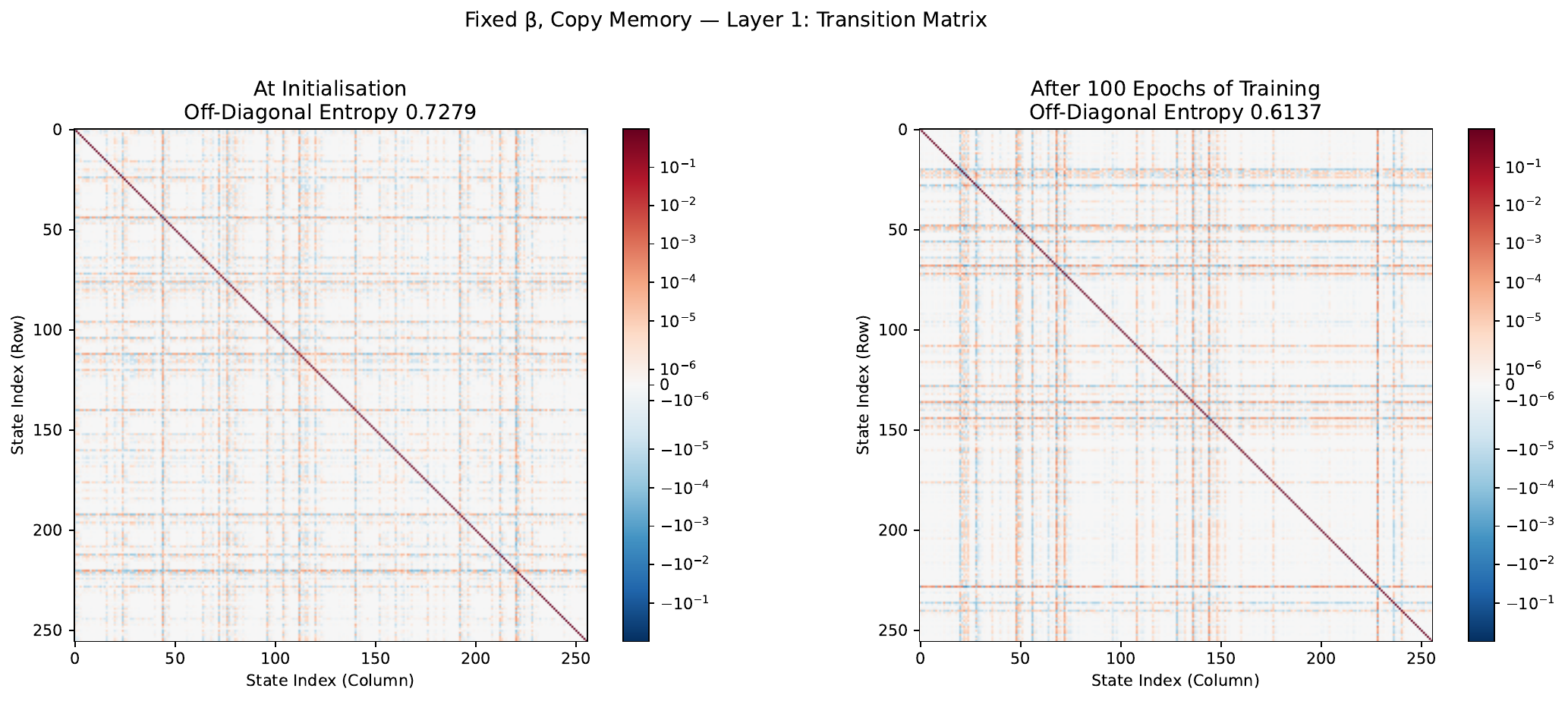}
            \caption{Copy Memory - Second Layer}
            \label{fig:cm_matrix_layer1}
        \end{subfigure}   
         \hfill
        \begin{subfigure}[b]{\textwidth}
            \centering        
            \includegraphics[width=\textwidth]{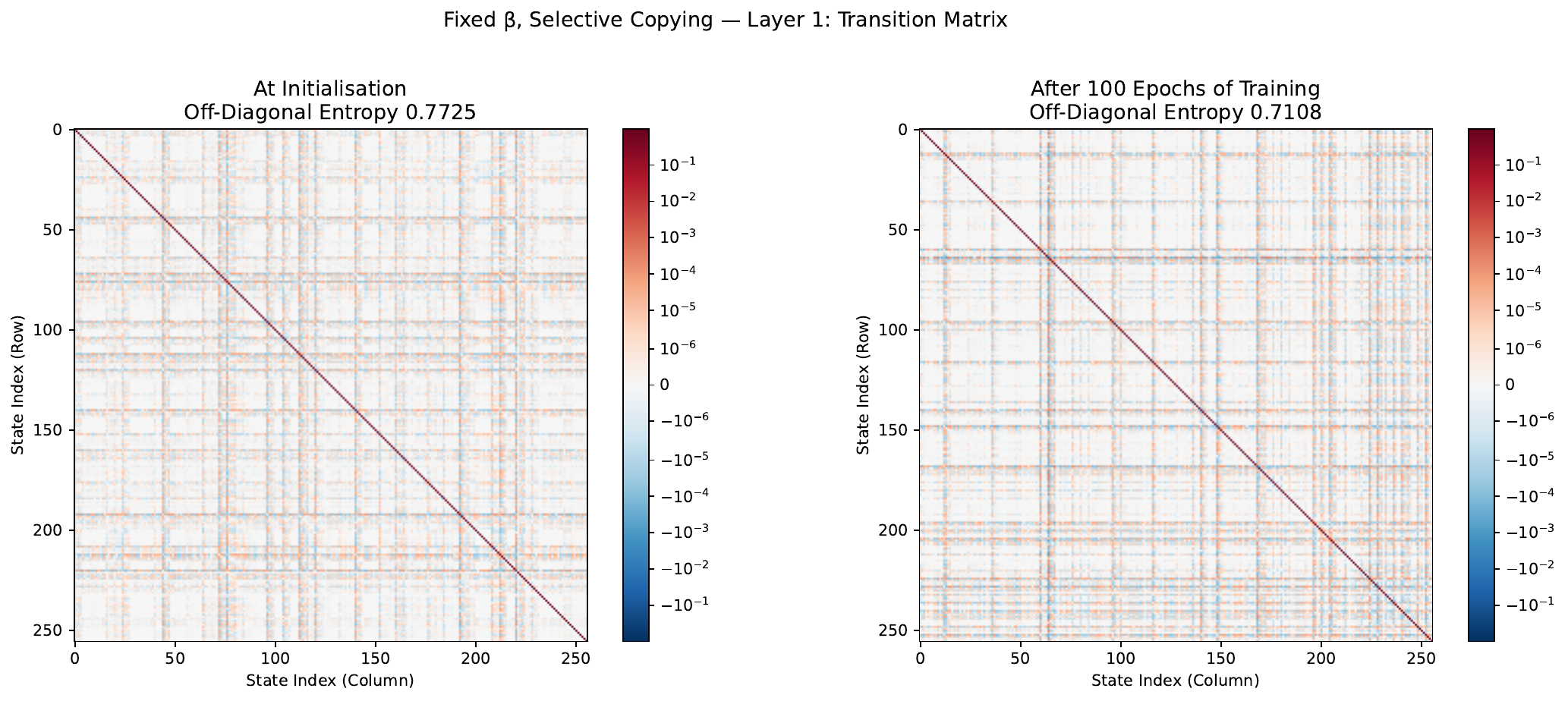}
            \caption{Selective Copy - Second Layer}
            \label{fig:sc_matrix_layer1}
        \end{subfigure}        
        \hfill
    \caption{\textbf{Recurrent Matrix $\mathbf{M}$ Materialisation during Training - Second Layers}. This figure accompanies Figure~\ref{fig:materialise_matrices_layer0}, showing the effect of training on different tasks on the materialised recurrent matrix $M$. While the effect on the second layers is less dramatic than on the first, there is still a visiable difference in entropy, with more evenly spread out connectivity for Selective Copying than Copy Memory.}
    \label{fig:materialise_matrices_layer1}
\end{figure}

\textbf{DEER Iterations} Figure~\ref{fig:sel_copy_deer} shows how Conv-DEER iterations required for convergence change after training and between tasks. Firstly, the same trend from Figure~\ref{fig:trained_damping} is also visible for both Copy Memory and Selective Copying. Namely, convergence slows down as training progresses. Furthermore, as seen in Figure~\ref{fig:sel_copy_lyapunov}, eigenvalue spread is higher in the deeper layers, and, accordingly, DEER convergence is slower on the second layers in all experimental conditions. Corroborating Figure~\ref{fig:eigenvalue_dt_min_max}, on all layers for both tasks, it can be seen that pushing $\beta$ to the $0.025$ ceiling after training invariably increases Newton iterations. 

Interestingly, however, undermining the trends set out in Section~\ref{sec:deer_ablation_results}, Figure~\ref{fig:sel_copy_deer} shows how the convergence for Copy Memory is slightly slower than that of Selective Copying. This is despite the fact that, as highlighted in Figure~\ref{fig:sel_copy_lyapunov}, Selective Copy induces a higher eigenvalue spread. One could hypothesise that higher off-diagonal entropy partially mitigates the effect of higher spread by effectively smoothing information across recurrent neurons. While $\beta$ remains a reliable predictor of DEER convergence, perhaps eigenvalue spread also requires consideration of effective recurrent connectivity to predict DEER convergence.  

\begin{figure}[H] 
    \centering
        \centering        
        \includegraphics[width=0.75\textwidth]{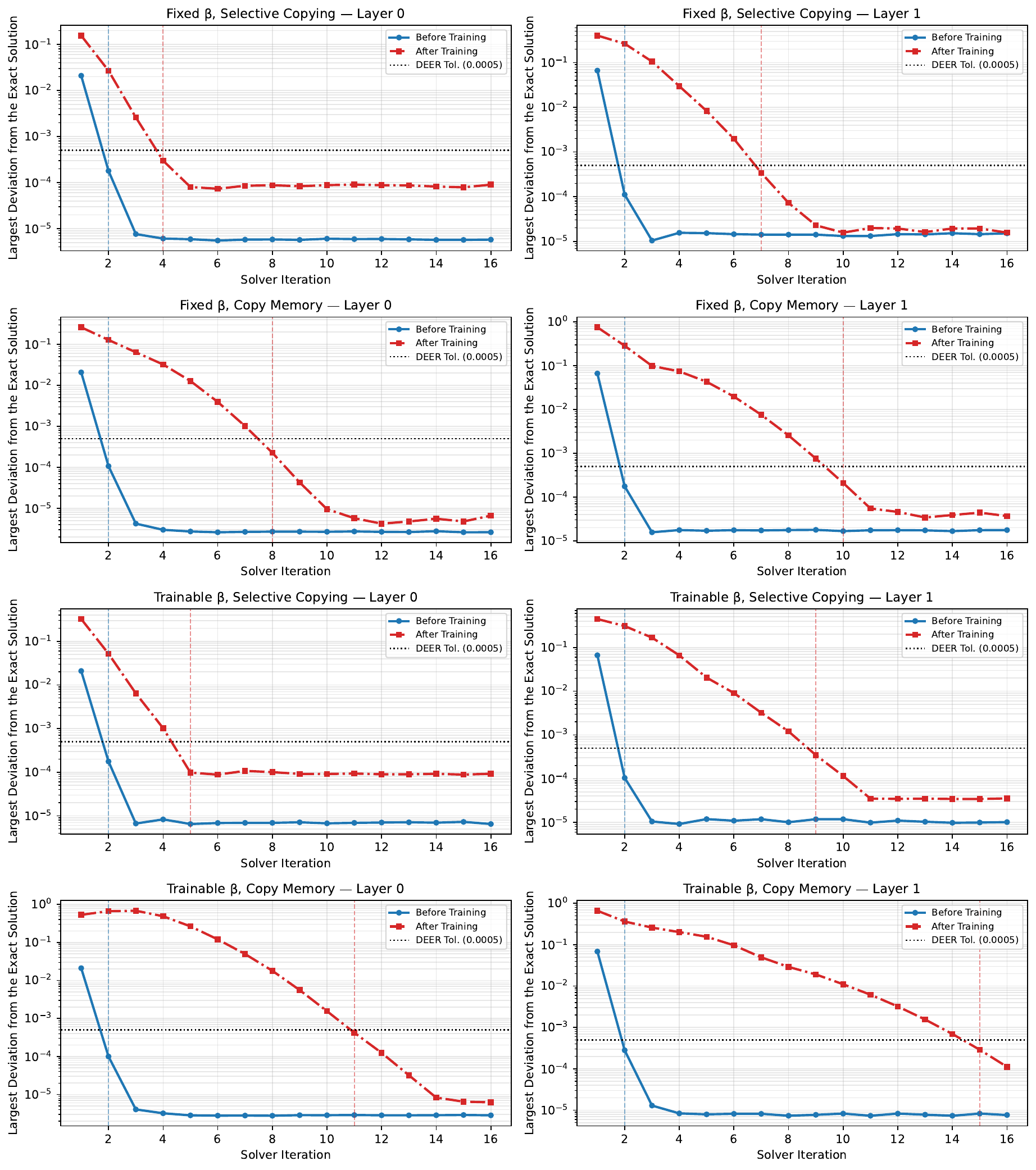}
        \hfill
        \caption{\textbf{Training Effect on DEER Convergence}. This figure uses the models from Table~\ref{table:trainable_beta_sc} and Figure~\ref{fig:sel_copy_lyapunov}, i.e., the trained and fixed $\beta$ models trained on either Selective Copying or Copy Memory. Using Conv-DEER, it can be seen that deeper layers are slower to converge across all settings. Notably, Selective Copy appear to converge more quickly than Copy Memory, despite the higher spread eigenspectrum (see Fig.~\ref{fig:sel_copy_lyapunov}).}
        \label{fig:sel_copy_deer}
\end{figure}

\subsection{Measuring Loose Coupling Expressivity on Mechanistic Architecture Design}  \label{sec:mad_results}

The key question being answered in this section is whether a small $\beta =0.00125$, which induces loose perturbation-like coupling, is expressive enough to compete with existing Transformer efficient alternatives. For this purpose, this study uses the Mechanistic Architecture Design (MAD) benchmark \citep{poli2024mechanistic}. 

The MAD benchmark is a synthetic test suite designed to predict large-scale language modelling capabilities from small token manipulation tasks. This enables identification of how specific model characteristics compare with established sequence architectures, including self-attention and efficient alternatives such as SSMs. Furthermore, all tasks within MAD use a fixed model dimension, standardised training hyperparameter sweeps, and the same task difficulty settings, maximising comparability across architectures. Thus, in this work, MAD serves as a proxy for validating ADPTNet adaptability, as required for language modelling, against prior methods. 

One of the tasks in MAD is the now familiar Selective Copying. The overall MAD score for Selective Copying is computed by varying its difficulty through increased numbers of distractors, larger vocabulary sizes, more target tokens to store and smaller training dataset sizes (see Appendix~\ref{appendix:mad_setup} for details). 

Another task is \textit{Compression}, which tests a network's ability to compress the concatenated input token sequence into a fixed-sized context (e.g., $\mathtt{input}: [\text{a, b, c, d, e}] \rightarrow \mathtt{output}: \text{cat(abcde)} $). The basic premise is reminiscent of the Long-Range Arena (LRA) benchmark \citep{tay2020long}, which tests a network's ability to model long-range dependencies and implicitly requires a degree of context compression. However, while some of the LRA's tasks comprised tokens simply consisting of greyscale pixels, Compression synthetically and systematically constructs larger and more diverse vocabularies, requiring a higher degree of expressivity from the network. As with Selective Copying, the details of the difficulty tunings for Compression are in Appendix~\ref{appendix:mad_setup}. 

The final MAD task employed here is \textit{Memorisation}, which tests key-value associative retrieval of factual memories.  This differs from in-context recall tasks because it only requires learning fixed key-value pairs of facts from the training data as ground truth and answering prompts at test time from the same fact collection (e.g., \texttt{training set (key, value) pairs}: [(a, b), (c, d), (e, f), $\dots$], \texttt{test prompt}: [$a\rightarrow?$, $c\rightarrow?$]). Memorisation can be considered a precursor to LLMs storing facts such as the \textit{(Key: Capital of Romania, Value: Bucharest)}. Once again, MAD averages over several Memorisation difficulty settings, including the number of facts to store and training dataset size (Appendix~\ref{appendix:mad_setup}). It should be noted that, while they do test computational primitives required for language modelling, neither Memorisation nor Compression requires in-context adaptation to the same degree as Selective Copying. Furthermore, key-value associative fact knowledge is typically stored within position-wise Feed-Forward Networks (FFN) / MLPs  in LLMs \citep{geva2021transformer}. As such, the Memorisation task is likely to be a stronger reflection of those components, rather than the sequence-mixing layers such as ADPTNet. 

Table~\ref{table:mad_results} shows how ADPTNet within the perturbation regime ($\beta=0.00125$) compares against existing MAD baselines. Notably, the only architectures with vector states in the table are ADPTNet and Hawk \citep{de2024griffinmixinggatedlinear}. The rest not only use matrix-valued states, but also expand the recurrent dimension to $\in \mathbb{R}^{2048\times 2048}$, much larger than $\mathbb{C}^{512}$  for ADPTNet. Despite this, for the Compression task, ADPTNet outperforms all reported matrix-valued architectures and is outperformed only by the Transformer and Hawk. In contrast, Memorisation appears more challenging for the tested ADPTNet configuration, as it only outperforms Mamba2 \citep{dao2024transformers} and DeltaNet. 

Before examining Selective Copying performance, it is important to establish in more detail how Hawk relates to ADPTNet. As proposed  in \citet{de2024griffinmixinggatedlinear}, Hawk relies on a selective Real-Gated LRU (RG-LRU) sequence-mixing layer: 

\begin{equation} \label{eq:hawk_rec}
    h_t = a_t \odot h_{t-1} + \sqrt{1 - a_t^2} \odot (i_t \odot x_t)
\end{equation}

Where $a_t = a^{cr_t}$, and the input-driven gates are $r_t = \sigma(W_r x_t + b_r)$, and $i_t = \sigma(W_i x_t + b_i)$, where $\sigma$ denotes a \texttt{softplus} activation, $a$ is parametrically constrained $0 \le a \le 1$, and $c$ is constant. As argued in \citet{de2024griffinmixinggatedlinear}, this recurrence rule differs from other selective SSMs such as Mamba because it cannot arbitrarily discard information. More precisely, in Mamba, each new state is obtained by interpolating between new inputs and the previous state, so there is no lower bound on the information that can be discarded at any given step. In other words, the network can decide to completely override its memory with the new input. By contrast, RG-LRU interpolates between its past state and an LRU update. Information is either added or not, but nothing is forgotten faster than the decay rates set by the network's fixed eigenspectrum through $a$. In principle, this constitutes a direct alternative linear solution to ADPTNet's stated goal of controlling the effective long-term dynamics of the recurrence, while enabling adaptability. 

Therefore, as highlighted in Table~\ref{table:mad_results}, for Selective Copy, the task that tests adaptability most directly among the MAD tasks included here, it is important to note that ADPTNet, at a perturbation level $\beta=0.00125$, manages to outperform Hawk ($80.9$ vs $77.0$). This provides evidence that the topological conjugate adaptability implemented in ADPTNet outperforms its closest equivalent in existing literature. Furthermore, it does so using weak perturbative coupling. 

\begin{table}[H]
    \centering
    \begin{tabular}{c c c c}
        \hline
        Model & Memorisation & Selective Copy & Compression \\
        \midrule
        Mamba2 \citep{dao2024transformers} & 42.0 & 95.4 & 41.3\\
        GLA \citep{yang2023gated} & 82.5 & 96.1 & 42.3 \\
        xLSTM \citep{beck2024xlstm} & 79.8 & 95.4 & 43.4\\
        DeltaNet \citep{yang2024parallelizing} & 40.8 & 98.8 & 43.3\\
        Gated DeltaNet\citep{yang2025gated} & 81.7 & 95.7 & 45.0 \\
        MesaNet \citep{von2025mesanet} & 77.2 & 99.2 & 45.4 \\
        Transformer \citep{von2025mesanet} & 84.7 & 96.0 & 49.5 \\
       \midrule
         \textbf{Hawk} \citep{de2024griffinmixinggatedlinear} & \textbf{91.3} & \textbf{77.0} & \textbf{47.7} \\
       \textbf{ADPTNet} ($\beta=0.00125$) & \textbf{46.1} & \textbf{80.9} &  \textbf{45.9} \\
       \midrule
    \end{tabular}
    \caption{\textbf{Accuracies on MAD Tasks} The ADPTNet network tested here is situated in a perturbation regime, where off-diagonal couplings are scaled by $\beta=0.00125$. The baseline results included in this table are reported from \citet{von2025mesanet}. GLA stands for Gated Linear Attention \citep{yang2023gated}. All networks besides Hawk, the Transformer, and ADPTNet, compress information in a matrix-valued state. The Transformer evidently keeps the entire sequence as its state, while ADPTNet and Hawk both have vector-valued states.  The most notable results are that ADPTNet performs competitively against models with much larger states on Compression, while outperforming its closest equivalent alternative, Hawk, on Selective Copying.}
    \label{table:mad_results}
\end{table}

\subsection{Measuring Long-Range Temporal Modelling with Sequential-CIFAR10} \label{sec:scifar}

Sections~\ref{sec:state_tracking_results}, \ref{sec:selective_copy_results}, and \ref{sec:mad_results} showed how the data-dependent similarity transforms within ADPTNet recurrence can enable state tracking and selectivity in SSMs. The key question, then, is whether ADPTNet preserves the desirable long-range modelling qualities of LTI SSMs. For instance, Mamba achieves a markedly lower accuracy than S4 on the LRA~\citep{yu2026block}. Could the topological conjugation operation ruin the coherence of long-range features? 

To check this, this study uses Sequential CIFAR-10 (sCIFAR), a well-established benchmark for modelling long-range temporal dependencies \citep{krizhevsky2009learning,hafner2017learning, gu2021efficiently, tay2020long}. It consists of $32 \times 32$ images from the CIFAR-10 dataset being converted into $L=1024$-long 1-dimensional temporal signals by processing each pixel at a time. The sequences are produced by concatenating each image's rows together. 

Tables~\ref{table:linear_scifar_results} and\ref{table:non-linear_scifar_results} show how ADPTNet, with a $\beta=0.0125$ which improves performance on both Selective Copying and state tracking (see Sections~\ref{sec:state_tracking_results} and \ref{sec:selective_copy_results}), compares to the baseline $\beta=0$ as well as existing non-linear recurrent or selective/adaptive architectures. Firstly, Table~\ref{table:non-linear_scifar_results} shows how state-driven non-linear ADPTNet does not degrade performance compared to the same baseline model with $\beta=0$, and actually marginally improves accuracy. It should also be noted that, as mentioned in Section~\ref{sec:eigvalue_param}, ADPTNet is only theoretically guaranteed to preserve, within given bounds, the decay rates encoded in the eigenvalue spectrum. Because the baseline $\beta=0$ and the non-linear $\beta=0.0125$ models both use S4D-Inv complex initialisation and parametrisation, the slight improvement in accuracy suggests that ADPTNet preserves the imaginary/oscillatory components of the recurrent eigenspectrum to some extent in its long-term dynamics. Compared to existing non-linear RNNs, Table~\ref{table:non-linear_scifar_results} also shows how small 2-layer non-linear ADPTNet is capable of outperforming them. 

Secondly, Table~\ref{table:linear_scifar_results} shows how a larger 4-layer input-driven LinearADPTNet outperforms current selective/adaptive architectures. To the best of the authors' knowledge, this is the highest recorded accuracy on this particular task (RGB~colour~sCIFAR) for a selective model. Moreover, the baselines in Table~\ref{table:linear_scifar_results} use deeper networks (6 layers) than both ADPTNet networks tested, yet even the smaller 2-layer non-linear ADPTNet configuration outperforms them. Nevertheless, the main takeaway here is that ADPTNet retains baseline LTI SSM performance for long-range sequence modelling while also improving state tracking and Selective Copy. The topological conjugation data-dependent adjustments improve adaptability while not damaging long-range features. 

\begin{table}[H]
    \centering
    \renewcommand{\arraystretch}{1.5}
    \begin{subtable}[t]{0.46\textwidth}
        \centering
        \begin{tabular}{c c}
            \hline
            Model & Acc (\%) \\
            \hline
            Llama  \citep{touvron2023llama} & 0.629 \\
            Mamba \citep{gu2023mamba} & 0.765 \\
            RWKV-4 \citep{peng2023rwkv} & 0.757 \\
            LSTM & 0.756 \\
            xLSTM & 0.761 \\
            \hline
            ADPTNet ($\beta=0.0125$ & \textbf{81.34} \\
        \hline
        \end{tabular}
        \caption{\textbf{LinearADPTNet - Four Layers}}
        \label{table:linear_scifar_results}
    \end{subtable}
    \hfill
    \begin{subtable}[t]{0.46\textwidth}
        \centering
        \begin{tabular}{c c}
        \hline
        Model & Acc (\%) \\
        \hline
        Transformer \citep{trinh2018learning} & 62.2\\
        CKConv \citep{romero2021ckconv} & 63.74\\
        TrellisNet \citep{bai2018trellis} & 73.42\\
        r-LSTM \citep{trinh2018learning} & 72.2\\
        UR-LSTM \citep{gu2020improving} & 71.00\\
        UR-GRU \citep{gu2020improving} & 74.4 \\
        HiPPO-RNN \citep{gu2020hippo} & 61.1\\
        LipschitzRNN  \citep{erichson2020lipschitz} & 64.2\\
        \hline
        Ablation S4D-Inv ($\beta=0$) & 76.39\\
        ADPTNet ($\beta=0.0125$) & \textbf{76.62}\\
        \hline
        \end{tabular}
        \caption{\textbf{Non-Linear ADPTNet - Two Layers}}
        \label{table:non-linear_scifar_results}
    \end{subtable}
    \caption{\textbf{sCifar Accuracy} Details of the experimental setup are available in Appendix~\ref{appendix:cifar_setup}. The baseline accuracies in Table~\ref{table:linear_scifar_results} are reported from \citet{beck2024xlstm}, while Table~\ref{table:non-linear_scifar_results} collects baseline results from \citet{gu2021efficiently}. It should be observed that both the 2-layer non-linear ADPTNet and the 4-layer LinearADPTNet outperform all of the larger selective architectures in Table~\ref{table:linear_scifar_results}. Most importantly, however, the non-linear ADPTNet configuration matches and even slightly outperforms the baseline S4D-Inv model ($\beta=0$) of the same size. }
    \label{table:ADPTNet_sc_results}
\end{table}

\subsection{Neuromorphic Speech Processing with Spiking Speech Commands} \label{sec:ssc}

As mentioned in Section~\ref{sec:intro}, the neuroscientific inspiration for ADPTNet stems from the auditory cortex's input-invariant time scales. As argued before, this implies that non-linear neuronal activity, synaptic modulation, and plasticity do not affect the temporal scales of population-level dynamics as captured by localised electrodes. Moreover, the same section detailed that the ultimate goal of ADPTNet is to provide a viable energy-efficient alternative to the Transformer, which could be achieved through an ADPTNet-based SNN (Section~\ref{sec:spiking_ADPTNet}). Therefore, this section addresses whether SpikingADPTNet can play a role in neuromorphic auditory processing. The de facto standard tasks in this regard are the Spiking Heidelberg Digits (SHD) and Spiking Speech Commands (SSC) \citep{cramer2020heidelberg}. Both datasets are constructed by passing spoken language recordings through a biologically plausible artificial cochlea model. For SSC, the original recordings consist of the Google Speech Commands dataset, which is also prevalent in mainstream sequence processing \citep{warden2018speech, gu2021efficiently}. It should be noted that the SHD dataset is smaller, and state-of-the-art models have recently approached saturation near 100\% accuracy \citep{sun2025towards}. Therefore, this work focuses only on the SSC task. 

The SpikingADPTNet setup used in this section is effectively a drop-in replacement for the LIF neurons within SiLIF \citep{fabre2025structured}. The training, hyperparameters, network topology, and eigenvalue initialisation are held constant with the experiments in \citet{fabre2025structured}. This is achieved by adding spiking activations to CoupledADPTNet (see Section~\ref{sec:selective_ADPTNet}, which mirrors the AdLIF backbone of SiLIF (see Section~\ref{sec:long_range_snn}). Both architectures are non-linear spiking units coupled with a slower-evolving "adaptation" linear state. To keep the parameter count comparable, LowParam linear layers are deployed within the input, $K$, and $V$ projections. Full details of the training setup are included in Appendix~\ref{appendix:ssc_setup}. Two SpikingADPTNet configurations are tested: one setting has its parameter count and recurrent state dimension matched to SiLIF, while the other expands the recurrent state $\times 4$ to match the size of the state-of-the-art spiking Transformer baseline, the SpikCommander \citep{wang2026spikcommander}. 

Table~\ref{table:ssc_results} shows the performance of SpikingADPTNet compared to the current state-of-the-art on SSC\footnote{Official ranking is available at https://zenkelab.org/resources/spiking-heidelberg-datasets-shd/.}. The first important takeaway is that the parameter-matched SpikingADPTNet model outperforms the baseline SiLIF result ($82.5 > 82.03\pm 0.25$) and bridges the gap to the current state-of-the-art, DelRec \citep{queant2025delrec}. Although only one seed is tested for the small SpikingADPTNet configuration, it lands slightly higher than DelRec's mean accuracy ($82.42$). Scaling up, the larger SpikingADPTNet configuration sets a new state-of-the-art accuracy on the SSC dataset, averaging $83.56$ over 4 random seeds, more than $1$ point higher than the previous state-of-the-art. Improvements in accuracy with higher parameter counts are not trivial on SSC, as Table~\ref{table:ssc_results} shows several architectures within the $\approx 1M$ parameter range still outperformed by the methods proposed here. Moreover, as reported in \citep{queant2025delrec}, DelRec is based on delay learning, which requires buffering past outputs, even as far back as 59 time steps. Therefore, while the recurrent state size is set to $256$ (Table~\ref{table:ssc_results}), the effective memory footprint is closer to the large ADPTNet configuration, strengthening the comparability of the two models. 

It should be mentioned that the accuracies reported in Table~\ref{table:ssc_results} contain only models trained on SSC without data augmentation. The highest accuracy recorded on SSC in \citet{schone2024scalable} relies on several spike-data augmentation techniques and does not use spiking activations within the model itself, leading to $> 88\%$ accuracy. Similarly, SpikCommander, proposed in \citet{wang2026spikcommander}, a Transformer-based spiking architecture, reaches at most $85.98\%$, using spike-dropping data augmentation. Interestingly, the 1.1M SpikCommander configuration reaches $83.26\%$, which, even with data augmentation, is below the comparably-sized SpikingADPTNet proposed here.

\begin{table}[H]
    \centering
    \begin{tabular}{c c c c c}
        \hline
        Model & Acc(\%) & State Size & N. Layers & N. Param. \\
        \midrule
        RSNN \citep{cramer2020heidelberg} & $51.1 \pm1.1$ & -- & -- & -- \\
        Adaptive RSNN \citep{yin2021accurate} & $74.2$ & -- & -- & --\\
        EventProp \citep{meszaros2025efficient} & $76.1\pm 1.0$ & -- & -- & --\\
        RSNN with Adaptation \citep{bittar2022surrogate} & $77.4$ & -- & -- & -- \\
        d-cAdLIF \citep{deckers2024co} & $80.23\pm0.07$ &-- & -- &  0.35M \\
        SE-adLIF \citep{baronig2025advancing} & $80.4 \pm 0.3$ &-- & -- &  1.6M \\
        DCLS \citep{hammouamri2024learning} & $80.7\pm0.2$ &-- & -- &  1.2M \\
        Adapt. Skip Rec. Connection SNN \citep{xu2026asrc} & $81.93$ &-- & -- & -- \\
        SiLIF \cite{fabre2025structured} & $82.03 \pm0.25$ & 512($\times2$) & 2 & 0.35M \\
        RSNN DelRec \citep{queant2025delrec} & $82.42 \pm 0.23$ & 256(*) & 3 &  0.37M \\
        \midrule 
       \multirow{2}{*}{\textbf{SpikingADPTNet}} & $82.5$ & 512($\times2$) & 2 & 355,689\\
       & $\mathbf{83.56 \pm 0.15}$ ($\max: 83.76$) & 2048($\times2$) & 2 & 1,069,481 \\
       \midrule
    \end{tabular}
    \caption{\textbf{Accuracies on SSC} The ranking of prior state-of-the-art work reported in this table is reported from the official SSC leaderboard (Available at https://zenkelab.org/resources/spiking-heidelberg-datasets-shd/). As shown, the smaller SpikingADPTNet configuration performs on par with the current state-of-the-art, while the larger one sets a new state-of-the-art by over $1$ point. (*) denotes that the state size in practice is larger due to axonal/synaptic delays requiring explicit buffers of past outputs. }
    \label{table:ssc_results}
\end{table}

\section{Discussion}

\textbf{Motivation Revisited} As stated in Section~\ref{sec:intro}, the goal of this work is to lay the foundations for large-scale neuromorphic systems as viable alternatives to the energy-intensive Transformer-GPU paradigm that dominates AI today. This requires an efficient neuromorphic alternative to the Transformer. Therefore, ADPTNet is proposed here, combining: (i) non-linear recurrent dynamics, (ii) adaptability, (iii) fine-grained parametric control over the time-scales of the system (to achieve long-term temporal modelling), and (iv) parallelisability. These four qualities are either already present and crucial to the Transformer's performance (i.e., adaptability (ii), long-range sequence modelling (iii), and parallelisability (iv)), or a key shortcoming that sets it back (i.e., vanilla fixed-depth Transformers lack scalable non-linear recurrence (i), which is now considered important for tasks such as reasoning \citep{jolicoeur2025less, schone2025implicit, geiping2026scaling}). 

The challenge in combining all four properties within a single recurrent architecture is that they are, to some extent, antithetical (Fig.~\ref{fig:tradeoffs}). Non-linear recurrence (i) impedes GPU-parallelism (iv).  In addition, non-linear recurrence (i) and adaptability (ii) typically hinder fine-grained control over timescales, and thus long-range sequence modelling as well (iii). As shown through both theoretical proofs and empirical evidence throughout this study, ADPTNet either mitigates or outright removes these trade-offs.

\textbf{Non-Linearity and Parallelism Trade-Off} Until recently, non-linear recurrence traditionally enforced sequential step-by-step simulation, which effectively prevented GPU-parallel training. With the popularisation of parallel-in-time simulation methods for RNNs, such as DEER, this limitation has gradually softened. Nevertheless, as highlighted in \citet{gonzalez2026predictability}, non-linear RNN parallelisation still depends on the properties of the dynamics, with penalties on larger Lipschitz constants and LLEs, and thus, implicitly on expressivity and long-term memory \citep{erichson2020lipschitz}. Moreover, training traditional non-linear RNNs such as LSTMs or GRUs is susceptible to bifurcations \citep{eisenmann2023bifurcations}, whereby recurrent dynamics can unpredictably become chaotic and, thus, difficult to parallelise. While researchers have parallelised novel recurrent architectures by design in the past \citep{farsang2026parallelization}, ADPTNet is the first to use dynamical systems to provide theoretically principled control over its parallelisation. Namely, as prescribed in Sections~\ref{sec:alt_deer} and \ref{sec:ADPTNet_def}, and empirically validated in Section~\ref{sec:deer_ablation_results}, the step size $\beta$ and recurrent eigenvalues $\bLb$ reliably predict and control the network's non-linearity and long-term behaviour, and, thus, its DEER convergence behaviour. One can parametrically and gradually "dial" the level of parallelism vs non-linearity within ADPTNet, unlike existing non-linear RNNs. 

Furthermore, the ADPTNet recurrence rule requires only efficient low-rank/element-wise multiplications and two dense linear projections, for a total computational cost comparable to GRUs and LSTMs. However, knowing the recurrent eigenvalues a priori, and thus also long-term behaviour, enables more efficient adaptations to Quasi-DEER compared to those established architectures. Namely, this work introduces Conv-DEER, which requires zero computational overhead and no derivatives or Jacobians. In addition, Conv-DEER reduces memory overhead by using a fixed kernel for the entire batch, and allows a choice between FFT-based signal convolutions and parallel scans for computing DEER iterations. Even with the efficiency gains, Conv-DEER retains average convergence rates similar to Quasi-DEER. With little additional cost and still no Jacobians, Forward-DEER matches baseline convergence. To the authors' knowledge, this is the first example of co-designing non-linear RNN dynamics and efficient parallel simulation algorithms, with direct user control over both.  

\textbf{Adaptability and Long-Range Sequence Modelling Trade-Off} As showcased by Mamba's degradation in performance compared to S4 on long-range sequence modelling tasks \citep{yu2026block}, in recurrent architectures, adaptability can come at the cost of stable long-term memory. Existing methods such as Hawk \citep{de2024griffinmixinggatedlinear} aim to mitigate this by disallowing arbitrarily strong forget gates and enforcing a lower bound on memory decay rates. In principle, ADPTNet follows a similar strategy. Information can only decay as fast as the fastest recurrent eigenvalue, perturbed by a term controlled by $\beta$. However, in practice, ADPTNet outperforms Hawk in adaptability, as quantified by the superior accuracy on the Selective Copying task (Section~\ref{sec:mad_results}). Furthermore, Section~\ref{sec:scifar} shows how this does not come at the cost of long-range sequence modelling performance. Both linear and non-linear ADPTNet outperform existing selective/adaptive architectures, while using significantly fewer parameters on sCIFAR, matching baseline LTI SSM performance. For both the MAD Selective Copying benchmark and sCIFAR, ADPTNet employs complex S4D-Inv and a perturbation-level $\beta$. Hence, interestingly, this study shows for the first time how linear oscillators with loose but dynamic and non-linear coupling can enable competitive selectivity compared to traditional forget gates, while retaining long-range memory. This finding contributes to the recently increasing interest in oscillators as a computational primitive in sequence models \citep{darlow2026continuous, unconventionalai2026un0}.

\textbf{Non-Linearity and Fine-Grained Timescale Control Trade-Off} Currently, controlling the non-linear dynamics of RNNs for long-range sequence modelling has been a matter of shaping the LLE. As showcased in Section~\ref{sec:lyapunov_results}, if one uses RNNs with orthogonally-constrained recurrent weights to set the LLE close to zero, this still does not prevent an arbitrarily quickly decaying rest of the Lyapunov spectrum. That differs from LTI SSMs where the entire spectrum can be fine-tuned to initialise and parametrise timescales according to task needs. The closest existing solution to attaining a form of full-spectrum control is Gradient Flossing \citep{engelken2023gradient}. Still, as detailed in Section~\ref{sec:lyapunov_exponents}, Gradient Flossing is a regularisation method relying on a loss term "nudging" dynamics towards the desired spectrum. It is sensitive to, and to some extent reliant on, the trajectories sampled for training. Furthermore, without continually optimising the flossing objective, the spectrum drifts unpredictably from its target. In terms of computational cost, it requires not only materialising full recurrent Jacobians but also computing their QR decomposition up to the number of exponents targeted for flossing, a cubically scaling operation. By contrast, ADPTNet has theoretical guarantees on parametric control of its Lyapunov spectrum without the need for loss-based regularisation, without spectral drift over training, and regardless of inputs. Lemmas~\ref{lemma:le_spectrum} and \ref{lemma:le_exact_match} provide the proofs for these guarantees, and Section~\ref{sec:lyapunov_results} reinforces them with empirical evidence. Furthermore, Section~\ref{sec:scifar} shows that non-linear ADPTNet networks can still retain SSM state-of-the-art long-range sequence modelling performance, enabled by adopting their powerful recurrent eigenspectrum parametrisation and initialisation. At the same time, Section~\ref{sec:state_tracking_results} shows that, while ADPTNet can offer guarantees on its Lyapunov spectrum, it is still sufficiently non-linear to perform state-tracking. Complementary to how $\beta$ can smoothly "dial" the selectivity and parallelisability of the network (Sections~\ref{sec:deer_ablation_results} and \ref{sec:selective_copy_results}), improvements in state tracking show by proxy how $\beta$ also incrementally increases non-linearity. 

It is worth emphasising how this interpolation behaviour differs from existing Almost Linear RNNs \citep{brenner2024almost}. Existing methods tune non-linearity by applying non-linear activations only to a subset of recurrent neurons (e.g., 10\% of neurons). That requires discrete increments in non-linear units, while $\beta$ is a continuous control. This enables, by contrast with Almost Linear RNNs, direct tuning of $\beta$ via gradient descent, allowing the network to set the level of non-linearity appropriate for the task (Table~\ref{table:trainable_beta_sc}). 

\textbf{State-of-the-Art Neuromorphic Speech Processing} ADPTNet takes inspiration from the computational and dynamical systems properties of the auditory cortex. This work argues that the input-invariant time scales in the auditory cortex can be abstracted as predictable Lyapunov exponents in a non-linear, adaptive RNN. In Section~\ref{sec:ssc}, it is shown how these brain-inspired principles lead to state-of-the-art accuracy on the SSC dataset, a de facto standard for benchmarking neuromorphic speech processing. Furthermore, it does so without the need to include large explicit buffers for delays as the previous state-of-the-art, which become more and more difficult to scale with sequence length (see Section~\ref{sec:long_range_snn}). SpikingADPTNet also outperforms state-of-the-art spiking Transformers on SSC (the SpikCommander), given a comparable parameter budget, without the need for data augmentation. 

Within the broader deep learning research landscape, ADPTNet is the first architecture to employ a theoretically principled adaptation of linear Transformers/DeltaNet online optimisation rules to Riemannian manifolds. Furthermore, this is the first study to connect modern linear Transformers/DeltaNet to dynamical systems theory using Chaos Theory (i.e., Lyapunov spectra). This latter connection extends the principles behind Neuromorphic Intermediate Representations (NIR) \citep{pedersen2024neuromorphic}. NIR connects heterogeneous neuromorphic hardware and software platforms through dynamical systems formalisms. This study likewise connects brain computing principles, GPU-parallelisation, long-range sequence modelling, and state-of-the-art notions of adaptability/selectivity through the shared framework of dynamical systems theory. 

\textbf{Significance} ADPTNet pushes forward the state-of-the-art in multiple distinct fields. It advances sequence modelling as the first combination of linear DeltaNet concepts with principles from dynamical systems theory, particularly Lyapunov exponents, and Riemannian manifold optimisation. It proposes a novel, more efficient paradigm for parallelising non-linear dynamics that exploits highly controllable long-term dynamics to reduce computational and memory overhead and offer fine-grained control over parallelisation properties. It also contributes to the RNN vanishing/exploding gradients mitigation literature, namely Gradient Flossing, as a more efficient alternative that achieves stable gradients by construction, not regularisation, to significantly improve efficacy and reliability. In addition, SpikingADPTNet sets a direct state-of-the-art result on real-world neuromorphic speech processing ($83.56\% \pm 0.15$), outperforming the current state-of-the-art spiking Transformer in the process. Ultimately, all the steps taken within this work bring neuromorphic architectures closer to competing with the Transformer, even outperforming it on long-range dependency modelling and state tracking. As a result, ADPTNet is also a step towards a viable neuromorphic and energy-efficient alternative to current LLM systems. This is a vital research direction, as an uncontrolled rise in AI energy consumption can have severe negative effects on both the environment and society at large. 

\section{Future Work}

\textbf{Limitations} Although ADPTNet makes strides towards capturing the qualities that support the Transformer's dominance, several challenges remain. First, while ADPTNet outperforms its closest existing equivalent, Hawk, it remains unclear whether it can outperform other adaptive networks on Selective Copying, since only a small perturbative $\beta=0.00125$ is tested here. Evidence from Table~\ref{table:ADPTNet_sc_results} suggests testing larger $\beta$ should further close that performance gap. Second, while Selective Copy (and MAD) and state tracking are synthetic predictors of language modelling and reasoning performance, respectively, they are not sufficient to conclude whether ADPTNet performs on par with the Transformer at scale. A more informative test would be to scale ADPTNet to billions of parameters and directly test it on full language modelling and reasoning benchmarks. Third, the ADPTNet and Conv/Forward DEER implementations are in PyTorch, which may be suboptimal compared to fused CUDA kernels and may also be limiting scalability at the moment.  Generally, to fully evaluate the potential of ADPTNet, more compute than was available for the scope of this work is needed. 

\textbf{Koopman-DEER} Briefly, Koopman Operator theory \citep{koopman1931hamiltonian} posits that any non-linear dynamics can be approximated by a linear operator $\mathcal{K}$ if lifted to an appropriate infinite-dimensional state space. In practice, a high-dimensional embedding, such as a delay embedding or the recurrent state of an SSM or ADPTNet, can serve as a proxy for computing the linear approximator $K \approx \mathcal{K}$, which can then effectively describe the system's non-linear behaviour. The topological conjugates that underpin ADPTNet recurrence have an established connection to the Koopman Operator. Dynamical Similarity Analysis (DSA) \citep{ostrow2023beyond} uses orthogonally-constrained topological conjugates $QKQ^T$ to define a metric distance between different non-linear dynamics. In general, finding a Koopman operator $K$ is done using Dynamic Mode Decomposition (DMD) \citep{tu2013dynamic}, which is essentially a linear least-squares problem of the form $\norm{(s_t - Ks _ {t-1})} $, where $s_t$ are all the states in a trajectory \footnote{It is worth mentioning that the DMD optimisation problem bears striking resemblance to the residual minimisation problem within DEER.}.

One direction for future work is to build on Conv-DEER toward a Koopman-DEER. Conv-DEER instability largely stems from errors caused by not considering the off-diagonal interactions between recurrent states. A Koopman-DEER implementation would similarly minimise computational overhead compared to traditional Quasi-DEER, but use a single dense matrix approximator $K$ instead of the diagonal eigenvalues $\bLb$.  Given the $Q \bLb Q^T$ structure of ADPTNet recurrent dynamics, finding an ADPTNet Koopman operator $K = \hat{K}\bLb \hat{K}^T$ may be a question of finding an optimal averaging of rotations $\hat{K} \approx \mathtt{mean}(Q_1, Q_2, \dots, Q_t)$. Computing DEER iterations is then just a matter of finding $K^t$ terms, which are $\hat{K} \bLb^t \hat{K}^T$, enjoying similar efficiency and convolution equivalence to Conv-DEER. 

\textbf{Proximal-DEER} More generally, future work could also explore taking advantage of manifold structure within the DEER update itself. In other words, DEER relies on unconstrained Newton steps in the direction minimising the residual. However, in networks such as ADPTNet, each recurrent step has a known and highly structured form. One could potentially devise new architectures and accompanying DEER extensions that can take advantage of Riemannian optimisation techniques to constrain DEER updates to valid manifold-bound state guesses. Depending on the constraints, one can use efficient proximal optimisation algorithms with accelerated convergence \citep{gokhale2024proximal}. 

\textbf{Topological Conjugate / Manifold-Constrained Online and Local Learning Rules} ADPTNet relies on similarity transforms and topological conjugations of its recurrent dynamics at each timestep. However, one could space out the updates and accumulate gradient/learning signals over larger sequences. This could yield an online learning rule similar to e-prop \citep{bellec2020solution} or STDP \citep{bengio2015stdp}, where, however, optimisation updates are constrained to the orthogonal manifold and used to topologically conjugate (i.e. apply similarity transforms to) SNN/RNN recurrent weights. This would mean that the fundamental timescales of the networks remain unchanged throughout learning, avoiding vanishing/exploding history compression pathologies.  

\textbf{Improve Energy-Efficiency} As presented, SpikingADPTNet could be further optimised for energy efficiency by adopting more neuromorphic computing principles. For instance, while ReLU is currently applied before obtaining $k$ and $v$, a spiking activation could avoid the dense and expensive $K$ and $V$ vector matrix multiplications. Furthermore, the loose coupling induced by the topological conjugates could also be optimised for sparsity and could eventually be derived stochastically rather than deterministically, to reduce memory overhead. 

\textbf{Dynamical Systems Modelling} The focus of this work is on stable long-range sequence modelling, and thus the eigespectrum is parametrised accordingly. However, the ADPTNet recurrence imposes no inherent constraints. One could also explore ADPTNet modelling of chaotic dynamical systems. Systems such as the Lorenz 63 attractor have well-known positive Lyapunov spectra which could be used as part of ADPTNet initialisation or as a form of inductive bias towards learning unstable dynamics. Furthermore, one could use the ADPTNet topological conjugate primitive to build interpretable models of different dynamical system topologies, with different similarity transforms for individual critical points/attractor basins in the spirit of recurrent switching SSMs \citep{linderman2016recurrent}.  

\textbf{Tighter Lyapunov Exponent Bounds} As shown in Section~\ref{sec:lyapunov_results}, the Lyapunov spectra empirically observed for ADPTNet at initialisation and during training lie within much tighter bounds of the eigenspectrum, relative to the limits prescribed by Lemma~\ref{lemma:le_spectrum}. Future work should establish whether the theoretical bounds could be further tightened, or whether the observed behaviour is an idiosyncrasy of the data and model setup used in this study.

\textbf{Architectural Extensions and Scaling} One of the key limits on ADPTNet expressiveness is the low-rank structure of Key-Value outer products projected onto the orthogonal manifold. For instance, \citet{siems2026deltaproduct} showed how taking multiple gradient descent steps over the DeltaNet learning objective improved state-tracking performance. Similar results could be a useful future direction for ADPTNet. In addition, ADPTNet currently relies on Riemannian gradient descent steps from the origin $I_n$ at each time step, which may also limit expressiveness. It may be worth exploring alternative starting points, or even dynamic starting points, potentially with matrix-valued memory similar to DeltaNet-like state-of-the-art recurrent architectures. Ultimately, the goal of ADPTNet is to propose computational alternatives to the Transformer. This evidently requires testing ADPTNet at billion-parameter scale

\section{Author Contributions Statement}
M.I.S. conceived the models under investigation, conducted the experiments, contributed to their design, prepared figures, and wrote the main manuscript text. O.R. supervised the research, contributed to devising the experiments and provided major revisions to the final manuscript text and figures. All authors reviewed the final manuscript.

\section{Data Availability Statement}

The datasets used in this study are publicly available:

\section{Additional Information}
Competing interests: The authors declare no competing interests.
\bibliographystyle{plainnat}
\bibliography{paper}  

\begin{appendices}

\section{Numerically Stable Lyapunov Exponent Algorithm}

Computing the Lyapunov spectrum from the long-term Jacobian of a dynamical system may encounter numerical stability issues due to vanishing or exploding values in repeated matrix products. This study uses Algorithm~\ref{algo:stable_lyapunov_comp} from  \citet{engelken2023lyapunov} and \citet{engelken2023gradient} to compute exponents accurately. 

\begin{algorithm}
\caption{Numerically Stable Lyapunov Spectrum Computation} \label{algo:stable_lyapunov_comp}
\begin{algorithmic}
\Require $T > 0$, $B$, \texttt{inputs}$\in \mathbb{R}^{\text{Batch Size} \times D \times T}$, \texttt{initial state} $\in \mathbb{R}^{\text{Batch Size} \times D}$, \texttt{states}$\in \mathbb{R}^{\text{Batch Size} \times D \times T}$, $\mathtt{n}_{\text{iters}}$
\Ensure $T \% B=0$
\State $N_{\text{blocks}} \gets T / B$
\State \texttt{input blocks} $ \gets \mathtt{inputs.chunk(N_{blocks}, dim=-1)}$
\State \texttt{states blocks} $ \gets \mathtt{states.chunk(N_{blocks}, dim=-1)}$
\State \texttt{states out} $\gets$ \texttt{empty list}
\State $i \gets 0$
\While{$i < n_{\text{blocks}}$}
    \State $j \gets 0$
    \State \texttt{states guess} $\gets$ \texttt{state blocks[i]}
    \While{$j <  n_{\text{iters}}$}
        \State \texttt{states guess} $\gets$ \texttt{DEER(states guess, initial guess, input blocks[i])}
        \State $j \gets j + 1$
    \EndWhile
    \State $i \gets i + 1$
    \State \texttt{initial state = states guess[..., -1]}
    \State \texttt{states out.append(states guess)}
\EndWhile

\Return \texttt{{cat(states out, dim=-1)}}
\end{algorithmic}
    
\end{algorithm}

\section{Code for Forward DEER MAC Estimate} \label{appendix:code_fwd_macs}

\begin{verbatim}
    class ForwardDEERMACs(nn.Module):
        def __init__(self):
            super().__init__()
    
        def forward(self, U, c, V, Lambda):
            L, R = U @ c, V 
            LR_diag = L[..., 0] * R[..., 0] + L[..., 1] * R[..., 1]
            RLL = R.matmul((L.mT * Lambda).matmul(L))#
            RLLR_diag = RLL[..., 0] * R[..., 0] + RLL[..., 1] * R[..., 1] 
            diag_approx = Lambda + 2 * Lambda * LR_diag + RLLR_diag 
            return diag_approx 
\end{verbatim}

\section{Training Setup for State Tracking} \label{appendix:s5_training_config}

\begin{table}[H] \centering\small\setlength{\tabcolsep}{6pt}
\begin{tabular}{l l l}
\toprule
& \textbf{$L=8$ sweep} & \textbf{$L=32$ sweep} \\
& \footnotesize\texttt{s5\_L8\_*\_beta\_sparsity} & \footnotesize\texttt{s5\_L32\_ln\_relu\_cdot} \\
\midrule
\multicolumn{3}{l}{\emph{Task and data}} \\
Group                  & \multicolumn{2}{l}{$S_5$, 120 elements, non-solvable; identity monoid, $p=0$} \\
Training sequences     & \multicolumn{2}{l}{$8{,}388{,}608$ ($512\times16384$), \texttt{uint8}} \\
Test sequences         & \multicolumn{2}{l}{$4{,}096$ per evaluation length} \\
Training length        & 8                       & 32 \\
Evaluation lengths     & 8, 16, 32 ($1/2/4\times$) & 32, 64, 128 ($1/2/4\times$) \\
\midrule
\multicolumn{3}{l}{\emph{Model}} \\
Layers                 & \multicolumn{2}{l}{1} \\
Width                  & \multicolumn{2}{l}{512} \\
Heads                  & \multicolumn{2}{l}{1} \\
Normalisation          & \multicolumn{2}{l}{batch norm} \\
Dropout                & \multicolumn{2}{l}{0.05} \\
Output projection      & \multicolumn{2}{l}{GLU} \\
Positional encoding    & \multicolumn{2}{l}{none (RoPE off)} \\
Input gate             & \multicolumn{2}{l}{on, per element; gate trainable} \\
Recurrence solver      & \multicolumn{2}{l}{Sequential Simulation} \\
\midrule
\multicolumn{3}{l}{\emph{Optimisation}} \\
Optimiser              & \multicolumn{2}{l}{AdamW (\texttt{adamw\_torch\_fused})} \\
Learning rate          & \multicolumn{2}{l}{$2\times10^{-3}$} \\
Weight decay           & \multicolumn{2}{l}{$0$} \\
$\beta_1,\beta_2$      & \multicolumn{2}{l}{$0.9$, $0.95$} \\
$\epsilon$             & \multicolumn{2}{l}{$10^{-8}$} \\
Gradient clipping      & \multicolumn{2}{l}{$0.25$} \\
Schedule               & \multicolumn{2}{l}{linear, $10\%$ warmup ($1{,}638$ steps)} \\
Batch size             & \multicolumn{2}{l}{512} \\
Epochs                 & \multicolumn{2}{l}{1 ($16{,}384$ steps)} \\
Seeds                  & \multicolumn{2}{l}{0, 1, 2} \\
Model selection        & \multicolumn{2}{l}{best \texttt{eval\_loss}, loaded at end} \\
Evaluation interval    & \multicolumn{2}{l}{every $1{,}000$ steps, all lengths} \\
\midrule
\multicolumn{3}{l}{\emph{Swept}} \\
Spectrum init          & \texttt{s4d-real}       & \texttt{continuous\_pm} (fixed) \\
$\beta$                & $0$, $0.00125$, $0.0125$, & $4$ (fixed) \\
                       & $0.125$, $0.5$, $1$, $4$  & \\
Sparsity               & $0\%$, $75\%$           & $75\%$ (fixed) \\
Negative half          & yes / no                & yes (fixed) \\
$k/v$ LayerNorm        & off                     & swept \\
\texttt{relu\_past\_out} & off                   & swept \\
Controlled dot         & off                     & swept, \texttt{dot\_value} $=0.5$ \\
Early stopping         & none                    & stop when $L{=}64$ is perfect \\
Configurations         & $7\times2\times2=28$    & see results table \\
\bottomrule
\end{tabular}
\caption{}
\label{}
\end{table}

\section{Training Setup for Selective Copy} \label{appendix:selective_copy_setup}

\begin{table}[H]\centering\small\setlength{\tabcolsep}{6pt}
\begin{tabular}{l|l|l}
\toprule
                 & Baseline & ADPTNet \\
\midrule
Vocabulary       & \multicolumn{2}{l}{$16$} \\
Tokens to copy   & \multicolumn{2}{l}{$16$} \\
Training samples & \multicolumn{2}{l}{$10{,}000$ per epoch, regenerated} \\
Test samples     & \multicolumn{2}{l}{$1{,}000$} \\
Layers           & \multicolumn{2}{l}{$2$} \\
Eig.\ Init.      & \multicolumn{2}{l}{S4D-Real} \\
Batch size       & \multicolumn{2}{l}{$64$} \\
Epochs           & \multicolumn{2}{l}{$200$} \\
\midrule
LR               & $10^{-3}$, $5\times10^{-3}$, $10^{-2}$ & $10^{-3}$ \\
WD               & $0$, $0.1$ & $0.1$ \\
No.\ Seeds       & $3$ & $1$ \\
\bottomrule
\end{tabular}
\caption{}
\label{}
\end{table}

\begin{table}[H]\centering\small\setlength{\tabcolsep}{6pt}
\begin{tabular}{l|l}
\toprule
$d_{\text{model}}$ & $64$ \\
$n_{\text{state}}$ & $4$ \\
Layers             & $2$ \\
Eig.\ Init.        & S4D-Inv \\
$\beta$            & 0.0125 \\
\midrule
Epochs             & $100$ \\
LR                 & $5\times10^{-4}$ \\
WD                 & $0$ \\
\bottomrule
\end{tabular}
\caption{}
\label{}
\end{table}

\section{MAD Benchmark Setup} \label{appendix:mad_setup}
\section{Sequential CIFAR-10 Setup} \label{appendix:cifar_setup}
\section{Spiking Speech Commands Setup} \label{appendix:ssc_setup}

\end{appendices}
\end{document}